\documentclass[10pt]{article}
\usepackage[preprint]{tmlr}
\def\month{11}
\def\year{2026}
\def\openreview{\url{https://openreview.net/forum?id=ESlL4ab9y8}}

\usepackage{amsmath,amssymb,amsfonts,amsthm}
\usepackage{graphicx}
\usepackage{adjustbox}
\usepackage{subcaption}
\usepackage{booktabs}
\usepackage{float}
\usepackage{hyperref}
\hypersetup{
  hidelinks,
  pdftitle={AAC: Admissible-by-Architecture Differentiable Landmark Compression for ALT},
  pdfauthor={An T. Le and Vien Ngo},
}
\usepackage{url}
\usepackage{multirow}
\usepackage[table]{xcolor}
\definecolor{aaccolor}{RGB}{233,241,248}
\definecolor{badgeMainBg}{RGB}{20,71,144}
\definecolor{badgeAblationBg}{RGB}{96,108,124}
\DeclareRobustCommand{\tabletag}[1]{\begingroup\setlength{\fboxsep}{1.5pt}\colorbox{badge#1Bg}{\color{white}\rule[-0.6pt]{0pt}{6pt}\scriptsize\bfseries\sffamily\,\MakeUppercase{#1}\,}\endgroup\hspace{0.25em}}
\usepackage{enumitem}
\usepackage{tabularx}
\usepackage{siunitx}
\graphicspath{{figures/}}

\definecolor{aacBlue}{RGB}{0,114,178}
\definecolor{stageGray}{RGB}{96,96,96}
\definecolor{nodeFill}{RGB}{248,249,251}

\newtheorem{theorem}{Theorem}
\newtheorem{proposition}[theorem]{Proposition}
\newtheorem{corollary}[theorem]{Corollary}
\newtheorem{lemma}[theorem]{Lemma}

\newtheorem{example}[theorem]{Example}
\theoremstyle{definition}
\newtheorem{definition}[theorem]{Definition}
\theoremstyle{plain}

\newcommand{\R}{\mathbb{R}}

\newcommand{\hALT}{h_{\mathrm{ALT}}}
\newcommand{\hFPS}{h_{\mathrm{FPS}}}
\newcommand{\hA}{h_{A}}
\newcommand{\argmax}{\operatorname{argmax}}
\newcommand{\Pbase}{P_{\mathrm{base}}}
\newcommand{\Bmem}{\mathsf{B}}
\newcommand{\smax}{\operatorname{smax}}
\newcommand{\smin}{\operatorname{smin}}

\title{AAC: Admissible-by-Architecture Differentiable Landmark Compression for ALT\thanks{Large language models assisted with manuscript editing and code review. The authors take full responsibility for all method development, theoretical investigations, and experimental design.}}

\makeatletter
\if@accepted
  \author{\name An T.~Le \email an@robot-learning.de \\
        \addr Center for AI Research, VinUniversity, Vietnam \\
        VinRobotics, Vietnam \\
        Intelligent Autonomous Systems, TU Darmstadt, Germany
        \AND
        \name Vien Ngo \email vien.na@vinuni.edu.vn \\
        \addr Center for AI Research, VinUniversity, Vietnam \\
        VinRobotics, Vietnam}
\fi
\makeatother

\setlist{topsep=3pt, partopsep=0pt, itemsep=1pt, parsep=0pt}

\begin{document}
\maketitle

\begin{abstract}
We introduce the Architecturally Admissible Compressor (AAC), which adapts landmark-based A*
heuristics to query demand under a fixed label-memory budget. AAC provides two continuous
relaxations of ALT (A*, Landmarks, and Triangle inequality) selection. With exact landmark distances,
every parameter value defines a lower bound on shortest-path distance. Both relaxations deploy ALT
using only the selected labels.
The row-stochastic relaxation takes convex combinations of landmark distances. We characterize
the maximal admissibility-preserving linear classes and prove that the relaxed and discrete
optima coincide. The fractional relaxation mixes subset heuristics and supports continuous greedy
(AAC-CG) with exact gradients and an entropy-regularized oracle for differentiable selection. For expected heuristic
value, we prove an approximation guarantee with finite-step, oracle-error, and finite-sample terms.
The limiting $1-1/e$ factor is optimal for arbitrary supplied pools and query distributions unless P${=}$NP.
At matched label memory, AAC-CG reduces A* expansions by $1.9$--$22.7$\,\% relative to
farthest-point ALT (FPS-ALT) in all $19$ tested combinations of graph and memory budget across
six road networks and two synthetic graphs. Each combination uses one $500$-query test split;
all methods share distance tables when distances are symmetric.
Canonical greedy also expands fewer nodes than FPS-ALT in all $19$ graph--budget pairs.
AAC-CG achieves $99.6$\,\% of canonical greedy's mean heuristic value on test queries, averaged
over these pairs. AAC makes admissibility independent of training success, while its analysis
and experiments distinguish selection quality from search efficiency.
\end{abstract}

\section{Introduction}
\label{sec:introduction}

Point-to-point shortest-path search underlies routing, robotics, and planning.
A*~\citep{hart1968formal} guides search with a heuristic estimate of the remaining distance.
A non-negative heuristic is \emph{admissible} if it satisfies $h(u,t)\leq d(u,t)$.
With node reopenings, this property guarantees optimal paths. Tighter consistent bounds can reduce search
expansions, subject to tie-breaking \citep{dechter1985generalized}. The learning problem is to adapt
these bounds to query demand without losing the optimality guarantee.

ALT~\citep{goldberg2005computing} constructs admissible bounds from stored landmark distances
using the triangle inequality.
At a fixed memory budget, the choice of landmarks determines which bounds remain available.
Farthest-point sampling (FPS) provides broad spatial coverage but ignores query demand.
Its $2$-approximation guarantee concerns metric covering radius \citep{gonzalez1985clustering};
we use symmetrized distances on directed graphs. Coverage bounds the heuristic gap, but a
smaller covering radius need not give a smaller gap (Corollary~\ref{cor:fps-covering}).
A query-aware selector can therefore seek stronger bounds without increasing label memory.

Learned search methods such as PHIL, Neural A*, TransPath, UPath, and iA* do not guarantee admissibility
\citep{pandy2022learning,yonetani2021path,kirilenko2023transpath,ananikian2026upath,chen2025ia}.
Loss constraints and validation offsets reduce violations without ruling them out for every
parameter value \citep{futuhi2026learning,sahay2026empirically}. Exact guarantees also exist,
including learned heuristics with external admissibility mechanisms \citep{li2022optimal}
and concurrent architectural cost partitioning \citep{barral2026cost}.
Within ALT, every subset of landmark bounds is admissible. Query-weighted subset selection
is already a monotone submodular problem with a greedy approximation guarantee
\citep{rayner2013subset}. We make this selection differentiable while preserving
admissibility throughout optimization. The selector can then train jointly with a cost encoder;
the admissibility guarantee applies whenever its landmark distances are exact.

\paragraph{AAC.} We formulate memory-constrained ALT as selecting $m$ landmark--direction items
from a pool of $K_0$ landmarks. We maximize the mean ALT bound over a query distribution
(Section~\ref{sec:objective}). Directly minimizing ALT search space is NP-hard
\citep{bauer2010hard}. The mean bound's submodular structure admits efficient approximation algorithms.
A higher mean bound does not imply pointwise dominance or fewer expansions, so we evaluate search effort separately.
The Architecturally Admissible Compressor (AAC) provides two continuous relaxations of this objective.
A row-stochastic matrix $A$ forms convex combinations of landmark distances, defining an
admissible heuristic $h_A$. We train this selector with a straight-through estimator (\textbf{AAC-ST});
one-hot rows give the deployed items. A fractional vector $x$ in the matroid polytope
of deployable sets defines an admissible mixture $h_x$ of subset heuristics.
Continuous greedy optimizes $x$, and pipage rounding returns a discrete set (\textbf{AAC-CG}).
An exact linear oracle serves static selection; an entropy-regularized oracle makes the
selection trajectory differentiable almost everywhere in the landmark bounds. Both methods deploy ordinary ALT
using only the selected labels.

Admissibility holds independently of whether training finds a good selection. We therefore
analyze representational capacity and optimization quality separately, then test how selection
quality translates into search efficiency. Our contributions are:

\textbf{(C1) Admissible relaxations and their expressive limits.}
We characterize the maximal linear classes that preserve admissibility on every graph
(Theorem~\ref{thm:maximal-class}). We prove that the row-stochastic relaxation has the same
optimum as discrete selection (Proposition~\ref{prop:convex-extension}), although
straight-through training need not find it. Admissibility-preserving compression is classical
\citep{felner2007pdb,goldenberg2011cdh,goldenberg2017cdh}; AAC adds differentiable
parameterizations of ALT selection and an analysis of their capacity.

\textbf{(C2) Approximation guarantees and limits.}
Building on continuous greedy and pipage rounding \citep{calinescu2011maximizing}, we derive a
guarantee for expected heuristic value with explicit finite-step, oracle-error, and finite-sample terms
(Theorem~\ref{thm:aac-guarantee}). We specialize the facility-location closed form of
\citet{iyer2014monotone} to obtain exact gradients for ALT bounds (Lemma~\ref{lem:closed-form})
and bound the penalty from the entropic oracle (Proposition~\ref{prop:eps-lmo}).
Fixed and free direction allocation give nested matroids, with a tight $1/2$ example for greedy
under the fixed split (Proposition~\ref{prop:free-split}). A maximum-coverage reduction shows
that the limiting $1-1/e$ factor is optimal for arbitrary supplied pools and query distributions
unless P${=}$NP (Theorem~\ref{thm:hardness}). This hardness result does not cover FPS-generated
pools, uniform all-pairs queries, or expansion minimization.

\textbf{(C3) Matched-memory evidence and optimizer diagnostics.}
Across six road networks and two synthetic graphs, AAC-CG and canonical greedy each expand
fewer nodes than FPS-ALT in all $19$ tested combinations of graph and memory budget.
Methods share distance tables when distances are symmetric, and each combination uses one
$500$-query test split. AAC-CG reduces expansions by $1.9$--$22.7$\,\% relative to FPS-ALT,
but expands fewer nodes than canonical greedy in only $4$ of the $19$ graph--budget pairs
(Table~\ref{tab:symmetric-accounting}). Re-evaluating eight of these pairs on $20$ fresh query sets
supports gains over FPS-ALT in seven; the paired interval for SBM at $32$\,B/v
includes zero (Section~\ref{sec:robustness}). Under directed accounting, AAC-ST's test objective is lower than FPS-ALT's on average, despite the zero relaxation gap (Table~\ref{tab:continuous-greedy}).
A contextual study shows that both selection layers train inside a differentiable cost-learning pipeline;
its fixed target is an FPS landmark, so it does not establish general search-efficiency gains (Appendix~\ref{sec:contextual-study}).

\section{Background and notation}
\label{sec:background}

\subsection{Graphs, A* search, and admissibility}
\label{sec:graphs-astar}

Let $G = (V, E, w)$ be a finite graph with vertex set $V$ and edge set $E$.
The edge weights $w: E \to \R_{>0}$ are positive. Let $d(u, v)$ denote the shortest-path distance from $u$ to $v$. The function $d$ satisfies the triangle inequality: $d(u,v) \leq d(u,z) + d(z,v)$ for all $u,z,v \in V$ with finite distances. For undirected graphs, $d(u,v) = d(v,u)$. For directed graphs, $d(u,v) \neq d(v,u)$ in general.

\paragraph{Reachability assumption.} Unless stated otherwise, the theorems assume finite shortest-path distances between the relevant vertex pairs.
We draw both query endpoints from the largest strongly connected component on directed graphs. On undirected graphs, we use the largest connected component. We abbreviate either component as LCC when the graph type is clear.
Every sampled query therefore has finite distance.
A* may still reach vertices outside that component, where some landmark distances are infinite.
The heuristic omits these terms unless they prove the target unreachable; it then returns $+\infty$ (Appendix~\ref{app:directed}).
Appendix~\ref{app:reproducibility} describes the sentinel value used for infinite distances.

\paragraph{A* search.} A*~\citep{hart1968formal} maintains a priority queue ordered by $f(v) = g(v) + h(v, t)$. Here $g(v)$ is the cost of the best known path from source $s$ to $v$. The term $h(v, t)$ is a heuristic estimate of the remaining distance to target $t$. A heuristic is \emph{admissible} if $h(v, t) \leq d(v, t)$ for all $v, t \in V$. We clamp heuristic values at zero, which preserves admissibility and gives $h(t,t)=0$. We use graph-search A* with node reopenings: a closed node whose $g$ value improves returns to the open list. With this normalization, admissibility suffices for optimality when we stop upon removing the target from the queue. The queue breaks ties in $f$ by the larger $g$, then by insertion order, so expansion counts are deterministic.

\paragraph{Consistency in exact arithmetic.} A heuristic is \emph{consistent} if $h(t,t) = 0$ and $h(u,t) \leq w(u,v) + h(v,t)$ for every arc $(u,v)$. With finite landmark distances, ALT on any item set (Eq.~\eqref{eq:hS}) is consistent.
Each item bound satisfies the arc inequality by the triangle inequality. A maximum of consistent terms is consistent. The $+\infty$ rule for unreachable landmarks keeps this property, including on arcs where the set of usable landmarks changes (Appendix~\ref{app:directed}). The row-stochastic and fractional heuristics of Section~\ref{sec:method} are consistent in exact arithmetic as well (Appendix~\ref{app:consistency}).

\paragraph{Exact guarantees and numerical checks.} The theorems assume exact distances and exact arithmetic.
Deployment stores labels in float32. Before clamping a bound at zero, it subtracts an allowance of two units in the last place (ULPs).
The ULP is evaluated at the largest finite label, using the coarser of the input and stored precisions.
This allowance covers float32 conversion and subtraction of exact input labels.
It does not bound errors accumulated while computing those labels (Appendix~\ref{app:directed}).

Under the exact-input assumption, the allowance preserves admissibility.
It can break consistency because rounding errors differ between adjacent vertices, so A* reopens nodes.
The experiments separately compare recorded path costs with Dijkstra's (Section~\ref{sec:admissibility-verification}).
The compressed differential heuristic (CDH) also needs reopenings: a pivot stored at one vertex may be absent at its neighbor (Section~\ref{sec:cdh-headtohead}).

\subsection{ALT: landmarks and triangle inequality}
\label{sec:alt-background}

\paragraph{ALT heuristics.} The ALT framework~\citep{goldberg2005computing} selects $K$ \emph{landmark} vertices $\{l_1, \ldots, l_K\}$.
It precomputes distances from each landmark to all vertices with Dijkstra's algorithm~\citep{dijkstra1959note}. For directed graphs, we store both \emph{forward} distances $d_{\mathrm{out}}(k, v) := d(l_k, v)$ and \emph{backward} distances $d_{\mathrm{in}}(k, v) := d(v, l_k)$. That costs $2K$ floats per vertex. The ALT heuristic is:
\begin{equation}
\label{eq:alt}
\hALT(u, t) = \max\!\Big(\max_{k} \big(d_{\mathrm{out}}(k, t) - d_{\mathrm{out}}(k, u)\big),\;\max_{k} \big(d_{\mathrm{in}}(k, u) - d_{\mathrm{in}}(k, t)\big),\; 0\Big).
\end{equation}

For undirected graphs, where $d_{\mathrm{out}} = d_{\mathrm{in}}$, the two bounds collapse to $\hALT(u,t) = \max_k |d(l_k, u) - d(l_k, t)|$.

\paragraph{Landmark--direction items.} We select individual landmark--direction bounds from Eq.~\eqref{eq:alt}. For a pool $L = \{l_1, \ldots, l_{K_0}\}$, the ground set $\mathcal{I} = \mathcal{I}_{\mathrm{fwd}} \cup \mathcal{I}_{\mathrm{bwd}}$ holds the $2K_0$ items $(l, \mathrm{fwd})$ and $(l, \mathrm{bwd})$, with bounds
\begin{equation}
\label{eq:item-bounds}
B_{(l,\mathrm{fwd})}(u,t) := d(l,t) - d(l,u), \qquad B_{(l,\mathrm{bwd})}(u,t) := d(u,l) - d(t,l).
\end{equation}
ALT on an item set $S \subseteq \mathcal{I}$ is
\begin{equation}
\label{eq:hS}
h_S(u,t) := \max\big(\{0\} \cup \{B_i(u,t) : i \in S\}\big).
\end{equation}
Classical ALT on a landmark set $L' \subseteq L$ is $h_S$ with $S = L' \times \{\mathrm{fwd}, \mathrm{bwd}\}$.
The full-pool teacher is $\hALT = h_{\mathcal{I}}$. An item set may include only one direction of a landmark.
On an undirected graph, the two directional bounds have opposite signs.
We therefore identify items with landmarks and use $B_l(u,t) := |d(l,u) - d(l,t)|$ on a ground set of $K_0$ items.

The classical ALT argument \citep{goldberg2005computing} applies to every item set.

\begin{theorem}[ALT admissibility on any item set]
\label{thm:alt-admissibility}
For every $S\subseteq\mathcal I$, if $d(u,t)$ and all distances used by $B_i(u,t)$,
$i\in S$, are finite, then $0\leq h_S(u,t)\leq d(u,t)$.
In particular, $\hALT(u,t)\leq d(u,t)$ when all pool labels are finite.
\end{theorem}

\begin{proof}
We bound each item.

\emph{Forward item.} By the triangle inequality, $d(l, t) \leq d(l, u) + d(u, t)$, so $B_{(l,\mathrm{fwd})}(u,t) = d(l, t) - d(l, u) \leq d(u, t)$.

\emph{Backward item.} By the triangle inequality, $d(u, l) \leq d(u, t) + d(t, l)$, so $B_{(l,\mathrm{bwd})}(u,t) = d(u, l) - d(t, l) \leq d(u, t)$.

On an undirected graph, the two directional inequalities give $|d(l,u)-d(l,t)|\leq d(u,t)$. Thus every argument of the maximum in Eq.~\eqref{eq:hS} is at most $d(u,t)$, and the constant $0$ gives the lower bound.
\end{proof}

This certifies every deployed item set, whichever optimizer chose it.
Appendix~\ref{app:directed} extends the heuristic and the guarantee to infinite landmark distances.

\paragraph{Landmark selection.} The quality of the ALT heuristic depends on the choice of landmarks. \emph{Farthest-point sampling} (FPS)~\citep{gonzalez1985clustering} iteratively selects the vertex farthest from the already-chosen set. We follow the farthest selection of \citet{goldberg2005computing}. We start at an arbitrary vertex and take the vertex farthest from it as the first landmark.
We do not retain the start vertex. In a metric, the resulting $K$-element subset is a $2$-approximation for the $K$-center problem.
Gonzalez's argument allows any first center. The maximum distance to the nearest landmark is at most twice the optimal radius.

Alternative strategies include random selection and planar partition landmarks. Random selection is fast, but it carries no coverage guarantee. Planar partition landmarks exploit graph planarity when it is available. The avoid and maxcover heuristics use lower-bound slack and arc coverage~\citep{goldberg2005external,goldberg2006better}. Earlier ALT optimization also scores landmark replacements by lower-bound gains on sampled vertex pairs~\citep{goldberg2005computing}. Degree-, centrality- and coverage-based selectors were also developed for landmark distance \emph{estimation} on large networks~\citep{potamias2009fast,tretyakov2011fast,takes2014adaptive}. Those selectors optimize estimation error and not an admissible bound. We use FPS as the default because it is simple and has a metric covering-radius guarantee.

\paragraph{Covering radius.} On an undirected graph, given a landmark subset $S = \{l_1, \ldots, l_m\} \subseteq L$, the \textbf{covering radius} of $S$ is
\begin{equation}
\label{eq:covering-radius-def}
r(S) := \max_{v \in V}\, \min_{l \in S}\, d(v, l).
\end{equation}
The radius measures the worst distance to the selected set. In a metric, FPS returns $S$
with $r(S)\leq2r_m^*$, where $r_m^*:=\min_{|S'|=m}r(S')$.
Theorem~\ref{thm:covering-radius} converts this to a bound on $d-h_S$.
The guarantee belongs to the set FPS returns; another subset of its pool need not inherit it
(Corollary~\ref{cor:fps-covering}). Radii decrease along nested selections and for optimal
$m$-center sets, but need not decrease for unrelated sets of increasing size.
On the unit path $P_5$, $r(\{2\})=2<3=r(\{0,1\})$. With positive edge weights,
$r(S)=0$ if and only if $S=V$.

For directed graphs, use $d_{\mathrm{sym}}(u,v):=\max(d(u,v),d(v,u))$ and
$r^{\mathrm{sym}}(S):=\max_{v\in V_{\mathrm{reach}}(S)}\min_{l\in S}d_{\mathrm{sym}}(v,l)$.
Here $V_{\mathrm{reach}}(S)$ consists of vertices mutually reachable with some landmark in $S$;
when $S$ lies in one strongly connected component, it is exactly that component.
FPS runs in $d_{\mathrm{sym}}$, which is a metric there. A landmark covers a vertex in this
metric only when both directional distances are finite. An individual ALT bound needs only
its own direction's distances.

\paragraph{Memory and complexity.} ALT stores $2K$ floats per vertex when forward and backward distances differ, and $K$ when they coincide (including on undirected graphs).
Preprocessing uses $K$ single-source shortest-path (SSSP) computations with Dijkstra's algorithm.
Each computation costs $O((|V|+|E|)\log|V|)$ with a binary heap.
Directed graphs require another $K$ computations on the reversed graph.
Evaluating $\hALT(u,t)$ requires $O(K)$ lookups. One expansion may evaluate the heuristic at several newly generated vertices.
The number of expansions depends on the graph and the tightness of the heuristic.

\section{Method: AAC}
\label{sec:method}

We first define the preprocessing pipeline, selection objective, and memory constraints.
We then introduce two admissible relaxations: row-stochastic matrices for AAC-ST
(Section~\ref{sec:compressor}) and fractional selections for AAC-CG
(Section~\ref{sec:theory-matroid}). Both deploy ALT on a selected item set
(Theorem~\ref{thm:alt-admissibility}). We conclude by relating the guarantees to hardness,
direction allocation, and coverage (Section~\ref{sec:theory-limits}).

An undirected selection stores $m$ labels per vertex from a pool of $K_0$ landmarks, with $m<K_0$ for compression.
Full ALT on a directed graph stores $2K_0$ directional labels. A selection stores $m = m_{\mathrm{fwd}} + m_{\mathrm{bwd}}$ labels, with $m_{\mathrm{fwd}}, m_{\mathrm{bwd}} \leq K_0$.
At $m = 2K_0$, it recovers the teacher when each direction retains every landmark.

\begin{table}[htbp]
\centering
\small
\caption{Core notation used throughout Section~\ref{sec:method}.}
\begin{tabular}{ll}
\toprule
Symbol & Meaning \\
\midrule
$K_0$ & number of pool (teacher) landmarks, chosen by FPS at preprocessing \\
$m$ & number of labels stored per vertex at deployment \\
$m_{\mathrm{fwd}}, m_{\mathrm{bwd}}$ & forward / backward labels on a directed graph ($m = m_{\mathrm{fwd}}{+}m_{\mathrm{bwd}}$) \\
$\mathcal{I}$, $B_i(u,t)$ & landmark--direction items and the bound item $i$ gives the query $(u,t)$ (Eq.~\eqref{eq:item-bounds}) \\
$S$, $h_S$ & a deployed item set and ALT on it (Eq.~\eqref{eq:hS}); $\hALT = h_{\mathcal{I}}$ is the full-pool teacher \\
$\hFPS$ & FPS-ALT at the stated baseline or hybrid-component budget \\
$A$, $\hA$ & row-stochastic matrix (directed: $A_{\mathrm{fwd}},A_{\mathrm{bwd}}$) and its heuristic (AAC-ST) \\
$\tilde{h}^{\beta}_{A}$ & $\hA$ with the outer maximum smoothed at inverse temperature $\beta$ (Theorem~\ref{thm:smooth}) \\
$x$, $h_x$ & fractional selection and the mixture $\mathbb{E}_{S \sim \mathrm{Bernoulli}(x)}[h_S]$ (AAC-CG) \\
$x_k$, $T$ & continuous-greedy iterate after step $k$, and number of steps \\
$Q$, $n$ & query distribution and number of independent training queries \\
$F(S)$, $\hat F_n(S)$ & population and empirical mean of $h_S$ over queries (Eq.~\eqref{eq:selection-objective}) \\
$\tilde F(x)$ & multilinear extension of $F$ \\
$F_{\mathrm{row}}(A)$ & mean of $\hA$ over queries, for a row-stochastic $A$ \\
$\mathcal{B}_{\mathrm{pair}}, \mathcal{B}_{\mathrm{fix}}, \mathcal{B}_{\mathrm{free}}$ & deployable families of item sets (Definition~\ref{def:deployable-families}) \\
$\Bmem$ & deployed label memory in bytes per vertex (the matched-memory budget) \\
$R_{\mathrm{uniq}}, R_{\mathrm{ent}}$ & uniqueness / entropy regularizers on $A$ \\
$\tau$, $\tau_{\mathrm g}$ & entropic-oracle and Gumbel-softmax temperatures \\
\bottomrule
\end{tabular}
\end{table}

\subsection{Preprocessing pipeline}
\label{sec:pipeline}

\begin{enumerate}
\item \textbf{Anchor selection.} We select $K_0$ landmarks by FPS on the largest strongly connected component, or the largest connected component for undirected graphs.
We use the symmetrized metric $d_{\mathrm{sym}}$, which gives Gonzalez's guarantee for the radius in our gap bound.
The first landmark is the vertex farthest from an arbitrary start vertex. We do not retain the start vertex.
\item \textbf{Teacher labels.} We compute $K_0$ forward SSSPs to obtain $d_{\mathrm{out}}(k, v)$ for all $k \in [K_0]$, $v \in V$.
For directed graphs, we also compute $K_0$ SSSPs on the transposed graph to obtain $d_{\mathrm{in}}(k, v)$.
\item \textbf{Selection.} We choose at most $m$ distinct items with continuous greedy (Section~\ref{sec:theory-matroid}) or a trained row-stochastic selector (Section~\ref{sec:compressor}). Continuous greedy returns a base of size $m$; the row selector may repeat items.
\item \textbf{Deployment.} We store the $m$ output labels per vertex in float32.
At query time, we evaluate $h_S(u,t)$ on the selected items in $O(m)$.
We apply the rounding allowance and the $+\infty$ rule from Section~\ref{sec:graphs-astar}.
\end{enumerate}

\paragraph{Cost breakdown.}
\begin{itemize}
\item \textbf{Preprocessing (offline):} $O(K_0(|V|+|E|)\log|V|)$ for $K_0$ SSSP computations from teacher landmarks.
\item \textbf{Straight-through compressor training (offline):} $O(E_{\mathrm{train}} \cdot b \cdot m \cdot K_0)$ for $E_{\mathrm{train}}$ epochs over the $m \times K_0$ selection matrix.
Each epoch is one optimizer step on $b$ query pairs. We typically use $200$ epochs and $b{=}256$.
\item \textbf{Exact-oracle continuous greedy (offline):} $O((n+T)K_0\log K_0+(T+K_0)nK_0)$ over $n$ training queries.
We sort the bounds once per query and sort the oracle weights at each optimization step. Each of the $T$ optimization steps and at most $2K_0$ pipage-rounding steps evaluates the objective or gradient in $O(nK_0)$ (Lemma~\ref{lem:closed-form}).
\item \textbf{Deployment:} We store $O(|V|m)$ floats and evaluate each heuristic value $h(u,t)$ in $O(m)$.
Total A* query time also depends on node expansions.
\end{itemize}
For the contextual variant, deployment also requires anchor SSSP on each predicted-cost graph.
Figure~\ref{fig:pipeline} summarizes the workflow. We precompute the teacher labels, select items
within the memory budget, and deploy the selected labels as an ALT heuristic.
The two selection methods use different continuous representations: row-stochastic matrices for
straight-through training and inclusion probabilities for continuous greedy.

\begin{figure}[t]
\centering
\includegraphics[width=\textwidth]{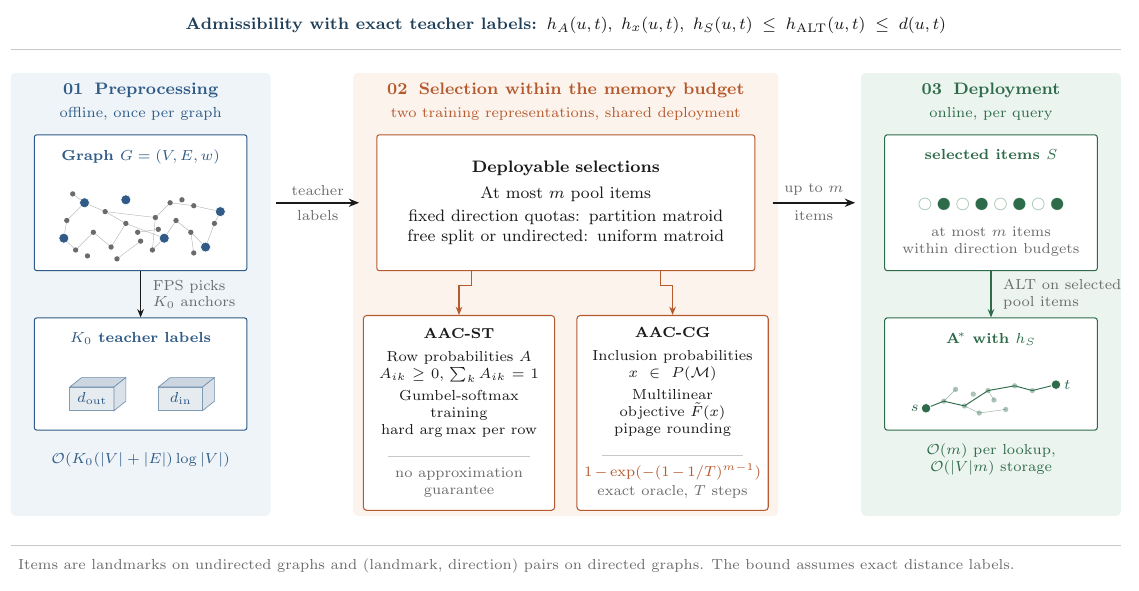}
\caption{\textbf{AAC method overview.} With exact labels and arithmetic, both training representations are admissible:
convex combinations of label differences for $A$
(Proposition~\ref{thm:admissibility}) and mixtures of subset heuristics for $x$
(Proposition~\ref{prop:matroid-cg}(c)). Both deploy ALT on a selected item set $S$, admissible
by Theorem~\ref{thm:alt-admissibility}.
The displayed finite-step factor applies to the exact-oracle AAC-CG objective; the entropic
oracle adds differentiability with the bound in Proposition~\ref{prop:eps-lmo}(c).
Section~\ref{sec:continuous-greedy} controls the other differences between the two methods.}
\label{fig:pipeline}
\end{figure}

\subsection{Selection objective and deployable families}
\label{sec:objective}

\paragraph{Objective.} A query distribution $Q$ scores an item set by its mean ALT bound. We restrict query endpoints to one component, strongly connected for directed graphs:
\begin{equation}
\label{eq:selection-objective}
F(S) := \mathbb{E}_{(s,t)\sim Q}\big[h_S(s,t)\big],
\qquad
\hat F_n(S) := \frac{1}{n}\sum_{j=1}^{n} h_S(s_j,t_j),
\end{equation}
with $h_S$ from Eq.~\eqref{eq:hS} and $\hat F_n$ its empirical version on $n$ i.i.d.\ training queries.
\citet{rayner2013subset} introduced this query-weighted objective for selecting admissible
heuristics, proved it monotone and submodular, and established greedy's $1-1/e$ cardinality
guarantee and the problem's NP-hardness.

We recall the two properties that underpin our guarantees. Fix a query and write
$g(S):=\max\big(0,\max_{i\in S}B_i\big)$. Enlarging $S$ cannot decrease the maximum, so $g$ is
monotone. For $S\subseteq S'$ and $j\notin S'$, the marginal gain
$g(S\cup\{j\})-g(S)=\max(0,B_j-g(S))$ is at least
$\max(0,B_j-g(S'))=g(S'\cup\{j\})-g(S')$. This is the diminishing-returns property of
submodularity. Taking expectations preserves both properties: $F=\mathbb E_Q[g]$ is a
monotone submodular facility-location objective, with $F(\emptyset)=0$ from the clamp.

Write $\Delta(i\mid S):=F(S\cup\{i\})-F(S)$ for the marginal gain of item $i$.
Section~\ref{sec:training} shows that the unregularized straight-through loss targets the
expected value of $F$ over sampled selections. AAC-ST and AAC-CG target the same discrete objective on the same ground set.

\paragraph{Scope of the objective.} $F$ evaluates each bound only at the query endpoints $(s,t)$.
A* evaluates $h(v,t)$ at every vertex it generates. A larger $F$ therefore need not give a larger
bound at every visited vertex, and our approximation guarantees concern $F$ alone.
Section~\ref{sec:objective-search} measures the relation between $F$ and expansions.
Unless stated otherwise, we sample both endpoints uniformly from the LCC, using the graph-specific definition in Section~\ref{sec:graphs-astar}. In the primary comparison and the optimizer study, selectors fit $\hat F_n$
on $n = 3{,}200$ training queries. We evaluate them on $500$ disjoint test queries using two metrics:
the mean heuristic value, called the \emph{test objective}, and the mean number of A* expansions.

\paragraph{Deployable families.} We consider three families of item sets, each using $m$ floats per vertex.

\begin{definition}[Deployable families]
\label{def:deployable-families}
On a directed graph, fix a pool $L$ of $K_0$ landmarks and an integer budget $1\leq m<2K_0$.
The free family below allows every such $m$; the paired and fixed families require even $m$,
with $m_{\mathrm{fwd}}=m_{\mathrm{bwd}}=m/2$. Define
\begin{align*}
\mathcal{B}_{\mathrm{pair}} &= \big\{L'\times\{\mathrm{fwd},\mathrm{bwd}\}: L'\subseteq L,\ |L'|=m/2\big\},\\
\mathcal{B}_{\mathrm{fix}} &= \big\{S\subseteq\mathcal I: |S\cap\mathcal I_{\mathrm{fwd}}|=m_{\mathrm{fwd}},\ |S\cap\mathcal I_{\mathrm{bwd}}|=m_{\mathrm{bwd}}\big\},\\
\mathcal{B}_{\mathrm{free}} &= \big\{S\subseteq\mathcal I: |S|=m\big\}.
\end{align*}
These are paired ALT, selection with fixed directional quotas, and selection with unrestricted
direction counts. Each stores $m$ floats per vertex. Write
$\mathrm{OPT}_{\mathcal S}:=\max_{S\in\mathcal S}F(S)$ for the optimal value over a family.
\end{definition}

\paragraph{Symmetric distances.} When $d$ is symmetric, one label supplies both directional
bounds. This includes undirected graphs and digraphs with equal-weight reverse arcs, as in
our DIMACS instances. We identify items with landmarks and use $B_l(u,t)=|d(l,u)-d(l,t)|$.
For any integer $1\leq m<K_0$, both $\mathcal B_{\mathrm{fix}}$ and $\mathcal B_{\mathrm{free}}$
then mean the bases of $U_{m,K_0}$, the uniform matroid of rank $m$ on $K_0$ landmarks.
The paired-storage comparison requires even $m$: it keeps two identical copies per landmark,
so $\mathcal B_{\mathrm{pair}}$ instead has rank $m/2$. This redundant storage is used only
where explicitly stated.

For a uniform family of rank $r$ on $K_0$ items, its size and maximum binary entropy are
$M=\binom{K_0}{r}$ and $H_{\max}=K_0h_2(r/K_0)$, respectively. Here
$h_2(p)=-p\log p-(1-p)\log(1-p)$, with $0\log0=0$; all logarithms are natural.
At either rank $r=m$ or $r=m/2$, both $\log M$ and $H_{\max}$ are at most
$m\log(2eK_0/m)$.

\paragraph{Matroid polytopes.} Under the fixed split, $\mathcal{B}_{\mathrm{fix}}$ is the set of
bases of the partition matroid $\mathcal{M}$ with one block per direction and budgets
$m_{\mathrm{fwd}}, m_{\mathrm{bwd}}$; under the free split, and on undirected graphs, the matroid is
uniform, with a single block. For blocks $\mathcal I_b$ with budgets $m_b$, the independence polytope
and base polytope are
\[
P(\mathcal M)=\{x\in[0,1]^{\mathcal I}:\textstyle\sum_{i\in\mathcal I_b}x_i\leq m_b\ \text{for all }b\},
\qquad
\Pbase(\mathcal M)=\{x\in P(\mathcal M):\textstyle\sum_{i\in\mathcal I_b}x_i= m_b\ \text{for all }b\}.
\]
On directed graphs with even $m$, the three families are nested and their optimal values can differ strictly
(Proposition~\ref{prop:free-split}, Section~\ref{sec:theory-limits}). The row-stochastic relaxation of Section~\ref{sec:compressor} parameterizes these
selections by matrices; continuous greedy (Section~\ref{sec:theory-matroid}) works in
$P(\mathcal{M})$ directly.

\subsection{Row-stochastic relaxation (AAC-ST)}
\label{sec:compressor}

\paragraph{Definition.} The compressor maintains a learnable logit matrix $W \in \R^{m \times K_0}$. For directed graphs, it splits the $m$ compressed dimensions into $m_{\mathrm{fwd}} = \lfloor m/2 \rfloor$ forward and $m_{\mathrm{bwd}} = m - m_{\mathrm{fwd}}$ backward dimensions. Each block keeps its own matrix, $W_{\mathrm{fwd}} \in \R^{m_{\mathrm{fwd}} \times K_0}$ and $W_{\mathrm{bwd}} \in \R^{m_{\mathrm{bwd}} \times K_0}$. Admissibility does not require this fixed split. Proposition~\ref{prop:free-split} allows the directional budgets to vary. Each compressed dimension selects one pool landmark at deployment. The architecture does not enforce distinct selections, so two rows can select the same item.

\paragraph{Training.} We sample $A$ using hard Gumbel-softmax over the logits $W$, with temperature $\tau_{\mathrm{g}}$ and a straight-through estimator~\citep{jang2017categorical,maddison2017concrete,bengio2013estimating}.
Each row draws i.i.d.\ $\mathrm{Gumbel}(0,1)$ noise $G$. Its soft sample is $\mathrm{softmax}((W + G)/\tau_{\mathrm{g}})$.
Its hard sample is the one-hot vector of $\argmax(W + G)$.
The forward pass uses hard rows; the backward pass uses the soft rows' gradient.
Every hard row has non-negative entries summing to 1, so $A$ is row-stochastic in every forward pass.

We decrease $\tau_{\mathrm{g}}$ exponentially from 1.0 to 0.1. This temperature affects only the backward Jacobian, not the hard sample.
It differs from the entropic-oracle temperature $\tau$ in Proposition~\ref{prop:eps-lmo}.

\paragraph{Design choices.} AAC-ST applies Gumbel-softmax independently to each row of $A$, with separate logits per row. The Concrete Autoencoder~\citep{balin2019concrete} and IP-CAE~\citep{nilsson2024ipcae} use the same row-stochastic form for dataset-level feature selection. L2X~\citep{chen2018l2x} instead draws multiple Concrete samples from shared, instance-dependent logits and combines them by an elementwise maximum. Other sampling methods include successive softmax~\citep{xie2019reparameterizable},
the Gumbel-top-$k$ trick and its unbiased gradient estimators~\citep{kool2019stochastic,kool2020estimating},
and score-function subset sampling~\citep{wijk2025sfess}.
Other relaxations include optimal transport and DSelect-$k$~\citep{xie2020differentiable,hazimeh2021dselect}.
Gradient estimators also include perturbed optimizers~\citep{berthet2020learning},
implicit maximum-likelihood estimation~\citep{niepert2021implicit}, and the exact-marginal
SIMPLE estimator~\citep{ahmed2023simple}. \citet{sun2023unified} give a common formulation of regularized and perturbed relaxations.
Section~\ref{sec:continuous-greedy} checks which selectors preserve the row-mass condition for admissibility. AAC-ST enforces this condition through its row-stochastic constraint.

We compute the compressed labels as:
\begin{equation}
\label{eq:compression}
y(v) = A \cdot \mathbf{d}(v), \quad \text{where } \mathbf{d}(v) = \big(d(l_1, v), \ldots, d(l_{K_0}, v)\big)^\top.
\end{equation}
For directed graphs, forward labels use $y_{\mathrm{fwd}}(v) = A_{\mathrm{fwd}} \cdot \mathbf{d}_{\mathrm{out}}(v)$ and backward labels use $y_{\mathrm{bwd}}(v) = A_{\mathrm{bwd}} \cdot \mathbf{d}_{\mathrm{in}}(v)$.

\paragraph{Inference.} At deployment, we take the hard argmax in each row of $W$.
Each output is one pool distance: $y_i(v) = d(l_{j^*(i)}, v)$ with $j^*(i) = \argmax_j W_{i,j}$.
The resulting one-hot rows are row-stochastic. The rows select an item set $S(A) \subseteq \mathcal{I}$.
Each forward row selects a forward item; each backward row selects a backward item.
The deployed heuristic is $h_{S(A)}$ (Proposition~\ref{thm:admissibility}(d)). Learning changes which items are chosen; deployment evaluates ordinary ALT on them.

\paragraph{Heuristic evaluation.} We define the compressed heuristic as:
\begin{equation}
\label{eq:hit-directed}
\hA(u, t) = \max\!\Big(0,\; \underbrace{\max_{i \in [m_{\mathrm{bwd}}]} \big(y^{\mathrm{bwd}}_i(u) - y^{\mathrm{bwd}}_i(t)\big)}_{\text{backward bound}},\; \underbrace{\max_{i \in [m_{\mathrm{fwd}}]} \big(y^{\mathrm{fwd}}_i(t) - y^{\mathrm{fwd}}_i(u)\big)}_{\text{forward bound}}\Big)
\end{equation}
for directed graphs, omitting empty directional blocks, and $\hA(u,t) = \max_i |y_i(u) - y_i(t)|$ for undirected graphs. Equation~\eqref{eq:hit-directed} and the chain $\hA \leq \hALT \leq d$ are statements about exact labels. At deployment, we evaluate Eq.~\eqref{eq:hit-directed} on float32 labels and subtract the rounding allowance before clamping at zero.
We apply the $+\infty$ rule under the precision assumptions in Section~\ref{sec:graphs-astar}.

\paragraph{Extension to directed graphs.}
AAC uses the directed ALT bounds of Eq.~\eqref{eq:alt} with separate row-stochastic matrices for forward and backward labels. Both inherit Proposition~\ref{thm:admissibility}'s convex-combination admissibility argument. When the two distance directions differ, ALT stores $2K$ floats per vertex; symmetric distances require only $K$. Section~\ref{sec:setup} states the budget rules, including the secondary studies that charge both directions separately.

The directed gap bound uses each direction's own landmark set and its symmetrized radius
(Theorem~\ref{thm:covering-radius}). Appendix~\ref{app:directed} explains the reachability
conditions and why Gonzalez's guarantee applies when FPS uses $d_{\mathrm{sym}}$.

\subsubsection{Architectural admissibility}
\label{sec:theory-arch}
\label{sec:theory}

Two elementary facts give the guarantee: the pointwise maximum of admissible heuristics is
admissible, and so is a convex combination of admissible bounds.

\begin{proposition}[Row-stochastic compression preserves admissibility]
\label{thm:admissibility}\label{prop:alt-special-case}
With finite, exact teacher labels:
\begin{enumerate}
\item[(a)] On an undirected graph, every row-stochastic $A\in\R^{m\times K_0}$ satisfies
$0\leq\hA(u,t)\leq\hALT(u,t)\leq d(u,t)$.
\item[(b)] On a directed graph, the same chain holds for Eq.~\eqref{eq:hit-directed} whenever
$A_{\mathrm{fwd}}$ and $A_{\mathrm{bwd}}$ are row-stochastic.
\item[(c)] One-hot rows satisfy both claims at every parameter value, including the hard
Gumbel-softmax forward pass and deployment by row-wise argmax.
\item[(d)] For one-hot rows, $\hA=h_{S(A)}$, where $S(A)$ is the set of distinct selected items.
On a digraph this is classical ALT on $L'$ when both directional blocks select $L'$.
Identity matrices recover the full teacher.
\end{enumerate}
\end{proposition}

\begin{proof}[Proof sketch]
Each row of $A$, $A_{\mathrm{fwd}}$ and $A_{\mathrm{bwd}}$ is a probability vector, so each compressed dimension is a convex combination of admissible pool differences. The row-wise inner products are therefore bounded by the maximum over pool landmarks of each bound type. Taking the maximum over compressed rows preserves the inequality; Theorem~\ref{thm:alt-admissibility} then bounds the teacher by $d(u,t)$. Appendix~\ref{app:admissibility-proof} gives the full proof. As a consequence, the Gumbel-softmax training trajectory stays admissible without projection. Theorem~\ref{thm:smooth}(b) extends the same chain to the smooth-max surrogate used when a gradient must pass through the maximum.
\end{proof}

On a digraph the two direction blocks may select different landmarks, so $h_{S(A)}$ need not
equal classical ALT on a single landmark set. On $u \leftrightarrow t$ at unit weight with
$L_{\mathrm{fwd}} = \{t\}$ and $L_{\mathrm{bwd}} = \{u\}$, both selected bounds are negative and the
clamp gives $h_{S(A)}(u,t) = 0$, where classical ALT on $\{u,t\}$ gives $1$. The item formulation retains ALT's evaluation cost, storage, and admissibility guarantee.
It coincides with classical ALT when $L_{\mathrm{fwd}} = L_{\mathrm{bwd}}$
(see Figure~\ref{fig:landmarks}).

\begin{figure}[t]
\centering
\includegraphics[width=\textwidth]{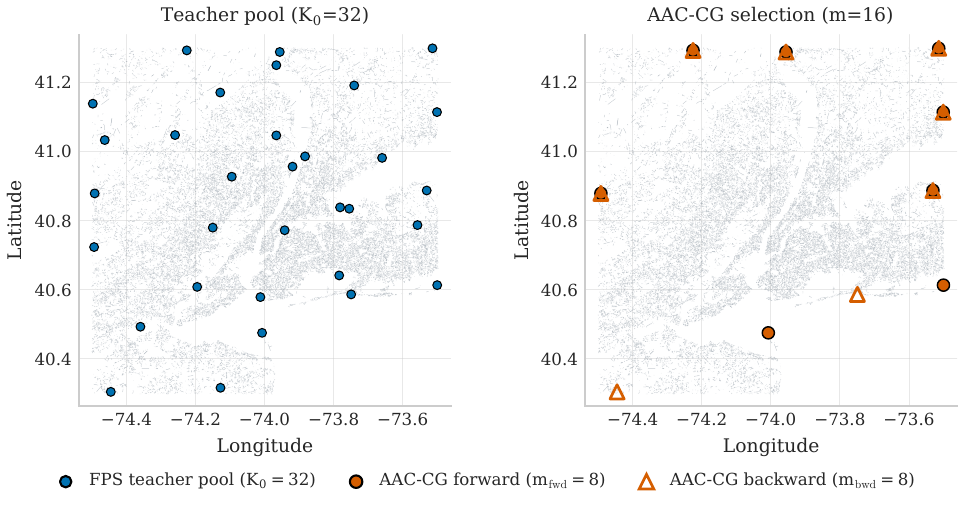}
\caption{Landmark selection on NY (DIMACS) by AAC-CG at $64$\,B/v, under the secondary directed accounting of Table~\ref{tab:continuous-greedy}. \textbf{Left}: the $K_0{=}32$ FPS teacher pool. \textbf{Right}: the $m{=}16$ directional rows that continuous greedy selects on the $3{,}200$ training queries ($8$ forward, $8$ backward). They select $10$ distinct landmarks, $6$ in both directions. NY's two distance tables coincide, so these $6$ pairs store the same column twice. The primary accounting shares those columns and selects $m{=}16$ landmarks from a $K_0{=}64$ pool at the same budget (Table~\ref{tab:symmetric-accounting}).}
\label{fig:landmarks}
\end{figure}

\subsubsection{Maximality of the admissible linear class}
\label{sec:theory-structural}

We next ask which linear coefficients preserve admissibility for every graph.
Row-stochasticity is sufficient, but the largest admissible class is broader:
it is the signed $\ell_1$ ball for undirected graphs and the sub-stochastic simplex for directed graphs.

\begin{theorem}[Maximal admissibility-preserving linear class]
\label{thm:maximal-class}
Fix $K_0\geq1$ and $a\in\R^{K_0}$. The following characterizations are universal over graphs, landmark
placements, and queries with finite distances.
\begin{enumerate}
\item[(a)] \emph{Undirected.} For $\delta_k=d(l_k,u)-d(l_k,t)$,
$|a^\top\delta|\leq d(u,t)$ on every instance if and only if $\|a\|_1\leq1$.
\item[(b)] \emph{Directed.} For $\delta_k^{\mathrm{fwd}}=d(l_k,t)-d(l_k,u)$,
$a^\top\delta^{\mathrm{fwd}}\leq d(u,t)$ on every instance if and only if
$a\geq0$ and $\mathbf1^\top a\leq1$. The same characterization holds for backward bounds.
\end{enumerate}
\end{theorem}
A fixed graph and pool may admit coefficients outside these universal classes.

\begin{proof}[Proof sketch]
For undirected graphs, $|\delta_k|\leq d(u,t)$ gives
$|a^\top\delta|\leq\sum_k |a_k|\,|\delta_k|\leq\|a\|_1 d(u,t)$.
For directed graphs, non-negative coefficients preserve the one-sided bounds
$\delta_k^{\mathrm{fwd}}\leq d(u,t)$, so
$a^\top\delta^{\mathrm{fwd}}\leq(\sum_k a_k)d(u,t)$.
For necessity on undirected graphs, attach landmark $l_k$ as a leaf to $t$ when $a_k > 0$
and to $u$ when $a_k < 0$. Then $\delta_k = \pm d(u,t)$, so any row with $\ell_1$ mass above one
violates admissibility. On directed graphs, placing every landmark upstream of $u$
shows that total coefficient mass cannot exceed one. Anchoring negative-weight landmarks at $t$ and the others at $u$ gives
$a^+d(u,t)+a^-R$, where $a^+$ and $a^-$ are the positive and negative coefficient masses.
If $a^->0$, choosing $R>d(u,t)/a^-$ violates admissibility. Appendix~\ref{app:maximal-class} gives the three constructions. Each construction uses finite edge weights and needs no limiting argument.
\end{proof}

\subsubsection{Zero relaxation gap and straight-through training}
\label{sec:training}

Theorem~\ref{thm:maximal-class} concerns a single row on every graph. The next result concerns
the whole matrix on a fixed instance: maximized over the row-stochastic class, the relaxed
objective equals the best deployable selection.

\begin{proposition}[Zero relaxation gap]
\label{prop:convex-extension}
Fix the graph, pool, and query distribution $Q$, with finite, exact labels.
\begin{enumerate}
\item[(a)] \emph{Undirected.} Let $1\leq m\leq K_0$ and
$F_{\mathrm{row}}(A)=\mathbb E_Q[\max_i|\sum_k A_{ik}(d(l_k,u)-d(l_k,t))|]$.
This objective is convex in $A$. Over products of $m$ simplices, sub-stochastic simplices,
or signed $\ell_1$ balls, it satisfies
\[
\max_A F_{\mathrm{row}}(A)=\max_{|S|\leq m}F(S)=\max_{|S|=m}F(S).
\]
\item[(b)] \emph{Directed.} Use the budgets of Definition~\ref{def:deployable-families}. For $A\in\R^{m\times|\mathcal I|}$ with sub-stochastic rows, let
$F_{\mathrm{row}}(A)=\mathbb E_Q[\max(0,\max_i\sum_{j\in\mathcal I}A_{ij}B_j(u,t))]$.
This objective is convex, and
\[
\max_A F_{\mathrm{row}}(A)=\mathrm{OPT}_{\mathrm{free}}.
\]
For even $m$, restricting $m_{\mathrm{fwd}}$ rows to forward items and
$m_{\mathrm{bwd}}$ rows to backward items
instead gives $\mathrm{OPT}_{\mathrm{fix}}$. Both equalities also hold when every row sums to one.
\end{enumerate}
In every case an optimal vertex exists, with one nonzero item per nonzero row.
\end{proposition}
The relaxation therefore contains a discrete optimum. It may also have mixed maximizers on
faces where the objective is constant.

\begin{proof}[Proof sketch]
In the undirected case, $F_{\mathrm{row}}$ averages maxima of absolute values of affine maps, so it is convex.
A convex function on a polytope attains its maximum at a vertex. Every vertex of the three sets has
rows $\pm e_j$ or $0$, so its value is $F(S)$ for the distinct landmarks selected by its rows.
Every non-empty $S$ with $|S| \leq m$ is represented by some vertex. Monotonicity then allows
us to restrict the maximum to $|S| = m$. Appendix~\ref{app:convex-extension} proves the directed case and explains why maximizing vertices are direction-pure, although other maximizers can mix directions.
\end{proof}

Proposition~\ref{thm:admissibility} keeps forward and backward labels in separate tables.
A row that mixes the two directions would require a combined label, still stored as one float
per vertex (Appendix~\ref{app:free-split-proof}). Restricting each deployed row to one direction
preserves the optimal value, because an optimal direction-pure vertex exists.

Zero relaxation gap does not make the objective easy to optimize. Theorem~\ref{thm:hardness}
proves NP-hardness for this class, and first-order ascent on $F_{\mathrm{row}}$ has no general
approximation guarantee. Section~\ref{sec:theory-matroid} gives a first-order method with a
guarantee on the \emph{multilinear} extension instead.

\paragraph{Straight-through training objective.} AAC-ST minimizes the expected gap between the
teacher and the compressed heuristic, with one row-entropy regularizer $R_{\mathrm{ent}}$ that
favors sharp selection:
\begin{equation}
\label{eq:loss}
\mathcal{L}(\theta) = \mathbb{E}_{Q,G} \Big[\big(\hALT(s,t) - h_{A_\theta(G)}(s,t)\big)_+\Big] + \lambda_{\mathrm{ent}} R_{\mathrm{ent}}(\theta),
\end{equation}
where $(\cdot)_+=\max(0,\cdot)$, $\theta$ denotes the logits $W$, and $A_\theta(G)$ is the hard
matrix sampled using independent Gumbel noise $G$. The expectation covers both queries and noise;
we suppress $G$ in $A_\theta$ below when no confusion arises.
The forward pass uses hard Gumbel-softmax rows and the hard maximum. The reported runs use no
uniqueness penalty; the Modena duplicate-selection ablation is in
Section~\ref{sec:selection-ablation}. For differentiable contextual training,
Appendix~\ref{app:smooth-max} defines a normalized log-sum-exp surrogate that stays below
the hard maximum. Its admissibility requires exact teacher labels; with smoothed contextual
labels, admissibility comes from exact distance recomputation at deployment.

\begin{proposition}[Gradient identity for the gap-to-teacher objective]
\label{prop:gradient-equiv}
For a fixed teacher with exact labels and any row-stochastic $A_\theta$,
\[
\mathbb E[(d-h_{A_\theta})_+]
=\mathbb E[(\hALT-h_{A_\theta})_+]+\mathbb E[d-\hALT].
\]
The two objectives have equal derivatives with respect to $\theta$ wherever those derivatives
exist. Expectations include both queries and any sampling noise.
\end{proposition}

\begin{proof}
By Proposition~\ref{thm:admissibility}, $h_{A_\theta}(s,t) \leq \hALT(s,t) \leq d(s,t)$ for every $(s,t)$.
Both gaps are non-negative, so we can remove the positive parts.
We have $d - h_{A_\theta} = (\hALT - h_{A_\theta}) + (d - \hALT)$.
The expected value of $d - \hALT$ is independent of $\theta$ because the graph and teacher are fixed.
The gradient identity follows wherever these objectives are differentiable. A straight-through update substitutes a surrogate Jacobian and need not equal this derivative. For AAC-ST, the positive-part implementation uses derivative $1$ at a zero gap, so replacing distance by the teacher leaves that surrogate update unchanged as well.
\end{proof}

For each sampled hard matrix, $h_{A_\theta(G)}=h_{S(A_\theta(G))}$. Hence the unregularized term is
\[
\mathbb{E}_{Q,G}\big[(\hALT-h_{A_\theta(G)})_+\big]
= \mathbb{E}_Q[\hALT] - \mathbb{E}_G\big[F(S(A_\theta(G)))\big].
\]
AAC-ST therefore targets the expected value of the same discrete objective $F$ over its sampled
selections. This expectation differs from evaluating $F_{\mathrm{row}}$ at the softmax of the
logits, and the straight-through update need not be the derivative of either objective.
Training needs only the pool labels, $O(K_0|V|)$ storage per direction and $O(mK_0)$ work per
query, with no all-pairs table. The loss does not target the covering-radius-optimal subset:
the objective averages query-specific bounds, whereas the covering radius takes a worst case
over vertices and ignores query demand. Section~\ref{sec:training-drift} documents the
difference empirically.

\paragraph{What straight-through can miss.} The straight-through update need not ascend
$F_{\mathrm{row}}$ and is not in general the gradient of a continuous extension of $F$;
Appendix~\ref{app:straight-through-example} gives a four-vertex counterexample for noise-free argmax. The class contains a discrete optimum, but the training rule may fail to find it. The SIMPLE estimator
\citep{ahmed2023simple} improves on AAC-ST's test objective in $18$ of the $19$
graph--budget pairs under directed accounting in Table~\ref{tab:continuous-greedy}. The mean test-objective ratios
to canonical greedy are $0.967$ for SIMPLE, $0.951$ for AAC-ST, and $0.995$ for AAC-CG.
SIMPLE therefore removes about one third of the objective gap between AAC-ST and AAC-CG.
Section~\ref{sec:continuous-greedy} varies the extension,
initialization, and rounding separately before interpreting the remaining difference.

\subsection{Continuous greedy on the matroid polytope (AAC-CG)}
\label{sec:theory-matroid}

Continuous greedy uses a different feasible space from AAC-ST.
AAC-ST optimizes a product of $m$ row simplices over $(\mathrm{slot}, \mathrm{landmark})$ choices.
Two slots can select the same landmark. The row constraint certifies $\hA$ (Theorem~\ref{thm:maximal-class}),
but the straight-through update has no approximation guarantee.

Continuous greedy instead operates on the matroid polytope $P(\mathcal{M})$ over landmark--direction items (Section~\ref{sec:objective}).
Each fractional point $x$ defines a mixture $h_x$ of subset heuristics.
Theorem~\ref{thm:alt-admissibility} and convexity make this mixture admissible.
It need not correspond to any fixed compression matrix $A$.

Under a fixed directional split, the distinct items selected by AAC-ST form an independent
set of the partition matroid $\mathcal M$. Every base is realizable by one-hot rows, and any
independent set can be padded to a base without lowering $F$ or exceeding the $m$-label budget.
Free direction allocation and shared tables for symmetric distances use a uniform matroid instead.

\begin{proposition}[Matroid structure and a guaranteed optimizer]
\label{prop:matroid-cg}
For the objective $F$ in Eq.~\eqref{eq:selection-objective}:
\begin{enumerate}
\item[(a)] Canonical greedy achieves $1-1/e$ under a cardinality constraint and $1/2$ under the
fixed split. A feasible base with no improving single-item exchange (a $1$-swap local optimum) achieves $1/2$ under either constraint.
The fixed-split $1/2$ bounds are tight, even for ALT instances.
\item[(b)] The multilinear extension
$\tilde F(x)=\mathbb E_{S\sim\mathrm{Bernoulli}(x)}[F(S)]$ is monotone and
diminishing-returns (DR) submodular.
For an integer $T\geq1$, start at $x_0=0$ and, for $k=0,\ldots,T-1$, set
\[
w_k=(\mathbf1-x_k)\odot\nabla\tilde F(x_k),\qquad
v_k\in\argmax_{v\in\Pbase(\mathcal M)}\langle v,w_k\rangle,\qquad
x_{k+1}=x_k+v_k/T.
\]
Choosing a vertex $v_k$ at each step gives approximation factor
$1-\exp(-(1-1/T)^{m-1})$ after pipage rounding, tending to $1-1/e$ as $T\to\infty$.
Both $\tilde F$ and its full gradient are computable in $O(K_0\log K_0)$ per query.
\item[(c)] Every $x\in[0,1]^{\mathcal I}$ defines an admissible heuristic
$h_x=\mathbb E_{S\sim\mathrm{Bernoulli}(x)}[h_S]$, with $0\leq h_x\leq d$.
For each query, $h_x$ is a non-negative combination of the clamped bounds
$B_i^+=\max(0,B_i)$ with total coefficient mass at most one.
\end{enumerate}
\end{proposition}
Monotonicity makes coordinate derivatives non-negative; DR-submodularity makes them
non-increasing in every coordinate. Part~(b) uses the marginal-gain weights of
\citet[Section~3.1]{calinescu2011maximizing}; Proposition~\ref{prop:eps-lmo}(a) proves the
finite-step factor, with $0^0:=1$. Lemma~\ref{lem:closed-form} gives the exact value and gradient
formulas, specializing the facility-location formulas of \citet{iyer2014monotone}.
Exchanging finite expectations gives $\tilde F(x)=\mathbb E_Q[h_x(s,t)]$: continuous greedy
optimizes the mean of an admissible heuristic.
The coefficients in~(c) depend on the query's bound order, so $h_x$ need not admit a fixed
compression matrix $A$.

\begin{proof}[Proof sketch]
The distinct items selected by one-hot rows satisfy each block's budget.
A feasible set is realizable by these rows if every block with a positive row count contains at least one selected item. Part~(a) applies
\citet{nemhauser1978analysis} and \citet{fisher1978analysis} to a monotone submodular objective, with
tightness from Example~\ref{ex:nine-vertex}. In part~(b), multilinearity gives a zero diagonal
Hessian, and submodularity makes every off-diagonal entry nonpositive. Thus all Hessian entries
are nonpositive, the twice-differentiable characterization of DR-submodularity
\citep{bian2017guaranteed}. Proposition~\ref{prop:eps-lmo}(a) gives the finite-$T$ factor for an
exact oracle whose outputs are vertices, and its limit is $1-1/e$. Part~(c) combines
Theorem~\ref{thm:alt-admissibility} for each $h_S$ with convexity of the admissible set. The
per-query expansion comes from sorting the clamped bounds. Appendix~\ref{app:matroid-cg} gives the full proofs.
\end{proof}

These are worst-case guarantees. At rank $r$, cardinality greedy has the stronger factor
$1-(1-1/r)^r$, equal to $0.6836$, $0.6564$, and $0.6439$ at $r=4,8,16$.
Its limit is $1-1/e$. Under a fixed split, both canonical greedy and the variant that fills
one direction first have factor $1/2$, although their objective values may differ.
The $1$-swap guarantee requires a local optimum. A run stopped after a fixed number of swaps may not reach one.

Example~\ref{ex:nine-vertex} makes the fixed-split bounds tight: greedy has ratio
$(N+1)/(2N)$, and the stated tie-break returns a $1$-swap local optimum.
The released greedy uses a different tie-break and returns the same value, but its set admits
an improving swap. Appendix~\ref{app:matroid-cg} also shows why a cardinality constraint does
not improve the $1$-swap guarantee to $1-1/e$. Section~\ref{sec:continuous-greedy} measures
the practical differences among these optimizers.

Every oracle base puts mass exactly $m_b$ on block $b$, so the iterate
$x_k = \tfrac{1}{T}\sum_{j<k} v_j$ has block mass $k\,m_b/T$ and only $x_T$ reaches the base
polytope. Our primary comparison uses forward Euler at $T{=}100$ with the exact oracle. The
certificate reported in Section~\ref{sec:continuous-greedy} is a submodular upper bound
computed from marginal gains, valid at any step count. Part~(c) requires clamping negative item bounds to zero;
Appendix~\ref{app:matroid-cg} gives a counterexample without this clamp. The mixture $h_x$ is a training-time object; the
deployed heuristic is $h_S$ for the rounded base $S$, admissible by
Theorem~\ref{thm:alt-admissibility}.

\paragraph{A differentiable linear oracle.} Each continuous-greedy iterate defines an admissible heuristic.
The exact linear oracle selects the top-$k$ items. Its map from landmark bounds to $x_T$ is piecewise constant, so its derivative is zero away from selection boundaries. An entropy-regularized oracle makes the map differentiable almost everywhere.
Within each block, we apply the sigmoid top-$k$ relaxation of \citet[Proposition~2]{wijk2025topk} to normalized marginal-gain weights.
Our continuous-greedy guarantee includes an additive oracle error and one more factor $1 - 1/T$ in the exponent of the approximation factor.

\begin{proposition}[Continuous greedy with an approximate oracle]
\label{prop:eps-lmo}
Let $F:2^{\mathcal I}\to\R_{\geq0}$ be monotone submodular with $F(\emptyset)=0$, $\tilde F$ its multilinear extension,
and $\mathcal M$ a rank-$m$ matroid, $m\geq1$, with independence polytope $P$ and base polytope
$\Pbase$. Let $\mathrm{OPT}$ maximize $F$ over bases and $f_{\max}=\max_i F(\{i\})$.
\begin{enumerate}
\item[(a)] \emph{Approximate oracle.} For an integer $T\geq1$, start at $x_0=0$ and set
$w_k=(\mathbf1-x_k)\odot\nabla\tilde F(x_k)$, $x_{k+1}=x_k+v_k/T$, where $v_k\in P$ and
\[
\langle v_k,w_k\rangle\geq\max_{v\in P}\langle v,w_k\rangle-\varepsilon_k,
\qquad \varepsilon_k\geq0.
\]
Define $\varepsilon=\max_{k<T}\varepsilon_k$ and
\begin{equation}
\label{eq:cg-factor}
c_T=\begin{cases}
(1-1/T)^{m-1},&\text{if every $v_k$ is a vertex of $P$},\\
(1-1/T)^m,&\text{otherwise},
\end{cases}
\qquad \alpha_T=1-e^{-c_T},
\end{equation}
with $0^0:=1$. Then $x_T\in P$ and
$\tilde F(x_T)\geq\alpha_T(F(\mathrm{OPT})-\varepsilon)$.
If every $v_k\in\Pbase$, deterministic pipage rounding returns a base $S$ with
$F(S)\geq\tilde F(x_T)$.
\item[(b)] \emph{Entropic oracle.} Suppose $\mathcal M$ is partition or uniform, with
$0<m_b<|\mathcal I_b|$ in every block. Normalize $w\geq0$ by each block's largest entry,
setting an all-zero block to zero, and denote the result by $\bar w$.
For $\tau>0$, define
\[
v_\tau(w)=\argmax_{v\in\Pbase}\{\langle v,\bar w\rangle+\tau H(v)\},
\qquad H(v)=\sum_i h_2(v_i),\qquad
H_{\max}=\max_{v\in\Pbase}H(v).
\]
For the weights in~(a), this oracle has error at most $\tau f_{\max}H_{\max}$, with
$H_{\max}\leq m\log(e|\mathcal I|/m)$.
It has a unique blockwise logistic solution and is $C^\infty$ wherever each block has a
unique positive largest weight.
\item[(c)] \emph{AAC-CG.} Apply (a)--(b) to the empirical objective $F=\hat F_n$ in
Eq.~\eqref{eq:selection-objective}, as a function of its $n\times|\mathcal I|$ bound matrix.
The entropic oracle is interior, so $c_T=(1-1/T)^m$, and pipage rounding gives
\[
F(S)\geq\alpha_T\big(F(\mathrm{OPT})-\tau f_{\max}H_{\max}\big).
\]
Every fractional heuristic is admissible. As functions of the bound matrix, $x_T$ and
every training-query value $h_{x_k}(s_j,t_j)$ are differentiable almost everywhere in the ambient matrix space.
\end{enumerate}
\end{proposition}
Both choices of $\alpha_T$ are at most $1-1/e$ and converge to it. The proof also gives the
sharper finite-step coefficient $1-(1-c_T/T)^T$.
For AAC, $|\mathcal I|\leq2K_0$, so $H_{\max}\leq m\log(2eK_0/m)$.
Differentiability in~(c) is an ambient-space statement: a structured family may lie entirely
on tie surfaces and needs separate analysis.
Appendix~\ref{app:eps-lmo-proof} proves the proposition and derives the logistic form and its scalar root equation.
The main argument for part~(a) is a coupling. With $R \sim \mathrm{Bernoulli}(x_k)$ and an independent $J \sim
\mathrm{Bernoulli}(v_k/T)$, every item lies in $R \cup J$ with probability at most $x_{k+1,i}$, so
monotonicity gives $\tilde F(x_{k+1}) \geq \mathbb{E}\,F(R \cup J) \geq \tilde F(x_k) + \sum_i
\Pr[J = \{i\}]\, w_{k,i}$. The probability $\Pr[J = \{i\}]$ is at least $v_{k,i}(1 - 1/T)^{m-1}/T$
at a vertex and $v_{k,i}(1 - 1/T)^{m}/T$ anywhere in $P$. Submodularity gives
$\langle \mathbf{1}_{\mathrm{OPT}}, w_k\rangle \geq F(\mathrm{OPT}) - \tilde F(x_k)$. With
$c$ that power of $1 - 1/T$ and $\Delta_k := F(\mathrm{OPT}) - \tilde F(x_k)$, the oracle then gives
$\Delta_{k+1} - \varepsilon \leq (1 - c/T)(\Delta_k - \varepsilon)$, and $(1 - c/T)^T \leq e^{-c}$
gives the bound. When $F(\mathrm{OPT}) < \varepsilon$, non-negativity gives the bound directly.

Part~(b) requires $\tau>0$ and an unsaturated budget in each block.
Positive temperature makes the regularized objective strictly concave. An unsaturated budget keeps the maximizer interior and the logistic root finite. A
saturated block has a single feasible point: we fix its oracle coordinates at one. This logistic form applies to partition and uniform matroids. Since
$\mathcal{B}_{\mathrm{pair}}$ is not a matroid on $\mathcal{I}$, it enters only through its own
ground set of $K_0$ paired items at rank $m/2$, where part~(a)'s exponent reads $m/2 - 1$ for a
vertex oracle and $m/2$ for a fractional one.

In~(c), changes of the bound order, the clamp, or the maximizing normalization entry can create nondifferentiable points; the appendix treats their finite-step composition.

Differentiability here refers to the trajectory and the fractional heuristic; the rounded
deployed base is piecewise constant in the bounds. The oracle term in~(c) is in the objective's
units, because $0 \leq w_i \leq \|\nabla \tilde F(x)\|_\infty \leq f_{\max} \leq D$ for the range
constant $D$ of Proposition~\ref{prop:erm}. For the partition and uniform matroids in~(b), every singleton is feasible. Hence, if $f_{\max}>0$, then $f_{\max} \leq F(\mathrm{OPT}) \leq m f_{\max}$.
The oracle deduction in~(c) is smaller than $F(\mathrm{OPT})$ whenever
$\tau H_{\max} < F(\mathrm{OPT})/f_{\max}$, a ratio between $1$ and $m$;
$\tau H_{\max} < 1$ suffices but is not necessary. Measured against the exact-oracle run, the ratio
of rounded training objectives stays within $[0.992, 1.007]$ on SBM and BA for $\tau$ from
$10^{-4}$ to $1$ (Table~\ref{tab:soft-cg}). On the tested road networks in
Table~\ref{tab:oracle-expansions}, at $\tau \in \{10^{-4}, 0.01, 0.1\}$, it stays within
$[0.988, 1.002]$, lowest on NY at $32$\,B/v with $\tau = 0.1$.

\subsubsection{Finite-sample guarantee}
\label{sec:finite-sample}

The guarantee so far concerns the empirical objective $\hat F_n$. We now bound the population
objective $F$ for selectors trained on $n$ queries from $Q$. The selected set must be a base
of a family in Definition~\ref{def:deployable-families}. This includes AAC-CG and the discrete baselines.
For an AAC-ST output with duplicate rows, the guarantee applies after padding its distinct items
to a base. We apply Hoeffding's inequality and a union bound to the finite family of deployable
sets. The family size and the range of the heuristic determine the resulting bound.

Fix the graph and pool independently of the training queries. Assume all pool landmarks and query endpoints lie in one component, strongly connected for directed graphs. Let $\mathcal S$ be a family of
Definition~\ref{def:deployable-families}, $M=|\mathcal S|$, and let $\hat F_n$ use $n\geq1$
i.i.d.\ queries from $Q$.

\begin{proposition}[Finite-sample selection]
\label{prop:erm}
Under this setup, define the query range and label range by
\[
D_Q=\max_{(s,t)\in\operatorname{supp}Q}d(s,t),\qquad
D_L=\max_{i\in\mathcal I}\left(\max_v z_i(v)-\min_v z_i(v)\right),
\]
where $z_{(l,\mathrm{fwd})}(v)=d(l,v)$, $z_{(l,\mathrm{bwd})}(v)=d(v,l)$, and $z_l(v)=d(l,v)$
in the undirected case; vertex extrema range over the query component.
\begin{enumerate}
\item[(a)] \emph{Range.} Every $S\in\mathcal S$ satisfies $0\leq h_S(s,t)\leq D:=\min(D_Q,D_L)$
on $\operatorname{supp}Q$, and $D_L$ is at most the component diameter.
\item[(b)] \emph{Uniform deviation.} For $\delta\in(0,1)$, with probability at least $1-\delta$,
\[
\sup_{S\in\mathcal S}|\hat F_n(S)-F(S)|\leq\varepsilon_n,
\qquad \varepsilon_n=D\sqrt{\frac{\log(2M/\delta)}{2n}}.
\]
For every stated family, $\log M\leq m\log(2eK_0/m)$.
\item[(c)] \emph{Approximate maximizers.} On this event, if $\hat S\in\mathcal S$, $\alpha\geq0$, and
$\hat F_n(\hat S)\geq\alpha\max_{S\in\mathcal S}\hat F_n(S)-\gamma$, then
\[
F(\hat S)\geq\alpha\max_{S\in\mathcal S}F(S)-\gamma-(1+\alpha)\varepsilon_n.
\]
An empirical maximizer therefore has population optimality gap at most $2\varepsilon_n$.
\end{enumerate}
\end{proposition}
Appendix~\ref{app:erm-proof} gives the proof.

The teacher table does not contain every query distance needed to compute $D_Q$.
It does give the upper bound $2 \min_{l} \max_{v} d_{\mathrm{sym}}(v, l)$.
We therefore use $D = \min\big(2 \min_{l} \max_{v} d_{\mathrm{sym}}(v, l),\ D_L\big)$
in the numerical evaluation, with the vertex maximum restricted to the query component.
Computing the label range requires one pass over the stored table.
Both range bounds depend on the graph and pool, rather than on a sampled query maximum.
The sizes of $\mathcal{B}_{\mathrm{pair}}$, $\mathcal{B}_{\mathrm{fix}}$ and
$\mathcal{B}_{\mathrm{free}}$ are $\binom{K_0}{m/2}$,
$\binom{K_0}{m_{\mathrm{fwd}}}\binom{K_0}{m_{\mathrm{bwd}}}$ and $\binom{2K_0}{m}$. Vandermonde's
identity bounds the fixed-split count by the free-split count.
For broader heuristic classes, \citet{sakaue2022sample} bound generalization through pseudo-dimension.
On undirected graphs, the deployable sets are bases of the uniform
matroid $U_{m,K_0}$, with $M = \binom{K_0}{m}$.

The generalization study uses SBM at $K_0{=}16$, $m{=}4$ and a fixed reference sample of
$20{,}000$ queries. Greedy fits $n$ independently drawn queries. We evaluate part~(c) using
greedy's approximation factor $\alpha$ and the range constant above.
When branch-and-bound certifies the reference sample's optimum within its node budget, we use
that optimum. Otherwise, we use greedy's value on the reference sample. The latter comparison
measures greedy's stability rather than excess risk.
Both comparisons use a finite reference sample, rather than the population optimum.
We also report the largest heuristic value observed on that sample.

\begin{theorem}[AAC-CG selection guarantee]
\label{thm:aac-guarantee}
Under Proposition~\ref{prop:erm}'s sampling assumptions, let
$\mathcal S\in\{\mathcal B_{\mathrm{fix}},\mathcal B_{\mathrm{free}}\}$ have rank $m$,
$M=|\mathcal S|$, and let $\mathrm{OPT}$ maximize the population objective $F$ over $\mathcal S$.
For integers $n,T\geq1$ and $\delta\in(0,1)$, run $T$ steps of Proposition~\ref{prop:eps-lmo}(a) on $\hat F_n$, with oracle error at most
$\varepsilon$ and every oracle output in the base polytope. Let $S$ be the pipage-rounded base.
Then $h_S$ and every fractional iterate are admissible, and with probability at least $1-\delta$,
\[
F(S)\geq\alpha_T F(\mathrm{OPT})-\alpha_T\varepsilon
-(1+\alpha_T)D\sqrt{\frac{\log(2M/\delta)}{2n}},
\]
where $\alpha_T$ is defined in Eq.~\eqref{eq:cg-factor} and $D=\min(D_Q,D_L)$.
For the entropic oracle of Proposition~\ref{prop:eps-lmo}(b),
$\varepsilon\leq\tau f_{\max}H_{\max}$ with $f_{\max}=\max_i\hat F_n(\{i\})$,
and the trajectory and fractional training-query values are differentiable almost everywhere in the
ambient bound-matrix space.
\end{theorem}
\begin{proof}
Proposition~\ref{prop:eps-lmo}(a) applied to $\hat F_n$, followed by pipage rounding, gives
$\hat F_n(S) \geq \alpha_T \big(\max_{S'} \hat F_n(S') - \varepsilon\big)$, i.e.\ $S$ is an
$\alpha$-approximate empirical maximizer with $\alpha = \alpha_T$ and $\gamma = \alpha_T \varepsilon$;
Proposition~\ref{prop:erm}(c) converts this to the population with loss $(1 + \alpha_T)\varepsilon_n$.
Theorem~\ref{thm:alt-admissibility} gives admissibility of $h_S$, and
Proposition~\ref{prop:matroid-cg}(c) gives admissibility of the fractional iterates.
Differentiability follows from Proposition~\ref{prop:eps-lmo}(c).
\end{proof}
\paragraph{Interpreting the guarantee.} The theorem requires exact distances, with the pool and
queries in one component, strongly connected for directed graphs. The pool must be fixed before drawing the training queries.
These queries must be i.i.d.\ from the evaluation distribution $Q$, and the oracle must return base-polytope points.

The bound separates optimization, oracle error, and sampling error.
The finite-step factor $\alpha_T$ depends only on $T$ and $m$ and approaches $1-1/e$.
In $c_T$, the power of $1-1/T$ is $m-1$ for the exact top-$k$ oracle and $m$ for the entropic oracle.
The oracle penalty $\alpha_T\varepsilon$ is zero for the exact oracle and at most
$\alpha_T\tau f_{\max}H_{\max}$ for the entropic oracle.
The sampling penalty converts the empirical guarantee to a population guarantee.
Its coefficient $1+\alpha_T$ is at most $2-1/e$.
Since $f_{\max}\leq D$, the label range also bounds the scale of the oracle penalty.

The comparator $\mathrm{OPT}$ is the best deployable set from the supplied pool.
The theorem does not compare against all landmarks in the graph or bound A* expansions.

At the deployed sample size, the bound does not certify a population advantage over
classical greedy. Appendix~\ref{app:guarantee-constants} reports the finite-step constants
and the sample-size calculations, distinguishing sample optima from population optima.

FPS has no positive worst-case approximation factor for the objective of Eq.~\eqref{eq:selection-objective}, even for selection from an FPS pool (Appendix~\ref{app:fps-covering-proof}).
Gonzalez's factor concerns the covering radius, which bounds but does not order that objective
(Corollary~\ref{cor:fps-covering}(b)). FPS inherits ALT's admissibility, but it uses no query sample
and provides no differentiable selection trajectory.

\subsection{Hardness, direction allocation, and coverage}
\label{sec:theory-limits}

\paragraph{Approximation limits.} Continuous greedy approaches a $1-1/e$ approximation.
We now ask whether a polynomial-time selector can improve this factor. For a monotone submodular \emph{value oracle} the answer
is classical \citep{nemhauser1978best}, but that result concerns a query model, and here the input
is given explicitly. Query-weighted selection from arbitrary admissible candidates is NP-hard
\citep[Prop.~1]{rayner2013subset}; that reduction does not restrict candidates to landmark bounds.
Hardness also holds for minimizing ALT search space \citep{bauer2010hard} and for a landmark-cover
objective used in distance estimation \citep{potamias2009fast}. \citet{fuchs2010alt} reduces MinALT
to maximum coverage, giving a greedy guarantee on the number of excluded vertices in his search-space model
(Section~\ref{sec:related}). The following reduction goes the other way: it realizes maximum
coverage as the objective of Eq.~\eqref{eq:selection-objective}, with a supplied pool and a
supplied query distribution.

Given a maximum-coverage instance $(U,\{C_1,\ldots,C_{K_0}\})$, with $U\neq\emptyset$ and
budget $1\leq m\leq K_0$, fix $w_0\geq1$ and construct the undirected graph
\[
V \;=\; \{a_e, b_e : e \in U\} \;\cup\; \{\ell_1, \dots, \ell_{K_0}\},
\qquad
\begin{aligned}
w(a_e, b_e) &= 1, \\
w(\ell_j, a_e) &= w_0, \\
w(\ell_j, b_e) &= w_0 + \mathbf{1}[e \in C_j],
\end{aligned}.
\]
Use pool $L=\{\ell_1,\ldots,\ell_{K_0}\}$ and let $Q$ be uniform on
$\{(a_e,b_e):e\in U\}$.

\begin{theorem}[Selection from a supplied pool is maximum-coverage-hard]
\label{thm:hardness}
For this construction and every $S\subseteq L$,
\begin{equation}
\label{eq:hardness-identity}
F(S) \;=\; \mathbb{E}_{(s,t) \sim Q}\big[h_S(s,t)\big] \;=\; \frac{1}{|U|}\,\Big|\bigcup_{j \,:\, \ell_j \in S} C_j\Big|,
\end{equation}
so maximizing $F$ under a cardinality constraint is NP-hard and, for every $\varepsilon > 0$, no
polynomial-time algorithm attains ratio $1 - 1/e + \varepsilon$ unless P${=}$NP
\citep{feige1998threshold}.
\end{theorem}

\begin{proof}[Proof sketch]
A detour through another landmark costs at least $3w_0$, and a detour within a query pair costs at least $w_0+1$. Since $w_0\geq1$, neither improves a direct landmark edge, so
$|d(\ell_j, a_e) - d(\ell_j, b_e)| = \mathbf{1}[e \in C_j]$. The heuristic on a query pair is
therefore the indicator that $e$ is covered. Averaging over $Q$ gives
Eq.~\eqref{eq:hardness-identity}, so any approximation ratio transfers to maximum coverage and
\citet{feige1998threshold} applies. Appendix~\ref{app:hardness} derives the distances.
\end{proof}

The construction has $K_0 + 2|U|$ vertices and $|U|(2K_0+1)$ edges, and it takes $O(K_0|U|)$ time.
The hardness holds already with $w_0=1$, using only edge weights $1$ and $2$, on an undirected graph. It also carries to
the symmetric digraph when we retain separate directional items and charge each copy. There the backward copy of a landmark carries the negative
of the forward bound, which the clamp sends to zero. Each family of Definition~\ref{def:deployable-families} then reduces to maximum coverage at the budget it can spend
forward: $m$ for $\mathcal{B}_{\mathrm{free}}$ on the restricted instances $m\leq K_0$, $m_{\mathrm{fwd}}$ for $\mathcal{B}_{\mathrm{fix}}$
and $m/2$ for $\mathcal{B}_{\mathrm{pair}}$ (Appendix~\ref{app:hardness}).

\paragraph{Scope of the hardness result.} The input to Theorem~\ref{thm:hardness} includes the
pool and the query distribution, as in the query-weighted utility of \citet{rayner2013subset}. The
theorem does not establish hardness for FPS-generated pools, uniform all-pairs queries, or the
minimization of A* expansions.
In the exact-selection experiments, we use FPS pools and uniform queries over the largest component.
Greedy is within $0.5$\,\% of the certified training optimum wherever the MILP proves optimality
(Table~\ref{tab:exact-selection}). Theorem~\ref{thm:hardness} therefore says that no
polynomial-time method improves the leading factor uniformly over all inputs; it does not explain
the observed gaps between selectors. Within that scope, AAC-CG's limiting factor is worst-case optimal, as is greedy's under a cardinality constraint.
An exact solver or local search may still reach a higher objective on a particular instance.

Under a cardinality constraint, which covers
DIMACS with shared distance tables and the undirected synthetic graphs, continuous greedy cannot improve on
greedy's guarantee. The experiments compare the objectives reached by exact-oracle continuous greedy and canonical greedy; Tables~\ref{tab:soft-cg} and~\ref{tab:oracle-expansions} separately evaluate the entropic oracle.
Under the fixed split, the separation the theory supports is between $1/2$ and $1 - 1/e$. Coverage
and the query objective also differ in approximability: the $m$-center radius over all vertices is
$2$-approximable by FPS, while selection from a supplied pool under a supplied $Q$ is not
$(1 - 1/e + \varepsilon)$-approximable. The two problems differ in candidate set and objective,
and neither result implies the other.

\paragraph{Direction allocation.} A per-direction quota of any value makes the constraint a
partition matroid. Deployed memory constrains the \emph{number} of items, not their directions, so
lifting the quota gives a larger family.

\begin{proposition}[Direction allocation at fixed memory]
\label{prop:free-split}
For the directed families of Definition~\ref{def:deployable-families} at even $m$:
\begin{enumerate}
\item[(a)] \emph{Nesting.}
$\mathcal B_{\mathrm{pair}}\subseteq\mathcal B_{\mathrm{fix}}\subseteq\mathcal B_{\mathrm{free}}$,
so $\mathrm{OPT}_{\mathrm{pair}}\leq\mathrm{OPT}_{\mathrm{fix}}\leq\mathrm{OPT}_{\mathrm{free}}$.
The fixed and free families are bases of a partition and a uniform matroid on $\mathcal I$.
The paired family is uniform of rank $m/2$ on $K_0$ paired items; it is not a matroid base
family on $\mathcal I$ itself.
\item[(b)] \emph{Strictness.} The nine-vertex instance of Example~\ref{ex:nine-vertex}, with
$m=2$ and $N\geq2$, has
\[
\mathrm{OPT}_{\mathrm{pair}}=\frac{N+1}{2N+1}
<\mathrm{OPT}_{\mathrm{fix}}=\frac{2N}{2N+1}<\mathrm{OPT}_{\mathrm{free}}=1.
\]
\item[(c)] \emph{Guarantees.} Canonical greedy achieves $1-1/e$ for the paired and free
families and a tight $1/2$ for the fixed split. Continuous greedy achieves
Proposition~\ref{prop:eps-lmo}(a)'s factor at the corresponding matroid rank, tending to
$1-1/e$; every fractional heuristic is admissible.
\item[(d)] \emph{Parameterization.} Top-$m$ selection from $2K_0$ scores parameterizes
$\mathcal B_{\mathrm{free}}$ with $O(K_0)$ parameters. Each output deploys an admissible
ALT heuristic with unrestricted direction counts.
\end{enumerate}
\end{proposition}
On symmetric distances, the fixed and free families coincide. A paired-storage base has
$m/2$ landmarks and extends to an $m$-landmark base, so the same ordering of optimal values
follows from monotonicity, without literal inclusion of the base families.
Appendix~\ref{app:free-split-proof} gives the proof.

\paragraph{Admissibility and selection constraints.} Theorem~\ref{thm:maximal-class} characterizes
the coefficients that preserve admissibility within a linear row. The matroid constrains
which items a deployed set may contain. The fractional heuristic $h_x$ is admissible because
it mixes admissible subset heuristics (Proposition~\ref{prop:matroid-cg}(c)); its approximation
guarantee comes from submodularity and the matroid constraint (Proposition~\ref{prop:eps-lmo}).
Thus admissibility and optimization quality require separate arguments.

On directed graphs, the direction quotas determine the matroid and its approximation guarantees. Under the fixed split, greedy and
$1$-swap local search have a tight approximation factor of $1/2$. Continuous greedy has factor
$\alpha_T$ at finite $T$, approaching $1 - 1/e$ as $T \to \infty$
(Appendix~\ref{app:guarantee-constants}). Under the free split, the constraint is a uniform matroid.
Greedy then has factor $1 - 1/e$, and the $1$-swap guarantee remains $1/2$.
The free-split optimum can also be larger. Section~\ref{sec:continuous-greedy} evaluates both splits.

\paragraph{Coverage and query quality.}
For an undirected graph, ALT on a landmark set of covering radius $r$ has gap at most $2r$.
The directed version uses separate forward and backward radii under the symmetrized distance.
These bounds constrain the worst-case gap but do not order selectors by expected heuristic
value or search effort. Appendix~\ref{sec:theory-coverage} states the bounds and gives a
seven-vertex example separating coverage from the query objective. Figure~\ref{fig:toy-p7} shows the distinction on a path.

\begin{figure}[t]
\centering
\includegraphics[width=\linewidth]{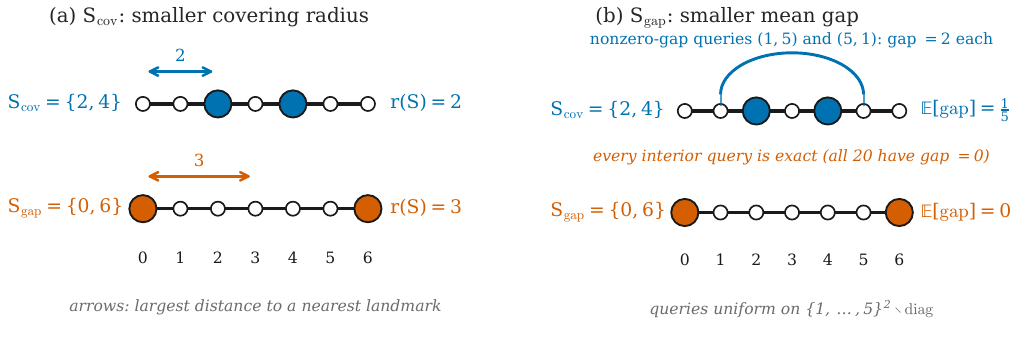}
\caption{\textbf{Covering radius and heuristic gap rank selections differently on the path $P_7$.}
$S_{\mathrm{cov}}$ has a smaller covering radius but a larger gap than $S_{\mathrm{gap}}$.
The plotted gap $\mathbb{E}_Q[d - h_S]$ is the mean difference between distance and the subset heuristic.
It differs from the gap-to-teacher loss by a constant independent of $S$
(Proposition~\ref{prop:gradient-equiv}), so both objectives rank subsets identically.
Eight of the $21$ two-subsets share the minimum radius, and seven of those have gap $0$.}

\label{fig:toy-p7}
\end{figure}

\subsection{Contextual variant}
\label{sec:contextual}
Both relaxations support context-dependent edge costs through an end-to-end differentiable
pipeline. An encoder predicts the costs. During training, a normalized soft-min Bellman--Ford
fixed point produces the landmark labels on the predicted-cost graph. At a hop into vertex $v$, smoothing adds at most $\log(n_v)/\beta$ to the distance bound.
Here $n_v$ counts eligible incoming arcs, and $\beta>0$ is the inverse temperature. One adjoint linear solve gives the implicit derivative of
the converged fixed point. At deployment, we recompute exact distances and evaluate ALT on the
selected items. The deployed heuristic is therefore admissible on the predicted-cost graph
(Theorem~\ref{thm:alt-admissibility}).
Appendix~\ref{app:contextual} specifies the pipeline and evaluates it on Warcraft using the
dataset's shortest-path labels.

\section{Related work}
\label{sec:related}

\paragraph{Exact routing and distance oracles.} Two-hop labels answer distance queries by
intersecting endpoint labels \citep{cohen2003reachability}. Contraction hierarchies, hub labeling, and
customizable routing exploit graph hierarchy for fast exact queries
\citep{geisberger2008contraction,abraham2012hierarchical,delling2011crp,dibbelt2016cch,bast2016route}.
CH-Potentials extracts an admissible A* heuristic from a contraction hierarchy
\citep{strasser2021chpotentials}. Pruned and landmark hub labeling answer distances through
per-vertex label intersections \citep{akiba2013pll,storandt2022lhl,storandt2024scalable,cauvi2024landmark}.
EHL* explicitly trades label memory for query time \citep{du2025ehlstar}.
We make no performance claim against these specialized engines. AAC selects an ALT lower bound
for A* and supports differentiable training; it does not answer exact distance queries directly.

\paragraph{Admissible compression.} Differential heuristics coincide with ALT on undirected
graphs \citep{sturtevant2009memory}. CDH stores a different pivot subset at each vertex
\citep{goldenberg2011cdh,goldenberg2017cdh}. Its goal-side bounding search supplies bounds for
missing pivot distances, at an additional expansion cost. AAC uses one selection map for all
vertices, so every retained dimension is directly available at both query endpoints.
Section~\ref{sec:cdh-headtohead} measures this tradeoff, including the bounding search.
FastMap with the linear residual projection of \citet{cohen2018fastmap} gives an admissible
$\ell_1$ heuristic on undirected graphs; the original squared-residual projection
\citep{faloutsos1995fastmap} lacks this guarantee. FM+DH combines the former with pivot bounds.
Euclidean heuristic optimization learns embeddings under admissibility and consistency
constraints \citep{rayner2011euclidean}. AAC instead constrains linear combinations of
precomputed landmark distances.

Architectural admissibility has broader predecessors in pattern databases, compressed pattern
databases, and merge-and-shrink abstractions
\citep{culberson1998pdb,felner2007pdb,helmert2014mergeandshrink}.
AAC applies this principle through differentiable convex combinations in the ALT family.
NEAR uses neural relaxations for program synthesis; its approximate admissibility depends on
representation capacity, optimization, and statistical error \citep[Section~3.1]{shah2020near}. Other learned optimal-search heuristics obtain
admissibility outside the architecture \citep{li2022optimal}; concurrent cost partitioning
provides an architectural guarantee for a different heuristic class \citep{barral2026cost}.

\paragraph{Learning and differentiable planning.} Truncated-Gaussian training lower-bounds a
learned distribution by a classical heuristic; it does not upper-bound every prediction by the
true distance \citep{nunezmolina2024admissible}. CEA-loss training targets near-admissibility
and comes with a generalization analysis \citep{futuhi2026learning}. \citet{hadar2026beyond} train heuristics with limited-horizon Bellman updates, without
imposing architectural admissibility.
The neural heuristics cited in Section~\ref{sec:introduction} generally lack an architectural admissibility guarantee.
A search wrapper can recover a guarantee only under its own assumptions, such as an admissible
anchor; weighted A* alone does not certify an arbitrary learned heuristic.
Distance-estimation methods likewise need not provide lower bounds \citep{choudhary2026empirical}.
Blackbox differentiation through Dijkstra \citep{vlastelica2020differentiation}, DataSP
\citep{lahoud2024datasp}, and CombOptNet \citep{paulus2021comboptnet} instead make combinatorial
solvers trainable. AAC supplies a differentiable heuristic layer with ordinary ALT deployment.

\paragraph{Submodular selection and differentiable optimization.}
\citet{rayner2013subset} formulated query-weighted admissible-heuristic selection as
monotone submodular maximization, proved NP-hardness, and obtained greedy's $1-1/e$ cardinality
bound. Their directed differential heuristics distinguish to- and from-pivot items.
\citet[Theorem~8]{fuchs2010alt} reduces MinALT to maximum coverage. In his search-space model, greedy excludes at
least $1-1/e$ of the optimal number of excluded vertices; this factor applies to excluded
vertices, not to the search-space size that MinALT minimizes.
Our hardness reduction runs in the opposite direction and realizes maximum coverage as the
expected ALT bound of an undirected instance with a supplied pool and query distribution. We add matroid formulations for fixed and free direction allocation.
We also give admissible differentiable optimization with finite-step, oracle-error, and finite-sample terms.

Differentiable greedy methods smooth either the sequence distribution
\citep{tschiatschek2018differentiable} or each discrete argmax
\citep{sakaue2021differentiable}; the latter nearly preserves greedy's guarantee in expectation.
\citet{wilder2019melding} optimize the multilinear extension by projected stochastic ascent
with lossless rounding in a decision-focused pipeline. We instead entropy-regularize continuous
greedy's linear oracle over the base polytope. This is Nesterov smoothing, equivalently a
Fenchel--Young prediction map \citep{nesterov2005smooth,blondel2020fenchel}, and yields the
deterministic oracle penalty of Proposition~\ref{prop:eps-lmo}.
For nonnegative monotone DR-submodular objectives over convex feasible sets, stationary points give a $1/2$ approximation \citep[Theorem~4.2]{hassani2017gradient};
non-oblivious ascent and stochastic continuous greedy offer $1-1/e$ variants
\citep{zhang2022boosting,mokhtari2018conditional}.
Concurrent work uses adaptive active sets to reduce continuous greedy's communication cost
\citep{rostami2026adaptive}.

\paragraph{Heuristic quality and generalization.}
\citet{eden2022embeddings} bound the storage versus search-time tradeoff for feature-based
heuristics. Our covering theorem instead bounds the heuristic gap within ALT, and does not
order A* running time. \citet{sakaue2022sample} study sample complexity for learned search
heuristics. Proposition~\ref{prop:erm} gives the finite-family bound needed for our deployable
landmark selections.

\section{Experiments}
\label{sec:experiments}

We test whether query-adaptive selection reduces search effort at fixed label memory.
We first compare AAC-CG, canonical greedy, and FPS-ALT (Section~\ref{sec:primary}).
We then examine optimizer choices, the relation between objective and search effort, and robustness to query changes.
The remaining studies examine straight-through training, check path optimality, and test compatibility with an upstream encoder.
Each study states its memory accounting and query sample.

The test objective is the mean heuristic value on test queries. Unless stated otherwise, we
report this value relative to canonical greedy on the same queries. A ratio of $1$ means equal
values; larger ratios are better. A graph--budget pair combines one graph with one deployed
label-memory budget in bytes per vertex. Unless a study specifies additional factors, counts and
averages refer to these pairs. Each pair receives equal weight when we average ratios.
We report expansions separately: a larger test objective need not mean fewer expansions.

\paragraph{Methods and naming.} AAC-ST trains the row-stochastic selector with a straight-through
estimator. AAC-CG uses continuous greedy and pipage rounding. We vary their oracles, rounding, and
initialization in the ablations. \emph{Greedy} denotes canonical matroid greedy.
At each step, it adds the feasible item with the largest marginal gain.
\emph{Block} uses the same rule within each direction block. Under fixed direction quotas, it fills the forward block first.
The $k$-swap local-search and coordinate-descent (CD) baselines start from
Greedy. Projected gradient ascent (PGA), SIMPLE, and the other differentiable selectors are
optimizer controls using the same objective and candidate pool. A mixed-integer linear program
(MILP) supplies reference optima and certified upper bounds.

\emph{Greedy-Max} is a landmark-level greedy baseline used in the selection-strategy ablation
and the non-road study. On directed graphs it selects $m/2$ landmarks and stores both directions
of each. On undirected graphs, it selects $m$ landmarks. We fit it to a separate $100$-query selection
sample. We use $3{,}200$ training queries for item-level Greedy unless stated otherwise.
Item-level Greedy can choose different forward and backward landmarks. FPS-ALT denotes ALT with
farthest-point landmarks. We shorten this name to ALT in tables where the choice is unambiguous.

\subsection{Methodological protocol}
\label{sec:methodological-protocol}
We compare ALT-family methods on road, weighted SBM, BA, and OGB-arXiv graphs, plus a Warcraft
compatibility task. All use positive-weight graph distances. We do not benchmark dynamic graphs,
specialized routing engines, or loss-constrained neural heuristics head to head.

We use two one-sided tests (TOST) with an equivalence margin of $\Delta=1$ percentage point
\citep{lakens2017tost}. Equivalence within this margin does not mean equality or statistical
indistinguishability. We also report paired Wilcoxon tests with false discovery rate (FDR) correction.
Where repeated-seed data are available, we use crossed seed-by-query intervals to account for both sources of variation.
We count wins on one $500$-query test split per graph--budget pair in the primary comparison.
These counts are descriptive. We re-evaluate a subset on fresh queries in Section~\ref{sec:robustness}
and report uncertainty for OGB-arXiv in Section~\ref{sec:ogb-main}.

We report differences between expansion-reduction percentages in percentage points (pp).
We label relative reductions in expansion counts explicitly. Section~\ref{sec:setup} defines the memory accounting.

\subsection{Experimental setup}
\label{sec:setup}

\paragraph{Hardware.} We ran all experiments on one machine with an Intel Core Ultra 9 285K CPU and 128\,GB RAM.
The GPU was an NVIDIA GeForce RTX 5090 with 32\,GB VRAM. We used Python 3.14 and PyTorch 2.14.1+cu130.

\paragraph{Seeds and queries.} Unless stated otherwise, we use five seeds
$\{42,123,456,789,1024\}$ and report mean $\pm$ one standard deviation. We evaluate methods on one
shared source--target sample drawn with seed $42$. AAC-ST seeds control initialization, Gumbel draws,
and training queries. FPS is deterministic because we fix its component anchor.
FastMap seeds control pivot-search starts. In the contextual study, we use eight seeds for the
CNN at $K_0=16,m=8$ and three for each other configuration (Appendix~\ref{sec:contextual-study}).

\paragraph{Datasets.}
\begin{itemize}
\item \textbf{DIMACS} road networks~\citep{demetrescu2009shortest}: NY ($\sim$264K nodes), BAY ($\sim$321K), COL ($\sim$436K), FLA ($\sim$1.07M). Weighted arc lists with equal-weight reverse arcs.
\item \textbf{OSMnx} road networks~\citep{boeing2017osmnx}: Modena ($25{,}642$ nodes), Manhattan ($13{,}169$), Berlin ($41{,}209$), Los Angeles ($92{,}310$). Directed.
\item \textbf{Warcraft} 12$\times$12 grid maps~\citep{vlastelica2020differentiation}: 10{,}000 train / 1{,}000 test. 144 nodes, 8-connected. RGB terrain images (96$\times$96).
\end{itemize}

\paragraph{Straight-through training.} Static studies use $K_0\in\{32,64,128\}$ and
$m\in\{8,16,32,64\}$ unless stated otherwise. We run AdamW for $200$ epochs with batch size $256$.
We use learning rate $10^{-3}$, zero weight decay, $\lambda_{\mathrm{ent}}=0.01$, and $\lambda_{\mathrm{uniq}}=0$.
The Gumbel temperature decreases exponentially from $1.0$ to $0.1$.
Each epoch is one minibatch update. We turn off the entropy penalty in optimizer comparisons.

\paragraph{Baselines.}
\begin{itemize}
\item \textbf{ALT}: $K \in \{4, 8, 16, 32\}$ landmarks, farthest-point sampling. Memory: $2K \times 4$ bytes/vertex on directed graphs, or $K \times 4$ when forward and backward tables coincide.
\item \textbf{FastMap}~\citep{cohen2018fastmap}: $d \in \{8, 16, 32, 64\}$ dimensions. L1 (Manhattan) distance in the FastMap embedding. Memory: $d \times 4$ bytes/vertex.
\item \textbf{Dijkstra}: no heuristic (expansion baseline).
\end{itemize}

\paragraph{Matched deployed label memory.} We match the bytes occupied by deployed labels,
not preprocessing cost or query latency. We compute distances in float64 and store them in
float32. At $\Bmem$ bytes per vertex, AAC stores $m=\Bmem/4$ labels. Directed ALT stores
$K=\Bmem/8$ landmarks with both distance directions; undirected ALT stores $K=\Bmem/4$.
The primary comparison also uses shared tables on DIMACS, whose reverse arcs have equal weights.
Secondary directed-accounting studies charge each direction separately. Directed selectors split
$m$ into $\lfloor m/2\rfloor$ forward and $\lceil m/2\rceil$ backward labels by default.
The free-split study relaxes these quotas for AAC-CG and greedy. Table~\ref{tab:multi-axis-cost} reports the other costs.

\paragraph{Validation and test protocols.} By default, we evaluate $100$ shared queries across five training seeds.
In the validation-split protocol, we select $K_0$ on a separate $100$-query validation sample (Table~\ref{tab:valsplit}).
We report all $K_0$ values in Table~\ref{tab:main_results} without selecting one.
For Table~\ref{tab:osmnx}, we select $K_0$ separately for each seed on validation queries, using a smaller candidate set.
This restriction explains the different ordering on Modena at $128$\,B/v.

The primary comparison uses an FPS candidate pool of $K_0=4K$, where $K$ is the
matched-memory FPS-ALT landmark count. The optimizer study uses $3{,}200$ training and
$500$ disjoint test queries, with three seeds for learned selectors. Table notes specify departures from these protocols.

Every result table carries a category label. \tabletag{Main} marks comparative results;
\tabletag{Ablation} marks controlled variations and diagnostics. Table~\ref{tab:symmetric-accounting} is the primary comparison.
Blue shading marks AAC-ST and AAC-CG results, including ablations, hybrids, and derived statistics.
We apply each table's bolding rule to unrounded values, so identical rounded values can differ in boldface.

\subsection{Primary comparison}
\label{sec:primary}

On the four DIMACS road graphs, every arc has an equal-weight reverse. One distance table per
landmark therefore supplies both directions. Table~\ref{tab:symmetric-accounting} uses this storage
rule for all methods. Manhattan and Modena remain directed; SBM and BA are undirected.
We compare $19$ graph--budget pairs. For each pair, we use $500$ test queries disjoint from the $3{,}200$ training queries. AAC-CG and canonical greedy each expand fewer nodes than FPS-ALT in all $19$ pairs.
AAC-CG reduces expansions by $1.9$--$22.7$\,\% relative to FPS-ALT.
Its test-objective ratio to canonical greedy averages $0.996$ over these pairs.
Compared with canonical greedy, AAC-CG expands fewer nodes in $4$ pairs, the same number in $1$,
and more in $14$.

We use the exact oracle in these runs. Table~\ref{tab:oracle-expansions} compares it with the entropic
oracle under secondary directed accounting. The win counts describe one test split per graph--budget pair;
they do not establish statistical significance for each pair.
Section~\ref{sec:robustness} re-evaluates the same selections on fresh queries for the eight
graph--budget pairs on Manhattan, Modena, SBM, and BA. The paired intervals for improvement over FPS-ALT
exclude zero in $7$ of these eight pairs. The exception is SBM at $32$\,B/v, with an interval of
$[-0.04, +0.16]$ points.

\begin{table}[!htbp]
\centering
\caption{\tabletag{Main} Primary comparison at matched label memory.}
\label{tab:symmetric-accounting}
\small
\setlength{\tabcolsep}{4pt}
\adjustbox{max width=\linewidth}{\begin{tabular}{l r r r S[table-format=6.0] S[table-format=6.0] S[table-format=6.0] S[table-format=+2.1] S[table-format=1.3]}
\toprule
\textbf{Graph} & {\textbf{B/v}} & {\textbf{$m$}} & {\textbf{$K_0$}}
  & \multicolumn{3}{c}{\textbf{Mean A* expansions}} & {\textbf{Margin (\%)}} & {$F_{\mathrm{test}}/F_{\mathrm{greedy}}$} \\
\cmidrule(lr){5-7}
& & & & {FPS-ALT} & {Greedy} & {AAC-CG} & & \\
\midrule
NY        &  32 &   8 &   32 &  12498 & \bfseries  10889 & \cellcolor{aaccolor} \bfseries  11008 & \cellcolor{aaccolor} +11.9 & \cellcolor{aaccolor} 0.999 \\
NY        &  64 &  16 &   64 &   8153 & \bfseries   7260 & \cellcolor{aaccolor} \bfseries   7298 & \cellcolor{aaccolor} +10.5 & \cellcolor{aaccolor} 0.999 \\
BAY       &  32 &   8 &   32 &  15549 & \bfseries  15166 & \cellcolor{aaccolor} \bfseries  14895 & \cellcolor{aaccolor} +4.2 & \cellcolor{aaccolor} 0.999 \\
BAY       &  64 &  16 &   64 &  11700 & \bfseries   9910 & \cellcolor{aaccolor} \bfseries  10492 & \cellcolor{aaccolor} +10.3 & \cellcolor{aaccolor} 0.998 \\
BAY       & 128 &  32 &  128 &   9431 & \bfseries   7591 & \cellcolor{aaccolor} \bfseries   7588 & \cellcolor{aaccolor} +19.5 & \cellcolor{aaccolor} 0.999 \\
COL       &  32 &   8 &   32 &  24753 & \bfseries  22020 & \cellcolor{aaccolor} \bfseries  22020 & \cellcolor{aaccolor} +11.0 & \cellcolor{aaccolor} 1.000 \\
COL       &  64 &  16 &   64 &  18420 & \bfseries  14966 & \cellcolor{aaccolor} \bfseries  14940 & \cellcolor{aaccolor} +18.9 & \cellcolor{aaccolor} 0.999 \\
COL       & 128 &  32 &  128 &  13946 & \bfseries  10882 & \cellcolor{aaccolor} \bfseries  11116 & \cellcolor{aaccolor} +20.3 & \cellcolor{aaccolor} 1.000 \\
FLA       &  32 &   8 &   32 &  66932 & \bfseries  53068 & \cellcolor{aaccolor} \bfseries  53343 & \cellcolor{aaccolor} +20.3 & \cellcolor{aaccolor} 0.999 \\
FLA       &  64 &  16 &   64 &  42819 & \bfseries  36025 & \cellcolor{aaccolor} \bfseries  34765 & \cellcolor{aaccolor} +18.8 & \cellcolor{aaccolor} 1.000 \\
FLA       & 128 &  32 &  128 &  31826 & \bfseries  25523 & \cellcolor{aaccolor} \bfseries  26921 & \cellcolor{aaccolor} +15.4 & \cellcolor{aaccolor} 0.999 \\
Manhattan &  32 &   8 &   16 &   1449 & \bfseries   1139 & \cellcolor{aaccolor} \bfseries   1295 & \cellcolor{aaccolor} +10.6 & \cellcolor{aaccolor} 1.001 \\
Manhattan &  64 &  16 &   32 &    949 & \bfseries    708 & \cellcolor{aaccolor} \bfseries    734 & \cellcolor{aaccolor} +22.7 & \cellcolor{aaccolor} 0.997 \\
Modena    &  32 &   8 &   16 &   2076 & \bfseries   1767 & \cellcolor{aaccolor} \bfseries   2002 & \cellcolor{aaccolor} +3.6 & \cellcolor{aaccolor} 0.978 \\
Modena    &  64 &  16 &   32 &   1371 & \bfseries   1110 & \cellcolor{aaccolor} \bfseries   1125 & \cellcolor{aaccolor} +17.9 & \cellcolor{aaccolor} 0.999 \\
SBM       &  32 &   8 &   32 &    492 & \bfseries    468 & \cellcolor{aaccolor} \bfseries    469 & \cellcolor{aaccolor} +4.6 & \cellcolor{aaccolor} 0.982 \\
SBM       &  64 &  16 &   64 &    276 & \bfseries    254 & \cellcolor{aaccolor} \bfseries    260 & \cellcolor{aaccolor} +6.0 & \cellcolor{aaccolor} 0.990 \\
BA        &  32 &   8 &   32 &    540 & \bfseries    488 & \cellcolor{aaccolor} \bfseries    530 & \cellcolor{aaccolor} +1.9 & \cellcolor{aaccolor} 0.995 \\
BA        &  64 &  16 &   64 &    298 & \bfseries    278 & \cellcolor{aaccolor} \bfseries    287 & \cellcolor{aaccolor} +3.6 & \cellcolor{aaccolor} 0.986 \\
\bottomrule
\end{tabular}
}
\par\smallskip
\begin{minipage}{\linewidth}
\footnotesize\raggedright
The four DIMACS road graphs have equal-weight reverse arcs. We share one distance table per landmark across directions for every method. The ground set consists of landmarks, with constraint $U_{m,K_0}$. Manhattan and Modena use directed accounting; SBM and BA are undirected. \par Each row is one graph--budget pair. AAC-CG and greedy use the same candidate pool, objective, and $3{,}200$ training queries. FPS-ALT uses the query-independent FPS prefix. All three methods use the same $500$ disjoint test queries and deployed label memory per vertex. AAC-CG uses $T{=}100$ steps and the exact per-block top-$k$ oracle. \par \textbf{Margin} is AAC-CG's relative expansion reduction compared with FPS-ALT. Bold marks fewer expansions than FPS-ALT. Of 19 pairs, AAC-CG expands fewer nodes than FPS-ALT in 19 and canonical greedy does so in 19. These counts describe one test split per graph--budget pair. \par The last column divides AAC-CG's mean test heuristic value by canonical greedy's value on the same queries. The ratio averages $0.996$ over 19 pairs. We evaluated AAC-ST and SIMPLE only under directed accounting (Table~\ref{tab:continuous-greedy}).
\end{minipage}
\end{table}

\begin{figure}[t]
\centering
\includegraphics[width=\textwidth]{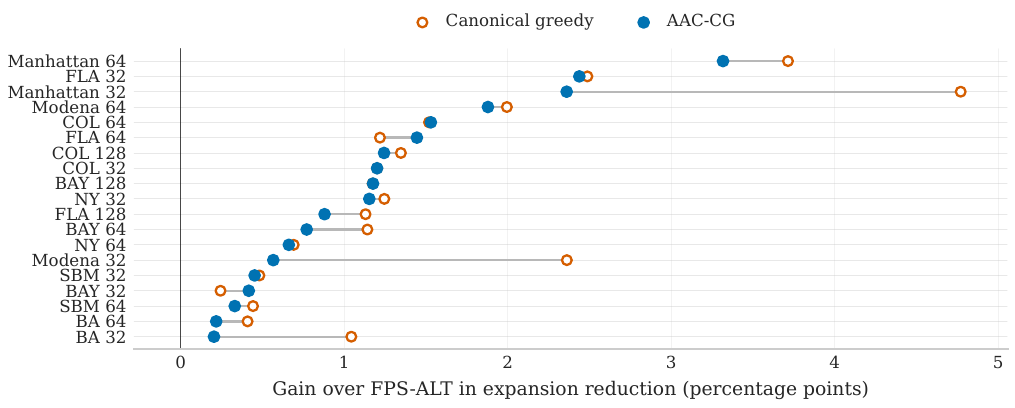}
\caption{\textbf{Query-adaptive selection at matched memory.} Each row compares canonical greedy
and AAC-CG with FPS-ALT under the primary accounting of Table~\ref{tab:symmetric-accounting}.
The horizontal axis is the difference in expansion reduction over Dijkstra, in percentage points;
positive values favor the selector. Both methods expand fewer nodes than FPS-ALT in all $19$
tested graph--budget pairs, using one test split per pair.}
\label{fig:optimizer-margin}
\end{figure}

\subsection{Optimizer comparison and controls}
\label{sec:continuous-greedy}

We compare optimizers on the same selection objective, pool, and deployable matroid.
We use $3{,}200$ training queries and $500$ disjoint test queries. AAC-CG uses $T=100$ steps, the closed-form
gradient, and deterministic pipage rounding. The static bound matrix is fixed, so we do not need
to backpropagate through the exact top-$k$ oracle. We use the entropic oracle to differentiate through selection when we train an upstream encoder (Appendix~\ref{app:contextual}). Tables~\ref{tab:oracle-expansions} and~\ref{tab:soft-cg}
compare the two oracles under the secondary protocol.

AAC-CG differs from AAC-ST in its continuous representation, sampling, and rounding.
Both methods omit the entropy regularizer in this study. Their comparison therefore measures
the combined effect of several design choices. The controls below vary these choices separately.

\begin{table}[!htbp]
\centering
\caption{\tabletag{Ablation} Search with exact and entropic selection oracles.}
\label{tab:oracle-expansions}
\footnotesize
\setlength{\tabcolsep}{5pt}
\adjustbox{max width=\linewidth}{\begin{tabular}{l r S[table-format=6.0] S[table-format=6.0] S[table-format=6.0] S[table-format=6.0] S[table-format=6.0] c}
\toprule
& & & \multicolumn{4}{c}{\textbf{AAC-CG}} & \\
\cmidrule(lr){4-7}
\textbf{Graph} & {\textbf{B/v}} & {\textbf{FPS-ALT}} & {\textbf{exact}} & {$\tau{=}0.0001$} & {$\tau{=}0.01$} & {$\tau{=}0.1$} & \textbf{Same set} \\
\midrule
NY        &  32 &  18122 & \cellcolor{aaccolor}  18303 & \cellcolor{aaccolor}  18303 & \cellcolor{aaccolor}  18303 & \cellcolor{aaccolor} \bfseries  17470 & \cellcolor{aaccolor} $\checkmark$$\checkmark$\ensuremath{-} \\
NY        &  64 &  12498 & \cellcolor{aaccolor} \bfseries  11362 & \cellcolor{aaccolor} \bfseries  11362 & \cellcolor{aaccolor} \bfseries  11362 & \cellcolor{aaccolor} \bfseries  11255 & \cellcolor{aaccolor} $\checkmark$$\checkmark$\ensuremath{-} \\
BAY       &  32 &  23201 & \cellcolor{aaccolor}  23880 & \cellcolor{aaccolor}  23880 & \cellcolor{aaccolor}  23880 & \cellcolor{aaccolor}  23880 & \cellcolor{aaccolor} $\checkmark$$\checkmark$$\checkmark$ \\
BAY       &  64 &  15549 & \cellcolor{aaccolor} \bfseries  14386 & \cellcolor{aaccolor} \bfseries  14386 & \cellcolor{aaccolor} \bfseries  14386 & \cellcolor{aaccolor} \bfseries  14989 & \cellcolor{aaccolor} $\checkmark$$\checkmark$\ensuremath{-} \\
BAY       & 128 &  11700 & \cellcolor{aaccolor} \bfseries  10008 & \cellcolor{aaccolor} \bfseries  10008 & \cellcolor{aaccolor} \bfseries  10008 & \cellcolor{aaccolor} \bfseries   9986 & \cellcolor{aaccolor} $\checkmark$$\checkmark$\ensuremath{-} \\
COL       &  32 &  34443 & \cellcolor{aaccolor} \bfseries  32753 & \cellcolor{aaccolor} \bfseries  32753 & \cellcolor{aaccolor} \bfseries  32753 & \cellcolor{aaccolor}  34772 & \cellcolor{aaccolor} $\checkmark$$\checkmark$\ensuremath{-} \\
COL       &  64 &  24753 & \cellcolor{aaccolor} \bfseries  22283 & \cellcolor{aaccolor} \bfseries  22283 & \cellcolor{aaccolor} \bfseries  22283 & \cellcolor{aaccolor} \bfseries  23234 & \cellcolor{aaccolor} $\checkmark$$\checkmark$\ensuremath{-} \\
COL       & 128 &  18420 & \cellcolor{aaccolor} \bfseries  15064 & \cellcolor{aaccolor} \bfseries  15064 & \cellcolor{aaccolor} \bfseries  15441 & \cellcolor{aaccolor} \bfseries  15274 & \cellcolor{aaccolor} $\checkmark$\ensuremath{-}\ensuremath{-} \\
FLA       &  32 &  85054 & \cellcolor{aaccolor} 101103 & \cellcolor{aaccolor}  85512 & \cellcolor{aaccolor}  85512 & \cellcolor{aaccolor}  97535 & \cellcolor{aaccolor} \ensuremath{-}\ensuremath{-}\ensuremath{-} \\
FLA       &  64 &  66932 & \cellcolor{aaccolor} \bfseries  53706 & \cellcolor{aaccolor} \bfseries  53706 & \cellcolor{aaccolor} \bfseries  56069 & \cellcolor{aaccolor} \bfseries  58704 & \cellcolor{aaccolor} $\checkmark$\ensuremath{-}\ensuremath{-} \\
FLA       & 128 &  42819 & \cellcolor{aaccolor} \bfseries  36921 & \cellcolor{aaccolor} \bfseries  36921 & \cellcolor{aaccolor} \bfseries  35672 & \cellcolor{aaccolor} \bfseries  36727 & \cellcolor{aaccolor} $\checkmark$\ensuremath{-}\ensuremath{-} \\
Manhattan &  32 &   1449 & \cellcolor{aaccolor} \bfseries   1295 & \cellcolor{aaccolor} \bfseries   1295 & \cellcolor{aaccolor} \bfseries   1295 & \cellcolor{aaccolor} \bfseries   1228 & \cellcolor{aaccolor} $\checkmark$$\checkmark$\ensuremath{-} \\
Manhattan &  64 &    949 & \cellcolor{aaccolor} \bfseries    734 & \cellcolor{aaccolor} \bfseries    734 & \cellcolor{aaccolor} \bfseries    734 & \cellcolor{aaccolor} \bfseries    771 & \cellcolor{aaccolor} $\checkmark$$\checkmark$\ensuremath{-} \\
Modena    &  32 &   2076 & \cellcolor{aaccolor} \bfseries   2002 & \cellcolor{aaccolor} \bfseries   2002 & \cellcolor{aaccolor} \bfseries   2002 & \cellcolor{aaccolor} \bfseries   2052 & \cellcolor{aaccolor} $\checkmark$$\checkmark$\ensuremath{-} \\
Modena    &  64 &   1371 & \cellcolor{aaccolor} \bfseries   1125 & \cellcolor{aaccolor} \bfseries   1125 & \cellcolor{aaccolor} \bfseries   1125 & \cellcolor{aaccolor} \bfseries   1104 & \cellcolor{aaccolor} $\checkmark$$\checkmark$\ensuremath{-} \\
SBM       &  32 &    492 & \cellcolor{aaccolor} \bfseries    469 & \cellcolor{aaccolor} \bfseries    469 & \cellcolor{aaccolor} \bfseries    469 & \cellcolor{aaccolor} \bfseries    469 & \cellcolor{aaccolor} $\checkmark$$\checkmark$$\checkmark$ \\
SBM       &  64 &    276 & \cellcolor{aaccolor} \bfseries    260 & \cellcolor{aaccolor} \bfseries    260 & \cellcolor{aaccolor} \bfseries    260 & \cellcolor{aaccolor} \bfseries    266 & \cellcolor{aaccolor} $\checkmark$$\checkmark$\ensuremath{-} \\
BA        &  32 &    540 & \cellcolor{aaccolor} \bfseries    530 & \cellcolor{aaccolor} \bfseries    530 & \cellcolor{aaccolor} \bfseries    530 & \cellcolor{aaccolor} \bfseries    530 & \cellcolor{aaccolor} $\checkmark$$\checkmark$$\checkmark$ \\
BA        &  64 &    298 & \cellcolor{aaccolor} \bfseries    287 & \cellcolor{aaccolor} \bfseries    287 & \cellcolor{aaccolor} \bfseries    287 & \cellcolor{aaccolor} \bfseries    284 & \cellcolor{aaccolor} $\checkmark$$\checkmark$\ensuremath{-} \\
\bottomrule
\end{tabular}
}
\par\smallskip
\begin{minipage}{\linewidth}
\footnotesize\raggedright
Entries are mean A* expansions on $500$ disjoint test queries; lower is better. DIMACS uses secondary directed accounting. We hold the pool, objective, matroid, memory, and test queries fixed, and change only the oracle (Proposition~\ref{prop:eps-lmo}). Bold marks fewer expansions than FPS-ALT. The exact oracle gives fewer expansions than FPS-ALT in 16 of 19 graph--budget pairs; the counts for the entropic oracle are $\tau{=}0.0001$: \textbf{16}, $\tau{=}0.01$: \textbf{16}, $\tau{=}0.1$: \textbf{16}. \par \textbf{Same set} gives one mark per temperature, in column order. A check means pipage rounding returns the exact oracle's subset, with the same expansions. A dash means a different subset. The subsets agree in 36 of 57 runs. The exact oracle has $\varepsilon=0$; the entropic oracle is differentiable with $\varepsilon\leq\tau f_{\max}H_{\max}$. Static selection needs no oracle derivative. Appendix~\ref{app:contextual} tests composition with an encoder.
\end{minipage}
\end{table}

\begin{table}[!htbp]
\centering
\caption{\tabletag{Ablation} Objective and gradient behavior of the entropic oracle.}
\label{tab:soft-cg}
\footnotesize
\setlength{\tabcolsep}{4pt}
\adjustbox{max width=\linewidth}{\begin{tabular}{l r S[table-format=1.4] S[table-format=1.4] S[table-format=1.4] S[table-format=1.4] S[table-format=1.3] S[table-format=1.2e+2]}
\toprule
\textbf{Graph} & {\textbf{B/v}} & {\textbf{$\tau$}} & {\textbf{$F(S_0)$}} & {\textbf{$F(S_\tau)$}}
  & {\textbf{ratio}} & {\textbf{$\max|\Delta x|$}} & {\textbf{$\|\nabla\|$}} \\
\midrule
\multirow{10}{*}{BA} &  32 & 0.0001 & \cellcolor{aaccolor} 4.8535 & \cellcolor{aaccolor} \bfseries 4.8535 & \cellcolor{aaccolor} 1.0000 & \cellcolor{aaccolor} 0.008 & \cellcolor{aaccolor} 4.16e+22 \\
                     &  32 & 0.0010 & \cellcolor{aaccolor} 4.8535 & \cellcolor{aaccolor} \bfseries 4.8535 & \cellcolor{aaccolor} 1.0000 & \cellcolor{aaccolor} 0.007 & \cellcolor{aaccolor} 1.77e+01 \\
                     &  32 & 0.0100 & \cellcolor{aaccolor} 4.8535 & \cellcolor{aaccolor} \bfseries 4.8535 & \cellcolor{aaccolor} 1.0000 & \cellcolor{aaccolor} 0.012 & \cellcolor{aaccolor} 9.59e-03 \\
                     &  32 & 0.1000 & \cellcolor{aaccolor} 4.8535 & \cellcolor{aaccolor} \bfseries 4.8535 & \cellcolor{aaccolor} 1.0000 & \cellcolor{aaccolor} 0.125 & \cellcolor{aaccolor} 7.86e-03 \\
                     &  32 & 1.0000 & \cellcolor{aaccolor} 4.8535 & \cellcolor{aaccolor} \bfseries 4.8535 & \cellcolor{aaccolor} 1.0000 & \cellcolor{aaccolor} 0.222 & \cellcolor{aaccolor} 6.72e-03 \\
                     &  64 & 0.0001 & \cellcolor{aaccolor} 5.6651 & \cellcolor{aaccolor} \bfseries 5.6651 & \cellcolor{aaccolor} 1.0000 & \cellcolor{aaccolor} 0.010 & \cellcolor{aaccolor} 1.29e+28 \\
                     &  64 & 0.0010 & \cellcolor{aaccolor} 5.6651 & \cellcolor{aaccolor} \bfseries 5.6651 & \cellcolor{aaccolor} 1.0000 & \cellcolor{aaccolor} 0.008 & \cellcolor{aaccolor} 5.99e+05 \\
                     &  64 & 0.0100 & \cellcolor{aaccolor} 5.6651 & \cellcolor{aaccolor} \bfseries 5.6651 & \cellcolor{aaccolor} 1.0000 & \cellcolor{aaccolor} 0.013 & \cellcolor{aaccolor} 1.02e-02 \\
                     &  64 & 0.1000 & \cellcolor{aaccolor} 5.6651 & \cellcolor{aaccolor} 5.6591 & \cellcolor{aaccolor} 0.9989 & \cellcolor{aaccolor} 0.113 & \cellcolor{aaccolor} 8.64e-03 \\
                     &  64 & 1.0000 & \cellcolor{aaccolor} 5.6651 & \cellcolor{aaccolor} 5.6195 & \cellcolor{aaccolor} 0.9919 & \cellcolor{aaccolor} 0.252 & \cellcolor{aaccolor} 6.78e-03 \\
\midrule
\multirow{10}{*}{SBM} &  32 & 0.0001 & \cellcolor{aaccolor} 2.4454 & \cellcolor{aaccolor} \bfseries 2.4454 & \cellcolor{aaccolor} 1.0000 & \cellcolor{aaccolor} 0.008 & \cellcolor{aaccolor} 5.92e+26 \\
                     &  32 & 0.0010 & \cellcolor{aaccolor} 2.4454 & \cellcolor{aaccolor} \bfseries 2.4454 & \cellcolor{aaccolor} 1.0000 & \cellcolor{aaccolor} 0.006 & \cellcolor{aaccolor} 7.55e+05 \\
                     &  32 & 0.0100 & \cellcolor{aaccolor} 2.4454 & \cellcolor{aaccolor} \bfseries 2.4454 & \cellcolor{aaccolor} 1.0000 & \cellcolor{aaccolor} 0.006 & \cellcolor{aaccolor} 8.83e-03 \\
                     &  32 & 0.1000 & \cellcolor{aaccolor} 2.4454 & \cellcolor{aaccolor} \bfseries 2.4454 & \cellcolor{aaccolor} 1.0000 & \cellcolor{aaccolor} 0.081 & \cellcolor{aaccolor} 7.76e-03 \\
                     &  32 & 1.0000 & \cellcolor{aaccolor} 2.4454 & \cellcolor{aaccolor} 2.4627 & \cellcolor{aaccolor} 1.0070 & \cellcolor{aaccolor} 0.221 & \cellcolor{aaccolor} 6.72e-03 \\
                     &  64 & 0.0001 & \cellcolor{aaccolor} 2.8628 & \cellcolor{aaccolor} \bfseries 2.8628 & \cellcolor{aaccolor} 1.0000 & \cellcolor{aaccolor} 0.010 & \cellcolor{aaccolor} 8.34e+25 \\
                     &  64 & 0.0010 & \cellcolor{aaccolor} 2.8628 & \cellcolor{aaccolor} \bfseries 2.8628 & \cellcolor{aaccolor} 1.0000 & \cellcolor{aaccolor} 0.007 & \cellcolor{aaccolor} 8.26e+09 \\
                     &  64 & 0.0100 & \cellcolor{aaccolor} 2.8628 & \cellcolor{aaccolor} \bfseries 2.8628 & \cellcolor{aaccolor} 1.0000 & \cellcolor{aaccolor} 0.008 & \cellcolor{aaccolor} 8.95e-03 \\
                     &  64 & 0.1000 & \cellcolor{aaccolor} 2.8628 & \cellcolor{aaccolor} 2.8619 & \cellcolor{aaccolor} 0.9997 & \cellcolor{aaccolor} 0.075 & \cellcolor{aaccolor} 7.98e-03 \\
                     &  64 & 1.0000 & \cellcolor{aaccolor} 2.8628 & \cellcolor{aaccolor} 2.8558 & \cellcolor{aaccolor} 0.9976 & \cellcolor{aaccolor} 0.219 & \cellcolor{aaccolor} 6.74e-03 \\
\bottomrule
\end{tabular}
}
\par\smallskip
\begin{minipage}{\linewidth}
\footnotesize\raggedright
We compare both oracles with the same pipage rounding and $T{=}100$ steps. $F(S_0)$ and $F(S_\tau)$ are the rounded training objectives for the exact and entropic oracles, respectively. Bold marks entropic objectives within relative tolerance $10^{-6}$ of the exact-oracle objective. The ratio $F(S_\tau)/F(S_0)$ stays within $[0.9919, 1.0070]$ across the tested temperatures and graph--budget pairs. The largest loss is $0.81\,\%$. \par $\max|\Delta x|$ is the largest coordinate difference from the exact-oracle trajectory. It is largest at $\tau=1$ in 4 of 4 graph--budget pairs. It increases monotonically with temperature in 0 of 4 pairs. Every iterate is admissible in 40 of 40 runs. Pipage rounding is valid in 40 of 40 runs (both $T$ values counted). \par $\|\nabla\|$ is the Euclidean norm of the soft objective's gradient with respect to the bound matrix. Large values at low temperature indicate poor conditioning; contextual training uses $\tau=0.1$. The $T{=}50$ rows give the same ratio to four decimals for each tested temperature and graph--budget pair and are omitted here for space.
\end{minipage}
\end{table}

\paragraph{Secondary optimizer comparison.} Table~\ref{tab:continuous-greedy} in
Appendix~\ref{app:directed-optimizer} compares the optimizers under directed accounting on DIMACS.
We charge each landmark--direction table separately, even when a selector retains both directions of one landmark.
The storage savings from sharing tables therefore depend on the selected items.
The remaining optimizer controls and the fresh-query evaluation of Section~\ref{sec:robustness}
use this secondary protocol. We ran AAC-ST and SIMPLE only under directed accounting.

\paragraph{Soft-selector feasibility.} We probe seven selectors at initialization on three graphs at two budgets and on a path witness.
We use three seeds, for $147$ runs.
Row-Gumbel, successive softmax \citep{xie2019reparameterizable}, and SIMPLE
\citep{ahmed2023simple} have unit row mass and are admissible in all $21$ runs each.
For successive softmax, we retain each selection step as a separate row; summing those rows
would instead produce total mass $m$. SIMPLE samples a subset of the prescribed size within
each block ($m$ on an undirected graph, $m_b$ for direction $b$) and represents it as one-hot rows.
The sigmoid, perturbed, convex-analysis, and optimal-transport top-$k$ marginals sum to $m$
\citep{wijk2025topk,berthet2020learning,sander2023sparse,xie2020differentiable}.
When we use each marginal vector as one compression row, it violates admissibility in all $21$ probe runs.
The largest observed $h/d$ is $12.18$. Hard top-$m$ rounding restores
admissible deployment. The violation comes from using marginals as one compression row; inclusion
probabilities can still define an admissible mixture, as $h_x$ does in AAC-CG.

We train five of these selectors on the same pool, objective, and queries (Table~\ref{tab:feasibility}). Across its six graph--budget pairs, SIMPLE has the highest mean test-objective ratio to canonical
greedy, $0.966$, compared with $0.963$ for AAC-ST. All five selectors have lower test objectives
than AAC-CG in the four graph--budget pairs shared with Table~\ref{tab:continuous-greedy}.
AAC-CG also has a higher test objective than the AAC-ST seed mean at all $16$ combinations of graph and training sample size in Table~\ref{tab:optimiser-gap-vs-n}.

\begin{table}[!htbp]
\centering
\caption{\tabletag{Ablation} Selector feasibility during training.}
\label{tab:feasibility}
\small
\setlength{\tabcolsep}{3pt}
\adjustbox{max width=\linewidth}{\begin{tabular}{l r l S[table-format=2.2] S[table-format=1.3(1)] S[table-format=5.0] l}
\toprule
\textbf{Graph} & {\textbf{B/v}} & \textbf{Selector}
  & {\textbf{Max row mass}} & {\textbf{$F$/Greedy}} & {\textbf{Expansions}} & \textbf{Class} \\
\midrule
\multirow{12}{*}{SBM}        &  32 & FPS-ALT (reference) & {--} & 0.990 &   492 & {--} \\
                             &  32 & \cellcolor{aaccolor} AAC-ST (row-Gumbel)            & \cellcolor{aaccolor} 1.00 & \cellcolor{aaccolor} 0.974 \pm 0.033 & \cellcolor{aaccolor}   530 & \cellcolor{aaccolor} admissible \\
                             &  32 & successive softmax             & 1.00 & 0.976 \pm 0.019 &   542 & admissible \\
                             &  32 & SIMPLE                         & 1.00 & 0.975 \pm 0.032 &   542 & admissible \\
                             &  32 & sigmoid top-$k$                & 8.00 & 0.942 \pm 0.012 &   648 & \textbf{leaves class} \\
                             &  32 & perturbed top-$k$              & 8.00 & 0.963 \pm 0.003 &   554 & \textbf{leaves class} \\
\cmidrule{2-7}
                             &  64 & FPS-ALT (reference) & {--} & 0.972 &   276 & {--} \\
                             &  64 & \cellcolor{aaccolor} AAC-ST (row-Gumbel)            & \cellcolor{aaccolor} 1.00 & \cellcolor{aaccolor} 0.967 \pm 0.009 & \cellcolor{aaccolor}   292 & \cellcolor{aaccolor} admissible \\
                             &  64 & successive softmax             & 1.00 & 0.966 \pm 0.005 &   287 & admissible \\
                             &  64 & SIMPLE                         & 1.00 & 0.971 \pm 0.003 &   283 & admissible \\
                             &  64 & sigmoid top-$k$                & 16.00 & 0.940 \pm 0.022 &   324 & \textbf{leaves class} \\
                             &  64 & perturbed top-$k$              & 16.00 & 0.947 \pm 0.019 &   308 & \textbf{leaves class} \\
\midrule
\multirow{12}{*}{BA}         &  32 & FPS-ALT (reference) & {--} & 1.003 &   540 & {--} \\
                             &  32 & \cellcolor{aaccolor} AAC-ST (row-Gumbel)            & \cellcolor{aaccolor} 1.00 & \cellcolor{aaccolor} 0.962 \pm 0.010 & \cellcolor{aaccolor}   572 & \cellcolor{aaccolor} admissible \\
                             &  32 & successive softmax             & 1.00 & 0.972 \pm 0.011 &   569 & admissible \\
                             &  32 & SIMPLE                         & 1.00 & 0.975 \pm 0.008 &   560 & admissible \\
                             &  32 & sigmoid top-$k$                & 8.00 & 0.965 \pm 0.012 &   577 & \textbf{leaves class} \\
                             &  32 & perturbed top-$k$              & 8.00 & 0.973 \pm 0.006 &   566 & \textbf{leaves class} \\
\cmidrule{2-7}
                             &  64 & FPS-ALT (reference) & {--} & 0.966 &   298 & {--} \\
                             &  64 & \cellcolor{aaccolor} AAC-ST (row-Gumbel)            & \cellcolor{aaccolor} 1.00 & \cellcolor{aaccolor} 0.952 \pm 0.023 & \cellcolor{aaccolor}   312 & \cellcolor{aaccolor} admissible \\
                             &  64 & successive softmax             & 1.00 & 0.940 \pm 0.007 &   322 & admissible \\
                             &  64 & SIMPLE                         & 1.00 & 0.953 \pm 0.007 &   307 & admissible \\
                             &  64 & sigmoid top-$k$                & 16.00 & 0.952 \pm 0.012 &   311 & \textbf{leaves class} \\
                             &  64 & perturbed top-$k$              & 16.00 & 0.950 \pm 0.007 &   312 & \textbf{leaves class} \\
\midrule
\multirow{12}{*}{OGB-arXiv}  &  32 & FPS-ALT (reference) & {--} & 0.943 & 11662 & {--} \\
                             &  32 & \cellcolor{aaccolor} AAC-ST (row-Gumbel)            & \cellcolor{aaccolor} 1.00 & \cellcolor{aaccolor} 0.958 \pm 0.008 & \cellcolor{aaccolor} 10157 & \cellcolor{aaccolor} admissible \\
                             &  32 & successive softmax             & 1.00 & 0.951 \pm 0.026 & 11128 & admissible \\
                             &  32 & SIMPLE                         & 1.00 & 0.951 \pm 0.022 &  9924 & admissible \\
                             &  32 & sigmoid top-$k$                & 8.00 & 0.970 \pm 0.012 & 10702 & \textbf{leaves class} \\
                             &  32 & perturbed top-$k$              & 8.00 & 0.955 \pm 0.025 & 11130 & \textbf{leaves class} \\
\cmidrule{2-7}
                             &  64 & FPS-ALT (reference) & {--} & 0.927 &  6615 & {--} \\
                             &  64 & \cellcolor{aaccolor} AAC-ST (row-Gumbel)            & \cellcolor{aaccolor} 1.00 & \cellcolor{aaccolor} 0.967 \pm 0.002 & \cellcolor{aaccolor}  5506 & \cellcolor{aaccolor} admissible \\
                             &  64 & successive softmax             & 1.00 & 0.944 \pm 0.019 &  5532 & admissible \\
                             &  64 & SIMPLE                         & 1.00 & 0.974 \pm 0.012 &  4944 & admissible \\
                             &  64 & sigmoid top-$k$                & 16.00 & 0.964 \pm 0.009 &  5497 & \textbf{leaves class} \\
                             &  64 & perturbed top-$k$              & 16.00 & 0.969 \pm 0.002 &  5389 & \textbf{leaves class} \\
\bottomrule
\end{tabular}
}
\par\smallskip
\begin{minipage}{\linewidth}
\footnotesize\raggedright
We train 5 selectors on the same objective, $K_0$ FPS pool, and 3{,}200 training queries for $200$ epochs over 3 seeds. We disable auxiliary regularizers and evaluate on 500 disjoint test queries. \emph{Max row mass} reports initialization probes for row-stochastic selectors and maxima along training traces for marginal selectors. \par For non-negative coefficients, the admissible linear class requires row mass $\leq1$ (Theorem~\ref{thm:maximal-class}). This condition suffices on directed and undirected graphs. Row-Gumbel and successive softmax have unit row mass by construction. SIMPLE samples a fixed-size subset within each block and represents it as one-hot rows. The sigmoid top-$k$ \citep{wijk2025topk} and perturbed top-$k$ \citep{berthet2020learning} marginals saturate at $m$ (measured $8$ at $m{=}8$, $16$ at $m{=}16$). Their training heuristics exceed the teacher by up to $m\times$ and leave the admissible class. Hard top-$m$ rounding restores admissibility at deployment; every reported configuration has zero suboptimal paths. \par $F$/Greedy divides the test objective by canonical greedy's value on the same queries (mean $\pm$ std across seeds). Expansions are mean A* expansions on test queries. FPS-ALT is the matched-memory reference. Across 6 graph--budget pairs, SIMPLE has the highest mean objective ratio ($0.966$). The sigmoid and perturbed marginals reach $0.955$ and $0.960$, respectively.
\end{minipage}
\end{table}

\begin{table}[!htbp]
\centering
\caption{\tabletag{Ablation} Selection quality versus training-sample size.}
\label{tab:optimiser-gap-vs-n}
\small
\adjustbox{max width=\linewidth}{\begin{tabular}{l r r S[table-format=1.4] S[table-format=1.4(1)] S[table-format=1.4(1)] S[table-format=1.4]}
\toprule
\textbf{Graph} & {\textbf{B/v}} & {\textbf{$n$}} & {\textbf{FPS-ALT}} & {\textbf{SIMPLE}} & {\textbf{AAC-ST}} & {\textbf{AAC-CG}} \\
\midrule
\multirow{8}{*}{BA}  &  32 &   25 & \bfseries 1.0025 & 0.9632 \pm 0.0335 & \cellcolor{aaccolor} 0.9680 \pm 0.0214 & \cellcolor{aaccolor} 0.9894 \\
                     &  32 &   50 & 1.0172 & 0.9603 \pm 0.0338 & \cellcolor{aaccolor} 0.9856 \pm 0.0250 & \cellcolor{aaccolor} \bfseries 1.0295 \\
                     &  32 &  100 & \bfseries 1.0367 & 0.9856 \pm 0.0173 & \cellcolor{aaccolor} 1.0032 \pm 0.0356 & \cellcolor{aaccolor} 1.0126 \\
                     &  32 &  200 & 0.9855 & 0.9624 \pm 0.0171 & \cellcolor{aaccolor} 0.9676 \pm 0.0397 & \cellcolor{aaccolor} \bfseries 1.0000 \\
                     &  32 &  400 & 0.9810 & 0.9591 \pm 0.0224 & \cellcolor{aaccolor} 0.9607 \pm 0.0257 & \cellcolor{aaccolor} \bfseries 1.0024 \\
                     &  32 &  800 & 0.9669 & 0.9339 \pm 0.0067 & \cellcolor{aaccolor} 0.9408 \pm 0.0231 & \cellcolor{aaccolor} \bfseries 0.9952 \\
                     &  32 & 1600 & 0.9682 & 0.9360 \pm 0.0053 & \cellcolor{aaccolor} 0.9438 \pm 0.0324 & \cellcolor{aaccolor} \bfseries 0.9971 \\
                     &  32 & 3200 & 1.0093 & 0.9728 \pm 0.0073 & \cellcolor{aaccolor} 0.9783 \pm 0.0255 & \cellcolor{aaccolor} \bfseries 1.0240 \\
\midrule
\multirow{8}{*}{SBM} &  32 &   25 & \bfseries 0.9656 & 0.9275 \pm 0.0254 & \cellcolor{aaccolor} 0.9371 \pm 0.0190 & \cellcolor{aaccolor} 0.9504 \\
                     &  32 &   50 & \bfseries 0.9797 & 0.9746 \pm 0.0213 & \cellcolor{aaccolor} 0.9497 \pm 0.0194 & \cellcolor{aaccolor} 0.9781 \\
                     &  32 &  100 & 0.9542 & 0.9325 \pm 0.0094 & \cellcolor{aaccolor} 0.9321 \pm 0.0256 & \cellcolor{aaccolor} \bfseries 0.9726 \\
                     &  32 &  200 & 0.9939 & 0.9752 \pm 0.0161 & \cellcolor{aaccolor} 0.9827 \pm 0.0071 & \cellcolor{aaccolor} \bfseries 1.0556 \\
                     &  32 &  400 & 0.9794 & 0.9547 \pm 0.0210 & \cellcolor{aaccolor} 0.9676 \pm 0.0099 & \cellcolor{aaccolor} \bfseries 1.0099 \\
                     &  32 &  800 & 0.9852 & 0.9604 \pm 0.0212 & \cellcolor{aaccolor} 0.9667 \pm 0.0338 & \cellcolor{aaccolor} \bfseries 0.9957 \\
                     &  32 & 1600 & 0.9711 & 0.9466 \pm 0.0209 & \cellcolor{aaccolor} 0.9625 \pm 0.0317 & \cellcolor{aaccolor} \bfseries 1.0106 \\
                     &  32 & 3200 & \bfseries 0.9808 & 0.9560 \pm 0.0211 & \cellcolor{aaccolor} 0.9636 \pm 0.0268 & \cellcolor{aaccolor} 0.9692 \\
\bottomrule
\end{tabular}
}
\par\smallskip
\begin{minipage}{\linewidth}
\footnotesize\raggedright
Each entry is the test objective divided by canonical greedy's value on the same test queries. The training-query count $n$ varies; all variants use the same objective, pool, and matroid. AAC-ST and SIMPLE report mean $\pm$ std over 3 seeds; FPS-ALT and AAC-CG are deterministic. Bold marks the highest mean in each row. A gap that persists as $n$ grows is consistent with optimization error, but this sweep does not isolate estimator and parameterization effects.
\end{minipage}
\end{table}

\begin{table}[!htbp]
\centering
\caption{\tabletag{Ablation} Candidate-pool size at fixed deployment memory.}
\label{tab:pool-sweep}
\small
\setlength{\tabcolsep}{3pt}
\adjustbox{max width=\linewidth}{\begin{tabular}{l r r r S[table-format=1.4] S[table-format=1.4] S[table-format=6.0] S[table-format=6.0] S[table-format=6.0] S[table-format=+2.1]}
\toprule
\textbf{Graph} & {\textbf{B/v}} & {\textbf{$K_0/K$}} & {\textbf{$K_0$}}
  & {\textbf{Greedy $F$/cert.}} & {\textbf{AAC-CG $F$/Greedy}}
  & {\textbf{Exp.\ FPS}} & {\textbf{Exp.\ Greedy}} & {\textbf{Exp.\ AAC-CG}} & {\textbf{Red.\ (\%)}} \\
\midrule
\multirow{6}{*}{NY}  &  32 &  4 &   16 & 0.9544 & \cellcolor{aaccolor} 0.9940 &  18122 &  16709 & \cellcolor{aaccolor}  18303 & \cellcolor{aaccolor} -1.0 \\
                     &  32 &  8 &   32 & 0.9628 & \cellcolor{aaccolor} 0.9951 &  18122 &  16196 & \cellcolor{aaccolor}  16783 & \cellcolor{aaccolor} +7.4 \\
                     &  32 & 16 &   64 & 0.9618 & \cellcolor{aaccolor} 1.0017 &  18122 &  16583 & \cellcolor{aaccolor}  16407 & \cellcolor{aaccolor} +9.5 \\
                     &  64 &  4 &   32 & 0.9771 & \cellcolor{aaccolor} 0.9979 &  12498 &  10889 & \cellcolor{aaccolor}  11362 & \cellcolor{aaccolor} +9.1 \\
                     &  64 &  8 &   64 & 0.9778 & \cellcolor{aaccolor} 1.0026 &  12498 &  10672 & \cellcolor{aaccolor}  10081 & \cellcolor{aaccolor} +19.3 \\
                     &  64 & 16 &  128 & 0.9756 & \cellcolor{aaccolor} 1.0000 &  12498 &  10727 & \cellcolor{aaccolor}  10893 & \cellcolor{aaccolor} +12.8 \\
\midrule
\multirow{6}{*}{Manhattan} &  32 &  4 &   16 & 0.9692 & \cellcolor{aaccolor} 1.0009 &   1449 &   1139 & \cellcolor{aaccolor}   1295 & \cellcolor{aaccolor} +10.6 \\
                     &  32 &  8 &   32 & 0.9669 & \cellcolor{aaccolor} 0.9998 &   1449 &   1051 & \cellcolor{aaccolor}   1130 & \cellcolor{aaccolor} +22.0 \\
                     &  32 & 16 &   64 & 0.9706 & \cellcolor{aaccolor} 1.0002 &   1449 &   1128 & \cellcolor{aaccolor}   1107 & \cellcolor{aaccolor} +23.6 \\
                     &  64 &  4 &   32 & 0.9851 & \cellcolor{aaccolor} 0.9975 &    949 &    708 & \cellcolor{aaccolor}    734 & \cellcolor{aaccolor} +22.7 \\
                     &  64 &  8 &   64 & 0.9803 & \cellcolor{aaccolor} 1.0010 &    949 &    714 & \cellcolor{aaccolor}    688 & \cellcolor{aaccolor} +27.5 \\
                     &  64 & 16 &  128 & 0.9802 & \cellcolor{aaccolor} 0.9991 &    949 &    694 & \cellcolor{aaccolor}    703 & \cellcolor{aaccolor} +26.0 \\
\midrule
\multirow{6}{*}{Modena} &  32 &  4 &   16 & 0.9475 & \cellcolor{aaccolor} 0.9781 &   2076 &   1767 & \cellcolor{aaccolor}   2002 & \cellcolor{aaccolor} +3.6 \\
                     &  32 &  8 &   32 & 0.9484 & \cellcolor{aaccolor} 1.0006 &   2076 &   1717 & \cellcolor{aaccolor}   1743 & \cellcolor{aaccolor} +16.1 \\
                     &  32 & 16 &   64 & 0.9442 & \cellcolor{aaccolor} 0.9830 &   2076 &   1686 & \cellcolor{aaccolor}   1868 & \cellcolor{aaccolor} +10.0 \\
                     &  64 &  4 &   32 & 0.9772 & \cellcolor{aaccolor} 0.9994 &   1371 &   1110 & \cellcolor{aaccolor}   1125 & \cellcolor{aaccolor} +17.9 \\
                     &  64 &  8 &   64 & 0.9741 & \cellcolor{aaccolor} 0.9982 &   1371 &   1117 & \cellcolor{aaccolor}   1168 & \cellcolor{aaccolor} +14.8 \\
                     &  64 & 16 &  128 & 0.9682 & \cellcolor{aaccolor} 1.0015 &   1371 &   1084 & \cellcolor{aaccolor}   1137 & \cellcolor{aaccolor} +17.0 \\
\midrule
\multirow{6}{*}{SBM} &  32 &  4 &   32 & 0.8351 & \cellcolor{aaccolor} 0.9824 &    492 &    468 & \cellcolor{aaccolor}    469 & \cellcolor{aaccolor} +4.6 \\
                     &  32 &  8 &   64 & 0.8300 & \cellcolor{aaccolor} 0.9948 &    492 &    464 & \cellcolor{aaccolor}    469 & \cellcolor{aaccolor} +4.6 \\
                     &  32 & 16 &  128 & 0.8259 & \cellcolor{aaccolor} 1.0052 &    492 &    449 & \cellcolor{aaccolor}    473 & \cellcolor{aaccolor} +3.9 \\
                     &  64 &  4 &   64 & 0.8571 & \cellcolor{aaccolor} 0.9902 &    276 &    254 & \cellcolor{aaccolor}    260 & \cellcolor{aaccolor} +6.0 \\
                     &  64 &  8 &  128 & 0.8554 & \cellcolor{aaccolor} 0.9978 &    276 &    251 & \cellcolor{aaccolor}    244 & \cellcolor{aaccolor} +11.6 \\
                     &  64 & 16 &  256 & 0.8506 & \cellcolor{aaccolor} 0.9985 &    276 &    242 & \cellcolor{aaccolor}    249 & \cellcolor{aaccolor} +9.7 \\
\midrule
\multirow{6}{*}{BA}  &  32 &  4 &   32 & 0.8317 & \cellcolor{aaccolor} 0.9951 &    540 &    488 & \cellcolor{aaccolor}    530 & \cellcolor{aaccolor} +1.9 \\
                     &  32 &  8 &   64 & 0.8359 & \cellcolor{aaccolor} 1.0030 &    540 &    496 & \cellcolor{aaccolor}    526 & \cellcolor{aaccolor} +2.5 \\
                     &  32 & 16 &  128 & 0.8330 & \cellcolor{aaccolor} 0.9912 &    540 &    501 & \cellcolor{aaccolor}    530 & \cellcolor{aaccolor} +1.7 \\
                     &  64 &  4 &   64 & 0.8488 & \cellcolor{aaccolor} 0.9859 &    298 &    278 & \cellcolor{aaccolor}    287 & \cellcolor{aaccolor} +3.6 \\
                     &  64 &  8 &  128 & 0.8463 & \cellcolor{aaccolor} 1.0031 &    298 &    275 & \cellcolor{aaccolor}    273 & \cellcolor{aaccolor} +8.5 \\
                     &  64 & 16 &  256 & 0.8417 & \cellcolor{aaccolor} 1.0036 &    298 &    271 & \cellcolor{aaccolor}    284 & \cellcolor{aaccolor} +4.7 \\
\bottomrule
\end{tabular}
}
\par\smallskip
\begin{minipage}{\linewidth}
\footnotesize\raggedright
We vary $K_0/K\in\{4,8,16\}$ at fixed $m$, where $K=m$ on undirected graphs and $K=m/2$ otherwise. Greedy's training $F$/cert. divides its objective by the Nemhauser online upper bound. AAC-CG's $F$/Greedy divides its test objective by canonical greedy's value on the same queries. \textbf{Red.} is AAC-CG's relative expansion reduction compared with FPS-ALT at matched memory. \par Increasing the pool from $4K$ to $16K$ raises expansion reduction on NY ($+7.1$\,pp), Manhattan ($+8.1$\,pp), Modena ($+2.8$\,pp), SBM ($+1.5$\,pp), BA ($+0.5$\,pp). These changes average over budgets within each graph. AAC-CG's mean $F$/Greedy changes by $+0.0005$ to $+0.0155$ across graphs.
\end{minipage}
\end{table}

\paragraph{Pool size and optimizer controls.} Increasing $K_0$ from $4K$ to $16K$ at fixed
memory improves AAC-CG's mean relative expansion reduction over FPS-ALT on all five graphs.
Its test-objective ratio to canonical greedy ranges from $0.978$ to $1.005$
(Table~\ref{tab:pool-sweep}). The improvement is not monotone:
the relative reduction falls in seven of ten graph--budget pairs between $8K$ and $16K$.
More candidates therefore do not ensure fewer expansions. A larger pool also leaves AAC-ST with more expansions than FPS-ALT on Manhattan.

We compare projected ascent on the per-query concave closure, Lov\'asz extension, and multilinear
extension with the same feasible set, initialization, step budget, and rounding.
The test-objective ratios to canonical greedy range from $0.974$ to $1.020$
(Table~\ref{tab:factorial}; U denotes the uniform start). No extension gives the highest value for every graph--budget pair.
Replacing pipage rounding with top-$k$ rounding changes the test-objective ratio by
$-0.93$ to $+2.58$ percentage points, averaging $+0.47$.
Changing the extension or rounding alone does not explain AAC-ST's lower test objective.
These controls still combine estimator effects with differences in sampling and parameterization.
We use the multilinear extension for its approximation guarantee, although other extensions can give higher test objectives.

\begin{table}[!htbp]
\centering
\caption{\tabletag{Ablation} Extension, initialization, and rounding controls.}
\label{tab:factorial}
\small
\setlength{\tabcolsep}{3pt}
\textbf{(a) Test objective / Greedy (higher is better)}\par\smallskip
\begin{tabular*}{\linewidth}{@{\extracolsep{\fill}}l r r r r r r r r r r@{}}
\toprule
Variant & \multicolumn{2}{c}{NY} & \multicolumn{2}{c}{Manhattan} & \multicolumn{2}{c}{Modena} & \multicolumn{2}{c}{SBM} & \multicolumn{2}{c}{BA} \\
\cmidrule(lr){2-3}\cmidrule(lr){4-5}\cmidrule(lr){6-7}\cmidrule(lr){8-9}\cmidrule(lr){10-11}
Budget (B/v) & 32 & 64 & 32 & 64 & 32 & 64 & 32 & 64 & 32 & 64 \\
\midrule
PGA: closure, $\varepsilon$ & 0.9913 & 0.9956 & \bfseries 1.0021 & 0.9961 & 0.9781 & 1.0010 & 0.9824 & 0.9841 & \bfseries 0.9951 & 0.9859 \\
PGA: closure, U & 0.9913 & 0.9956 & \bfseries 1.0021 & 0.9961 & 0.9781 & 0.9988 & 0.9824 & 0.9841 & \bfseries 0.9951 & 0.9859 \\
PGA: Lov\'asz, $\varepsilon$ & 0.9843 & 0.9942 & 0.9874 & 0.9881 & 0.9890 & 0.9961 & \bfseries 1.0200 & 0.9821 & 0.9931 & 0.9851 \\
PGA: Lov\'asz, U & 0.9755 & 0.9956 & 0.9874 & 0.9880 & 0.9744 & 0.9919 & \bfseries 1.0200 & 0.9821 & 0.9931 & 0.9851 \\
PGA: multilinear, $\varepsilon$ & 0.9913 & 0.9967 & 1.0002 & 0.9961 & 0.9781 & \bfseries 1.0012 & 0.9824 & 0.9902 & \bfseries 0.9951 & 0.9859 \\
PGA: multilinear, U & 0.9921 & 0.9967 & 1.0002 & 0.9961 & 0.9781 & 1.0003 & 0.9916 & 0.9783 & \bfseries 0.9951 & 0.9859 \\
\cellcolor{aaccolor} AAC-CG: top-$k$ & \cellcolor{aaccolor} 0.9908 & \cellcolor{aaccolor} 0.9948 & \cellcolor{aaccolor} 0.9987 & \cellcolor{aaccolor} \bfseries 0.9983 & \cellcolor{aaccolor} \bfseries 1.0039 & \cellcolor{aaccolor} 0.9994 & \cellcolor{aaccolor} 1.0074 & \cellcolor{aaccolor} \bfseries 0.9975 & \cellcolor{aaccolor} 0.9859 & \cellcolor{aaccolor} \bfseries 0.9912 \\
\cellcolor{aaccolor} AAC-CG: pipage & \cellcolor{aaccolor} \bfseries 0.9940 & \cellcolor{aaccolor} \bfseries 0.9979 & \cellcolor{aaccolor} 1.0009 & \cellcolor{aaccolor} 0.9975 & \cellcolor{aaccolor} 0.9781 & \cellcolor{aaccolor} 0.9994 & \cellcolor{aaccolor} 0.9824 & \cellcolor{aaccolor} 0.9902 & \cellcolor{aaccolor} \bfseries 0.9951 & \cellcolor{aaccolor} 0.9859 \\
\bottomrule
\end{tabular*}
\par\medskip
\textbf{(b) Mean A* expansions (lower is better)}\par\smallskip
\begin{tabular*}{\linewidth}{@{\extracolsep{\fill}}l r r r r r r r r r r@{}}
\toprule
Variant & \multicolumn{2}{c}{NY} & \multicolumn{2}{c}{Manhattan} & \multicolumn{2}{c}{Modena} & \multicolumn{2}{c}{SBM} & \multicolumn{2}{c}{BA} \\
\cmidrule(lr){2-3}\cmidrule(lr){4-5}\cmidrule(lr){6-7}\cmidrule(lr){8-9}\cmidrule(lr){10-11}
Budget (B/v) & 32 & 64 & 32 & 64 & 32 & 64 & 32 & 64 & 32 & 64 \\
\midrule
PGA: closure, $\varepsilon$ & \bfseries 17470 & 11316 & \bfseries 1211 & 771 & 2002 & 1126 & 469 & 266 & \bfseries 530 & \bfseries 284 \\
PGA: closure, U & \bfseries 17470 & 11316 & \bfseries 1211 & 771 & 2002 & 1170 & 469 & 266 & \bfseries 530 & \bfseries 284 \\
PGA: Lov\'asz, $\varepsilon$ & 18354 & 11406 & 1287 & 819 & 1918 & 1124 & \bfseries 464 & 265 & 537 & 293 \\
PGA: Lov\'asz, U & 19146 & 11349 & 1287 & 822 & 2122 & 1150 & \bfseries 464 & 265 & 537 & 293 \\
PGA: multilinear, $\varepsilon$ & \bfseries 17470 & \bfseries 11255 & 1292 & 771 & 2002 & 1104 & 469 & \bfseries 260 & \bfseries 530 & \bfseries 284 \\
PGA: multilinear, U & 17738 & \bfseries 11255 & 1292 & 771 & 2002 & \bfseries 1097 & 481 & 261 & \bfseries 530 & 287 \\
\cellcolor{aaccolor} AAC-CG: top-$k$ & \cellcolor{aaccolor} 18559 & \cellcolor{aaccolor} 11925 & \cellcolor{aaccolor} 1228 & \cellcolor{aaccolor} \bfseries 719 & \cellcolor{aaccolor} \bfseries 1846 & \cellcolor{aaccolor} 1125 & \cellcolor{aaccolor} 487 & \cellcolor{aaccolor} 261 & \cellcolor{aaccolor} 537 & \cellcolor{aaccolor} 294 \\
\cellcolor{aaccolor} AAC-CG: pipage & \cellcolor{aaccolor} 18303 & \cellcolor{aaccolor} 11362 & \cellcolor{aaccolor} 1295 & \cellcolor{aaccolor} 734 & \cellcolor{aaccolor} 2002 & \cellcolor{aaccolor} 1125 & \cellcolor{aaccolor} 469 & \cellcolor{aaccolor} \bfseries 260 & \cellcolor{aaccolor} \bfseries 530 & \cellcolor{aaccolor} 287 \\
\bottomrule
\end{tabular*}
\par\smallskip
\begin{minipage}{\linewidth}
\footnotesize\raggedright
We compare three continuous extensions of Eq.~\eqref{eq:selection-objective} with projected gradient ascent (PGA). These are the mean per-query concave closure (the LP relaxation), the multilinear extension, and the Lov\'asz extension. Each uses two starts. The uniform start (U) sets $x_{0,i}=m/K_0$ on undirected graphs and $m/(2K_0)$ on directed graphs. The near-zero start ($\varepsilon$) sets $x_{0,i}=10^{-3}$ for every item. \par All PGA variants use the same feasible set, $T{=}200$ steps, and pipage rounding. AAC-CG uses the multilinear extension, a zero start, and the same $T$. Its top-$k$ variant replaces pipage rounding with top-$k$ rounding. $F$/Greedy divides the test objective by canonical greedy's value on the same queries. Bold marks the highest ratio in panel~(a) and the fewest expansions in panel~(b), for each graph--budget pair. \par PGA's ratio ranges from $0.974$ to $1.020$ across extensions and starts (per-query closure $0.978$--$1.002$; Lov\'asz $0.974$--$1.020$; multilinear $0.978$--$1.001$). Replacing pipage with top-$k$ rounding changes AAC-CG's ratio by $-0.93$ to $+2.58$\,pp, with mean $+0.47$\,pp over 10 graph--budget pairs. We also test the constructed witness (Example~\ref{ex:nine-vertex}) with the free split and $T{=}25$. All 7 combinations of extension, start, and rounding reach $1.0000$ of greedy at every query repetition count $N \in \{4, 16, 64, 256\}$. We omit those rows from the table.
\end{minipage}
\end{table}

\begin{table}[!htbp]
\centering
\caption{\tabletag{Ablation} Fixed and free direction budgets.}
\label{tab:free-split}
\small
\setlength{\tabcolsep}{4pt}
\textbf{(a) Training objective / upper bound (higher is better)}\par\smallskip
\begin{tabular*}{\linewidth}{@{\extracolsep{\fill}}l r r r r r r@{}}
\toprule
& & \multicolumn{2}{c}{Fixed split} & \multicolumn{3}{c}{Free split} \\
\cmidrule(lr){3-4}\cmidrule(lr){5-7}
Graph & Budget & Greedy & AAC-CG & Greedy & Block & AAC-CG \\
\midrule
Witness & $N{=}4$ & 0.5556 & \cellcolor{aaccolor} 0.8889 & \bfseries 1.0000 & \bfseries 1.0000 & \cellcolor{aaccolor} \bfseries 1.0000 \\
 & $N{=}16$ & 0.5152 & \cellcolor{aaccolor} 0.9697 & \bfseries 1.0000 & \bfseries 1.0000 & \cellcolor{aaccolor} \bfseries 1.0000 \\
 & $N{=}64$ & 0.5039 & \cellcolor{aaccolor} 0.9922 & \bfseries 1.0000 & \bfseries 1.0000 & \cellcolor{aaccolor} \bfseries 1.0000 \\
 & $N{=}256$ & 0.5010 & \cellcolor{aaccolor} 0.9981 & \bfseries 1.0000 & \bfseries 1.0000 & \cellcolor{aaccolor} \bfseries 1.0000 \\
\addlinespace[3pt]
SBM &  32 & \bfseries 0.8351 & \cellcolor{aaccolor} 0.8247 & \bfseries 0.8351 & \bfseries 0.8351 & \cellcolor{aaccolor} 0.8247 \\
 &  64 & \bfseries 0.8571 & \cellcolor{aaccolor} 0.8540 & \bfseries 0.8571 & \bfseries 0.8571 & \cellcolor{aaccolor} 0.8540 \\
\addlinespace[3pt]
BA &  32 & \bfseries 0.8317 & \cellcolor{aaccolor} 0.8280 & \bfseries 0.8317 & \bfseries 0.8317 & \cellcolor{aaccolor} 0.8280 \\
 &  64 & 0.8488 & \cellcolor{aaccolor} \bfseries 0.8491 & 0.8488 & 0.8488 & \cellcolor{aaccolor} \bfseries 0.8491 \\
\addlinespace[3pt]
NY &  32 & \bfseries 0.9544 & \cellcolor{aaccolor} 0.9516 & \bfseries 0.9544 & \bfseries 0.9544 & \cellcolor{aaccolor} 0.9516 \\
 &  64 & \bfseries 0.9771 & \cellcolor{aaccolor} 0.9766 & \bfseries 0.9771 & \bfseries 0.9771 & \cellcolor{aaccolor} 0.9766 \\
\addlinespace[3pt]
Manhattan &  32 & 0.9692 & \cellcolor{aaccolor} 0.9694 & 0.9728 & 0.9728 & \cellcolor{aaccolor} \bfseries 0.9739 \\
 &  64 & \bfseries 0.9851 & \cellcolor{aaccolor} 0.9844 & \bfseries 0.9851 & \bfseries 0.9851 & \cellcolor{aaccolor} 0.9844 \\
\addlinespace[3pt]
Modena &  32 & \bfseries 0.9475 & \cellcolor{aaccolor} 0.9314 & \bfseries 0.9475 & \bfseries 0.9475 & \cellcolor{aaccolor} 0.9419 \\
 &  64 & 0.9772 & \cellcolor{aaccolor} 0.9773 & 0.9772 & 0.9772 & \cellcolor{aaccolor} \bfseries 0.9773 \\
\bottomrule
\end{tabular*}
\par\medskip
\textbf{(b) Mean A* expansions (lower is better); forward/backward label counts in parentheses}\par\smallskip
\begin{tabular*}{\linewidth}{@{\extracolsep{\fill}}l r r r r r r@{}}
\toprule
& & \multicolumn{2}{c}{Fixed split} & \multicolumn{3}{c}{Free split} \\
\cmidrule(lr){3-4}\cmidrule(lr){5-7}
Graph & Budget & Greedy & AAC-CG & Greedy & Block & AAC-CG \\
\midrule
Witness & $N{=}4$ & \bfseries 2 & \cellcolor{aaccolor} \bfseries 2 & \bfseries 2 \normalfont (2/0) & \bfseries 2 \normalfont (2/0) & \cellcolor{aaccolor} \bfseries 2 \normalfont (2/0) \\
 & $N{=}16$ & \bfseries 2 & \cellcolor{aaccolor} \bfseries 2 & \bfseries 2 \normalfont (2/0) & \bfseries 2 \normalfont (2/0) & \cellcolor{aaccolor} \bfseries 2 \normalfont (2/0) \\
 & $N{=}64$ & \bfseries 2 & \cellcolor{aaccolor} \bfseries 2 & \bfseries 2 \normalfont (2/0) & \bfseries 2 \normalfont (2/0) & \cellcolor{aaccolor} \bfseries 2 \normalfont (2/0) \\
 & $N{=}256$ & \bfseries 2 & \cellcolor{aaccolor} \bfseries 2 & \bfseries 2 \normalfont (2/0) & \bfseries 2 \normalfont (2/0) & \cellcolor{aaccolor} \bfseries 2 \normalfont (2/0) \\
\addlinespace[3pt]
SBM &  32 & \bfseries 468 & \cellcolor{aaccolor} 469 & \bfseries 468 \normalfont (8/0) & \bfseries 468 \normalfont (8/0) & \cellcolor{aaccolor} 469 \normalfont (8/0) \\
 &  64 & \bfseries 254 & \cellcolor{aaccolor} 260 & \bfseries 254 \normalfont (16/0) & \bfseries 254 \normalfont (16/0) & \cellcolor{aaccolor} 260 \normalfont (16/0) \\
\addlinespace[3pt]
BA &  32 & \bfseries 488 & \cellcolor{aaccolor} 530 & \bfseries 488 \normalfont (8/0) & \bfseries 488 \normalfont (8/0) & \cellcolor{aaccolor} 530 \normalfont (8/0) \\
 &  64 & \bfseries 278 & \cellcolor{aaccolor} 287 & \bfseries 278 \normalfont (16/0) & \bfseries 278 \normalfont (16/0) & \cellcolor{aaccolor} 287 \normalfont (16/0) \\
\addlinespace[3pt]
NY &  32 & \bfseries 16709 & \cellcolor{aaccolor} 18303 & \bfseries 16709 \normalfont (4/4) & \bfseries 16709 \normalfont (4/4) & \cellcolor{aaccolor} 18303 \normalfont (4/4) \\
 &  64 & \bfseries 10889 & \cellcolor{aaccolor} 11362 & \bfseries 10889 \normalfont (8/8) & \bfseries 10889 \normalfont (8/8) & \cellcolor{aaccolor} 11362 \normalfont (8/8) \\
\addlinespace[3pt]
Manhattan &  32 & \bfseries 1139 & \cellcolor{aaccolor} 1295 & 1206 \normalfont (3/5) & 1206 \normalfont (3/5) & \cellcolor{aaccolor} 1208 \normalfont (3/5) \\
 &  64 & \bfseries 708 & \cellcolor{aaccolor} 734 & \bfseries 708 \normalfont (8/8) & \bfseries 708 \normalfont (8/8) & \cellcolor{aaccolor} 734 \normalfont (8/8) \\
\addlinespace[3pt]
Modena &  32 & \bfseries 1767 & \cellcolor{aaccolor} 2002 & \bfseries 1767 \normalfont (4/4) & \bfseries 1767 \normalfont (4/4) & \cellcolor{aaccolor} 2014 \normalfont (3/5) \\
 &  64 & \bfseries 1110 & \cellcolor{aaccolor} 1125 & \bfseries 1110 \normalfont (8/8) & \bfseries 1110 \normalfont (8/8) & \cellcolor{aaccolor} 1125 \normalfont (8/8) \\
\bottomrule
\end{tabular*}
\par\smallskip
\begin{minipage}{\linewidth}
\footnotesize\raggedright
We compare AAC-CG with fixed and free direction splits (Proposition~\ref{prop:free-split}). The fixed partition matroid assigns $m/2$ labels to each direction. The free uniform matroid lets continuous greedy allocate all $m$ labels between directions. In panel~(b), parentheses give the deployed forward/backward label counts $m_{\mathrm{fwd}}/m_{\mathrm{bwd}}$. Fixed splits are equal on directed graphs; undirected SBM and BA use forward labels only. Free splits follow the objective and need not be equal. \par Greedy is canonical item-level greedy. Block fills the forward block before the backward block when direction quotas apply. Objective is $F_{\mathrm{train}}$ divided by the Nemhauser online upper bound of its split. Bold marks the largest ratio in panel~(a) and the fewest expansions in panel~(b), within each row. We index witness rows by query repetition count $N$; each deploys $8$\,B/v. Free AAC-CG has a higher certificate ratio in 7 of 14 settings, including all four witness repetition counts. The ratio changes by $+0.00$ to $+11.11$\,pp across settings. Mean expansions change by $-6.7$ to $+0.6$\,\% relative to fixed.
\end{minipage}
\end{table}

\paragraph{Direction allocation.} On symmetric DIMACS graphs, forward and backward distance tables coincide, but their signed bounds are opposite.
A fixed direction quota can therefore affect the objective when we charge the two tables separately. The free-split study
covers NY, Manhattan, and Modena at two budgets; only the two NY graph--budget pairs are DIMACS
(Table~\ref{tab:free-split}). We observe limited changes in the selected direction counts. Canonical greedy
retains an equal allocation in five of six graph--budget pairs, and AAC-CG in four. Relative to the fixed split, AAC-CG's expansion count changes by $-6.7$ to $+0.6$\,\%.
The largest reduction occurs on Manhattan at $32$\,B/v. On the
nine-vertex witness (Example~\ref{ex:nine-vertex}), fixed-split greedy attains $(N+1)/(2N)$ of $\mathrm{OPT}_{\mathrm{fix}}$, or
$(N+1)/(2N+1)$ of $\mathrm{OPT}_{\mathrm{free}}$, while free-split greedy attains
$\mathrm{OPT}_{\mathrm{free}}$. The witness shows that a fixed direction quota can limit the optimum, despite the small changes on the tested road graphs.

\paragraph{Secondary results.} Under directed accounting, the mean test-objective ratios to
canonical greedy are $0.995$ for AAC-CG and $0.951$ for AAC-ST.
AAC-CG expands fewer nodes than FPS-ALT in $16$ of $19$ graph--budget pairs; AAC-ST does so in $0$ of $19$
(Table~\ref{tab:continuous-greedy}). AAC-CG expands more nodes on NY, BAY, and FLA at $32$\,B/v.
Under the primary protocol, AAC-CG expands fewer nodes than FPS-ALT in these three graph--budget pairs too.
That protocol shares distance tables and doubles the candidate pool under the $K_0=4K$ rule.
We did not isolate these two effects. These results use one test split per graph--budget pair.

\paragraph{Guarantees and empirical optima.} A stronger worst-case guarantee need not give a
better result on a particular graph. Under a partition matroid, AAC-CG's limiting approximation
factor exceeds greedy's $1/2$. Four primary graph--budget pairs use a partition matroid: Manhattan and Modena at $32$ and $64$\,B/v.
Greedy expands fewer nodes than AAC-CG in all four and has the higher test objective in three
(Table~\ref{tab:symmetric-accounting}). The other $15$ primary graph--budget pairs have a cardinality constraint,
where greedy's finite-rank factor exceeds AAC-CG's at $T=100$. Both approach the optimal leading
factor $1-1/e$ (Appendix~\ref{app:guarantee-constants}). We differentiate through the entropic-oracle trajectory; the primary results use the exact oracle.

Greedy is within $0.5$\,\% of the training optimum wherever MILP proves optimality (Table~\ref{tab:exact-selection}).
For two synthetic graph--budget pairs at $3{,}200$ training queries, we use dual bounds with gaps at most $2.4$\,\%. Local-search variants also usually exceed AAC-CG
on the test objective. Exact training optimization does not guarantee fewer test expansions:
in seven of eight graph--budget pairs trained on $100$ queries, the certified optimum expands more nodes on test queries than
greedy. Four also have a lower test objective, consistent with overfitting; the other three
have a higher objective but more expansions. We therefore evaluate both objective and expansions
on disjoint queries.

\begin{table}[!htbp]
\centering
\caption{\tabletag{Main} Selection objectives relative to the MILP bound.}
\label{tab:exact-selection}
\footnotesize
\setlength{\tabcolsep}{2pt}
\adjustbox{max width=\linewidth}{\begin{tabular}{l r S[table-format=1.4] S[table-format=1.4] S[table-format=1.4] S[table-format=1.4] S[table-format=1.4] S[table-format=1.4] S[table-format=1.4] S[table-format=1.4] S[table-format=1.4] S[table-format=3.2]}
\toprule
\multicolumn{2}{l}{\emph{$n{=}100$}} & \multicolumn{9}{c}{\textbf{$F_{\mathrm{train}}$ / optimum}} & {\textbf{MILP (s)}} \\
\cmidrule(lr){3-11}
\textbf{Graph} & {\textbf{B/v}} & {FPS-ALT} & {Greedy} & {Block} & {1-swap} & {2-swap} & {3-swap} & {CD} & {AAC-ST} & {SIMPLE} & \\
\midrule
\multirow{2}{*}{NY}        &  32 & 0.9850 & 0.9950 & 0.9937 & \bfseries 1.0000 & \bfseries 1.0000 & \bfseries 1.0000 & \bfseries 1.0000 & \cellcolor{aaccolor} 0.8938 & 0.9036 & 0.03 \\
                           &  64 & 0.9827 & 0.9962 & 0.9945 & 0.9993 & \bfseries 1.0000 & {--} & 0.9983 & \cellcolor{aaccolor} 0.9537 & 0.9633 & 0.03 \\
\midrule
\multirow{2}{*}{SBM}       &  32 & 0.9169 & 0.9989 & 0.9989 & \bfseries 1.0000 & \bfseries 1.0000 & \bfseries 1.0000 & \bfseries 1.0000 & \cellcolor{aaccolor} 0.9090 & 0.9016 & 0.48 \\
                           &  64 & 0.9070 & 0.9957 & 0.9957 & 0.9993 & \bfseries 1.0000 & {--} & 0.9993 & \cellcolor{aaccolor} 0.9097 & 0.9056 & 0.48 \\
\midrule
\multirow{2}{*}{BA}        &  32 & 0.8558 & 0.9981 & 0.9981 & \bfseries 1.0000 & \bfseries 1.0000 & \bfseries 1.0000 & \bfseries 1.0000 & \cellcolor{aaccolor} 0.9008 & 0.9198 & 0.07 \\
                           &  64 & 0.8069 & 0.9973 & 0.9973 & \bfseries 1.0000 & \bfseries 1.0000 & {--} & \bfseries 1.0000 & \cellcolor{aaccolor} 0.8577 & 0.9115 & 0.05 \\
\midrule
\multirow{2}{*}{OGB-arXiv} &  32 & 0.9132 & 0.9957 & 0.9957 & 0.9999 & \bfseries 1.0000 & \bfseries 1.0000 & 0.9999 & \cellcolor{aaccolor} 0.9374 & 0.9213 & 0.24 \\
                           &  64 & 0.8719 & 0.9995 & 0.9995 & \bfseries 1.0000 & \bfseries 1.0000 & {--} & \bfseries 1.0000 & \cellcolor{aaccolor} 0.9181 & 0.9359 & 0.04 \\
\bottomrule
\end{tabular}
}
\par\vspace{4pt}
\adjustbox{max width=\linewidth}{\begin{tabular}{l r S[table-format=1.4] S[table-format=1.4] S[table-format=1.4] S[table-format=1.4] S[table-format=1.4] S[table-format=1.4] S[table-format=1.4] l}
\toprule
\multicolumn{2}{l}{\emph{$n{=}3{,}200$}} & \multicolumn{7}{c}{\textbf{$F_{\mathrm{train}}$ / MILP bound}} & \textbf{Solver status} \\
\cmidrule(lr){3-9}
\textbf{Graph} & {\textbf{B/v}} & {FPS-ALT} & {Greedy} & {1-swap} & {CD} & {AAC-ST} & {SIMPLE} & {MILP} & \\
\midrule
\multirow{2}{*}{NY}        &  32 & 0.9947 & \bfseries 1.0000 & \bfseries 1.0000 & \bfseries 1.0000 & \cellcolor{aaccolor} 0.8726 & 0.9137 & \bfseries 1.0000 & proven, 30\,s \\
                           &  64 & 0.9883 & 0.9995 & 1.0000 & 1.0000 & \cellcolor{aaccolor} 0.9700 & 0.9840 & \bfseries 1.0000 & proven, 104\,s \\
\midrule
\multirow{2}{*}{SBM}       &  32 & 0.9844 & \bfseries 1.0000 & \bfseries 1.0000 & \bfseries 1.0000 & \cellcolor{aaccolor} 0.9561 & 0.9519 & \bfseries 1.0000 & proven, 3{,}190\,s \\
                           &  64 & 0.9543 & 0.9762 & 0.9765 & 0.9765 & \cellcolor{aaccolor} 0.9317 & 0.9387 & 0.9765 & gap 2.4\,\% at 9{,}719\,s \\
\midrule
\multirow{2}{*}{BA}        &  32 & 0.9685 & \bfseries 1.0000 & \bfseries 1.0000 & \bfseries 1.0000 & \cellcolor{aaccolor} 0.9593 & 0.9605 & \bfseries 1.0000 & proven, 1{,}819\,s \\
                           &  64 & 0.9483 & 0.9826 & 0.9847 & 0.9847 & \cellcolor{aaccolor} 0.9359 & 0.9399 & 0.9852 & gap 1.5\,\% at 9{,}129\,s \\
\midrule
\multirow{2}{*}{OGB-arXiv} &  32 & 0.9409 & \bfseries 1.0000 & \bfseries 1.0000 & \bfseries 1.0000 & \cellcolor{aaccolor} 0.9553 & 0.9427 & \bfseries 1.0000 & proven, 300\,s \\
                           &  64 & 0.9302 & 0.9988 & 0.9998 & \bfseries 1.0000 & \cellcolor{aaccolor} 0.9582 & 0.9729 & \bfseries 1.0000 & proven, 397\,s \\
\bottomrule
\end{tabular}
}
\par\smallskip
\begin{minipage}{\linewidth}
\footnotesize\raggedright
Every variant selects $m$ items from the same $K_0$ FPS pool and optimizes Eq.~\eqref{eq:selection-objective}. Entries divide $F_{\mathrm{train}}$ by the MILP bound: the certified optimum when available, or the dual upper bound otherwise. With the dual bound, the reported ratio is a lower bound on the ratio to the true optimum. Bold marks agreement with the bound to relative tolerance $10^{-9}$. \par The top block uses $n{=}100$ training queries. The MILP proves optimality for every graph--budget pair in at most $0.5$\,s. Greedy attains $0.995$--$1.000$ of the optimum. The $1$- and $2$-swap variants attain the optimum in 5 and 8 of 8 pairs, respectively. The $3$-swap variant does so in 4 of the 4 pairs with $m\leq8$ where it runs. Block is block greedy; $k$-swap and coordinate descent (CD) start from Greedy. AAC-ST and SIMPLE are seed means. \par The bottom block uses the paper's training sample, $n{=}3{,}200$. We cap the MILP at 1{,}000 branch-and-bound nodes so that it stops at the same point on every host. It proves optimality on NY-32, NY-64, SBM-32, BA-32, OGB-arXiv-32, OGB-arXiv-64; elsewhere it reports its incumbent with the certified gap. The status column gives the wall time. \par In 7 of 8 pairs trained on $n{=}100$ queries, the certified optimum expands more test nodes than greedy. Table~\ref{tab:continuous-greedy} evaluates the $n{=}3{,}200$ variants on both objective and expansions.
\end{minipage}
\end{table}

\paragraph{Constructed instances.} Table~\ref{tab:hardness-gadget} instantiates the
maximum-coverage reduction. Greedy and AAC-CG both satisfy their finite-step approximation bounds. In the constructed
greedy-gap example, greedy attains its cardinality factor $0.75$ at $m=2$, while AAC-CG attains
the optimum. The separation depends on the specified tie-break; reversing it lets greedy
attain the optimum too.

\begin{table}[!htbp]
\centering
\caption{\tabletag{Main} Selection on maximum-coverage graphs.}
\label{tab:hardness-gadget}
\small
\setlength{\tabcolsep}{5pt}
\adjustbox{max width=\linewidth}{\begin{tabular}{l r r r S[table-format=1.4] S[table-format=1.4] S[table-format=1.4] S[table-format=1.4] S[table-format=1.4]}
\toprule
\textbf{Instance} & {\textbf{$|U|$}} & {\textbf{$K_0$}} & {\textbf{$m$}} & {\textbf{OPT}}
  & {\textbf{Greedy}} & {\textbf{AAC-CG}} & {\textbf{Greedy/OPT}} & {\textbf{AAC-CG/OPT}} \\
\midrule
\texttt{greedy-gap} & 4 & 3 & 2 & 1.0000 & 0.7500 & \cellcolor{aaccolor} 1.0000 & 0.7500 & \cellcolor{aaccolor} 1.0000 \\
\texttt{rand-U12-n6-k2} & 12 & 6 & 2 & 0.4167 & 0.4167 & \cellcolor{aaccolor} 0.4167 & 1.0000 & \cellcolor{aaccolor} 1.0000 \\
\texttt{rand-U16-n8-k3} & 16 & 8 & 3 & 0.6250 & 0.6250 & \cellcolor{aaccolor} 0.6250 & 1.0000 & \cellcolor{aaccolor} 1.0000 \\
\texttt{rand-U20-n10-k3} & 20 & 10 & 3 & 0.6500 & 0.6500 & \cellcolor{aaccolor} 0.6500 & 1.0000 & \cellcolor{aaccolor} 1.0000 \\
\texttt{rand-U24-n10-k4} & 24 & 10 & 4 & 0.6667 & 0.6667 & \cellcolor{aaccolor} 0.6667 & 1.0000 & \cellcolor{aaccolor} 1.0000 \\
\texttt{rand-U32-n12-k4} & 32 & 12 & 4 & 0.6875 & 0.6875 & \cellcolor{aaccolor} 0.6875 & 1.0000 & \cellcolor{aaccolor} 1.0000 \\
\bottomrule
\end{tabular}
}
\par\smallskip
\begin{minipage}{\linewidth}
\footnotesize\raggedright
Each instance uses the construction of Theorem~\ref{thm:hardness} with $w_0{=}2$ and uniform query weights. We compute OPT by enumerating all $\binom{K_0}{m}$ selections. AAC-CG ($T{=}100$) reaches OPT in all 6 instances; greedy's lowest ratio is 0.7500. The finite-step guarantees range over 0.6836--0.7500 for greedy and 0.6210--0.6284 for AAC-CG. These small instances illustrate optimizer behavior. The approximation hardness follows from the theorem's reduction.
\end{minipage}
\end{table}

\subsection{Objective versus search}
\label{sec:objective-search}

A* evaluates bounds at every vertex it generates. The selection objective averages bounds
only at sampled query endpoints (Section~\ref{sec:objective}), which may not represent the vertices
that control search effort. A higher objective therefore
need not give fewer expansions. On Manhattan at $32$\,B/v, AAC-CG's test-objective ratio to
canonical greedy is $1.0009$, but it expands $1{,}295$ nodes compared with $1{,}139$ for greedy
(Table~\ref{tab:continuous-greedy}).
We compare variants within $131$ experimental configurations.
Within each configuration, we fix the graph, memory budget, label count, and candidate pool.
We also fix the training and test sample sizes, memory accounting, and direction-allocation rule.
Each configuration includes at least four variants. The mean Spearman correlation between test objective and expansions is $-0.68$, with a nine-graph
cluster-bootstrap interval $[-0.81,-0.57]$. Per-graph means range from $-0.87$ on FLA to $-0.48$ on
BA. Thus a higher test objective is associated with fewer expansions on average, but does not guarantee fewer expansions for each selection.

\subsection{Robustness and latency}
\label{sec:robustness}

\paragraph{Fresh queries.} We re-evaluate the fixed selections on $20$ fresh $500$-query sets per graph--budget pair under the secondary protocol (Table~\ref{tab:continuous-greedy-fresh}).
For the eight pairs on Manhattan, Modena, SBM, and BA, the selections also match those in the primary comparison. AAC-CG
has a higher mean expansion reduction than FPS-ALT in all ten graph--budget pairs, and eight paired intervals
exclude zero; the two that do not are NY and SBM at $32$\,B/v. The pooled margin is $+1.14$ points
($[+0.32,+2.15]$, two-level bootstrap). Against the best discrete variant, AAC-CG's pooled margin is $-0.61$ points ($[-0.97,-0.32]$).
AAC-CG has a positive mean margin over that variant only on BA at $64$\,B/v.
These intervals support fewer expansions than FPS-ALT in eight pairs, but more expansions than the best discrete variant in the pooled comparison.

\begin{table}[!htbp]
\centering
\caption{\tabletag{Main} Generalization to fresh query sets.}
\label{tab:continuous-greedy-fresh}
\small
\setlength{\tabcolsep}{4pt}
\adjustbox{max width=\linewidth}{\begin{tabular}{l r r S[table-format=2.2] l S[table-format=2.2] S[table-format=2.2] l l}
\toprule
& & & \multicolumn{4}{c}{\textbf{Reduction (\%)}} & \multicolumn{2}{c}{\textbf{$\Delta$ AAC-CG (pp) [95\,\% CI]}} \\
\cmidrule(lr){4-7}\cmidrule(lr){8-9}
\textbf{Graph} & {\textbf{B/v}} & {\textbf{$n$}} & {FPS-ALT} & best-disc. & {best-disc.} & {AAC-CG} & vs.\ FPS-ALT & vs.\ best-disc. \\
\midrule
\multirow{3}{*}{NY}  &  32 & 20 & 86.57 & Greedy & 87.54 & \cellcolor{aaccolor} 86.63 & \cellcolor{aaccolor} $+0.06$ [$-0.05$,\,$+0.18$] & \cellcolor{aaccolor} $-0.91$ [$-1.04$,\,$-0.77$] \\
                     &  64 & 20 & 90.65 & 1-swap & 91.87 & \cellcolor{aaccolor} 91.74 & \cellcolor{aaccolor} $+1.09$ [$+0.96$,\,$+1.21$] & \cellcolor{aaccolor} $-0.13$ [$-0.21$,\,$-0.07$] \\
                     & all & 20 & {--} & {--} & {--} & {--} & \cellcolor{aaccolor} $+0.58$ [$+0.47$,\,$+0.68$] & \cellcolor{aaccolor} $-0.52$ [$-0.60$,\,$-0.45$] \\
\midrule
\multirow{3}{*}{Manhattan} &  32 & 20 & 77.38 & 1-swap & 82.66 & \cellcolor{aaccolor} 80.27 & \cellcolor{aaccolor} $+2.89$ [$+2.70$,\,$+3.09$] & \cellcolor{aaccolor} $-2.39$ [$-2.53$,\,$-2.25$] \\
                     &  64 & 20 & 85.15 & Greedy & 88.34 & \cellcolor{aaccolor} 88.16 & \cellcolor{aaccolor} $+3.01$ [$+2.86$,\,$+3.16$] & \cellcolor{aaccolor} $-0.18$ [$-0.20$,\,$-0.15$] \\
                     & all & 20 & {--} & {--} & {--} & {--} & \cellcolor{aaccolor} $+2.95$ [$+2.83$,\,$+3.07$] & \cellcolor{aaccolor} $-1.28$ [$-1.35$,\,$-1.21$] \\
\midrule
\multirow{3}{*}{Modena} &  32 & 20 & 83.18 & 1-swap & 85.50 & \cellcolor{aaccolor} 84.20 & \cellcolor{aaccolor} $+1.02$ [$+0.83$,\,$+1.22$] & \cellcolor{aaccolor} $-1.30$ [$-1.48$,\,$-1.11$] \\
                     &  64 & 20 & 88.59 & 1-swap & 91.03 & \cellcolor{aaccolor} 90.91 & \cellcolor{aaccolor} $+2.33$ [$+2.17$,\,$+2.48$] & \cellcolor{aaccolor} $-0.12$ [$-0.18$,\,$-0.07$] \\
                     & all & 20 & {--} & {--} & {--} & {--} & \cellcolor{aaccolor} $+1.67$ [$+1.55$,\,$+1.80$] & \cellcolor{aaccolor} $-0.71$ [$-0.80$,\,$-0.61$] \\
\midrule
\multirow{3}{*}{SBM} &  32 & 20 & 90.21 & Greedy & 90.70 & \cellcolor{aaccolor} 90.27 & \cellcolor{aaccolor} $+0.06$ [$-0.04$,\,$+0.16$] & \cellcolor{aaccolor} $-0.43$ [$-0.53$,\,$-0.32$] \\
                     &  64 & 20 & 94.55 & 1-swap & 94.93 & \cellcolor{aaccolor} 94.77 & \cellcolor{aaccolor} $+0.22$ [$+0.16$,\,$+0.29$] & \cellcolor{aaccolor} $-0.15$ [$-0.20$,\,$-0.10$] \\
                     & all & 20 & {--} & {--} & {--} & {--} & \cellcolor{aaccolor} $+0.14$ [$+0.08$,\,$+0.21$] & \cellcolor{aaccolor} $-0.29$ [$-0.35$,\,$-0.22$] \\
\midrule
\multirow{3}{*}{BA}  &  32 & 20 & 88.87 & Greedy & 89.77 & \cellcolor{aaccolor} 89.18 & \cellcolor{aaccolor} $+0.31$ [$+0.21$,\,$+0.41$] & \cellcolor{aaccolor} $-0.59$ [$-0.67$,\,$-0.51$] \\
                     &  64 & 20 & 93.90 & 1-swap & 94.20 & \cellcolor{aaccolor} 94.30 & \cellcolor{aaccolor} $+0.40$ [$+0.32$,\,$+0.48$] & \cellcolor{aaccolor} $+0.09$ [$+0.05$,\,$+0.13$] \\
                     & all & 20 & {--} & {--} & {--} & {--} & \cellcolor{aaccolor} $+0.36$ [$+0.27$,\,$+0.44$] & \cellcolor{aaccolor} $-0.25$ [$-0.30$,\,$-0.20$] \\
\midrule
\textbf{All} & all & 100 & {--} & {--} & {--} & {--} & \cellcolor{aaccolor} $+1.14$ [$+0.32$,\,$+2.15$] & \cellcolor{aaccolor} $-0.61$ [$-0.97$,\,$-0.32$] \\
\bottomrule
\end{tabular}
}
\par\smallskip
\begin{minipage}{\linewidth}
\footnotesize\raggedright
We re-evaluate fixed selections from FPS-ALT, Greedy, Block, 1-swap, CD and AAC-CG on 20 fresh sets of 500 queries each. All selectors except FPS-ALT fit the $3{,}200$-query training set. Every fresh set is disjoint from training. We use seeds that were never used to choose $K_0$ or other hyperparameters. \par Entries are mean expansion reduction over Dijkstra (\%) for FPS-ALT, the best discrete variant, and AAC-CG ($T{=}100$). Before evaluation, we choose the discrete variant with the highest training objective among Greedy, Block, 1-swap, and CD (Section~\ref{sec:experiments}). The \textbf{$\Delta$} columns give paired mean differences in reduction (pp), pairing methods within each fresh set. Positive values favor AAC-CG. We use 10{,}000 percentile-bootstrap resamples with a fixed seed for 95\,\% confidence intervals. \par All budgets of a graph reuse the same fresh sets. For each \emph{all} row, we average over budgets within each set, then resample sets. For the final pooled row, we resample graphs, then sets within each graph. The column $n$ counts query sets. \par AAC-CG has higher mean reduction than FPS-ALT in 10 of 10 graph--budget pairs and than the best discrete variant in 1. The pooled differences are $+1.14$\,pp against FPS-ALT and $-0.61$\,pp against the best discrete variant.
\end{minipage}
\end{table}

\paragraph{Query shift.} We evaluate NY, Manhattan, and Modena at two memory budgets under three query distributions:
uniform, hotspot, and power-law. First, we use the same distribution for training and testing.
Canonical greedy expands fewer nodes than FPS-ALT in all $18$ combinations of graph, budget, and query distribution.
AAC-CG does so in $17$ combinations. Its only deficit is $0.005$ points on NY at $32$\,B/v under uniform queries. Under train-test distribution
shift, AAC-CG expands fewer nodes than FPS-ALT in $31$ of $36$ combinations of graph, budget,
training distribution, and test distribution. Its largest seed-mean shift penalty is
$1.9$ percentage points, measured against AAC-CG trained on the test distribution
(Table~\ref{tab:query-distribution-cross}). Hotspot samples concentrate on a small set of pairs:
$68$\,\% of test queries repeat a training pair. Appendix~\ref{app:query-distribution-cost}
specifies this protocol and the disjoint uniform and power-law samples. The hotspot results
therefore measure performance on repeated demand as well as unseen pairs.

\paragraph{Latency and preprocessing.} At $64$\,B/v, FPS-ALT has lower observed median (p50) latency than
AAC-ST on three of four DIMACS graphs and lower observed 95th-percentile (p95) latency on all four
(Table~\ref{tab:latency}).
These are point estimates from one training seed, not tests of percentile differences.
AAC-ST has higher mean latency on every graph and incurs $7.3$--$10.1\times$ the offline
cost (Table~\ref{tab:multi-axis-cost}). The BAY p50 exception shows that medians and means can order
methods differently.

We evaluate FPS-ALT, AAC-ST, and AAC-CG with the same split-table layout and heuristic implementation (Table~\ref{tab:latency-layout}).
We use three seeds, each with $60$ queries per graph--budget pair. Under a warm
cache, cycles per expansion differ by at most $10.5$\,\% on roads and $15.1$\,\% on SBM and BA.
Compared with FPS-ALT, AAC-CG has lower observed seed-mean latency in all $12$ graph--budget pairs
and lower p50 in $9$.
Using those mean latency savings, we estimate that preprocessing is amortized after $170$--$3{,}210$ queries.
Only $3$ of $36$ per-seed latency-ratio intervals for AAC-CG exclude $1$, all in its favor.
For AAC-ST, $8$ of $36$ exclude $1$, all favoring FPS-ALT. The table reports descriptive
seed means. We did not retain per-query timings, so we cannot compute pooled confidence intervals.

These timing queries are separate from the optimizer study. On FLA at $32$\,B/v, AAC-CG expands fewer nodes than FPS-ALT in the timing sample but more in the secondary optimizer sample.
The latter uses $500$ queries. We therefore report latency and expansion conclusions with their respective query protocols.

\begin{table}[!htbp]
\centering
\caption{\tabletag{Main} Median and tail query latency on DIMACS.}
\label{tab:latency}
\small
\adjustbox{max width=\linewidth}{\begin{tabular}{l l S[table-format=3.1] S[table-format=4.1]}
\toprule
\textbf{Graph} & \textbf{Method}
  & {\textbf{p50 (ms)}} & {\textbf{p95 (ms)}} \\
\midrule
\multirow{2}{*}{NY}  & \cellcolor{aaccolor} AAC-ST & \cellcolor{aaccolor} 18.3 & \cellcolor{aaccolor} 84.4 \\
                     & ALT &     \bfseries 12.3 &     \bfseries 64.6 \\
\multirow{2}{*}{BAY} & \cellcolor{aaccolor} AAC-ST & \cellcolor{aaccolor} \bfseries 19.0 & \cellcolor{aaccolor} 118.4 \\
                     & ALT &               21.8 &     \bfseries 91.7 \\
\multirow{2}{*}{COL} & \cellcolor{aaccolor} AAC-ST & \cellcolor{aaccolor} 41.5 & \cellcolor{aaccolor} 183.6 \\
                     & ALT &     \bfseries 30.2 &    \bfseries 158.3 \\
\multirow{2}{*}{FLA} & \cellcolor{aaccolor} AAC-ST & \cellcolor{aaccolor} 92.7 & \cellcolor{aaccolor} 649.0 \\
                     & ALT &     \bfseries 73.5 &    \bfseries 319.9 \\
\bottomrule
\end{tabular}
}
\par\smallskip
\begin{minipage}{\linewidth}
\footnotesize\raggedright
We measure median and tail query latency (p50/p95, in ms) on DIMACS at $\Bmem{=}64$\,B/v. AAC-ST uses $K_0{=}64,m{=}16$; ALT uses $K{=}8$. Both use the same deployed label memory. Bold marks the smaller observed percentile. These comparisons are descriptive; we do not report percentile confidence intervals. \par FPS-ALT has the smaller point estimate in 7 of the 8 entries. This includes 3 of the 4 p50 values, by $1.26$--$1.49\times$ and 4 of the 4 p95 values, by $1.16$--$2.03\times$. The straight-through selector has the lower point estimate on BAY at p50 ($19.0$ against $21.8$\,ms). \par We use the same heuristic implementation and timing loop for both methods. We randomize method order and use a warm cache, without flushing it. AAC-ST uses seed $42$. We report the minimum of three repetitions for each query. Table~\ref{tab:latency-layout} includes AAC-CG in the shared-layout comparison.
\end{minipage}
\end{table}

\begin{table}[!htbp]
\centering
\caption{\tabletag{Ablation} Query latency with a shared storage layout.}
\label{tab:latency-layout}
\footnotesize
\setlength{\tabcolsep}{4.5pt}
\adjustbox{max width=\linewidth}{\begin{tabular}{@{}l r l l S[table-format=6.0] l r S[table-format=3.1] S[table-format=6.0]@{}}
\toprule\addlinespace[3pt]
\textbf{Graph} & {\textbf{B/v}} & \textbf{Variant} & \textbf{p50 (ms)} & {\textbf{Exp.}}
  & \shortstack{\textbf{p50 ratio}\\\textbf{ALT/method}} & \shortstack{\textbf{Break-even}\\\textbf{queries}} & {\shortstack{\textbf{LLC misses}\\\textbf{per exp.}}} & {\shortstack{\textbf{Cycles}\\\textbf{per exp.}}} \\
\addlinespace[2pt]\midrule
\multirow{6}{*}{NY}        &  32 & FPS-ALT & 18.9 &  15708 & --                   &      -- & 2.6 &  10560 \\
                           &  32 & \cellcolor{aaccolor} AAC-ST     & \cellcolor{aaccolor} 25.3 & \cellcolor{aaccolor}  23694 & \cellcolor{aaccolor} 0.74                 & \cellcolor{aaccolor}      -- & \cellcolor{aaccolor} 2.9 & \cellcolor{aaccolor}  10366 \\
                           &  32 & \cellcolor{aaccolor} AAC-CG  & \cellcolor{aaccolor} 20.1 & \cellcolor{aaccolor}  15342 & \cellcolor{aaccolor} 0.94                 & \cellcolor{aaccolor}     910 & \cellcolor{aaccolor} 4.3 & \cellcolor{aaccolor}  11003 \\
\cmidrule{2-9}
                           &  64 & FPS-ALT & 14.7 &  10520 & --                   &      -- & 1.7 &  10408 \\
                           &  64 & \cellcolor{aaccolor} AAC-ST     & \cellcolor{aaccolor} 12.7 & \cellcolor{aaccolor}  11668 & \cellcolor{aaccolor} 1.15                 & \cellcolor{aaccolor}      -- & \cellcolor{aaccolor} 2.3 & \cellcolor{aaccolor}  10525 \\
                           &  64 & \cellcolor{aaccolor} AAC-CG  & \cellcolor{aaccolor} 12.8 & \cellcolor{aaccolor}   9603 & \cellcolor{aaccolor} 1.14                 & \cellcolor{aaccolor}     805 & \cellcolor{aaccolor} 2.3 & \cellcolor{aaccolor}  10627 \\
\midrule
\multirow{6}{*}{FLA}       &  32 & FPS-ALT & 112.8 &  82714 & --                   &      -- & 5.2 &  11005 \\
                           &  32 & \cellcolor{aaccolor} AAC-ST     & \cellcolor{aaccolor} 173.2 & \cellcolor{aaccolor} 122512 & \cellcolor{aaccolor} 0.65                 & \cellcolor{aaccolor}      -- & \cellcolor{aaccolor} 5.4 & \cellcolor{aaccolor}  10952 \\
                           &  32 & \cellcolor{aaccolor} AAC-CG  & \cellcolor{aaccolor} 127.0 & \cellcolor{aaccolor}  81389 & \cellcolor{aaccolor} 0.89                 & \cellcolor{aaccolor}     715 & \cellcolor{aaccolor} 4.8 & \cellcolor{aaccolor}  10908 \\
\cmidrule{2-9}
                           &  64 & FPS-ALT & 83.8 &  65173 & --                   &      -- & 4.3 &  10931 \\
                           &  64 & \cellcolor{aaccolor} AAC-ST     & \cellcolor{aaccolor} 117.6 & \cellcolor{aaccolor}  83299 & \cellcolor{aaccolor} 0.71                 & \cellcolor{aaccolor}      -- & \cellcolor{aaccolor} 4.7 & \cellcolor{aaccolor}  10935 \\
                           &  64 & \cellcolor{aaccolor} AAC-CG  & \cellcolor{aaccolor} 74.4 & \cellcolor{aaccolor}  50776 & \cellcolor{aaccolor} 1.13                 & \cellcolor{aaccolor}     170 & \cellcolor{aaccolor} 4.0 & \cellcolor{aaccolor}  11031 \\
\midrule
\multirow{6}{*}{Manhattan} &  32 & FPS-ALT & 1.5 &   1344 & --                   &      -- & 0.7 &  10049 \\
                           &  32 & \cellcolor{aaccolor} AAC-ST     & \cellcolor{aaccolor} 2.1 & \cellcolor{aaccolor}   1684 & \cellcolor{aaccolor} 0.74                 & \cellcolor{aaccolor}      -- & \cellcolor{aaccolor} 0.8 & \cellcolor{aaccolor}   9813 \\
                           &  32 & \cellcolor{aaccolor} AAC-CG  & \cellcolor{aaccolor} 1.4 & \cellcolor{aaccolor}   1084 & \cellcolor{aaccolor} 1.09                 & \cellcolor{aaccolor}     435 & \cellcolor{aaccolor} 0.9 & \cellcolor{aaccolor}  10282 \\
\cmidrule{2-9}
                           &  64 & FPS-ALT & 1.2 &    947 & --                   &      -- & 1.0 &  10467 \\
                           &  64 & \cellcolor{aaccolor} AAC-ST     & \cellcolor{aaccolor} 1.4 & \cellcolor{aaccolor}   1187 & \cellcolor{aaccolor} 0.85                 & \cellcolor{aaccolor}      -- & \cellcolor{aaccolor} 0.9 & \cellcolor{aaccolor}  10115 \\
                           &  64 & \cellcolor{aaccolor} AAC-CG  & \cellcolor{aaccolor} 1.0 & \cellcolor{aaccolor}    748 & \cellcolor{aaccolor} 1.25                 & \cellcolor{aaccolor} 1{,}746 & \cellcolor{aaccolor} 1.7 & \cellcolor{aaccolor}  11174 \\
\midrule
\multirow{6}{*}{Modena}    &  32 & FPS-ALT & 2.1 &   1963 & --                   &      -- & 2.2 &   9924 \\
                           &  32 & \cellcolor{aaccolor} AAC-ST     & \cellcolor{aaccolor} 2.8 & \cellcolor{aaccolor}   2518 & \cellcolor{aaccolor} 0.77                 & \cellcolor{aaccolor}      -- & \cellcolor{aaccolor} 2.5 & \cellcolor{aaccolor}  10280 \\
                           &  32 & \cellcolor{aaccolor} AAC-CG  & \cellcolor{aaccolor} 1.9 & \cellcolor{aaccolor}   1821 & \cellcolor{aaccolor} 1.14                 & \cellcolor{aaccolor}     795 & \cellcolor{aaccolor} 3.0 & \cellcolor{aaccolor}  10593 \\
\cmidrule{2-9}
                           &  64 & FPS-ALT & 1.4 &   1368 & --                   &      -- & 1.9 &   9974 \\
                           &  64 & \cellcolor{aaccolor} AAC-ST     & \cellcolor{aaccolor} 1.7 & \cellcolor{aaccolor}   1702 & \cellcolor{aaccolor} 0.84                 & \cellcolor{aaccolor}      -- & \cellcolor{aaccolor} 2.3 & \cellcolor{aaccolor}  10057 \\
                           &  64 & \cellcolor{aaccolor} AAC-CG  & \cellcolor{aaccolor} 1.2 & \cellcolor{aaccolor}   1053 & \cellcolor{aaccolor} 1.17                 & \cellcolor{aaccolor} 1{,}024 & \cellcolor{aaccolor} 2.7 & \cellcolor{aaccolor}  10634 \\
\midrule
\multirow{6}{*}{SBM}       &  32 & FPS-ALT & 5.7 &    506 & --                   &      -- & 169.4 & 107787 \\
                           &  32 & \cellcolor{aaccolor} AAC-ST     & \cellcolor{aaccolor} 6.0 & \cellcolor{aaccolor}    529 & \cellcolor{aaccolor} 0.96                 & \cellcolor{aaccolor}      -- & \cellcolor{aaccolor} 196.0 & \cellcolor{aaccolor} 111644 \\
                           &  32 & \cellcolor{aaccolor} AAC-CG  & \cellcolor{aaccolor} 5.5 & \cellcolor{aaccolor}    475 & \cellcolor{aaccolor} 1.04                 & \cellcolor{aaccolor} 1{,}184 & \cellcolor{aaccolor} 190.2 & \cellcolor{aaccolor} 113582 \\
\cmidrule{2-9}
                           &  64 & FPS-ALT & 3.2 &    268 & --                   &      -- & 225.1 & 138257 \\
                           &  64 & \cellcolor{aaccolor} AAC-ST     & \cellcolor{aaccolor} 3.8 & \cellcolor{aaccolor}    300 & \cellcolor{aaccolor} 0.83                 & \cellcolor{aaccolor}      -- & \cellcolor{aaccolor} 211.6 & \cellcolor{aaccolor} 130189 \\
                           &  64 & \cellcolor{aaccolor} AAC-CG  & \cellcolor{aaccolor} 3.5 & \cellcolor{aaccolor}    247 & \cellcolor{aaccolor} 0.90                 & \cellcolor{aaccolor} 2{,}827 & \cellcolor{aaccolor} 232.6 & \cellcolor{aaccolor} 140287 \\
\midrule
\multirow{6}{*}{BA}        &  32 & FPS-ALT & 3.3 &    617 & --                   &      -- & 10.2 &  36774 \\
                           &  32 & \cellcolor{aaccolor} AAC-ST     & \cellcolor{aaccolor} 3.1 & \cellcolor{aaccolor}    620 & \cellcolor{aaccolor} 1.06                 & \cellcolor{aaccolor} 13{,}586 & \cellcolor{aaccolor} 36.1 & \cellcolor{aaccolor}  41462 \\
                           &  32 & \cellcolor{aaccolor} AAC-CG  & \cellcolor{aaccolor} 3.2 & \cellcolor{aaccolor}    584 & \cellcolor{aaccolor} 1.03                 & \cellcolor{aaccolor} 1{,}035 & \cellcolor{aaccolor} 35.4 & \cellcolor{aaccolor}  42288 \\
\cmidrule{2-9}
                           &  64 & FPS-ALT & 1.9 &    340 & --                   &      -- & 11.5 &  42127 \\
                           &  64 & \cellcolor{aaccolor} AAC-ST     & \cellcolor{aaccolor} 2.1 & \cellcolor{aaccolor}    362 & \cellcolor{aaccolor} 0.93                 & \cellcolor{aaccolor}      -- & \cellcolor{aaccolor} 7.2 & \cellcolor{aaccolor}  40475 \\
                           &  64 & \cellcolor{aaccolor} AAC-CG  & \cellcolor{aaccolor} 1.7 & \cellcolor{aaccolor}    310 & \cellcolor{aaccolor} 1.10                 & \cellcolor{aaccolor} 3{,}210 & \cellcolor{aaccolor} 30.2 & \cellcolor{aaccolor}  46596 \\
\bottomrule
\end{tabular}
}
\par\smallskip
\begin{minipage}{\linewidth}
\footnotesize\raggedright
We deploy FPS-ALT, straight-through AAC-ST, and AAC-CG through the same split-table layout and heuristic implementation, with a warm cache. We use 3 seeds and 60 queries per graph--budget pair. Entries are descriptive seed means; DIMACS uses secondary directed accounting. \par The ratio is FPS-ALT p50 divided by method p50; values above $1$ favor the method. Break-even queries repay the extra offline cost through mean query-time savings; a dash means no mean latency saving. We compute ratios and break-even counts from seed means. For each seed, we use 10{,}000 bootstrap resamples for percentile 95\,\% intervals of p50 and the paired ratio of medians. We did not retain per-query timings, so pooled intervals are unavailable. \par AAC-CG has lower p50 than FPS-ALT in 9 of 12 graph--budget pairs; straight-through AAC-ST does in 2. The ratio interval contains $1$ in 61 of 72 seed-level comparisons. Latency and expansion differences have the same sign in 19 of 24 comparisons with FPS-ALT, at a fixed graph and budget. These comparisons cover both AAC-ST and AAC-CG. We measure last-level cache (LLC) misses and CPU cycles per expansion with hardware counters. Cycles per expansion agree across the three methods to within $15.1\,\%$ for every graph--budget pair.
\end{minipage}
\end{table}

\subsection{AAC-ST diagnostics on road networks}
\label{sec:st-roads}

Sections~\ref{sec:st-roads}--\ref{sec:training-drift} study the AAC-ST selector and the
classical baselines. Table~\ref{tab:main-matched-memory} summarizes AAC-ST across graph classes;
these runs use the directionality-aware budgets of Section~\ref{sec:setup}, which charge both
directions on DIMACS.

\providecommand{\tostpass}{\ensuremath{\checkmark}}
\providecommand{\tostfail}{\ensuremath{\times}}
\providecommand{\tostna}{\textendash}
\begin{table}[!htbp]
\centering
\caption{\tabletag{Ablation} Straight-through selection across graph classes.}
\label{tab:main-matched-memory}
\small
\begin{tabularx}{\textwidth}{@{}l X
    S[table-format=2.2(2)]
    c
    c c@{}}
\toprule
\textbf{Graph class} & \textbf{Graph (size)}
  & {\textbf{AAC-ST (\%)}} & {\textbf{FPS-ALT (\%)}} & {\textbf{CDH best (\%)}}
  & {\textbf{TOST $\Delta{=}1$\,pp}} \\
\midrule
\multirow{4}{*}{DIMACS road}
  & NY (264K)          & \cellcolor{aaccolor} 88.57 \pm 0.76 & $\mathbf{90.23}$ & 90.7 & \tostna \\
  & BAY (321K)         & \cellcolor{aaccolor} 85.94 \pm 2.83 & $\mathbf{89.96}$ & {--} & \tostna \\
  & COL (436K)         & \cellcolor{aaccolor} 84.54 \pm 0.97 & $\mathbf{88.97}$ & {--} & \tostna \\
  & FLA (1.07M)        & \cellcolor{aaccolor} 83.96 \pm 2.20 & $\mathbf{87.75}$ & {--} & \tostna \\
\midrule
\multirow{4}{*}{OSMnx road}
  & Modena (26K)       & \cellcolor{aaccolor} 87.89 \pm 0.64 & $\mathbf{88.98}$ & 92.2 & \tostna \\
  & Manhattan (13K)    & \cellcolor{aaccolor} 83.17 \pm 4.68 & $\mathbf{86.97}$ & {--} & \tostna \\
  & Berlin (41K)       & \cellcolor{aaccolor} 90.86 \pm 1.16 & $\mathbf{92.12}$ & {--} & \tostna \\
  & Los Angeles (92K)  & \cellcolor{aaccolor} 90.51 \pm 0.67 & $\mathbf{91.17}$ & {--} & \tostna \\
\midrule
\multirow{2}{*}{Synthetic}
  & SBM (10K)          & \cellcolor{aaccolor} \bfseries 94.01 \pm 0.30 & 93.88 & 93.3 & \tostpass \\
  & BA (10K)           & \cellcolor{aaccolor} 93.81 \pm 0.37 & $\mathbf{93.85}$ & {--} & \tostpass \\
\midrule
Citation
  & OGB-arXiv (169K)   & \cellcolor{aaccolor} \bfseries 92.93 \pm 0.70 & 91.81 & {--} & \tostfail\,\textsuperscript{a,\,\dag} \\
\bottomrule
\end{tabularx}
\par\vspace{2pt}
{\footnotesize
$^{a}$\,OGB-arXiv: equivalence not established ($p_{\mathrm{TOST}}{=}0.602$; mean AAC-ST$-$ALT difference $+1.12$\,pp), and neither margin test rejects ($p{=}0.398$ for above $+\Delta$, $p{=}1.000$ for below $-\Delta$).
$^{\dagger}$\,AAC-ST result is bolded only because AAC-ST is numerically higher; TOST equivalence at $\Delta{=}1$\,pp is \emph{not} established there (Section~\ref{sec:ogb-main}).
}
\par\smallskip
\begin{minipage}{\linewidth}
\footnotesize\raggedright
We compare expansion reduction at $\Bmem{=}64$\,B/v across four graph classes. We use 8 road networks (4 DIMACS, 4 OSMnx), 2 synthetic graphs (SBM, BA), and 1 citation network (OGB-arXiv). Entries are mean expansion reduction over Dijkstra (\% $\pm$ std across 5 seeds). We match FPS-ALT and AAC-ST label memory under the accounting in Section~\ref{sec:setup}. AAC-ST uses $K_0{=}64, m{=}16$ on DIMACS. On OSMnx, we select $K_0$ on validation queries disjoint from the test queries. \par The CDH column gives the best seed-averaged configuration from Table~\ref{tab:cdh-sweep} on each of the three graphs where we evaluated CDH. The sweep varies pool size, storage rule, and BPMX on undirected SBM. CDH expansions include the goal-side bounding search. CDH uses the sweep's queries, which differ from the AAC-ST and FPS-ALT queries. Bold compares only AAC-ST with FPS-ALT. \par TOST uses margin $\Delta{=}1$ percentage point: \tostpass{} means equivalent within $\Delta$, \tostfail{} means equivalence not established, and \tostna{} means not applicable. For DIMACS, we use the FDR-corrected Wilcoxon tests in Table~\ref{tab:dimacs-wilcoxon-percell} instead of TOST (Section~\ref{sec:dimacs}). We perform no significance test for OSMnx. Footnotes explain the cases where TOST does not establish equivalence. \par Tables~\ref{tab:valsplit},~\ref{tab:osmnx}, and~\ref{tab:main_results} give road results at each budget; Table~\ref{tab:matched-hybrid} gives non-road results. Table~\ref{tab:cdh-reference} gives the full CDH comparison. ALT has zero seed variation because FPS uses a fixed LCC start vertex.
\end{minipage}
\end{table}

\subsubsection{DIMACS}
\label{sec:dimacs}

We apply the validation-split protocol to DIMACS NY and FLA and OSMnx Modena and Manhattan.
FPS-ALT expands fewer nodes than AAC-ST in $10$ of $12$ graph--budget pairs.
AAC-ST expands fewer nodes on both OSMnx graphs at $128$\,B/v (Table~\ref{tab:valsplit}). Validation-to-test
gaps remain within $2.0$ points for AAC-ST and differ from ALT's by at most $1.7$ points, so
validation selection does not explain the ordering between AAC-ST and FPS-ALT.

\paragraph{Paired analysis.} On DIMACS, we average the model seeds within each of the $100$ shared queries.
We then apply one two-sided Wilcoxon test per graph--budget pair, with Benjamini--Hochberg FDR $q=0.05$ across $12$ pairs. The tests favor FPS-ALT in all $12$. These tests condition on the trained
models. Crossed random-effects intervals also account for seed variation and exclude zero in
$9$ graph--budget pairs; the intervals for BAY, COL, and FLA at $128$\,B/v include zero
(Table~\ref{tab:dimacs-wilcoxon-percell}).

\paragraph{FastMap on DIMACS.}
\label{sec:fastmap}
We observe no admissibility violations across $800$ DIMACS queries for FastMap with the linear residual projection~\citep{cohen2018fastmap}.
All returned paths match Dijkstra's cost.
As memory increases from $32$ to $256$\,B/v, its expansion reduction stays within
$81.4$--$85.5$\,\%. ALT and AAC-ST exceed $95$\,\% expansion reduction on NY and reach $94.4$ and $94.6$\,\% on FLA. FM+DH allocates half of the same bytes to differential-heuristic pivots.
It continues to improve, reaching $94.8$\,\% on NY and $94.7$\,\% on FLA at $256$\,B/v.
At these budgets, additional pivots reduce expansions more than additional FastMap coordinates. Appendix~\ref{app:fastmap} specifies the projection and correctness checks.

\subsubsection{OSMnx directed road networks}
\label{sec:osmnx}

FPS-ALT expands fewer nodes than AAC-ST in $11$ of $12$ graph--budget pairs in Table~\ref{tab:osmnx}.
For this comparison, we select AAC-ST's $K_0$ separately for each seed on validation queries.
FPS-ALT's expansion-reduction advantage narrows from $1.7$--$6.0$ points at $32$\,B/v to $0.4$--$0.9$ at $128$\,B/v,
where AAC-ST leads only on Manhattan. The validation-split protocol also favors AAC-ST on Modena at
$128$\,B/v because it selects $K_0$ from more candidates (Table~\ref{tab:valsplit}). AAC-CG, by contrast, expands fewer nodes
than FPS-ALT on Manhattan and Modena at both $32$ and $64$\,B/v in the primary comparison (Table~\ref{tab:symmetric-accounting}).

\subsubsection{Hybrid \texorpdfstring{$\max(\hA, \hFPS)$}{max(AAC-ST, FPS-ALT)} evaluation}
\label{sec:hybrid}

At each component's budget, write $\hFPS$ for FPS-ALT and $h_{\mathrm{CG}}$ for AAC-CG's rounded
heuristic. The pointwise maximum of AAC-ST and FPS-ALT is admissible. We distinguish a \emph{matched-memory}
half/half hybrid, which allocates $\Bmem/2$ to each component, from a \emph{doubled-memory}
hybrid that gives each component $\Bmem$.
At matched total memory, pure FPS-ALT expands fewer nodes than the road hybrid in five of six comparisons.
On Manhattan at $96$\,B/v, the hybrid leads by $0.3$ points, within the variation across seeds
(Table~\ref{tab:hybrid}). The non-road half/half hybrid also expands more nodes than the better pure variant in
all nine non-road graph--budget pairs (Section~\ref{sec:matched-hybrid}).

The doubled-memory maximum expands fewer nodes than either smaller component.
However, a single AAC-CG selector with $2m$ labels expands fewer nodes than both AAC-ST$+$ALT and AAC-CG$+$ALT.
This holds in all $18$ tested combinations of graph, budget, and query distribution (Table~\ref{tab:hybrid-stability}).
We therefore obtain lower expansion counts by giving the full budget to one selector.

\subsubsection{Selection strategy ablation}
\label{sec:selection-ablation}

At $64$\,B/v on Modena and Manhattan, Greedy-Max improves expansion reduction over FPS-ALT
by $1.6$ and $1.2$ points, respectively (Table~\ref{tab:selection-ablation}). Random-restart
FPS does not improve on the canonical prefix. AAC-ST expands fewer nodes than random selection only on
Modena. All tested default models select distinct items within each direction, so duplicate selections do not explain these deficits.
We also sweep $\lambda_{\mathrm{uniq}}\in\{0,0.01,0.1\}$ on Modena at $K_0=64,m=16$ over five seeds.
Every run selects distinct items within each direction. For each seed, all penalty strengths give identical expansion counts.
These margins apply to the tested budgets and queries. They do not bound the improvement possible with other selections.

\paragraph{Coverage.} We measure the symmetrized radius over the same reachable vertex set
in both directions. Separate directional radii can underestimate it arbitrarily
(Theorem~\ref{thm:covering-radius}(c)). At $K=16$, the FPS radius is $22{,}007$\,m on Modena
and $34{,}204$\,m on NY; at $K=64$, it falls to $10{,}317$ and $16{,}439$\,m.
We convert NY arc lengths from $0.1$\,m units. Gonzalez's factor gives no tighter
instance-specific certificate, so these radii neither establish near-optimal coverage nor limit
query-adaptive gains.

\subsection{AAC-ST on non-road graphs}
\label{sec:synthetic}

These studies use the directionality-aware budgets of Section~\ref{sec:setup}. FPS-ALT leads
AAC-ST in five of six graph--budget pairs on SBM and BA, and AAC-ST leads on OGB-arXiv at all three budgets (Table~\ref{tab:matched-hybrid}).
The studies test whether query adaptation helps on community and hub-dominated graphs.

\subsubsection{Stochastic block model (SBM, weighted)}
\label{sec:synthetic-sbm-ba}

We use $10{,}000$ nodes in five equal communities, $p_{\mathrm{in}}=0.05$,
$p_{\mathrm{out}}=0.001$, and undirected edge weights uniform in $[1,10]$.
At $32/64/128$\,B/v, AAC-ST uses $m=8/16/32$ and $K_0=4m$, against ALT with $K=m$.
Over five seeds and $100$ evaluation queries, AAC-ST's mean expansion-reduction margins over
FPS-ALT are $-1.08$, $+0.13$, and $-0.20$ points. Crossed seed-by-query TOST at
$\Delta=1$ point and $\alpha=0.05$ establishes equivalence at $64$ and $128$\,B/v.
At $32$\,B/v it establishes neither equivalence ($p_{\mathrm{TOST}}=0.56$) nor an ALT lead
beyond $\Delta$ ($p=0.44$).

Greedy-Max selects from the same $4m$ FPS pool using queries disjoint from evaluation.
Its expansion reduction exceeds FPS-ALT's by $0.95$ and $0.99$ points at $32$ and $64$\,B/v, respectively.
It matches FPS-ALT to two decimals at $128$\,B/v (Table~\ref{tab:matched-hybrid}). Other selections from the pool may give larger improvements.

\subsubsection{Barab\'{a}si--Albert (BA, with edge costs)}
\label{sec:synthetic-ba}

We use $10{,}000$ nodes, attachment parameter $m_{\mathrm{attach}}=5$, undirected edge
weights uniform in $[1,10]$, and the same budgets as SBM. AAC-ST's expansion reduction is lower than FPS-ALT's by
$0.94$, $0.04$, and $0.35$ points at $32/64/128$\,B/v. TOST establishes equivalence
within one point at the two larger budgets, but not at $32$\,B/v
($p_{\mathrm{TOST}}=0.45$). Greedy-Max also has lower expansion reduction than FPS-ALT at every budget, by
$1.65$, $0.17$, and $0.31$ points.

\subsubsection{OGB-arXiv (citation network)}
\label{sec:ogb-main}

On the symmetrized $169{,}343$-node OGB-arXiv graph~\citep{hu2020ogb}, AAC-ST with
$K_0=4m$ improves expansion reduction over FPS-ALT by $1.17\pm1.30$, $1.12\pm0.70$,
and $0.45\pm0.22$ percentage points at $32/64/128$\,B/v over five seeds.
Crossed intervals support a positive lead at the two larger budgets (one-sided $p=0.013$
and $0.011$), but not at $32$\,B/v ($p=0.11$). Only $128$\,B/v establishes equivalence
within one percentage point ($p_{\mathrm{TOST}}=0.003$).
Failure to establish equivalence does not establish a difference beyond the one-point margin.
A small positive lead can still fall within that margin.

The optimizer comparison uses the same seeds, $100$ test queries, and $4m$ pool, with
$3{,}200$ fixed training queries per seed and no entropy penalty. Its AAC-ST control gains
$1.03$, $1.11$, and $0.25$ points. AAC-CG gains $2.08\pm0.80$, $1.56\pm0.40$, and
$0.94\pm0.14$ points, with crossed intervals supporting a lead at every budget.
Doubling the pool to $8m$ at fixed deployed memory raises AAC-CG's mean gains to
$2.33$, $2.13$, and $1.05$ points. Thus both optimizer and pool size affect the result.

The symmetrized positive-weight shortest-path distance is a metric, so
Theorem~\ref{thm:covering-radius} applies. Metricity therefore does not distinguish this graph from roads.
All $15$ combinations of training seed and memory budget have zero admissibility violations. Greedy-Max gains
$0.7$--$1.5$ points over FPS-ALT (Table~\ref{tab:matched-hybrid}).
Appendix~\ref{app:reproducibility} specifies the weighted instance; Section~\ref{sec:setup} defines the memory accounting.

\subsubsection{Cross-graph matched-memory hybrid summary}
\label{sec:matched-hybrid}

Table~\ref{tab:matched-hybrid} compares pure AAC-ST, FPS-ALT, the half/half hybrid, the
doubled-memory maximum, and Greedy-Max. The half/half hybrid expands more nodes than the better pure method in all nine non-road graph--budget pairs.
In eight pairs, it expands more nodes than both pure methods. The doubled-memory maximum
improves on each smaller component by $1.14$--$5.65$ points, but uses twice the storage.
For example, the SBM hybrid built from $64$\,B/v components reaches $96.46$\,\% reduction,
below pure ALT's $97.01$\,\% and AAC-ST's $96.81$\,\% at the same $128$\,B/v total budget.
These results favor spending the budget on one selector.

\subsubsection{Reference CDH baseline at matched memory}
\label{sec:cdh-headtohead}

Our CDH reference follows \citet{goldenberg2011cdh,goldenberg2017cdh}: Avoid places
$P_{\mathrm{cdh}}$ pivots, and each vertex stores $r$ float32 distances. The canonical
$P_{\mathrm{cdh}}=64$ configuration uses round-robin storage with implicit pivot identities.
We store no pivot-to-pivot table. A goal-side Dijkstra search bounds missing pivot distances,
stopping once each missing pivot has been seen at four storing states. Its settled states
count toward total expansions.

We evaluate $372$ configurations across SBM, Modena, and NY, varying pool size and round-robin,
farthest, nearest, or random storage. Random storage averages three seeds; other configurations
are deterministic given the component anchor. CDH and FPS-ALT share $100$ queries.
We evaluate one-step bidirectional pathmax (BPMX) only on undirected SBM, where every edge has an equal-weight reverse.
It retains raised heuristic values for the rest of the query.
All configurations return Dijkstra-optimal path costs with reopenings.

\paragraph{Results.} Tuned CDH expands fewer nodes than FPS-ALT at all four Modena budgets and on NY at $64$ and $128$\,B/v.
FPS-ALT expands fewer nodes in the other six graph--budget pairs. The NY $128$\,B/v lead is only
$0.004$ points. On Modena, tuned CDH gains $2.2$--$14.2$ points; on SBM it trails by
$0.6$--$22.1$ (Tables~\ref{tab:cdh-reference} and~\ref{tab:cdh-sweep}). AAC-CG's recorded expansion reductions are higher than tuned CDH's on SBM and NY and lower on Modena.
These results use different test queries, so they do not constitute a paired comparison.
BPMX improves reduction by $2.0$ points on average over $120$ matched SBM configurations
evaluated with and without BPMX, and never increases expansions.

\paragraph{Storage tradeoff.} Round-robin gives the fewest total expansions at all twelve graph--budget pairs.
It spreads pivot holders across the graph, limiting the missing-pivot search.
Farthest and nearest storage can require more goal-side expansions than Dijkstra's entire
query. Across $330$ configurations with $r<P_{\mathrm{cdh}}$, goal-side bounding cost and expansion
reduction have Spearman correlation $-0.88$. On SBM at $32$\,B/v and above, the best
configuration retains its whole pool and therefore performs no compression.
CDH trades vertex-specific storage for extra query work; AAC keeps every retained dimension
directly readable. The comparison depends on graph and pivot placement, and does not isolate
storage from Avoid's landmark choice. It also does not cover the grid-game regime emphasized
in the original CDH evaluations.

\begin{table}[!htbp]
\centering
\caption{\tabletag{Main} ALT, AAC variants, and compressed differential heuristics.}
\label{tab:cdh-reference}
\small
\setlength{\tabcolsep}{4pt}
\adjustbox{max width=\linewidth}{\begin{tabular}{l l c c c}
\toprule
\textbf{Graph} & \textbf{Method}
  & {\textbf{32 B/v}} & {\textbf{64 B/v}} & {\textbf{128 B/v}} \\
\midrule
\multirow{7}{*}{SBM 10K (undir.)}
              & \cellcolor{aaccolor} AAC-CG & \cellcolor{aaccolor} 90.6\,\% & \cellcolor{aaccolor} 94.8\,\% & \cellcolor{aaccolor} 97.3\,\% \\
              & \cellcolor{aaccolor} AAC-ST & \cellcolor{aaccolor} 88.8\,\% & \cellcolor{aaccolor} 94.0\,\% & \cellcolor{aaccolor} 96.8\,\% \\
              & FPS-ALT          & \bfseries 89.9\,\% & \bfseries 93.9\,\% & \bfseries 97.0\,\% \\
              & CDH              & 55.9\,\% & 78.6\,\% & 92.5\,\% \\
              & CDH+BPMX         & 62.8\,\% & 81.7\,\% & 93.2\,\% \\
              & CDH tuned        & 88.0\,\% & 93.3\,\% & 96.4\,\% \\
              & \quad at & $P{=}8$, round-robin & $P{=}16$, round-robin & $P{=}32$, round-robin \\
\midrule
\multirow{6}{*}{OSMnx Modena (26K, dir.)}
              & \cellcolor{aaccolor} AAC-CG & \cellcolor{aaccolor} 84.7\,\% & \cellcolor{aaccolor} 91.4\,\% & \cellcolor{aaccolor} 93.8\,\% \\
              & \cellcolor{aaccolor} AAC-ST & \cellcolor{aaccolor} 81.1\,\% & \cellcolor{aaccolor} 87.9\,\% & \cellcolor{aaccolor} 91.2\,\% \\
              & FPS-ALT          & 83.9\,\% & 89.0\,\% & 92.1\,\% \\
              & CDH              & 85.9\,\% & 91.4\,\% & \bfseries 94.3\,\% \\
              & CDH tuned        & \bfseries 89.1\,\% & \bfseries 92.2\,\% & \bfseries 94.3\,\% \\
              & \quad at & $P{=}16$, round-robin & $P{=}32$, round-robin & $P{=}64$, round-robin \\
\midrule
\multirow{6}{*}{DIMACS NY (264K, dir.)}
              & \cellcolor{aaccolor} AAC-CG & \cellcolor{aaccolor} 85.8\,\% & \cellcolor{aaccolor} 91.2\,\% & \cellcolor{aaccolor} 94.4\,\% \\
              & \cellcolor{aaccolor} AAC-ST & \cellcolor{aaccolor} 81.4\,\% & \cellcolor{aaccolor} 89.0\,\% & \cellcolor{aaccolor} 93.2\,\% \\
              & FPS-ALT          & \bfseries 85.5\,\% & 90.2\,\% & 93.6\,\% \\
              & CDH              & 70.4\,\% & 87.4\,\% & \bfseries 93.6\,\% \\
              & CDH tuned        & 85.3\,\% & \bfseries 90.7\,\% & \bfseries 93.6\,\% \\
              & \quad at & $P{=}8$, round-robin & $P{=}16$, round-robin & $P{=}64$, round-robin \\
\bottomrule
\end{tabular}
}
\par\smallskip
\begin{minipage}{\linewidth}
\footnotesize\raggedright
We compare compressed differential heuristics (CDH) at matched label memory, following \citet{goldenberg2011cdh,goldenberg2017cdh}. Each vertex stores $r$ of the $P_{\mathrm{cdh}}$ Avoid pivot distances in float32, without a side table. We choose the largest $r$ that fits $\Bmem$ B/v. A Dijkstra search from the goal bounds missing pivot distances. The search stops once it has seen each missing pivot at 4 states that store it. We count all states settled by this search as expansions. \par CDH+BPMX adds one-step bidirectional pathmax during A* expansion. Its upward step requires equal-weight reverse arcs, which our search assumes only for undirected graphs. We therefore use BPMX only on SBM; Modena and NY are stored as directed graphs. The CDH and CDH+BPMX rows use round-robin storage at $P_{\mathrm{cdh}}{=}64$. \emph{CDH tuned} selects the best seed-averaged result over pool size, storage rule, and variant (Table~\ref{tab:cdh-sweep}). The row below identifies that configuration. \par Entries are mean expansion reduction over Dijkstra (\%). A fixed LCC start vertex determines the deterministic rules; random storage averages 3 seeds. FPS-ALT and CDH share one query sample per graph. Bold marks the highest reduction among these methods. AAC-CG and AAC-ST use different test queries, so we exclude them from bolding. \par AAC-CG results come from Table~\ref{tab:continuous-greedy} at $32$ and $64$\,B/v and the same script at $128$\,B/v. AAC-ST results come from Table~\ref{tab:matched-hybrid} on SBM, validation-selected configurations in Table~\ref{tab:osmnx} on Modena, and Table~\ref{tab:valsplit}'s test column on NY. AAC-ST denotes the straight-through selector; AAC-CG uses continuous greedy. Section~\ref{sec:cdh-headtohead} discusses the results.
\end{minipage}
\end{table}

\subsection{Training-objective drift: a controlled ablation}
\label{sec:training-drift}

With the same FPS start, the first $K$ entries of a $K_0$-landmark pool are exactly the
$K$-landmark FPS baseline. Identity selection therefore recovers FPS-ALT at matched memory:
$K=m$ on undirected graphs and $K=m/2$ per direction on directed graphs. This equality holds
across the eight graphs tested. The architecture can therefore represent the baseline; training determines whether it recovers that selection.

Table~\ref{tab:training-drift} isolates the effect on SBM at $K_0=32,m=8$, seed $42$.
The default block-sparse start is $0.91$ points below FPS-ALT; $200$ epochs narrow the gap,
and $500$ produce a $0.51$-point lead. This single-seed improvement need not hold after we average across seeds.

\begin{table}[!htbp]
\centering
\small
\caption{\tabletag{Ablation} Initialization and training on SBM.}
\label{tab:training-drift}
\adjustbox{max width=\linewidth}{\begin{tabular}{l S[table-format=3.1] S[table-format=2.2, table-space-text-post={\,\%}]}
\toprule
\textbf{Configuration}
  & {\textbf{Mean expansions}} & {\textbf{Reduction}} \\
\midrule
FPS-ALT, $K{=}8$                                                & 540.6 & 89.86\,\% \\
\rowcolor{aaccolor} AAC-ST $K_0{=}32, m{=}8$, forced one-hot on first 8 (no training)  & 540.6 & 89.86\,\% \\
\rowcolor{aaccolor} AAC-ST $K_0{=}8, m{=}8$ trained 200 epochs (identity case)         & 540.6 & 89.86\,\% \\
\rowcolor{aaccolor} AAC-ST $K_0{=}32, m{=}8$ block-sparse initialization (0 epochs)    & 589.0 & 88.95\,\% \\
\rowcolor{aaccolor} AAC-ST $K_0{=}32, m{=}8$ trained 200 epochs                        & 555.9 & 89.57\,\% \\
\rowcolor{aaccolor} AAC-ST $K_0{=}32, m{=}8$ trained 500 epochs                        & 513.2 & 90.37\,\% \\
\bottomrule
\end{tabular}
}
\par\smallskip
\begin{minipage}{\linewidth}
\footnotesize\raggedright
Forced-first-$m$ diagnostic on SBM at matched memory (32\,B/v, $K_0{=}32$, $m{=}8$, seed 42). The first 8 landmarks of the $K_0{=}32$ FPS pool coincide with ALT's $K{=}8$ landmarks; forcing AAC-ST to select them recovers ALT exactly. The default block-sparse initialization starts on a different subset, $0.91$ points below ALT; on this seed training puts it $0.29$ points below ALT at $200$ epochs and $0.51$ points above ALT at $500$ epochs.
\end{minipage}
\end{table}

We repeat the comparison on SBM and BA at $32$ and $64$\,B/v, with five seeds
($42$--$46$) and budgets of $0,50,200,500,1000$ epochs. Each budget is an independent run
with its own temperature schedule. At $32$\,B/v, the seed-mean expansion reduction is lower than FPS-ALT's at every training budget.
At $64$\,B/v, it is slightly higher after $1{,}000$ epochs (Figure~\ref{fig:training-drift}).
Longer training does not consistently improve selection. The road replication on NY retains a $0.19$-point deficit after
$1000$ epochs (Table~\ref{tab:training-drift-road}).

Initialization on the first $m$ pool landmarks matches FPS-ALT at the reported precision
(Table~\ref{tab:training-drift-hp}), whereas a soft coverage penalty does not
(Appendix~\ref{app:coverage-aware}). All runs use zero weight decay.
These diagnostics identify a limitation of the straight-through procedure. They do not
isolate its estimator from sampling and parameterization; we examine those choices separately in Section~\ref{sec:continuous-greedy}.

\begin{figure}[t]
\centering
\includegraphics[width=\textwidth]{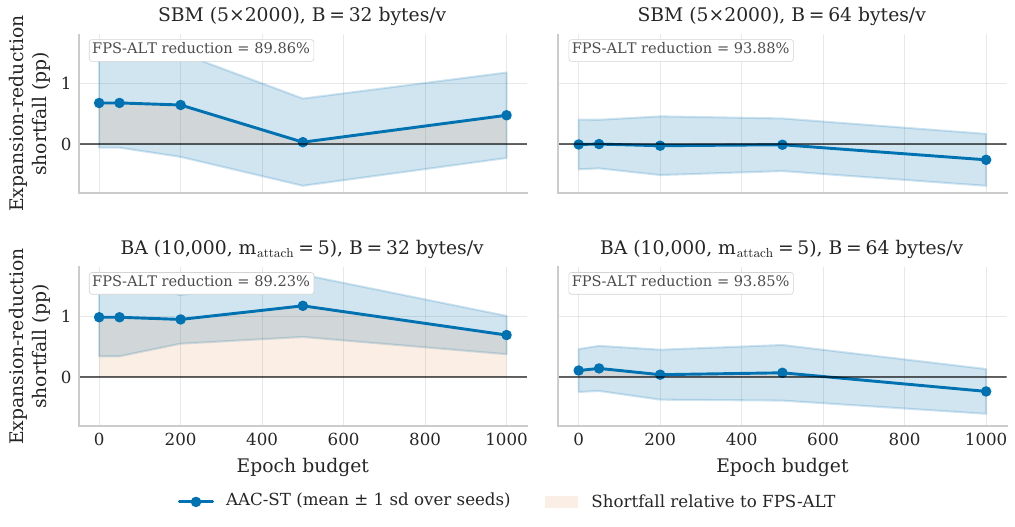}
\caption{\textbf{Training from the default initialization need not recover FPS-ALT.}
Zero denotes FPS-ALT with $K=m$, recovered by selecting the first $m$ pool landmarks.
Each panel states FPS-ALT's expansion reduction. The curves show FPS-ALT minus AAC-ST in percentage points; positive values mean AAC-ST expands more nodes.
At $32$\,B/v, the seed mean is positive at all ten combinations of graph and epoch budget.
Training for $1{,}000$ epochs narrows the gap but does not close it.
At $64$\,B/v, the mean stays within $0.14$ points of zero through $500$ epochs.
At $1{,}000$ epochs, it is below zero on both graphs, with four of five seeds below zero.
The compressor can represent FPS-ALT exactly, so these deficits arise from training.
Table~\ref{tab:training-drift-road} reports the road-graph study.}
\label{fig:training-drift}
\end{figure}

\subsection{Admissibility under early stopping and path optimality}
\label{sec:admissibility-verification}

The guarantees of Section~\ref{sec:method} hold in exact arithmetic. We check the float32 deployments numerically. On Modena at $K_0=64,m=16$, no tested query violates admissibility at epoch budgets
$\{1,5,10,50,200\}$ across five seeds (Table~\ref{tab:admissibility-early-stopping}).
Each budget is an independent run with its own temperature schedule. Expansion reduction rises
from $84.9$\,\% at one epoch to $87.1$\,\% at $200$. These numerical observations agree with Proposition~\ref{thm:admissibility}.

\paragraph{Path optimality.} All $125$ released record files that store optimality flags or
suboptimal-path counts report zero failures. The $6{,}000$ per-query DIMACS records also match
Dijkstra's path costs. In validation-split, OSMnx, and hybrid runs, we check costs at run time but do not retain per-query costs. Because A* reopens nodes, correctness needs admissibility and not floating-point
consistency.

\subsection{Compatibility case study: Warcraft \texorpdfstring{$12{\times}12$}{12 by 12} grids}
\label{sec:contextual-exp}

We evaluate AAC-ST and AAC-CG in a differentiable planning pipeline on Warcraft $12{\times}12$ grids~\citep{vlastelica2020differentiation} (Appendix~\ref{app:contextual}). The encoder learns while
the deployed heuristic stays admissible on the predicted-cost graph. AAC-CG trains through the
entropic oracle at $\tau=0.1$ and deploys with the exact oracle. At $K_0=16,m=8$,
the mean test-objective ratio to the FPS prefix is $F/F_{\mathrm{FPS}}=1.046$ for AAC-CG
and $0.989$ for AAC-ST over eight seeds (Table~\ref{tab:contextual-ablation}).

Every path query has the same corner-to-corner target, which is the first FPS landmark.
Retaining that landmark gives an exact heuristic in real arithmetic. Frozen FPS retains
that landmark and expands fewer nodes than AAC-CG in all five tested combinations of encoder,
pool size, and selected-label count. On optimal-cost accuracy, the ResNet
encoder reaches $85.3$\,\% with a frozen compressor and $87.0$\,\% with AAC-CG, below
the $96.0$\,\% reported by \citet[Table~2]{vlastelica2020differentiation} under a different training protocol.
The study establishes compatibility with the differentiable pipeline, but not general search-efficiency gains.

\section{Discussion}
\label{sec:discussion}

The results separate representational capacity from optimization. The row-stochastic class
contains the best pool subset, yet straight-through training often misses useful selections. Both
canonical greedy and AAC-CG improve on FPS-ALT under the primary accounting, on the test split of
each graph--budget pair. AAC-CG's main additional capability is its differentiable entropic-oracle
trajectory; the experiments do not establish an objective or search advantage over discrete
greedy. The accounting rule also matters: equal-weight reverse arcs let each stored landmark serve
both directions.

\subsection{Limitations}
\label{sec:limitations}
\begin{itemize}[leftmargin=1.2em,topsep=2pt,itemsep=2pt]
\item \emph{Scope.} We study static ALT-family selection and a contextual compatibility task,
not a replacement for specialized routing engines or an evaluation on dynamic graphs.
The query models and small seed counts do not cover production demand.
\item \emph{Preprocessing and latency.} Larger teacher pools cost more offline than FPS-ALT.
The estimated AAC-CG break-even point is $170$--$3{,}210$ queries at the observed seed-mean
savings in the shared-layout study; most per-seed latency intervals include equality (Table~\ref{tab:latency-layout}).
Expansion improvements alone do not establish a latency advantage.
\item \emph{Optimization and search.} The guarantee concerns expected heuristic value,
not expansions. At the deployed sample size, our evaluated population bound does not exceed the $1/2$ guarantee for greedy with access to the population objective
(Appendix~\ref{app:guarantee-constants}). In cardinality-constrained settings, greedy's finite-rank factor also exceeds
AAC-CG's at $T=100$. Coverage bounds do not limit the possible selection gain or rank selectors by search effort.
\item \emph{Conditioning and the oracle term.} The entropic layer is differentiable almost
everywhere, but its bound-matrix gradient norm reaches $1.29\times10^{28}$ at $\tau=10^{-4}$
(Table~\ref{tab:soft-cg}); contextual training uses $\tau=0.1$. On its undirected
$K_0{=}16$, $m{=}8$ instance, $H_{\max} = 16\log 2$ and $\tau H_{\max} \approx 1.1$, so the oracle
term of Proposition~\ref{prop:eps-lmo}(c) is informative only where $F(\mathrm{OPT})$ exceeds
about $1.1 f_{\max}$; we do not report that ratio for the contextual maps.
\item \emph{Contextual performance.} Frozen FPS expands fewer nodes than AAC-CG in all five tested
combinations of encoder, pool size, and selected-label count, on a fixed query whose target is the first FPS landmark
(Section~\ref{sec:contextual-exp}). Imitation performance trails the published blackbox baseline
under different training protocols (Appendix~\ref{app:contextual}). The experiment demonstrates compatibility with a differentiable pipeline but does not
establish competitive planning performance.
\end{itemize}

\subsection{Future directions}
\label{sec:future}
Richer candidate pools could provide bounds that no FPS-pool selector can recover.
Under instance-dependent costs, query demand may diverge from coverage.
It remains open whether combining AAC with classical admissible heuristics, such as CDH or
CH-Potentials, improves search at matched total memory. The pointwise maximum preserves
admissibility, but the hybrid experiments show that this alone does not justify its memory cost.

\section{Conclusion}
\label{sec:conclusion}

We formulated memory-constrained ALT selection through two continuous relaxations: row-stochastic
matrices and fractional mixtures of subset heuristics. Both preserve admissibility with exact
distances and arithmetic, and both deploy ordinary ALT on selected labels. The row-stochastic
relaxation contains a discrete optimum, but straight-through training need not find it.
For the fractional relaxation, continuous greedy gives a guarantee with explicit finite-step,
oracle-error, and finite-sample terms (Theorem~\ref{thm:aac-guarantee}). Its entropic oracle makes
the trajectory differentiable almost everywhere in the ambient bound-matrix space.
The limiting $1-1/e$ factor is optimal for arbitrary supplied pools and query distributions
unless P${=}$NP (Theorem~\ref{thm:hardness}).

Under the primary symmetry-aware accounting, AAC-CG and canonical greedy each expand fewer nodes
than FPS-ALT in all $19$ tested graph--budget pairs, on one $500$-query test split per pair.
AAC-CG's test-objective ratio to canonical greedy averages $0.996$ over these pairs (Table~\ref{tab:symmetric-accounting}).
These experiments do not establish an advantage over canonical greedy.
Under directed accounting, AAC-ST has lower test objectives than AAC-CG and expands more nodes than FPS-ALT in all $19$ pairs (Table~\ref{tab:continuous-greedy}).
Both relaxations train inside the contextual pipeline. Frozen FPS expands fewer nodes than AAC-CG
in all five tested configurations and fewer than AAC-ST in four (Appendix~\ref{sec:contextual-study}).
The fixed query target is an FPS landmark, so this task does not establish general search-efficiency gains.

\makeatletter
\if@accepted
  \section*{Acknowledgments}
  The authors acknowledge support from the VinUni Center for AI Research (CAIR).
\fi
\makeatother

\newpage
\appendix

\section{Optimizer comparison under directed accounting}
\label{app:directed-optimizer}

We compare optimizer families using two landmark--direction label tables per DIMACS landmark (Table~\ref{tab:continuous-greedy}).
Table~\ref{tab:symmetric-accounting} gives the primary comparison with shared tables.
We evaluated AAC-ST and SIMPLE only under directed accounting, so we omit them from the primary table.
Under directed accounting, SIMPLE expands fewer nodes than FPS-ALT in $2$ of $19$ graph--budget pairs.

\begin{table}[!htbp]
\centering
\footnotesize
\setlength{\tabcolsep}{3pt}
\caption{\tabletag{Main} Optimizer comparison under directed accounting.}
\label{tab:continuous-greedy}
\adjustbox{max width=\linewidth}{\begin{tabular}{l r l S[table-format=1.4] S[table-format=1.4] S[table-format=1.4] S[table-format=1.4] S[table-format=6.0] S[table-format=6.0] S[table-format=6.0] S[table-format=6.0]}
\toprule
& & & \multicolumn{4}{c}{\textbf{Objective} ($F_{\mathrm{test}}$ / Greedy)} & \multicolumn{4}{c}{\textbf{Mean A* expansions}} \\
\cmidrule(lr){4-7}\cmidrule(lr){8-11}
\textbf{Graph} & {\textbf{B/v}} & \textbf{Matroid}
  & {FPS-ALT} & {SIMPLE} & {AAC-ST} & {AAC-CG}
  & {FPS-ALT} & {Greedy} & {AAC-ST} & {AAC-CG} \\
\midrule
NY        &  32 & partition   & 0.9963 & 0.9095 & \cellcolor{aaccolor} 0.8660 & \cellcolor{aaccolor} 0.9940 & 18122 & \bfseries 16709 & \cellcolor{aaccolor} 29339 & \cellcolor{aaccolor} 18303 \\
NY        &  64 & partition   & 0.9881 & 0.9839 & \cellcolor{aaccolor} 0.9664 & \cellcolor{aaccolor} 0.9979 & 12498 & \bfseries 10889 & \cellcolor{aaccolor} 13767 & \cellcolor{aaccolor} \bfseries 11362 \\
BAY       &  32 & partition   & 0.9804 & 0.9349 & \cellcolor{aaccolor} 0.8615 & \cellcolor{aaccolor} 0.9905 & 23201 & \bfseries 21966 & \cellcolor{aaccolor} 43622 & \cellcolor{aaccolor} 23880 \\
BAY       &  64 & partition   & 0.9935 & 0.9883 & \cellcolor{aaccolor} 0.9656 & \cellcolor{aaccolor} \bfseries 1.0021 & 15549 & \bfseries 15133 & \cellcolor{aaccolor} 21405 & \cellcolor{aaccolor} \bfseries 14386 \\
BAY       & 128 & partition   & 0.9954 & 0.9943 & \cellcolor{aaccolor} 0.9888 & \cellcolor{aaccolor} 1.0000 & 11700 & \bfseries  9872 & \cellcolor{aaccolor} 13118 & \cellcolor{aaccolor} \bfseries 10008 \\
COL       &  32 & partition   & \bfseries 1.0000 & 0.9241 & \cellcolor{aaccolor} 0.9230 & \cellcolor{aaccolor} 1.0000 & 34443 & 34443 & \cellcolor{aaccolor} 55281 & \cellcolor{aaccolor} \bfseries 32753 \\
COL       &  64 & partition   & 0.9845 & 0.9749 & \cellcolor{aaccolor} 0.9643 & \cellcolor{aaccolor} \bfseries 1.0000 & 24753 & \bfseries 22283 & \cellcolor{aaccolor} 31205 & \cellcolor{aaccolor} \bfseries 22283 \\
COL       & 128 & partition   & 0.9921 & 0.9912 & \cellcolor{aaccolor} 0.9830 & \cellcolor{aaccolor} 0.9986 & 18420 & \bfseries 14898 & \cellcolor{aaccolor} 21119 & \cellcolor{aaccolor} \bfseries 15064 \\
FLA       &  32 & partition   & \bfseries 1.0000 & 0.9683 & \cellcolor{aaccolor} 0.9697 & \cellcolor{aaccolor} 0.9950 & 85054 & 85054 & \cellcolor{aaccolor} 119851 & \cellcolor{aaccolor} 101103 \\
FLA       &  64 & partition   & 0.9887 & 0.9888 & \cellcolor{aaccolor} 0.9737 & \cellcolor{aaccolor} 0.9981 & 66932 & \bfseries 54179 & \cellcolor{aaccolor} 85052 & \cellcolor{aaccolor} \bfseries 53706 \\
FLA       & 128 & partition   & 0.9931 & 0.9946 & \cellcolor{aaccolor} 0.9891 & \cellcolor{aaccolor} 0.9990 & 42819 & \bfseries 35150 & \cellcolor{aaccolor} 49321 & \cellcolor{aaccolor} \bfseries 36921 \\
Manhattan &  32 & partition   & 0.9732 & 0.9561 & \cellcolor{aaccolor} 0.9342 & \cellcolor{aaccolor} \bfseries 1.0009 &  1449 & \bfseries  1139 & \cellcolor{aaccolor}  1793 & \cellcolor{aaccolor} \bfseries  1295 \\
Manhattan &  64 & partition   & 0.9816 & 0.9786 & \cellcolor{aaccolor} 0.9643 & \cellcolor{aaccolor} 0.9975 &   949 & \bfseries   708 & \cellcolor{aaccolor}  1078 & \cellcolor{aaccolor} \bfseries   734 \\
Modena    &  32 & partition   & 0.9841 & 0.9346 & \cellcolor{aaccolor} 0.9317 & \cellcolor{aaccolor} 0.9781 &  2076 & \bfseries  1767 & \cellcolor{aaccolor}  2644 & \cellcolor{aaccolor} \bfseries  2002 \\
Modena    &  64 & partition   & 0.9809 & 0.9783 & \cellcolor{aaccolor} 0.9295 & \cellcolor{aaccolor} 0.9994 &  1371 & \bfseries  1110 & \cellcolor{aaccolor}  1917 & \cellcolor{aaccolor} \bfseries  1125 \\
SBM       &  32 & uniform     & 0.9899 & 0.9747 & \cellcolor{aaccolor} 0.9735 & \cellcolor{aaccolor} 0.9824 &   492 & \bfseries   468 & \cellcolor{aaccolor}   530 & \cellcolor{aaccolor} \bfseries   469 \\
SBM       &  64 & uniform     & 0.9717 & 0.9706 & \cellcolor{aaccolor} 0.9666 & \cellcolor{aaccolor} 0.9902 &   276 & \bfseries   254 & \cellcolor{aaccolor}   292 & \cellcolor{aaccolor} \bfseries   260 \\
BA        &  32 & uniform     & \bfseries 1.0026 & 0.9755 & \cellcolor{aaccolor} 0.9623 & \cellcolor{aaccolor} 0.9951 &   540 & \bfseries   488 & \cellcolor{aaccolor}   572 & \cellcolor{aaccolor} \bfseries   530 \\
BA        &  64 & uniform     & 0.9659 & 0.9535 & \cellcolor{aaccolor} 0.9523 & \cellcolor{aaccolor} 0.9859 &   298 & \bfseries   278 & \cellcolor{aaccolor}   312 & \cellcolor{aaccolor} \bfseries   287 \\
\midrule
\textbf{Mean} & & & 0.9875 & 0.9671 & \cellcolor{aaccolor} 0.9508 & \cellcolor{aaccolor} 0.9950 & \multicolumn{4}{c}{Wins: 0/19 AAC-ST, 17/19 Greedy, 16/19 AAC-CG} \\
\bottomrule
\end{tabular}
}
\par\smallskip
\begin{minipage}{\linewidth}
\footnotesize\raggedright
We store two direction-tagged tables per DIMACS landmark. Table~\ref{tab:symmetric-accounting} gives the primary comparison with one shared table. The query-adaptive selectors use the same objective (Eq.~\eqref{eq:selection-objective}), ground set, matroid, and $3{,}200$ training queries. FPS-ALT uses the query-independent FPS prefix. All variants use $500$ disjoint test queries and deploy ALT on their selections. \par The left block divides each method's mean test heuristic value by canonical greedy's value on the same queries (Proposition~\ref{prop:matroid-cg}(a)). Bold marks ratios at least $1$, with relative tolerance $10^{-9}$. The right block reports mean A* expansions; bold marks fewer expansions than FPS-ALT at matched memory. The footer counts graph--budget pairs with fewer expansions than FPS-ALT. \par AAC-ST is the straight-through selector of Section~\ref{sec:training}. AAC-CG uses continuous greedy with pipage rounding and $T{=}100$. They differ in continuous representation, sampling, and rounding, so the comparison does not isolate one design choice. AAC-ST and SIMPLE entries are means over 3 training seeds; the other variants are deterministic. Every AAC-CG iterate is admissible ($\max$ coefficient mass $=1.0$ across the $5{,}839$ trace rows for these graph--budget pairs).
\end{minipage}
\end{table}

\section{Extended proofs}
\label{app:proofs}

\subsection{Row-stochastic admissibility (Proposition~\ref{thm:admissibility})}
The proof uses closure of admissible bounds under convex combinations and pointwise maxima.
\label{app:admissibility-proof}

\begin{proof}
We prove both cases by showing that row-stochastic linear maps cannot amplify coordinate-wise differences.

\emph{Part (a): undirected case.}
Define $\delta_k := d(l_k, u) - d(l_k, t)$ for $k \in [K_0]$. From Eq.~\eqref{eq:compression}:
\[
y_i(u) - y_i(t) = \sum_{k=1}^{K_0} A_{i,k}\, \delta_k.
\]
Non-negativity and the unit row sum give
\[
\left|\sum_k A_{i,k}\delta_k\right|
\leq\sum_k A_{i,k}|\delta_k|
\leq\max_k|\delta_k|.
\]

Thus $|y_i(u) - y_i(t)| \leq \max_k |\delta_k|$ for every $i$.
Taking the maximum over $i$, we obtain
\[
\hA(u,t) = \max_i |y_i(u) - y_i(t)| \leq \max_k |d(l_k,u) - d(l_k,t)| = \hALT(u,t).
\]
The final inequality $\hALT(u,t) \leq d(u,t)$ follows from Theorem~\ref{thm:alt-admissibility}.

\emph{Part (b): directed case.}
For the backward bound, we define $\delta^{\mathrm{bwd}}_k := d_{\mathrm{in}}(k,u) - d_{\mathrm{in}}(k,t) = d(u,l_k) - d(t,l_k)$. Then:
\[
y^{\mathrm{bwd}}_i(u) - y^{\mathrm{bwd}}_i(t) = \sum_k A^{\mathrm{bwd}}_{i,k}\, \delta^{\mathrm{bwd}}_k \;\leq\; \max_k \delta^{\mathrm{bwd}}_k = \max_k\big(d(u,l_k) - d(t,l_k)\big),
\]
by the convex combination bound. For the forward bound, we define $\delta^{\mathrm{fwd}}_k := d_{\mathrm{out}}(k,t) - d_{\mathrm{out}}(k,u) = d(l_k,t) - d(l_k,u)$. Then:
\[
y^{\mathrm{fwd}}_i(t) - y^{\mathrm{fwd}}_i(u) = \sum_k A^{\mathrm{fwd}}_{i,k}\, \delta^{\mathrm{fwd}}_k \;\leq\; \max_k \delta^{\mathrm{fwd}}_k = \max_k\big(d(l_k,t) - d(l_k,u)\big).
\]
We take the maximum over $i$ for each direction, then combine both directions with the clamp at zero:
\[
\hA(u,t) \leq \max\!\Big(\max_k\big(d(u,l_k) - d(t,l_k)\big),\; \max_k\big(d(l_k,t) - d(l_k,u)\big),\; 0\Big) = \hALT(u,t).
\]
Admissibility follows from Theorem~\ref{thm:alt-admissibility}.

\emph{Parts (c)--(d).} A one-hot row is a probability vector and copies one teacher coordinate. Repeated rows do not change the maximum, so $\hA=h_{S(A)}$. Identity matrices retain all items and recover $\hALT$.
\end{proof}

\subsection{Maximal admissible linear classes (Theorem~\ref{thm:maximal-class})}
\label{app:maximal-class}

\begin{proof}

\emph{Sufficiency.} On undirected graphs, $|\delta_k| \leq d(u,t)$ by the triangle inequality applied in both
directions, so $|\sum_k a_k \delta_k| \leq \sum_k |a_k|\,|\delta_k| \leq \|a\|_1\, d(u,t)$.
For directed graphs, $\delta^{\mathrm{fwd}}_k \leq d(u,t)$ by Theorem~\ref{thm:alt-admissibility}.
Non-negative coefficients of total mass at most one therefore give
$\sum_k a_k\delta^{\mathrm{fwd}}_k\leq(\sum_k a_k)d(u,t)\leq d(u,t)$.

\emph{Necessity, undirected.} Take a query edge $u$--$t$ of length $\ell_0 > 0$ and attach landmark
$l_k$ as a leaf of length $\eta_k > 0$ off $t$ if $a_k \geq 0$ and off $u$ if $a_k < 0$. A leaf off $t$ has
$d(l_k,t) = \eta_k$ and $d(l_k,u) = \eta_k + \ell_0$, so $\delta_k = +\ell_0$ exactly. A leaf off $u$ gives
$\delta_k = -\ell_0$ exactly. Hence $\sum_k a_k \delta_k = \|a\|_1\, \ell_0$, and any $\|a\|_1 > 1$
exceeds $d(u,t) = \ell_0$ by exactly that factor.

\emph{Necessity, directed, mass.} Place every landmark upstream of $u$ (arc $l_k \to u$ of length
$\eta_k$), with the query arc $u \to t$ of length $\ell_0$ and return arcs $t \to l_k$ of length $R$.
The return arcs keep the digraph strongly connected, so all distances in the counterexample are finite. Then $d(l_k,t) = \eta_k + \ell_0$ and $d(l_k,u) = \eta_k$, so
$\delta^{\mathrm{fwd}}_k = \ell_0$ for every $k$ and
$\sum_k a_k \delta^{\mathrm{fwd}}_k = (\sum_k a_k)\, \ell_0$. A non-negative row of mass above one
exceeds $d(u,t) = \ell_0$ by a factor equal to its mass.

\emph{Necessity, directed, sign.} Suppose $a_j < 0$ for some $j$, and write
$a^+ := \sum_{a_k \geq 0} a_k \geq 0$ and $a^- := \sum_{a_k < 0} |a_k| > 0$. We build a digraph on $\{u, t\} \cup \{l_k\}$. We add the query arc $u \to t$ of length $\ell_0 > 0$ and the return arc $t \to u$ of length $R$. For each $k$, add a private pair of arcs of length $\eta_k > 0$ in both directions
between $l_k$ and its anchor. The anchor is $u$ if $a_k \geq 0$ and $t$ if $a_k < 0$. Each landmark is a distinct
vertex, and the digraph is strongly connected. The only arcs leaving $u$ are $u \to t$ and the arcs $u \to l_k$ for landmarks anchored at $u$.
These landmark vertices are dead ends, so $d(u,t) = \ell_0$. A landmark anchored at $u$ has $d(l_k,u) = \eta_k$ and
$d(l_k,t) = \eta_k + \ell_0$, so $\delta^{\mathrm{fwd}}_k = \ell_0$. A landmark anchored at $t$
has $d(l_k,t) = \eta_k$ and $d(l_k,u) = \eta_k + R$, so $\delta^{\mathrm{fwd}}_k = -R$. Hence,
for \emph{every} choice of the $\eta_k$,
\[
\textstyle\sum_k a_k\, \delta^{\mathrm{fwd}}_k \;=\; a^+\ell_0 \;+\; a^- R .
\]
Choosing $R > \ell_0/a^-$ makes this at least $a^- R > \ell_0 = d(u,t)$, so the row is inadmissible
on this instance. Such an $R$ exists for every vector with a negative entry.
For backward bounds, reverse every arc: $d^\top(x,y)=d(y,x)$.
Then $d(u,l_k)-d(t,l_k)$ is the forward bound of query $(t,u)$ in $G^\top$, and
$d^\top(t,u)=d(u,t)$, so the same characterization applies.
\end{proof}

\subsection{Zero relaxation gap (Proposition~\ref{prop:convex-extension})}
\label{app:convex-extension}

\begin{proof}

\emph{(a)} Each $\sum_k A_{i,k}\delta_k$ is affine in $A$, $|\cdot|$ is convex, and both $\max_i$ and
$\mathbb{E}_Q$ preserve convexity, so $F_{\mathrm{row}}$ is convex. All three sets are polytopes.
Every point is a convex combination of vertices, so convexity bounds its objective by the
same combination of vertex objectives \citep{rockafellar1970convex}. At least one vertex
therefore has an objective value as large as that of the point.

The vertices of $\{\|a\|_1 \leq 1\}$ are $\pm e_j$, of the sub-stochastic simplex $e_j$ and $0$, and of
the simplex $e_j$. At such a vertex every row is $\pm e_j$ or $0$. The absolute value removes the sign, and repeated or zero rows do not change the maximum. The vertex value is therefore exactly $F(S)$ for
the set $S$ of distinct landmarks selected by its rows, with $|S| \leq m$. Conversely, we can represent
every non-empty $S$ with $|S| \leq m$ by repeating one of its landmarks in the spare rows. The empty
set is not represented by any vertex of the $\ell_1$ ball or of the plain simplex, neither of which has $0$
as a vertex. But $F(\emptyset) = 0 \leq F(S)$ for every $S$, so omitting it cannot lower the maximum. Hence
$\max_A F_{\mathrm{row}} = \max_{|S| \leq m} F(S)$ on each of the three sets, and the second equality
is monotonicity of $F$.

\emph{(b)} The same argument applies with $\mathcal{I}$ as the ground set and the
clamp $\max(0,\cdot)$ included in $F$. The map is convex, and the vertices of the sub-stochastic
simplex over $\mathcal{I}$ are $e_j$ and $0$. Each matrix vertex selects at most $m$ distinct
items; padding this set gives a base in $\mathcal{B}_{\mathrm{free}}$ without lowering $F$.
Restricting row supports to the direction blocks gives a set within each block's budget.
Zero and repeated rows may leave unused capacity. Padding within each block gives a base in
$\mathcal{B}_{\mathrm{fix}}$, again without lowering $F$.
Row-stochastic matrices form a subset of this class and contain one-hot realizations of every
base, so restricting to them preserves both optimal values.

Direction-purity is a property of the vertices, not of the interior. Every non-zero row of a vertex is
a single $e_j$, and $j$ carries one direction. Admissibility holds at every point, not only at the
vertices, because a sub-stochastic mixture of lower bounds is a lower bound
(Theorem~\ref{thm:maximal-class}(b)).
\end{proof}

\subsection{A straight-through counterexample}
\label{app:straight-through-example}

The straight-through update can point opposite to $\nabla F_{\mathrm{row}}$.
In general, it is not the gradient of any continuous extension of $F$. We use noise-free argmax to isolate the estimator's bias.
Take the undirected graph on $\{u,t,l_1,l_2\}$ with $w(u,t) = 2$, $w(t,l_1) = 1$,
$w(u,l_2) = 1$ and $w(t,l_2) = 2$, so $d(u,t) = 2$ and $\delta = (2,-1)$. On the single row
$a = (p, 1-p)$ the soft objective is $F_{\mathrm{row}}(p) = |3p - 1|$, whose derivative at
$p = 0.4$ is $+3$. With noise-free argmax, the hard forward pass selects $l_2$. The straight-through backward pass
differentiates $|{-}1|$ through the soft Jacobian and returns $-3$. An ascent step therefore
\emph{decreases} $F_{\mathrm{row}}$, because the forward value and the backward direction come from
different functions. For two logits, $\nabla_W p=\tau_{\mathrm g}^{-1}p(1-p)(1,-1)$, so the sign reversal
also holds in logit space.

The failure is not specific to $F_{\mathrm{row}}$. That update is $-3$ for
$p < 1/2$ and $+3$ for $p > 1/2$, so it integrates to $0$ along the segment. Any function
agreeing with $F$ at the endpoints must gain $F(\{l_1\}) - F(\{l_2\}) = 1$ across it. No continuous extension of
$F$ has this update as its derivative away from the tie at $p=1/2$.

\subsection{Greedy and continuous-greedy guarantees (Proposition~\ref{prop:matroid-cg})}
\label{app:matroid-cg}

\begin{proof}

\emph{Matroid structure.} $\mathcal{M}$ is a partition matroid by definition. Its independent sets are
those meeting each block $\mathcal{I}_b$ in at most $m_b$ items, and this family is non-empty,
down-closed and satisfies exchange blockwise. A hard parameter setting selects one item per row, so the
set of distinct items it selects meets block $b$ in at most $m_b$ items and is independent.
Conversely, we can represent every base by rows that enumerate its items. A hard parameter
setting cannot represent an independent set that omits a block with positive row count,
including the empty set, because each row must select an item. Every other independent set
can be represented using duplicate rows. Padding an
independent set to a base adds items, which cannot lower the monotone $F$, and leaves the deployed
count at $m$.

\emph{(a) Greedy.} Let $\Gamma=\{g_1,\ldots,g_m\}$ be the greedy base in selection order,
$\Gamma_j=\{g_1,\ldots,g_j\}$, and $O$ an optimal base. Under a cardinality constraint,
submodularity bounds $F(O)-F(\Gamma_{j-1})$ by the sum of at most $m$ feasible marginals.
Greedy takes the largest, so the remaining gap contracts by at least a factor $1-1/m$
at each step. Starting from $F(\emptyset)=0$ gives
$F(\Gamma)\geq[1-(1-1/m)^m]F(O)\geq(1-1/e)F(O)$ \citep{nemhauser1978analysis}.
For the partition-matroid guarantee \citep{fisher1978analysis}, match items of $O$ to items of $\Gamma$
within each block, matching common items to themselves. If $o\in O\setminus \Gamma$ is matched
to $g_j$, then $o$ was feasible when $g_j$ was chosen. Thus
\[
\Delta(o\mid \Gamma)\leq\Delta(o\mid \Gamma_{j-1})
\leq\Delta(g_j\mid \Gamma_{j-1}).
\]
Summing and telescoping gives
$F(O)\leq F(\Gamma)+\sum_{o\in O\setminus \Gamma}\Delta(o\mid \Gamma)\leq2F(\Gamma)$.
The same argument applies when greedy fills one block at a time, because each comparison
uses an item in the block currently being filled.

\emph{Swap-local optimality.} Now let $\Gamma$ be any base with no improving feasible single swap.
Match $O\setminus \Gamma$ bijectively to $\Gamma\setminus O$ within each block, writing $g(o)$ for the match.
Submodularity and local optimality give
\[
\Delta(o\mid \Gamma)\leq\Delta(o\mid \Gamma\setminus\{g(o)\})
\leq F(\Gamma)-F(\Gamma\setminus\{g(o)\}).
\]
For any ordering $g_1,\ldots,g_m$ of $\Gamma$, diminishing returns gives
$F(\Gamma)-F(\Gamma\setminus\{g_j\})\leq F(\Gamma_j)-F(\Gamma_{j-1})$.
All increments $F(\Gamma_j)-F(\Gamma_{j-1})$ are non-negative, so the sum over matched items is at most $F(\Gamma)$.
The preceding bound on $F(O)$ again yields $F(O)\leq2F(\Gamma)$.
Example~\ref{ex:nine-vertex} attains ratio $(N+1)/(2N)$ for both algorithms, so neither
guarantee can exceed $1/2$ on this family.

Under a uniform matroid, greedy has a $1-1/e$ guarantee, but a swap-local optimum can have a lower approximation ratio.
We construct three disjoint sets $C_1, C_2, C_3$ of five elements each.
We also construct three pairwise disjoint sets $D_1, D_2, D_3$, with each $D_i$ containing one element from each $C_j$.
At $m = 3$, the $D_i$ cover $9$ elements, while the $C_j$ cover all $15$.
Exchanging any $D_i$ for any $C_j$ again covers $5 + 2 + 2 = 9$, so $\{D_1, D_2, D_3\}$ is a
$1$-swap local optimum at ratio $3/5 < 1 - 1/e$. Theorem~\ref{thm:hardness} embeds this instance as
an ALT instance, so landmark selection attains the ratio.

\emph{(b)} $\tilde F$ is multilinear, so $\partial^2 \tilde F/\partial x_i^2 = 0$, and for $i \neq j$
\[
\frac{\partial^2 \tilde F}{\partial x_i \partial x_j}(x)
= \mathbb{E}_{S \sim \mathrm{Bernoulli}(x|_{-ij})}\big[F(S \cup \{i,j\}) - F(S \cup \{i\}) - F(S \cup \{j\}) + F(S)\big] \;\leq\; 0,
\]
Each summand is non-positive by submodularity of $F$. A twice-differentiable function whose Hessian
has all entries $\leq 0$ is DR-submodular \citep{bian2017guaranteed}, and $\partial \tilde F/\partial x_i = \mathbb{E}[F(S \cup
\{i\}) - F(S \setminus \{i\})] \geq 0$ gives monotonicity.
The idealized continuous-greedy process maximizes $\langle v,\nabla\tilde F(x)\rangle$
over the down-closed polytope and attains $1-1/e$ \citep[Section~2.3]{calinescu2011maximizing}.
Their implementable algorithm (Section~3.1) instead uses sampled marginal-gain weights.
We compute these weights exactly and choose a base maximizing $\langle v,w(x)\rangle$, where
$w_i(x) := \mathbb{E}_{S \sim \mathrm{Bernoulli}(x)}[F(S \cup \{i\}) - F(S)] = (1 - x_i)\,
\partial \tilde F/\partial x_i(x)$. If $i \in S$, the difference $F(S \cup \{i\}) - F(S)$ is zero.
Otherwise, it depends only on the other selected items. This gives the second equality. These are the weights the
linear oracle of part~(b) maximizes. The finite-$T$ analysis of Proposition~\ref{prop:eps-lmo}(a)
uses these weights in a coupling argument and needs no bound on the Hessian. The closed
form is Lemma~\ref{lem:closed-form}.

\emph{(c)} $h_x = \mathbb{E}_{S \sim \mathrm{Bernoulli}(x)}[h_S]$ is a convex combination of the
heuristics $h_S$, each admissible by Theorem~\ref{thm:alt-admissibility}, so $0 \leq h_x \leq d$ pointwise. For the expanded
form, fix a query and sort the clamped bounds descending. For $B_{(j)}^+>0$, the event that rank $j$ is the first selected rank is $\{$rank $j$ selected, ranks $<j$ unselected$\}$, of probability
$c_{(j)} = x_{(j)}\prod_{k<j}(1-x_{(k)})$. These events are disjoint. Their union is the event that
some item with $B^+ > 0$ is selected, whose probability is $1 - \prod_{i : B_i^+ > 0}(1 - x_i) \leq 1$.
Summing gives both the expansion and the mass identity. The restriction to items with $B_i^+ > 0$ is
necessary. With a single item at $B_i = -1$ and $x = 1/2$ the true mixture is $0$,
while the unclamped sum is $-1/2$.
\end{proof}

\subsection{Hardness of pool selection (Theorem~\ref{thm:hardness})}
\label{app:hardness}

\begin{proof}

The vertex set splits into $\Lambda = \{\ell_j\}$ and $\Xi = \{a_e, b_e\}$, and every edge is either
the unit edge $a_e b_e$ inside $\Xi$ or a $\Lambda$--$\Xi$ edge of weight $w_0$ or $w_0+1$. There is no
$\Lambda$--$\Lambda$ edge. A walk from $\ell_j$ to $\Xi$ therefore takes one of two forms. The first
uses one $\Lambda$--$\Xi$ edge and then stays inside $\Xi$. There the only edges are the unit ones,
and any two distinct pairs $\{a_e,b_e\}$, $\{a_f,b_f\}$ are disconnected from each other. The second
re-enters $\Lambda$, which costs $w_0$ to leave, $w_0$ to return and $w_0$ to leave again, hence at
least $3w_0$. A detour inside the pair costs at least $w_0+1$, and a route through another
landmark costs at least $3w_0\geq w_0+1$. Neither improves a direct landmark edge. Likewise,
an alternative to $a_eb_e$ must visit a landmark and costs at least $2w_0>1$. Hence
\begin{align*}
d(\ell_j, a_e) &= w_0, &
d(\ell_j, b_e) &= w_0 + \mathbf{1}[e \in C_j], &
d(a_e,b_e)&=1.
\end{align*}
Therefore
$|d(\ell_j, a_e) - d(\ell_j, b_e)| = \mathbf{1}[e \in C_j]$ for all $j, e$, and since the graph is
undirected the ALT heuristic is $h_S(a_e, b_e) = \max_{\ell_j \in S} \mathbf{1}[e \in C_j] =
\mathbf{1}\big[e \in \bigcup_{\ell_j \in S} C_j\big]$ (with the convention $\max \emptyset = 0$, so
$F(\emptyset) = 0$). Averaging over the uniform $Q$ gives Eq.~\eqref{eq:hardness-identity}. The
identity is exact and holds for every subset, not only for maximizers. A $\rho$-approximation for
$\max_{|S| \leq m} F$ is therefore a $\rho$-approximation for maximum coverage on the original instance.
\citet{feige1998threshold} supplies the conclusion. The bounds
$|d(\ell_j,a_e) - d(\ell_j,b_e)| \in \{0,1\}$ are at most $d(a_e,b_e) = 1$, so every heuristic
value is admissible.

For the directed families, replace each edge by equal-weight reverse arcs and retain separate directional items, without sharing labels. Read the item bounds
of Eq.~\eqref{eq:item-bounds} on the same queries. The two items of $\ell_j$ have bounds with opposite signs. On the query $(a_e, b_e)$ the forward
copy contributes $d(\ell_j,b_e) - d(\ell_j,a_e) = \mathbf{1}[e \in C_j]$, and the backward copy
contributes its negative, which the clamp $\max(0,\cdot)$ in Eq.~\eqref{eq:selection-objective} sends to
$0$. A directed selection's value is therefore the covered fraction of the \emph{forward} landmarks it
selects; backward items contribute zero. For the restricted instances $m\leq K_0$, an optimum over $\mathcal{B}_{\mathrm{free}}$ can spend all $m$ items forward and is Max-$m$-Cover. This restriction suffices for hardness. Over $\mathcal{B}_{\mathrm{fix}}$ only the $m_{\mathrm{fwd}}$
forward slots contribute to the objective, giving Max-$m_{\mathrm{fwd}}$-Cover. Over $\mathcal{B}_{\mathrm{pair}}$ each
of the $m/2$ pairs contributes one useful forward copy, which is Max-$(m/2)$-Cover. All three are
inapproximable beyond $1 - 1/e$ by the same reduction.
\end{proof}

\subsection{Covering bounds and query quality}
\label{sec:theory-coverage}

Section~\ref{sec:theory-arch} bounds the compressor's heuristic by its teacher's heuristic.
The architecture also represents the FPS-ALT identity subset (Proposition~\ref{thm:admissibility}(d)).
We now compare ALT on a landmark set $S$ with the true distance.
For classical ALT, the gap is at most $2 r^{\mathrm{sym}}(S)$. If the two directions use different landmark sets, each uses its own radius (Theorem~\ref{thm:covering-radius}).
This bound does not imply that a smaller radius gives a smaller query gap (Figure~\ref{fig:toy-p7}).
Thus minimizing the gap-to-teacher loss need not minimize the covering radius.

AAC-CG selects pool items to increase the bounds on its training queries
(Figure~\ref{fig:landmarks}). On a directed graph, it may retain only one direction of a landmark.
Peripheral landmarks can give tight ALT bounds on long-range queries \citep{goldberg2005computing}.
Selecting such landmarks for a query distribution can increase the covering radius.
Section~\ref{sec:continuous-greedy} finds selections from the same pool that reduce expansions
relative to FPS-ALT.

On SBM, AAC-ST retains only $4$ of FPS-ALT's $16$ landmarks and places two to four in each
of the five communities. This different placement does not give a statistically established
advantage at matched memory
(Figure~\ref{fig:landmarks-sbm}, Section~\ref{sec:synthetic-sbm-ba}).

\paragraph{Covering radius bound.}
ALT on a subset $S \subseteq L$ cannot be tighter than on $L$. The covering radius
(Section~\ref{sec:alt-background}) bounds the loss.

For a nonempty landmark set $S$ on a directed graph, define
\[
V_{\mathrm{reach}}(S)=\{v:\exists l\in S,\ d(l,v),d(v,l)<\infty\},\qquad
d_{\mathrm{sym}}(v,S)=\min_{l\in S}\max(d(l,v),d(v,l)),
\]
and $r^{\mathrm{sym}}(S)=\max_{v\in V_{\mathrm{reach}}(S)}d_{\mathrm{sym}}(v,S)$.

\begin{theorem}[Covering radius bound]
\label{thm:covering-radius}
\textbf{(a) Undirected.} Let $G$ be undirected, so that $d$ is symmetric, let $S \subseteq L$ be an
$m$-element landmark subset with covering radius
$r(S) := \max_{v \in V}\, \min_{l \in S}\, d(v, l) < \infty$, and let $h_S$ denote ALT on
$S$ (Eq.~\eqref{eq:hS}, items identified with landmarks). Then for all $u, t \in V$ with $d(u,t) < \infty$:
\begin{equation}
\label{eq:covering-bound}
d(u,t) - h_S(u,t) \;\leq\; 2\,r(S).
\end{equation}
\textbf{(b) Directed.} Let $h$ use forward landmarks $L_{\mathrm{fwd}}$ and backward landmarks
$L_{\mathrm{bwd}}$. For $u\in V_{\mathrm{reach}}(L_{\mathrm{fwd}})$ and
$t\in V_{\mathrm{reach}}(L_{\mathrm{bwd}})$ with $d(u,t)<\infty$,
\begin{equation}
\label{eq:covering-bound-directed}
d(u,t) - h(u,t) \;\leq\; 2 \min\!\big(d_{\mathrm{sym}}(u, L_{\mathrm{fwd}}),\,
d_{\mathrm{sym}}(t, L_{\mathrm{bwd}})\big) \;\leq\; 2 \min\!\big(r^{\mathrm{sym}}(L_{\mathrm{fwd}}),\,
r^{\mathrm{sym}}(L_{\mathrm{bwd}})\big).
\end{equation}

\textbf{(c) The symmetrized radius is not a function of the directional ones.} Let
$r^{\mathrm{fwd}}(S) := \max_{v \in V'} \min_{l \in S} d(l,v)$ and $r^{\mathrm{bwd}}(S) :=
\max_{v \in V'} \min_{l \in S} d(v,l)$, with $r^{\mathrm{sym}}(S)$ taken over the same vertex set
$V'$, on which all three are finite. Then
$\max(r^{\mathrm{fwd}}(S),r^{\mathrm{bwd}}(S))\leq r^{\mathrm{sym}}(S)$, with no bound on the
right-hand side in terms of the two directional radii alone.
\end{theorem}
For~(c), the cycle $a\to b\to c\to a$ with arc weights $1,1,\omega>0$ and landmarks
$\{a,c\}$ has $r^{\mathrm{fwd}}=r^{\mathrm{bwd}}=1$ but $r^{\mathrm{sym}}=\omega+1$.
Parts~(a) and~(b) apply to both AAC-ST and AAC-CG deployments.

\begin{proof}[Proof sketch]
Take the landmark $l^*$ nearest $u$ and expand its forward bound. The triangle inequality along
$u \to l^* \to t$ leaves $d(l^*,u) + d(u,l^*)$, which equals $2\,d(u,l^*)\leq2r(S)$ when $d$ is
symmetric and is at most $2\,d_{\mathrm{sym}}(u,l^*)$ otherwise. That last step is the only place directedness
enters, and it is why~(b) bounds each direction against its own landmark set. For~(c),
$\min_l \max(d(l,v), d(v,l))$ dominates both one-directional minima at each $v$. On the witness
$b$ is covered forward and backward by two \emph{different} landmarks at cost $1$, and reaching
either one in the other direction traverses the closing arc.
Appendix~\ref{app:covering-radius-proof} gives the proof and the two-vertex witness ruling out the
union form.
\end{proof}

Outside the respective $V_{\mathrm{reach}}$, part~(b) gives no bound for that direction. The
deployed heuristic handles an unreachable landmark as follows (Appendix~\ref{app:directed}). If a
forward landmark reaches $u$ but not $t$, or $t$ reaches a backward landmark that $u$ does not, then
$t$ is unreachable from $u$, and the heuristic returns $+\infty$, which is exact. Every other
landmark with an infinite distance at either end drops out of the maximum, and the heuristic falls
back to the remaining bounds, or to the clamp value $0$ when none remains. The
query-dependent bound applies to these separate landmark sets. Classical ALT is the case
$L_{\mathrm{fwd}} = L_{\mathrm{bwd}} = S$, where the right-hand side is $2\,r^{\mathrm{sym}}(S)$.
Bounding each direction against its own set is necessary: the stronger inequality in the radius
of the union $L_{\mathrm{fwd}} \cup L_{\mathrm{bwd}}$ is false, and the directed proof gives a
two-vertex witness.

In~(a) the restriction to finite $d(u,t)$ is likewise needed. A finite covering radius does
not imply a connected graph. For an unreachable query, both $d(u,t)$ and the heuristic may equal $+\infty$, so their difference is undefined.

\begin{corollary}[Scope of Gonzalez's guarantee]
\label{cor:fps-covering}
\begin{enumerate}
\item[(a)] Let $S$ be selected by farthest-point sampling \citep{gonzalez1985clustering} in a metric.
For an undirected graph, we use $d$. For a directed graph, we use $d_{\mathrm{sym}}(x,y) := \max(d(x,y), d(y,x))$ on a strongly connected component.
Let $r_m^*$ be the optimal $m$-center radius in that metric. Then $r(S) \leq 2\,r_m^*$, for any first landmark.
Our pipeline takes the vertex farthest from an arbitrary start and does not keep the start.
Consequently, classical ALT $h$ on $S$ satisfies $d(u,t) - h(u,t) \leq 4\,r_m^*$.
In the directed case, $r = r^{\mathrm{sym}}$ and $L_{\mathrm{fwd}} = L_{\mathrm{bwd}} = S$.
\item[(b)] Part~(a) is a statement about the set FPS returns. A subset $S' \subseteq L$ of an FPS pool
$L$ chosen by another rule satisfies $r(S') \geq r(L)$, but inherits no Gonzalez $2$-approximation guarantee. The corollary therefore supplies no such radius or query-gap guarantee for a learned selection. It does not bound the possible improvement over
FPS on the objective of Eq.~\eqref{eq:selection-objective}, which the covering radius bounds but does
not order (Figure~\ref{fig:toy-p7}).
\item[(c)] Forward-only FPS (selection under $d(l, \cdot)$ alone) carries no such guarantee on
$r^{\mathrm{sym}}$. On the strongly connected digraph over $\{v_1, v_2, v_3, v_4\}$ with arcs
$v_1 \!\to\! v_2$ at cost $3$, $v_2 \!\to\! v_3$ at $2$, $v_3 \!\to\! v_4$ at $2$, $v_4 \!\to\! v_1$
at $8$, $v_3 \!\to\! v_1$ at $3$ and $v_1 \!\to\! v_4$ at $2$, forward-only FPS started at $v_1$ returns $\{v_3, v_2\}$. Its symmetrized radius is $11$, against an optimum of $5$, giving a ratio of $11/5 > 2$.
FPS under $d_{\mathrm{sym}}$ from the same start returns $\{v_4, v_3\}$, with radius $6 \leq 2 \cdot 5$, as part~(a) guarantees.
\end{enumerate}
\end{corollary}
Appendix~\ref{app:fps-covering-proof} proves the claim by deriving the twelve directed distances and
enumerating the six two-subsets. The pipeline of Section~\ref{sec:pipeline} selects under
$d_{\mathrm{sym}}$, so part~(a) applies to it. Approximating the one-directional radius is a
different problem: asymmetric $k$-center. It admits an $O(\log^* n)$ approximation
\citep{panigrahy1998asymmetric}; its approximation hardness is studied by
\citet{chuzhoy2005asymmetric}. Part~(c) neither uses nor implies that hardness.

\paragraph{Toy example: the path $P_7$.} Take $G = P_7$ (vertices $\{0,\dots,6\}$, unit edges), $m{=}2$, and $Q$ uniform on $\{1,\dots,5\}^2 \setminus \{s{=}t\}$ (20 ordered pairs). On a path, $|d(s,l) - d(l,t)| = d(s,t)$ whenever $l$ lies outside $[\min(s,t), \max(s,t)]$. The subset $S_{\mathrm{gap}}{=}\{0,6\}$ therefore puts a landmark outside every interior query interval. Its ALT bound is exact on all 20 queries. A symmetric $k$-center optimum $S_{\mathrm{cov}}{=}\{2,4\}$ has the smaller radius ($r(S_{\mathrm{cov}})=2$ against $r(S_{\mathrm{gap}})=3$). The pairs $(1,5)$ and $(5,1)$ span both selected landmarks. For $S_{\mathrm{cov}}$, the mean gap is $\mathbb{E}_Q[d - h_S] = 1/5$ (Figure~\ref{fig:toy-p7}).

Of all $\binom{7}{2}{=}21$ two-subsets, eight attain the minimum radius $r_2^*=2$.
Seven of these have zero gap; $S_{\mathrm{cov}}$ is the exception.
The subset $\{0,1\}$ has the largest radius, $5$, but still answers all 20 queries exactly.
This one-dimensional example disproves a general ordering of query gaps by covering radius.

We also test a penalty based on the mean soft nearest-landmark distance in a vertex batch (Appendix~\ref{app:coverage-aware}).
This $k$-median-style surrogate averages distances, whereas the covering radius takes a maximum.
The penalty improves mean expansion reduction on none of the five tested graphs.
No mean shift exceeds the standard deviation of its per-seed shifts.

\subsection{Proof of Theorem~\ref{thm:covering-radius}}
\label{app:covering-radius-proof}

We prove the directed case. The undirected one is its specialization at $d_{\mathrm{sym}} = d$ and
$L_{\mathrm{fwd}} = L_{\mathrm{bwd}} = S$, where the right-hand side becomes $2r(S)$.

We use $r^{\mathrm{sym}}$ and $V_{\mathrm{reach}}$ as defined before Theorem~\ref{thm:covering-radius}.
Each direction uses its own landmark set. The radius bound requires $u\in V_{\mathrm{reach}}(L_{\mathrm{fwd}})$ and
$t\in V_{\mathrm{reach}}(L_{\mathrm{bwd}})$. Without these premises it can fail. Take a single
landmark with no arcs, so that $V_{\mathrm{reach}}$ is that landmark alone and both radii are $0$. Take
also a query of positive length elsewhere in the graph. The landmark distances are infinite at both
endpoints, so the heuristic returns $0$ (Appendix~\ref{app:directed}). The gap is positive
despite the zero radii. Here
$V_{\mathrm{reach}}(S)$ is never empty for non-empty $S$, since $d(l,l) = 0$ in both directions. The failure occurs because the query endpoints lie outside this reachable set.
For directed graphs, we draw the queries and pool from one strongly connected component, so the premises hold (Section~\ref{sec:background}).
For undirected graphs, we use one connected component.

\begin{proof}
Fix such a $u$ and $t$ with $d(u,t) < \infty$, and bound the two gaps separately, each against the
radius of \emph{its own} landmark set.

\emph{Forward.} Choose $l^* \in L_{\mathrm{fwd}}$ minimizing $\max(d(l^*,u), d(u,l^*))$, so both are
at most $r^{\mathrm{sym}}(L_{\mathrm{fwd}})$ and in particular finite. The directed triangle
inequality along $u \to l^* \to t$ gives $d(l^*,t) \geq d(u,t) - d(u,l^*)$, so
\[
d(u,t) - \big[d(l^*,t) - d(l^*,u)\big] \;\leq\; d(l^*,u) + d(u,l^*) \;\leq\; 2\,r^{\mathrm{sym}}(L_{\mathrm{fwd}}).
\]

\emph{Backward.} Choose $l^\dagger \in L_{\mathrm{bwd}}$ minimizing $\max(d(l^\dagger,t), d(t,l^\dagger))$.
From $d(u,t) \leq d(u,l^\dagger) + d(l^\dagger,t)$,
\[
d(u,t) - \big[d(u,l^\dagger) - d(t,l^\dagger)\big] \;\leq\; d(l^\dagger,t) + d(t,l^\dagger) \;\leq\; 2\,r^{\mathrm{sym}}(L_{\mathrm{bwd}}).
\]

The deployed heuristic dominates both selected terms, so its gap is at most the smaller of their two bounds. Before replacing the landmark distances by global radii, the same argument gives the sharper query-dependent form
$d(u,t) - h(u,t) \leq 2 \min(d_{\mathrm{sym}}(u, L_{\mathrm{fwd}}), d_{\mathrm{sym}}(t, L_{\mathrm{bwd}}))$,
as stated in the theorem. This inequality extends to any
finite-distance query if $d_{\mathrm{sym}}(v,S)=+\infty$ outside $V_{\mathrm{reach}}(S)$:
each finite direction supplies its own bound, and an infinite right-hand side is vacuous.

Each direction requires its own landmark set. Drawing both landmarks
from the union $S = L_{\mathrm{fwd}} \cup L_{\mathrm{bwd}}$ and concluding a bound in
$r^{\mathrm{sym}}(S)$ would be strictly stronger, since a larger set has a smaller radius over the
same $V_{\mathrm{reach}}$, and that stronger statement is false. Take $V = \{u,t\}$ with arcs $u \to t$ and $t \to u$ of
weight $1$ and $L_{\mathrm{fwd}} = \{t\}$, $L_{\mathrm{bwd}} = \{u\}$. Here
$r^{\mathrm{sym}}(S) = 0$ while both selected bounds are $-1$, so $h(u,t) = 0$ and the gap is
$1 > 0$. The theorem's own right-hand side, $2\min(r^{\mathrm{sym}}(\{t\}), r^{\mathrm{sym}}(\{u\})) = 2$,
holds.

\emph{Part (c).} That $\max(r^{\mathrm{fwd}}, r^{\mathrm{bwd}}) \leq r^{\mathrm{sym}}$ holds at each
$v$, because $\min_l \max(d(l,v), d(v,l))$ dominates each one-directional minimum there. For the
gap, take the three-vertex cycle $a \to b \to c \to a$ with $w(a,b) = w(b,c) = 1$ and
$w(c,a) = \omega > 0$. Its distances are $d(a,b) = 1$, $d(b,c) = 1$, $d(c,a) = \omega$,
$d(a,c) = 2$, $d(b,a) = 1 + \omega$ and $d(c,b) = \omega + 1$. Over $V' = \{a,b,c\}$ with
$S = \{a,c\}$, vertices $a$ and $c$ are landmarks and contribute $0$ to all three radii, so each is
decided at $b$. There $\min(d(a,b), d(c,b)) = \min(1, \omega+1) = 1$ gives
$r^{\mathrm{fwd}} = 1$, and $\min(d(b,a), d(b,c)) = \min(1+\omega, 1) = 1$ gives
$r^{\mathrm{bwd}} = 1$, both for every $\omega > 0$. But $b$ is covered forward and backward by two
\emph{different} landmarks, and reaching either one in the other direction traverses the closing
arc: $\max(d(a,b), d(b,a)) = \max(d(c,b), d(b,c)) = \omega + 1$, so $r^{\mathrm{sym}} = \omega + 1$.
The two directional radii are therefore fixed while $r^{\mathrm{sym}}$ grows without bound. \qedhere
\end{proof}

\subsection{Proof of Corollary~\ref{cor:fps-covering}}
\label{app:fps-covering-proof}

\begin{proof}
\emph{(a)} Gonzalez's argument needs only that the selection runs in a metric, and it allows an
arbitrary first center. Our FPS follows the farthest selection of \citet{goldberg2005computing}. It
takes as its first center $c_1$ the vertex farthest from a start vertex that it does not keep, and then
adds at each step the point farthest from the current set. Let $\delta_j$ be the distance from $c_j$
to $\{c_1, \dots, c_{j-1}\}$ when $c_j$ is added, for $j \geq 2$. Distances to the set only shrink as it
grows, and each $c_j$ maximizes the current one, so $\delta_2 \geq \cdots \geq \delta_m \geq r(S)$,
where $r(S)$ is the distance the next step would realize. Two centers $c_i, c_j$ with $i < j$ are
therefore at least $\delta_j \geq r(S)$ apart. Hence if some vertex $v$
were farther than $2 r_m^*$ from all $m$ chosen centers, the $m+1$ points $S \cup \{v\}$ would be
pairwise more than $2 r_m^*$ apart. By the triangle inequality no single optimal center then serves
two of them within $r_m^*$, so $m$ optimal centers cannot cover $m+1$ such points, a contradiction. On an
undirected graph $d$ is a metric. On a strongly connected component of a directed graph
$d_{\mathrm{sym}}(x,y) = \max(d(x,y), d(y,x))$ is one. It is finite, symmetric by construction, zero
only on the diagonal, and $d_{\mathrm{sym}}(x,z) = \max(d(x,z), d(z,x)) \leq \max(d(x,y) + d(y,z),\,
d(z,y) + d(y,x)) \leq d_{\mathrm{sym}}(x,y) + d_{\mathrm{sym}}(y,z)$. So FPS in $d_{\mathrm{sym}}$
gives $r^{\mathrm{sym}}(S) \leq 2 r_m^*$ for the optimal $m$-center radius in that metric. Substitute
into Theorem~\ref{thm:covering-radius} with $L_{\mathrm{fwd}} = L_{\mathrm{bwd}} = S$:
\[
d(u,t) - h(u,t) \leq 2\,r^{\mathrm{sym}}(S) \leq 4\,r_m^*.
\]
\emph{(b)} The factor $2$ is a property of the greedy sequence FPS builds. The guarantee does not transfer to a subset $S' \subseteq L$ chosen by another rule.
We have $r(S') \geq r(L)$, but this inequality gives no upper bound on $r(S')$.
For example, let the FPS pool contain all vertices of the unit path $P_5$. The two-subset $\{0,1\}$ has radius $3$, whereas $r_2^*=1$, so it violates the factor $2$.

FPS also has no positive approximation factor for $F$, at any budget $m\geq1$. Take a star with center $o$, leaves $a_1,\ldots,a_m$ at edge length $2$, and leaves $b,c$ at edge length $1$. Starting from $o$, FPS selects all $a_i$ first; its $(m+2)$-landmark pool contains all leaves. For $Q$ concentrated on $(b,c)$, each selected bound is $|3-3|=0$, while any subset containing $b$ gives $d(b,c)=2$. Thus the FPS prefix has ratio zero against the best $m$-subset of its own pool.

\emph{(c)} The four-vertex digraph of the statement is strongly connected, with directed distances
\[
\begin{array}{l@{\quad}l@{\quad}l}
d(v_1,v_2) = 3, & d(v_1,v_3) = 5, & d(v_1,v_4) = 2,\\
d(v_2,v_1) = 5, & d(v_2,v_3) = 2, & d(v_2,v_4) = 4,\\
d(v_3,v_1) = 3, & d(v_3,v_2) = 6, & d(v_3,v_4) = 2,\\
d(v_4,v_1) = 8, & d(v_4,v_2) = 11, & d(v_4,v_3) = 13,
\end{array}
\]
each read off the six arcs. The symmetrized distances are therefore
$d_{\mathrm{sym}}(v_1,v_2) = 5$, $d_{\mathrm{sym}}(v_1,v_3) = 5$, $d_{\mathrm{sym}}(v_1,v_4) = 8$,
$d_{\mathrm{sym}}(v_2,v_3) = 6$, $d_{\mathrm{sym}}(v_2,v_4) = 11$ and $d_{\mathrm{sym}}(v_3,v_4) = 13$.
Forward-only FPS started at $v_1$ sees $d(v_1,\cdot) = (3, 5, 2)$ on $(v_2, v_3, v_4)$, so its first
landmark is $v_3$, and $v_1$ is not kept. It then sees $d(v_3,\cdot) = (3, 6, 2)$ on
$(v_1, v_2, v_4)$ and adds the farthest, $v_2$. The resulting $\{v_3, v_2\}$ leaves $v_4$ at
$\min(d_{\mathrm{sym}}(v_4,v_3), d_{\mathrm{sym}}(v_4,v_2)) = \min(13, 11) = 11$, so
$r^{\mathrm{sym}} = 11$.
Enumerating the six two-subsets gives radii $8, 8, 5, 11, 6, 6$ for $\{v_1,v_2\}$, $\{v_1,v_3\}$,
$\{v_1,v_4\}$, $\{v_2,v_3\}$, $\{v_2,v_4\}$, $\{v_3,v_4\}$. The optimum is $5$ and the ratio is
$11/5 > 2$. FPS run under $d_{\mathrm{sym}}$ from the same start sees $d_{\mathrm{sym}}(v_1,\cdot) =
(5, 5, 8)$ on $(v_2, v_3, v_4)$ and takes $v_4$ first. It then sees
$d_{\mathrm{sym}}(v_4,\cdot) = (8, 11, 13)$ on $(v_1, v_2, v_3)$ and adds $v_3$. It returns
$\{v_3, v_4\}$ at radius $6\leq2\cdot5$, as required. \qedhere

\end{proof}

\subsection{Consistency of the row-stochastic and fractional heuristics}
\label{app:consistency}

\paragraph{Admissible rows give consistent terms.} Let $w$ be the edge weights and $y(v) := \sum_k a_k\, d(l_k, v)$.
\emph{(a)} If $G$ is undirected and $\|a\|_1 \leq 1$, then $|y(u) - y(v)| \leq w(u,v)$ on every
edge and, for every target $t$, the term $h(v) := |y(v) - y(t)|$ is consistent:
$h(u) \leq w(u,v) + h(v)$.
\emph{(b)} If $G$ is directed and $a$ is sub-stochastic, then the forward term
$h(v) := y(t) - y(v)$ and the backward term $h(v) := \sum_k a_k\,(d(v,l_k) - d(t,l_k))$ are both
consistent.
In either case the clamp at $0$ and the maximum over rows preserve consistency.

\begin{proof}
\emph{(a)} Each teacher label is $1$-Lipschitz, $|d(l_k,u) - d(l_k,v)| \leq d(u,v) \leq w(u,v)$, and
a signed combination with $\|a\|_1 \leq 1$ preserves the bound, which is the sufficiency argument of
Theorem~\ref{thm:maximal-class}(a) with $\delta_k := d(l_k,u) - d(l_k,v)$. Then $|y(u) - y(t)| \leq
|y(u) - y(v)| + |y(v) - y(t)|$.
\emph{(b)} Consistency of the forward term is $y(v) - y(u) \leq w(u,v)$, and
$d(l_k,v) \leq d(l_k,u) + d(u,v)$ with $a \geq 0$, $\sum_k a_k \leq 1$ gives
$y(v) - y(u) \leq d(u,v) \leq w(u,v)$. For the backward term it is
$\sum_k a_k\, d(u,l_k) \leq w(u,v) + \sum_k a_k\, d(v,l_k)$, which follows termwise from
$d(u,l_k) \leq w(u,v) + d(v,l_k)$.
Finally the constant $0$ is consistent and a maximum of consistent terms is consistent.
\end{proof}

The fractional heuristic $h_x = \mathbb{E}_{S \sim \mathrm{Bernoulli}(x)}[h_S]$ is also consistent.
It is a convex combination of consistent heuristics $h_S$.
Such combinations preserve both the arc inequality and $h(t,t) = 0$.

Both parts assume finite labels and exact arithmetic. Appendix~\ref{app:directed}(b) extends the
consistency of the deployed heuristic to landmarks that cannot reach an endpoint. It also records
the assumptions under which the float32 allowance preserves admissibility, and why consistency can
fail under it.

\subsection{The nine-vertex witness and the closed-form multilinear extension}

\begin{example}[Nine-vertex witness for greedy's $1/2$]
\label{ex:nine-vertex}
Take vertices $a, b, c, s, t, p, q, z, H$ and arcs $a{\to}s$, $s{\to}t$, $t{\to}c$, $b{\to}p$, $p{\to}q$,
$s{\to}z$ of length $1$, $p{\to}c$ and $q{\to}c$ of length $2$, and arcs $H{\to}v$, $v{\to}H$ of
length $10$ for every $v \neq H$, making the digraph strongly connected. Use pool $\{a, b, c\}$,
budget $m = 2$, and an integer $N\geq1$. Let $Q$ put mass $N/(2N+1)$ on each of the queries $(s,t)$ and $(p,q)$ and
$1/(2N+1)$ on $(s,z)$, all of length $1$. In that order the clamped bounds of the six items are
$a_{\mathrm{fwd}} = (1,0,1)$, $b_{\mathrm{fwd}} = (0,1,0)$, $c_{\mathrm{bwd}} = (1,0,0)$, and $0$ for
$a_{\mathrm{bwd}}$, $b_{\mathrm{bwd}}$, $c_{\mathrm{fwd}}$ (the unclamped $c_{\mathrm{bwd}}$ is $-18$
on $(s,z)$). Hence $\mathrm{OPT}_{\mathrm{pair}} = F(\{a_{\mathrm{fwd}}, a_{\mathrm{bwd}}\}) =
\tfrac{N+1}{2N+1}$, $\mathrm{OPT}_{\mathrm{fix}} = F(\{b_{\mathrm{fwd}}, c_{\mathrm{bwd}}\}) =
\tfrac{2N}{2N+1}$ and $\mathrm{OPT}_{\mathrm{free}} = F(\{a_{\mathrm{fwd}}, b_{\mathrm{fwd}}\}) = 1$.
Canonical greedy under the fixed split takes $a_{\mathrm{fwd}}$ first, after which all three
backward items have marginal gain $0$; breaking that tie toward $a_{\mathrm{bwd}}$ it returns
$\{a_{\mathrm{fwd}}, a_{\mathrm{bwd}}\}$ at ratio $\tfrac{N+1}{2N}$, a $1$-swap local optimum from
which reaching $\mathrm{OPT}_{\mathrm{fix}}$ needs both items exchanged. Under the
free split greedy takes $b_{\mathrm{fwd}}$ second and returns the optimum.
\end{example}

The only distances the bounds use are
\begin{center}
\begin{tabular}{lllll}
\toprule
$d(a,s) = 1$ & $d(a,t) = 2$ & $d(a,z) = 2$ & $d(b,p) = 1$ & $d(b,q) = 2$ \\
$d(s,c) = 2$ & $d(t,c) = 1$ & $d(p,c) = 2$ & $d(q,c) = 2$ & other required pairs: $20$ \\
\bottomrule
\end{tabular}
\end{center}
because the only arcs between the three blocks $\{a,s,t,z\}$, $\{b,p,q\}$ and $\{c\}$ are
$t \to c$, $p \to c$, $q \to c$ and the hub arcs. Each remaining landmark--endpoint pair used in these bounds must pass
through $H$ at cost $10 + 10$. Reading off $B_{(l,\mathrm{fwd})} = d(l,t) - d(l,u)$ and
$B_{(l,\mathrm{bwd})} = d(u,l) - d(t,l)$ on the three queries gives the six vectors above. The three
$\mathrm{OPT}$ values are the best coordinatewise maximum available to each family. Greedy's
first step compares $\tfrac{N+1}{2N+1}$ for $a_{\mathrm{fwd}}$ against $\tfrac{N}{2N+1}$ for
$b_{\mathrm{fwd}}$ and for $c_{\mathrm{bwd}}$. These distances determine all three objective values directly. Strictness in Proposition~\ref{prop:free-split}(b) needs $N \geq 2$. At $N = 1$ the paired and
fixed optima coincide.

\paragraph{Tie-break.} The example stipulates the tie-break toward $a_{\mathrm{bwd}}$ at greedy's
second step. The released implementation breaks ties toward the highest item index and takes
$c_{\mathrm{bwd}}$ there. The set $\{a_{\mathrm{fwd}}, c_{\mathrm{bwd}}\}$ has the same value
$\tfrac{N+1}{2N+1}$, hence the same ratio $\tfrac{N+1}{2N}$, so the greedy statement of
Proposition~\ref{prop:matroid-cg}(a) holds under either rule. The $1$-swap statement does not
transfer: for $N \geq 2$, exchanging $a_{\mathrm{fwd}}$ for $b_{\mathrm{fwd}}$ reaches
$\mathrm{OPT}_{\mathrm{fix}}$, so the implementation's greedy set is not swap-local. The tightness of
the $1$-swap guarantee rests on $\{a_{\mathrm{fwd}}, a_{\mathrm{bwd}}\}$ itself, which is swap-local
whatever rule produced it.

\begin{lemma}[Closed form of the multilinear extension and its gradient]
\label{lem:closed-form}
Fix a query $(u,t)$, put $p:=|\mathcal I|$, and take $x\in[0,1]^p$. Sort its clamped bounds as $b_{(1)}\geq\cdots\geq b_{(p)}$, where $b_i:=B_i^+(u,t)$. Define $\tilde F_q(x):=\mathbb E_{S\sim\mathrm{Bernoulli}(x)}[h_S(u,t)]$ and $\pi_j:=\prod_{k<j}(1-x_{(k)})$. Then
\[
\tilde F_q(x) = \sum_{j} b_{(j)}\, x_{(j)}\, \pi_j,
\qquad
\frac{\partial \tilde F_q}{\partial x_{(j)}} = \pi_j \Big( b_{(j)} - \sum_{l > j} b_{(l)}\, x_{(l)}
\prod_{j < k < l} (1 - x_{(k)}) \Big),
\]
The inner sum is a suffix recurrence: $q_p=0$ and
$q_j=b_{(j+1)}x_{(j+1)}+(1-x_{(j+1)})q_{j+1}$.
Together with the prefix products $\pi_j$, it evaluates $\tilde F$ and its whole gradient
in one sort and two linear scans per query, $O(p\log p)$.
\end{lemma}
Closed-form multilinear extensions are known for coverage \citep{karimi2017stochastic} and
facility location, including exact gradients \citep[Section~4]{iyer2014monotone}. The lemma specializes
these formulas to the clamped ALT bounds. The sort uses $2K_0$ directional items, or $K_0$ landmarks when distances are symmetric.
\begin{proof}
$\max(0, \max_{i \in S} B_i) = b_{(j)}$ for the smallest rank $j$ in $S$, and under independent
Bernoulli$(x)$ selection the probability that rank $j$ is the first selected is $x_{(j)} \pi_j$.
Multilinearity gives $\partial \tilde F_q / \partial x_i = \tilde F_q(x \mid x_i{=}1) - \tilde F_q(x
\mid x_i{=}0)$. With $x_{(j)} = 1$ the ranks after $j$ are never first, with $x_{(j)} = 0$ the first
selected rank after $j$ contributes the inner sum, and the ranks before $j$ contribute identically to
both. What remains is $\pi_j$ times the bracket. $\tilde F_q$ is a polynomial in $x$, so ties in the sort change neither the value nor
$\nabla_x \tilde F_q$. They can make $\nabla_b\tilde F_q$ undefined; Proposition~\ref{prop:eps-lmo}(c)
treats differentiation with respect to these bounds.
\end{proof}

\subsection{Proof of Proposition~\ref{prop:free-split}}
\label{app:free-split-proof}

\begin{proof}
\emph{(a)} A paired set has $m/2$ forward and $m/2$ backward items, so it lies in
$\mathcal{B}_{\mathrm{fix}}$. A fixed-split set has $m$ items, so it lies in
$\mathcal{B}_{\mathrm{free}}$. The three families are the bases of three matroids. They are the
uniform matroid of rank $m/2$ on the $K_0$ paired items, the partition matroid of
Proposition~\ref{prop:matroid-cg}, and the uniform matroid of rank $m$ on $\mathcal{I}$.
On the directional ground set $\mathcal I$, exchanging one item between distinct paired bases breaks a pair, so base exchange fails. Each family deploys $m$ floats per vertex. A paired base is
$m/2$ pairs and hence $m$ directional labels, and a base of either other family is $m$ directional
labels. Inclusion gives the ordering of optimal values.

\emph{(b)} Example~\ref{ex:nine-vertex} gives the three optimal values and their strict ordering.

\emph{(c)} Canonical greedy has approximation factor $1 - 1/e$ under a cardinality constraint \citep{nemhauser1978analysis} and
$1/2$ under a partition matroid \citep{fisher1978analysis}. Continuous greedy has approximation factor $1 - 1/e$ over
any matroid polytope in the continuous-time limit \citep{calinescu2011maximizing}, and at finite $T$
the factor of Proposition~\ref{prop:eps-lmo}(a), with the rank $m/2$ in place of $m$ on
$\mathcal{B}_{\mathrm{pair}}$. For fixed and free items, fractional admissibility is Proposition~\ref{prop:matroid-cg}(c). For paired items, we instead average over subsets of pairs; each resulting ALT heuristic is admissible, so their mixture is too.

\emph{(d)} Top-$m$ selection from a score $s \in \R^{2K_0}$ returns a set of $m$ distinct items. Every $S \in
\mathcal{B}_{\mathrm{free}}$ is returned by $s = \mathbf{1}_S$, so the parameterization is onto.
Each item has one direction. A selected set therefore deploys as $A_{\mathrm{fwd}}, A_{\mathrm{bwd}}$
with one-hot rows and row counts $|S \cap \mathcal{I}_{\mathrm{fwd}}|$, $|S \cap
\mathcal{I}_{\mathrm{bwd}}|$, which Proposition~\ref{thm:admissibility}(b) covers for any counts. A
row mixing the two directions is not one of the two separately stored label tables the $m$-float
compressor of Proposition~\ref{thm:admissibility} deploys. It is still realizable at the same cost by
a \emph{combined} label $z(v) := \sum_k a^{\mathrm{fwd}}_k d(l_k,v) - \sum_k a^{\mathrm{bwd}}_k
d(v,l_k)$, since $z(t) - z(u)$ is the mixed bound and $z$ is one float per vertex. Its bound is
admissible as a sub-stochastic mixture of admissible bounds
(Proposition~\ref{prop:convex-extension}(b)). The parameterization here keeps the two tables apart.
Mixed rows therefore appear only during training. This restriction preserves the optimal value,
because an optimal direction-pure vertex exists.
\end{proof}

\subsection{Proof of Proposition~\ref{prop:eps-lmo}}
\label{app:eps-lmo-proof}

\begin{proof}
\emph{(a)} \emph{Weights.} Write $w_k := (1 - x_k) \odot \nabla \tilde F(x_k)$ and let
$R \sim \mathrm{Bernoulli}(x_k)$. Multilinearity gives $\partial \tilde F/\partial x_i(x_k) =
\mathbb{E}[F(R \cup \{i\}) - F(R \setminus \{i\})]$. The marginal $F(R \cup \{i\}) - F(R)$ vanishes
when $i \in R$ and equals $F((R \setminus \{i\}) \cup \{i\}) - F(R \setminus \{i\})$ when $i \notin
R$. The latter expression depends only on $R\setminus\{i\}$, which is independent of the inclusion indicator of $i$. Hence
\[
w_{k,i} \;=\; \mathbb{E}\big[F(R \cup \{i\}) - F(R)\big] \;\geq\; 0 .
\]
\emph{Feasibility.} $P$ is a down-closed polytope, so it contains $0$, and
$x_k = \frac{1}{T}\sum_{s<k} v_s$ is a convex combination of $v_0, \dots, v_{k-1}$ and $0$ for every
$k \leq T$. Hence $x_k \in P \subseteq [0,1]^{\mathcal{I}}$. Also $x_T = \frac{1}{T}\sum_k v_k$ is a
convex combination of the $v_k$, so $x_T \in \Pbase$ when every $v_k \in \Pbase$.

\emph{One step, by coupling.} Let $J$ be a random set, independent of $R$, that contains each $i$
independently with probability $v_{k,i}/T \leq 1$. Then $R \cup J$ contains each $i$ independently,
with probability $1 - (1 - x_{k,i})(1 - v_{k,i}/T) \leq x_{k,i} + v_{k,i}/T = x_{k+1,i}$. Each
$\partial \tilde F/\partial x_i \geq 0$ by monotonicity of $F$, so $\tilde F$ is non-decreasing in
every coordinate, and $\tilde F(x_{k+1}) \geq \mathbb{E}[F(R \cup J)]$. Monotonicity of $F$ also gives
$F(R \cup J) - F(R) \geq 0$ on every outcome, and on the event $J = \{i\}$ the difference is
$F(R \cup \{i\}) - F(R)$. Dropping every other event and using the independence of $R$ and $J$,
\[
\tilde F(x_{k+1}) - \tilde F(x_k) \;\geq\; \sum_i \Pr[J = \{i\}]\, w_{k,i},
\qquad
\Pr[J = \{i\}] \;=\; \frac{v_{k,i}}{T} \prod_{j \neq i} \Big(1 - \frac{v_{k,j}}{T}\Big).
\]
For $a \in [0,1]$ the weighted AM--GM inequality gives $(1 - 1/T)^a \cdot 1^{1-a} \leq a(1 - 1/T) +
(1 - a) = 1 - a/T$, with $0^0 = 1$ when $T = 1$. The product is therefore at least
$(1 - 1/T)^{\sum_{j \neq i} v_{k,j}}$. Every point of $P$
has coordinate sum at most the rank $m$. If $v_k$ is a vertex of $P$, that is, the indicator of an
independent set, the terms with $v_{k,i} = 1$ have $\sum_{j \neq i} v_{k,j} \leq m - 1$, and the
others vanish. For a general $v_k \in P$ the sum is at most $m$. Since $w_k \geq 0$,
\[
\tilde F(x_{k+1}) \;\geq\; \tilde F(x_k) + \frac{c}{T}\,\langle v_k, w_k\rangle,
\qquad
c := \begin{cases} (1 - 1/T)^{m-1} & \text{if every } v_k \text{ is a vertex of } P, \\
(1 - 1/T)^{m} & \text{otherwise.} \end{cases}
\]

\emph{Progress.} Submodularity gives $\sum_{i \in \mathrm{OPT}} [F(R \cup \{i\}) - F(R)] \geq F(R \cup
\mathrm{OPT}) - F(R)$, and monotonicity gives $F(R \cup \mathrm{OPT}) \geq F(\mathrm{OPT})$. Taking
expectations, $\langle \mathbf{1}_{\mathrm{OPT}}, w_k\rangle \geq F(\mathrm{OPT}) - \tilde F(x_k)$.
Since $\mathbf{1}_{\mathrm{OPT}} \in P$, the $\varepsilon_k$-oracle gives $\langle v_k, w_k\rangle \geq
F(\mathrm{OPT}) - \tilde F(x_k) - \varepsilon_k$.

\emph{Unrolling.} With $\Delta_k := F(\mathrm{OPT}) - \tilde F(x_k)$ and $\varepsilon := \max_k
\varepsilon_k$, the two displays give $\Delta_{k+1} \leq (1 - c/T)\,\Delta_k + (c/T)\,\varepsilon_k$,
hence $\Delta_{k+1} - \varepsilon \leq (1 - c/T)(\Delta_k - \varepsilon)$. Here $0 \leq 1 - c/T \leq 1$,
and $\Delta_0 = F(\mathrm{OPT})$ because $\tilde F(0) = F(\emptyset) = 0$. Unrolling $T$ steps gives
\[
\tilde F(x_T) \;\geq\; \Big(1 - \big(1 - \tfrac{c}{T}\big)^{T}\Big)\big(F(\mathrm{OPT}) - \varepsilon\big).
\]
If $F(\mathrm{OPT})\geq\varepsilon$, using $(1-c/T)^T\leq e^{-c}$ yields
$\tilde F(x_T)\geq(1-e^{-c})(F(\mathrm{OPT})-\varepsilon)$.
If $F(\mathrm{OPT})<\varepsilon$, the same conclusion follows from $\tilde F(x_T)\geq0$. With $c = (1 - 1/T)^{m-1}$ and
$\varepsilon = 0$ this is the exact top-$k$ oracle's factor
$1 - \exp(-(1 - 1/T)^{m-1})$, and the sharper $1 - (1 - (1 - 1/T)^{m-1}/T)^T$. With $c = (1 -
1/T)^m$ it is the factor for fractional oracle points. The argument requires no bound on the Hessian or $f_{\max}$, and both factors tend to
$1 - 1/e$ as $T \to \infty$. The guarantee holds for every sequence of oracle outputs
satisfying the error bounds.

\emph{Rounding.} General matroid pipage rounding is given by
\citet[Section~3.2]{calinescu2011maximizing}, following \citet{ageev2004pipage}.
For AAC's partition matroids the construction is elementary. First fill any unused block
mass, which cannot lower the monotone $\tilde F$. A nonintegral block of integer mass contains
at least two fractional coordinates $x_i,x_j$. Move along $x+t(e_i-e_j)$ for
\[
-\min(x_i,1-x_j)\leq t\leq\min(1-x_i,x_j).
\]
This preserves the block sum and feasibility. Along the segment, $\tilde F$ has second
derivative $-2\,\partial^2\tilde F/\partial x_i\partial x_j\geq0$, so one endpoint has value
at least that of the current point. Choose the better endpoint; at least one more coordinate
is now integral. Repeating gives a base $S$ with $F(S)\geq\tilde F(x_T)$.

\emph{(b)} The oracle acts on one step's weight vector, which we write $w \geq 0$; in~(a) and~(c)
it is $w_k = (1 - x_k) \odot \nabla \tilde F(x_k)$. Let $\bar w$ be $w$ normalized within each block
by that block's largest entry. Since $w \geq 0$ and every independent set pads to a base, $\max_{v \in
P} \langle v, w\rangle$ is attained on $\Pbase$. Let $v^* \in \Pbase$ maximize $\langle v, \bar
w\rangle$. Optimality of $v_\tau$ gives
$\langle v_\tau, \bar w\rangle + \tau H(v_\tau) \geq \langle v^*, \bar w\rangle + \tau H(v^*) \geq
\langle v^*, \bar w\rangle$, so $\langle v_\tau, \bar w\rangle \geq \langle v^*, \bar w\rangle - \tau
H_{\max}$. Both $H$ and the constraint separate over blocks, so this holds blockwise with
$H_{b,\max}$, and $v^*_b$ also maximizes $\langle \cdot, w_b\rangle$ on block $b$. Multiplying block
$b$'s inequality by its own scale $\|w_b\|_\infty$ and summing restores the original units and gives $\varepsilon \leq \tau \sum_b \|w_b\|_\infty H_{b,\max} \leq \tau \max_b \|w_b\|_\infty
H_{\max}$, the $\varepsilon$ stated. Each $\|w_b\|_\infty \leq f_{\max}$, because
$0 \leq w_i = \mathbb{E}[F(R \cup \{i\}) - F(R)] \leq F(\{i\})$ for $R \sim \mathrm{Bernoulli}(x)$ as in~(a), by submodularity, every marginal gain
of $i$ being at most its gain on $\emptyset$. \par On a block with $w_b = 0$ we set $\bar w_b := 0$.
The oracle then maximizes entropy alone there, and the bound is trivial since that block contributes
nothing to the linear objective. $H$ is strictly concave and symmetric within each block, so over
$\Pbase$ it is maximized at the uniform point $v_i = m_b / |\mathcal{I}_b|$. Hence $H_{\max} =
\sum_b |\mathcal{I}_b|\, h_2(m_b / |\mathcal{I}_b|)$ with $|\mathcal{I}_b| = K_0$ for the fixed split
and one block of size $2K_0$ for the free one; symmetric distances use one block of size $K_0$. For $0<r<N$, $N h_2(r/N)=r\log(N/r)+(N-r)\log(1+r/(N-r))\leq r\log(eN/r)$. Concavity of $\log$ therefore gives $H_{\max}\leq\sum_b m_b\log(e|\mathcal I_b|/m_b)\leq m\log(e|\mathcal I|/m)$. \par The objective $\langle v, \bar w\rangle + \tau H(v)$ is strictly concave on the
block-sum affine slices of $[0,1]^{\mathcal{I}}$, and $h_2'$ diverges at $0$ and $1$. The maximizer
is therefore interior and unique. Stationarity reads $\bar w_i + \tau \log((1 - v_i)/v_i) = \lambda_b$, i.e.\
$v_i = \sigma((\bar w_i - \lambda_b)/\tau)$. The block sum is continuous and strictly decreasing in
$\lambda_b$ from $|\mathcal{I}_b|$ to $0$, so its root is unique. The real-analytic implicit
function theorem applies because the derivative in $\lambda_b$ is strictly negative.
Thus $\lambda_b$ and $v_\tau$ are real-analytic in $\bar w$.
The Jacobian is explicit: within one block, set $s_i:=v_i(1-v_i)>0$. Differentiating
$\sum_i v_i=m_b$ first gives $\partial\lambda_b/\partial\bar w_j=s_j/\sum_r s_r$, and hence
\[
\frac{\partial v_i}{\partial\bar w_j}
=\frac{s_i}{\tau}\left(\mathbf 1[i=j]-\frac{s_j}{\sum_r s_r}\right).
\]
Each column sums to zero, as required to preserve the block budget.

\emph{(c)} The entropic oracle returns fractional points of $\Pbase$, so (a) applies with
$c = (1 - 1/T)^m$ and $\varepsilon_k = \tau \max_b \|w_{k,b}\|_\infty H_{\max} \leq \tau f_{\max}
H_{\max}$, and pipage rounding then returns the base $S$ with $F(S) \geq \tilde F(x_T)$. Admissibility of every iterate is Proposition~\ref{prop:matroid-cg}(c). It remains to justify differentiability through all $T$ steps.
Remove the finitely many hyperplanes on which a raw bound is zero or two positive bounds tie.
On each remaining open region the clamp and sort order are fixed. Lemma~\ref{lem:closed-form}
then makes $w(x)$ polynomial in $x$, with coefficients analytic in the bound matrix.

Inductively, suppose $x_k$ is real-analytic on an open set of full measure. On each connected
component, every difference $w_{k,i}-w_{k,j}$ is either identically zero or has a zero set of
measure zero. Remove the latter sets. The maximizing entry in each block is now locally fixed;
identically tied entries give the same normalization. A block maximum is likewise either
identically zero, where the zero-block convention is constant, or nonzero away from a null set.
Thus normalization, the analytic oracle of part~(b), and the update to $x_{k+1}$ are analytic
on an open set of full measure. There are only countably many connected components and finitely
many steps, so the removed sets remain null. Starting from $x_0=0$ proves the claim for $x_T$
and, by the same closed form, for every $h_{x_k}$.
\end{proof}

\subsection{Proof of Proposition~\ref{prop:erm}}
\label{app:erm-proof}

\begin{proof}
\emph{(a)} $0 \leq h_S \leq d$ is Theorem~\ref{thm:alt-admissibility}, and $d(u,t)
\leq D_Q$ on the support. For any landmark $l$ of the component, $d(u,t) \leq d(u,l) + d(l,t) \leq
d_{\mathrm{sym}}(u,l) + d_{\mathrm{sym}}(t,l) \leq 2 \max_v d_{\mathrm{sym}}(v,l)$. Minimize over
$l$. For the label range, $h_S(u,t)$ is either $0$ or $B_i(u,t)$ for some $i \in S$. Every
$B_i(u,t)$ is $z_i(t) - z_i(u)$ or $z_i(u) - z_i(t)$, a difference of two entries of the row $z_i$,
hence at most that row's range and so at most $D_L$. For the diameter bound, a forward row has
$z_i(v) - z_i(v') = d(l,v) - d(l,v') \leq d(v',v)$ and a backward row has $z_i(v) - z_i(v') = d(v,l) -
d(v',l) \leq d(v,v')$, both by the triangle inequality, so $D_L \leq \max_{u,t} d(u,t)$.

\emph{(b)} If $D=0$, every heuristic is zero and the claim is immediate. Assume $D>0$. For fixed $S$ the summands $h_S(u_j, t_j)$ are i.i.d.\ in $[0, D]$ with mean $F(S)$,
so Hoeffding's inequality gives $\Pr[|\hat F_n(S) - F(S)| > \varepsilon] \leq 2 \exp(-2n
\varepsilon^2 / D^2)$. A union bound over the $M$ sets and solving for $\varepsilon$ at total
failure probability $\delta$ gives $\varepsilon_n$. The three sizes are the numbers of bases of the
three matroids of Proposition~\ref{prop:free-split}(a), and $\binom{2K_0}{m} = \sum_j
\binom{K_0}{j}\binom{K_0}{m-j} \geq \binom{K_0}{m/2}^2 \geq \binom{K_0}{m/2}$ (Vandermonde). The
bound $\binom{2K_0}{m} \leq (2eK_0/m)^m$ is the standard binomial bound.

\emph{(c)} On the event of (b), and since $\hat S \in \mathcal{S}$,
$F(\hat S) \geq \hat F_n(\hat S) - \varepsilon_n \geq \alpha \hat F_n(\mathrm{OPT}) - \gamma -
\varepsilon_n \geq \alpha (F(\mathrm{OPT}) - \varepsilon_n) - \gamma - \varepsilon_n$, the last step
using $\alpha\geq0$.
\end{proof}

\subsection{Numerical interpretation of the selection guarantee}
\label{app:guarantee-constants}
We use the entropic oracle penalty $\tau f_{\max}H_{\max}$. The exact oracle has zero penalty
and a larger $\alpha_T$, yielding weaker requirements on sample size.
For comparison, greedy with access to the population objective has a $1/2$ guarantee under a
matroid constraint \citep{fisher1978analysis}. The AAC-CG population bound exceeds it exactly when
\[
\alpha_T\, F(\mathrm{OPT}) \;-\; \alpha_T\, \tau f_{\max} H_{\max} \;-\; (1 + \alpha_T)\, D
\sqrt{\frac{\log(2M/\delta)}{2n}} \;>\; \tfrac{1}{2}\, F(\mathrm{OPT}),
\]
that is, when the two penalties together stay below the margin $(\alpha_T - 1/2)\,F(\mathrm{OPT})$,
and it is non-vacuous when they stay below $\alpha_T F(\mathrm{OPT})$. Any allocation of that
margin gives a sufficient condition on $(\tau, n)$ separately. A bound on either penalty alone does not establish a bound on their sum.
If both methods instead fit the same $n$ queries, Proposition~\ref{prop:erm}(c) gives
greedy the population bound $\tfrac12 F(\mathrm{OPT})-\tfrac32\varepsilon_n$.
For $\alpha_T>1/2$, AAC-CG's bound exceeds this bound precisely when
\[
F(\mathrm{OPT})>\varepsilon_n+
\frac{\alpha_T\tau f_{\max}H_{\max}}{\alpha_T-1/2}.
\]
The numerical thresholds below use the population-objective benchmark.

The step count affects the guarantee only through $\alpha_T$.
We can check whether $\alpha_T>1/2$ without estimating $F(\mathrm{OPT})$. At the deployed $T{=}100$ the exact oracle gives $\alpha_T = 0.621$, $0.606$,
$0.577$, $0.519$ and $0.412$ at $m = 4$, $8$, $16$, $32$ and $64$, against the limit
$1 - 1/e = 0.632$. The margin is positive for every $m \leq 37$ under the exact oracle and every
$m \leq 36$ under the entropic one. Exceeding $1/2$ at $m = 8$, $16$, $32$ and $64$ requires
$T \geq 20$, $42$, $86$ and $173$ under the exact oracle, before accounting for oracle and sampling error.
Increasing $T$ increases selection time linearly.
On NY at $64$\,B/v, selection time rises from $0.49$ to $8.75$\,s when $T$ increases from $100$ to $2048$.
We sweep $T\in\{100,512,2048\}$ on NY, SBM, and BA at $32$ and $64$\,B/v.
All three step counts return identical selections, objectives, and expansion counts.

The sample-size thresholds also depend on the unknown population optimum $F(\mathrm{OPT})$.
We approximate it using the certified optimum on the generalization study's
$20{,}000$-query reference sample. Because we substitute a sample optimum, the thresholds below are estimates rather than certificates.
On SBM, we use $K_0{=}16$, $m{=}4$, $\delta = 0.05$, label range $D_L = 8.74$, and the exact oracle at $T{=}100$.
The estimated guarantee becomes positive at $n = 693$ and exceeds the $1/2$ benchmark at $n \approx 18{,}240$. Our largest training
set is $n = 3{,}200$ (Table~\ref{tab:optimiser-gap-vs-n}), which lies between the two. SBM is
undirected, so greedy has the stronger finite-rank factor $1 - (1-1/4)^4 = 0.6836$, and
$\alpha_T = 0.621$ stays below it at every $n$. The $1/2$ comparison therefore illustrates the partition-matroid threshold; it does not
establish an advantage over greedy on SBM.

The label range in Proposition~\ref{prop:erm}(a) gives a tighter range constant:
$D_L=8.74$, compared with $15.86$ for the diameter upper bound on the same instance. The largest bound the teacher realizes
on the evaluation sample is $7.64$, which is a sample maximum and not itself a valid constant. Here
$D_L$ exceeds this sample maximum by $14$\,\%, whereas the diameter upper bound is $2.1$ times the sample maximum. At our deployment settings, the theorem's lower bound on $F(S)$ does not exceed classical greedy's population-oracle bound. The bound uses a worst-case range $D$ and a union bound over $M$ selections.
Empirically, AAC-CG's test objective stays within $5.6$\,\% of greedy's at every $n$ in the sweep (Table~\ref{tab:optimiser-gap-vs-n}).

\subsection{Smooth-max differentiable surrogate}
\label{app:smooth-max}

With exact teacher labels, a smooth approximation to the maximum can preserve admissibility.
The contextual variant uses this surrogate during training. Its smoothed distance labels require
separate analysis, and deployment obtains admissibility by recomputing exact distances
(Section~\ref{sec:contextual}). Static AAC-ST uses the hard maximum with a straight-through
Gumbel-softmax estimator (Section~\ref{sec:compressor}).

\begin{theorem}[Smooth extrema and the end-to-end admissibility chain]
\label{thm:smooth}
\textbf{(a)} For $r\geq1$, $\mathbf{x}\in\R^r$, and $\beta>0$, define
\begin{equation}
\label{eq:smooth-max}
\smax_\beta(\mathbf{x}) := \frac{1}{\beta}\log\sum_{i=1}^{r} e^{\beta x_i} - \frac{\log r}{\beta},
\qquad
\smin_\beta(\mathbf{x}) := -\frac{1}{\beta}\log\sum_{i=1}^{r} e^{-\beta x_i}.
\end{equation}
Both are $C^\infty$ and satisfy
\[
\max_i x_i-\frac{\log r}{\beta}\leq\smax_\beta(\mathbf x)\leq\max_i x_i,
\qquad
\min_i x_i-\frac{\log r}{\beta}\leq\smin_\beta(\mathbf x)\leq\min_i x_i.
\]
\textbf{(b)} Assume finite, exact teacher labels and row-stochastic compression matrices. Let $\tilde{h}^{\beta}_{A}$
replace the outer maximum by $\smax_\beta$ applied to the full argument list: the $m$
directional bounds and $0$ on a digraph ($r=m+1$), or the $2m$ signed differences
$\pm(y_i(u)-y_i(t))$ and $0$ on an undirected graph ($r=2m+1$).
Then $\tilde{h}^{\beta}_{A}$ is $C^\infty$ in the labels and row-stochastic weights, and
\[
\tilde{h}^{\beta}_{A}(u,t)\leq\hA(u,t)\leq\hALT(u,t)\leq d(u,t).
\]
The surrogate is at least $-\log r/\beta$ and equals zero at $u=t$.
\end{theorem}
The contextual distance layer uses the normalized soft-min $-\smax_\beta(-\mathbf x)$,
which instead lies between $\min_i x_i$ and $\min_i x_i+\log r/\beta$. This distinction
explains why its training labels need a separate analysis.

\begin{proof}
(a) Let $x^* = \max_i x_i$. Since $e^{\beta x^*} \leq \sum_i e^{\beta x_i} \leq r\,e^{\beta x^*}$,
\begin{equation}
\label{eq:lse-bounds}
\beta x^* \;\leq\; \log\sum_{i=1}^{r} e^{\beta x_i} \;\leq\; \beta x^* + \log r.
\end{equation}
Divide by $\beta$ and subtract $\log r / \beta$. Applying Eq.~\eqref{eq:lse-bounds} to
$-\mathbf{x}$ and using $\max(-\mathbf{x}) = -\min(\mathbf{x})$ gives the second pair, through the
identity $\smin_\beta(\mathbf{x}) = -\smax_\beta(-\mathbf{x}) - \log r/\beta$. The same pair
applied to $-\mathbf{x}$ reads $-\min(\mathbf{x}) - \log r/\beta \leq \smax_\beta(-\mathbf{x}) \leq
-\min(\mathbf{x})$, which is the bound on the normalized soft-min. Smoothness is that
of $\log\sum\exp$.
(b) The maximum of the full argument list is $\hA(u,t)$, so part~(a) gives the
first inequality. The remaining inequalities are Proposition~\ref{thm:admissibility} and
Theorem~\ref{thm:alt-admissibility}. Each argument is polynomial in the finite labels and $A$,
so smoothness follows from part~(a). Using signed differences is necessary here: applying
$\smax_\beta$ to the $m$ absolute values $|y_i(u)-y_i(t)|$ would leave a kink wherever one of them
vanishes.
\end{proof}

In the static exact-label setting the chain preserves a lower bound throughout training. It does
not apply to the smoothed Bellman--Ford labels of the contextual variant
(Section~\ref{sec:contextual}, Appendix~\ref{app:contextual}). Those labels overestimate the exact
distances, so their differences need not lower-bound the true distance. There, deployment-time admissibility on the
predicted-cost graph comes from re-running exact Dijkstra from the selected landmarks.

\subsection{Directed-graph extension: covering radius and unreachable landmarks}
\label{app:directed}

The directed extension uses a symmetrized covering radius and an explicit rule for unreachable
landmarks. The following arguments establish the coverage bound, admissibility, and consistency.

\textbf{(a)~Covering radius.} On a strongly connected component, $d_{\mathrm{sym}}(x,y) := \max(d(x,y), d(y,x))$ is a metric.
Thus Gonzalez's $2$-approximation applies to farthest-point sampling in $d_{\mathrm{sym}}$.
It bounds $r^{\mathrm{sym}}(S)$, the radius in Theorem~\ref{thm:covering-radius}. Forward-only FPS optimizes a different distance and does not give this guarantee. The four-vertex digraph of Corollary~\ref{cor:fps-covering}(c) exhibits a ratio of $11/5 > 2$.

\textbf{(b)~Unreachable landmarks.} We store infinite landmark distances as sentinels
(Appendix~\ref{app:reproducibility}). We treat each landmark term according to which distances are infinite.
A forward term $d(l,t) - d(l,v)$ with $d(l,v) < \infty = d(l,t)$ proves that $t$ is unreachable from $v$.
Otherwise, a path from $v$ to $t$ would also give $l$ a path to $t$.
A backward term $d(v,l) - d(t,l)$ with $d(t,l) < \infty = d(v,l)$ proves the same.
Otherwise, a path from $v$ to $t$ would give $v$ a path to $l$ through $t$. The heuristic then returns $+\infty$, which is exact. If only the subtracted distance is infinite, the term gives no useful lower bound and is omitted.
This distance is $d(l,v)$ for a forward term and $d(t,l)$ for a backward term. A term infinite at both ends contributes $0$, which the clamp supplies anyway. On an undirected graph a landmark that reaches exactly one of $v$ and $t$ separates them, and the heuristic again returns $+\infty$.

CDH applies the same rule to its pivots. Every surviving term is a finite triangle-inequality bound, and $+\infty$ is returned only where $d(v,t) = +\infty$, so admissibility holds.

Dropping terms can only \emph{lower} the heuristic, hence only \emph{widen} the gap $d - h$. The rule is therefore not what preserves the bound of Theorem~\ref{thm:covering-radius}. The landmark $l^*$ that certifies the forward bound in its proof comes, by construction, from the premise $u \in V_{\mathrm{reach}}(L_{\mathrm{fwd}})$. Both $d(l^*, u)$ and $d(u, l^*)$ are therefore finite, and $d(l^*,t) \leq d(l^*,u) + d(u,t) < \infty$. Its term is present, and the deployed $\max$ is at least that term whatever else drops out. The backward landmark $l^\dagger$ is present likewise, since $d(u,l^\dagger) \leq d(u,t) + d(t,l^\dagger) < \infty$. For a query whose endpoint lies outside the respective $V_{\mathrm{reach}}$, the theorem makes no claim about that direction's bound. The heuristic uses the remaining terms. It is $0$ if no term is positive. If this holds at every vertex, the search is equivalent to Dijkstra's (Section~\ref{sec:graphs-astar}).

\emph{Consistency.} Under this rule, ALT on any forward and backward landmark sets is consistent in exact arithmetic.
This includes every deployed $h_S$ from either optimizer, even when finite terms change between the two ends of an arc. Fix $t$ and an arc $u \to v$. If $h(v) = +\infty$ there is nothing to show. If $h(u) = +\infty$, a forward certificate $l$ at $u$ has $d(l,v) \leq d(l,u) + w(u,v) < \infty = d(l,t)$, and a backward certificate has $d(t,l) < \infty$ and $d(v,l) = \infty$, because $d(u,l) \leq w(u,v) + d(v,l)$. Either way $l$ certifies at $v$, and $h(v) = +\infty$.

Otherwise both values are finite. A forward term present at $u$, with $d(l,u)$ and $d(l,t)$ finite, is present at $v$, because $d(l,v) \leq d(l,u) + w(u,v) < \infty$, and then $d(l,t) - d(l,u) \leq w(u,v) + d(l,t) - d(l,v)$. A backward term present at $u$, with $d(u,l)$ and $d(t,l)$ finite, has $d(v,l) < \infty$, because $h(v)$ is finite and $l$ does not certify at $v$, and then $d(u,l) - d(t,l) \leq w(u,v) + d(v,l) - d(t,l)$. Terms that contribute $0$ at $u$ satisfy the arc inequality because $h(v) \geq 0$, and terms present only at $v$ only raise $h(v)$. Taking the maximum gives $h(u) \leq w(u,v) + h(v)$. The undirected case is the same argument on $|d(l,u) - d(l,t)|$, since $u$ and $v$ lie in one component.

\emph{Rounding.} The deployed tables are float32. Before clamping at zero, the heuristic subtracts two ULPs of the coarser of input and stored precision at the largest finite label. For exact, non-negative input labels, conversion changes their difference by at most one ULP at that magnitude.
Float32 subtraction adds at most half an ULP. Subtracting the allowance adds at most another half ULP at the same scale (the implementation uses float64 for this step). Thus two ULPs cover storage and evaluation error.

This argument assumes no overflow and finite labels below the sentinel cutoff. It does not cover arbitrary error in the input distance computation. If each input label has absolute error at most $\eta$, an additional allowance of $2\eta$ suffices; the implementation does not estimate this quantity. Our mathematical guarantees use exact distances, while the numerical experiments separately compare returned path costs with Dijkstra.

Even with admissible rounding, consistency need not survive: the common offset cancels between positive adjacent heuristic values, while their rounding errors need not cancel. A* uses reopenings for this reason.

\section{Extended experiment tables}
\label{app:extended-experiments}

We report results at each memory budget and give the protocols for the robustness checks.

\subsection{Per-graph detail}
\label{app:per-graph-detail}

We report validation-split results on DIMACS NY and FLA and OSMnx Modena and Manhattan in Table~\ref{tab:valsplit}.
We evaluate the validation-selected configurations on a held-out test set.
Table~\ref{tab:osmnx} summarizes all four OSMnx graphs at each budget, followed by the full $(K_0, m, \mathrm{B/v})$ sweep.

\begin{figure}[t]
\centering
\includegraphics[width=\textwidth]{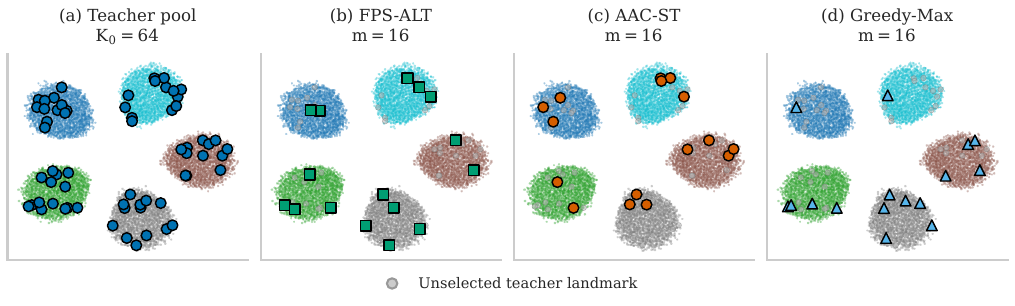}
\caption{Landmark selection on SBM ($5{\times}2000$, $p_{\mathrm{in}}{=}0.05$, $p_{\mathrm{out}}{=}0.001$).
All four panels use the same $K_0{=}64$ pool. Colors indicate communities.
The spring layout is schematic; plotted distances are not graph distances.
Panel~(b) is exactly FPS-ALT with $K{=}m$, by the forced-first-$m$ identity
(Section~\ref{sec:training-drift}). AAC-ST in~(c) retains only $4$ of those $16$ landmarks.
Its expansion reduction is $0.13$ points higher than FPS-ALT's at the illustrated $64$\,B/v budget.
It is $0.20$ and $1.08$ points lower at $128$ and $32$\,B/v, respectively
(Section~\ref{sec:synthetic-sbm-ba}). At $64$\,B/v, Greedy-Max in~(d) gives the largest expansion
reduction: $94.87$\,\%, compared with $93.88$\,\% for FPS-ALT.}

\label{fig:landmarks-sbm}
\end{figure}

\begin{table}[!htbp]
\centering
\caption{\tabletag{Main} Validation selection and test performance.}
\label{tab:valsplit}
\small
\adjustbox{max width=\linewidth}{\begin{tabular}{l l r
    S[table-format=2.1(1)] S[table-format=2.1(1)] S[table-format=+1.1] l}
\toprule
\textbf{Graph} & \textbf{Method} & \textbf{B/v}
  & {\textbf{Val (\%)}} & {\textbf{Test (\%)}} & {\textbf{Gap (pp)}}
  & \textbf{$K_0$ picks} \\
\midrule
\multirow{6}{*}{NY} & \cellcolor{aaccolor} AAC-ST (val-sel.) & \cellcolor{aaccolor}  32 & \cellcolor{aaccolor} 81.5 \pm 1.1 & \cellcolor{aaccolor} 81.4 \pm 1.5 & \cellcolor{aaccolor} +0.2 & \cellcolor{aaccolor} \{32, 64\} \\
                           & ALT &  32 & 86.6 \pm 0.0 & \bfseries 85.5 \pm 0.0 & +1.1 & {--} \\
                           & \cellcolor{aaccolor} AAC-ST (val-sel.) & \cellcolor{aaccolor}  64 & \cellcolor{aaccolor} 90.5 \pm 0.6 & \cellcolor{aaccolor} 89.0 \pm 0.9 & \cellcolor{aaccolor} +1.5 & \cellcolor{aaccolor} \{32, 64, 128\} \\
                           & ALT &  64 & 91.1 \pm 0.0 & \bfseries 90.2 \pm 0.0 & +0.9 & {--} \\
                           & \cellcolor{aaccolor} AAC-ST (val-sel.) & \cellcolor{aaccolor} 128 & \cellcolor{aaccolor} 93.5 \pm 0.8 & \cellcolor{aaccolor} 93.2 \pm 0.3 & \cellcolor{aaccolor} +0.3 & \cellcolor{aaccolor} \{32, 64\} \\
                           & ALT & 128 & 94.4 \pm 0.0 & \bfseries 93.6 \pm 0.0 & +0.8 & {--} \\
\midrule
\multirow{6}{*}{FLA} & \cellcolor{aaccolor} AAC-ST (val-sel.) & \cellcolor{aaccolor}  32 & \cellcolor{aaccolor} 80.8 \pm 1.7 & \cellcolor{aaccolor} 78.8 \pm 4.6 & \cellcolor{aaccolor} +2.0 & \cellcolor{aaccolor} \{32, 64, 128\} \\
                           & ALT &  32 & 87.8 \pm 0.0 & \bfseries 84.6 \pm 0.0 & +3.2 & {--} \\
                           & \cellcolor{aaccolor} AAC-ST (val-sel.) & \cellcolor{aaccolor}  64 & \cellcolor{aaccolor} 86.3 \pm 3.3 & \cellcolor{aaccolor} 84.4 \pm 1.8 & \cellcolor{aaccolor} +1.9 & \cellcolor{aaccolor} \{32, 64\} \\
                           & ALT &  64 & 90.1 \pm 0.0 & \bfseries 87.8 \pm 0.0 & +2.4 & {--} \\
                           & \cellcolor{aaccolor} AAC-ST (val-sel.) & \cellcolor{aaccolor} 128 & \cellcolor{aaccolor} 92.6 \pm 0.9 & \cellcolor{aaccolor} 91.9 \pm 0.3 & \cellcolor{aaccolor} +0.7 & \cellcolor{aaccolor} \{32, 64\} \\
                           & ALT & 128 & 92.9 \pm 0.0 & \bfseries 92.5 \pm 0.0 & +0.4 & {--} \\
\midrule
\multirow{6}{*}{Modena} & \cellcolor{aaccolor} AAC-ST (val-sel.) & \cellcolor{aaccolor}  32 & \cellcolor{aaccolor} 82.0 \pm 1.2 & \cellcolor{aaccolor} 81.9 \pm 0.9 & \cellcolor{aaccolor} +0.1 & \cellcolor{aaccolor} \{32, 64, 128\} \\
                           & ALT &  32 & 83.8 \pm 0.0 & \bfseries 83.9 \pm 0.0 & -0.1 & {--} \\
                           & \cellcolor{aaccolor} AAC-ST (val-sel.) & \cellcolor{aaccolor}  64 & \cellcolor{aaccolor} 87.8 \pm 0.5 & \cellcolor{aaccolor} 87.9 \pm 0.7 & \cellcolor{aaccolor} -0.1 & \cellcolor{aaccolor} \{32, 64, 128\} \\
                           & ALT &  64 & 88.3 \pm 0.0 & \bfseries 89.0 \pm 0.0 & -0.7 & {--} \\
                           & \cellcolor{aaccolor} AAC-ST (val-sel.) & \cellcolor{aaccolor} 128 & \cellcolor{aaccolor} 92.6 \pm 0.3 & \cellcolor{aaccolor} \bfseries 92.5 \pm 0.2 & \cellcolor{aaccolor} +0.1 & \cellcolor{aaccolor} \{32, 64\} \\
                           & ALT & 128 & 92.1 \pm 0.0 & 92.1 \pm 0.0 & +0.0 & {--} \\
\midrule
\multirow{6}{*}{Manhattan} & \cellcolor{aaccolor} AAC-ST (val-sel.) & \cellcolor{aaccolor}  32 & \cellcolor{aaccolor} 75.9 \pm 2.1 & \cellcolor{aaccolor} 77.0 \pm 2.2 & \cellcolor{aaccolor} -1.1 & \cellcolor{aaccolor} \{32, 64, 128\} \\
                           & ALT &  32 & 77.8 \pm 0.0 & \bfseries 80.6 \pm 0.0 & -2.9 & {--} \\
                           & \cellcolor{aaccolor} AAC-ST (val-sel.) & \cellcolor{aaccolor}  64 & \cellcolor{aaccolor} 84.9 \pm 2.2 & \cellcolor{aaccolor} 84.5 \pm 2.6 & \cellcolor{aaccolor} +0.4 & \cellcolor{aaccolor} \{32, 64, 128\} \\
                           & ALT &  64 & 87.0 \pm 0.0 & \bfseries 87.0 \pm 0.0 & +0.0 & {--} \\
                           & \cellcolor{aaccolor} AAC-ST (val-sel.) & \cellcolor{aaccolor} 128 & \cellcolor{aaccolor} 90.6 \pm 0.6 & \cellcolor{aaccolor} \bfseries 90.9 \pm 1.1 & \cellcolor{aaccolor} -0.3 & \cellcolor{aaccolor} \{32, 64\} \\
                           & ALT & 128 & 90.2 \pm 0.0 & 90.0 \pm 0.0 & +0.1 & {--} \\
\bottomrule
\end{tabular}
}
\par\smallskip
\begin{minipage}{\linewidth}
\footnotesize\raggedright
We select $K_0$ on validation queries and evaluate it on disjoint test queries. The test queries match those in the other matched-memory tables. We select $K_0$ separately for each seed and budget; both query sets are fixed across seeds. FPS-ALT has no hyperparameter to select in this comparison. Bold marks the higher test expansion reduction for each graph--budget pair. \par We report 5 seeds per row. ALT has zero standard deviation because FPS uses a fixed LCC start vertex. Test results: \textbf{AAC-ST wins 2 of the 12 graph--budget pairs}, all on OSMnx road graphs. AAC-ST's validation-minus-test gap stays within 2.0 points. FPS-ALT shows gaps of similar size between the two query samples, although it uses no validation selection.
\end{minipage}
\end{table}

\begin{table}[!htbp]
\centering
\caption{\tabletag{Main} OSMnx performance at matched label memory.}
\label{tab:osmnx}
\small
\adjustbox{max width=\linewidth}{\begin{tabular}{llrccc}
\toprule
\textbf{Scale / Graph} & \textbf{Method} & \textbf{Nodes} & \textbf{32 B/v} & \textbf{64 B/v} & \textbf{128 B/v} \\
\midrule
\multicolumn{6}{l}{\textit{Small-scale}} \\
\midrule
\rowcolor{aaccolor} Modena & AAC-ST & 26K & $81.1 \pm 2.1$ & $87.9 \pm 0.6$ & $91.2 \pm 1.1$ \\
                            & ALT & 26K & $\mathbf{83.9 \pm 0.0}$ & $\mathbf{89.0 \pm 0.0}$ & $\mathbf{92.1 \pm 0.0}$ \\
\midrule
\rowcolor{aaccolor} Manhattan & AAC-ST & 13K & $75.5 \pm 3.2$ & $83.2 \pm 4.7$ & $\mathbf{91.0 \pm 1.1}$ \\
                            & ALT & 13K & $\mathbf{80.6 \pm 0.0}$ & $\mathbf{87.0 \pm 0.0}$ & $90.0 \pm 0.0$ \\
\midrule
\multicolumn{6}{l}{\textit{City-scale}} \\
\midrule
\rowcolor{aaccolor} Berlin & AAC-ST & 41K & $82.1 \pm 3.7$ & $90.9 \pm 1.2$ & $93.3 \pm 1.0$ \\
                            & ALT & 41K & $\mathbf{88.1 \pm 0.0}$ & $\mathbf{92.1 \pm 0.0}$ & $\mathbf{94.2 \pm 0.0}$ \\
\midrule
\rowcolor{aaccolor} Los Angeles & AAC-ST & 92K & $82.5 \pm 3.4$ & $90.5 \pm 0.7$ & $93.9 \pm 0.6$ \\
                            & ALT & 92K & $\mathbf{84.2 \pm 0.0}$ & $\mathbf{91.2 \pm 0.0}$ & $\mathbf{94.3 \pm 0.0}$ \\
\bottomrule
\end{tabular}
}
\par\vspace{3pt}
{\small \textbf{Full configuration sweep.} We report every configuration as mean $\pm$ std over 5 seeds. We select no $K_0$ in this block. Bold marks the highest mean reduction per graph--budget pair.}\par\vspace{2pt}
\footnotesize
\setlength{\tabcolsep}{4pt}
\adjustbox{max width=\linewidth}{\begin{tabular}{llrrrrr}
\toprule
\textbf{Graph} & \textbf{Method} & $K_0$ & $m$ & \textbf{B/v} & \textbf{Reduction (\%)} & \textbf{Latency (ms)} \\
\midrule
\rowcolor{aaccolor}    Manhattan (13K) & AAC-ST & 16 & 8 & 32 & $75.2 \pm 3.6$ & 3 \\
\rowcolor{aaccolor}     & AAC-ST & 32 & 8 & 32 & $72.5 \pm 5.3$ & 3 \\
\rowcolor{aaccolor}     & AAC-ST & 32 & 16 & 64 & $82.5 \pm 5.1$ & 2 \\
\rowcolor{aaccolor}     & AAC-ST & 64 & 16 & 64 & $81.3 \pm 3.4$ & 2 \\
\rowcolor{aaccolor}     & AAC-ST & 64 & 32 & 128 & $\mathbf{91.0 \pm 1.1}$ & 1 \\
   & ALT & -- & -- & 32 & $\mathbf{80.6 \pm 0.0}$ & 2 \\
   & ALT & -- & -- & 64 & $\mathbf{87.0 \pm 0.0}$ & 2 \\
   & ALT & -- & -- & 128 & $90.0 \pm 0.0$ & 1 \\
   & ALT & -- & -- & 256 & $\mathbf{93.8 \pm 0.0}$ & 1 \\
\midrule
\rowcolor{aaccolor}    Modena (26K) & AAC-ST & 16 & 8 & 32 & $78.3 \pm 3.3$ & 5 \\
\rowcolor{aaccolor}     & AAC-ST & 32 & 8 & 32 & $80.4 \pm 2.2$ & 4 \\
\rowcolor{aaccolor}     & AAC-ST & 32 & 16 & 64 & $87.1 \pm 1.8$ & 3 \\
\rowcolor{aaccolor}     & AAC-ST & 64 & 16 & 64 & $87.1 \pm 0.3$ & 3 \\
\rowcolor{aaccolor}     & AAC-ST & 64 & 32 & 128 & $91.2 \pm 1.1$ & 2 \\
   & ALT & -- & -- & 32 & $\mathbf{83.9 \pm 0.0}$ & 3 \\
   & ALT & -- & -- & 64 & $\mathbf{89.0 \pm 0.0}$ & 2 \\
   & ALT & -- & -- & 128 & $\mathbf{92.1 \pm 0.0}$ & 2 \\
   & ALT & -- & -- & 256 & $\mathbf{93.5 \pm 0.0}$ & 2 \\
\midrule
\rowcolor{aaccolor}    Berlin (41K) & AAC-ST & 16 & 8 & 32 & $81.5 \pm 3.6$ & 7 \\
\rowcolor{aaccolor}     & AAC-ST & 32 & 8 & 32 & $81.1 \pm 3.4$ & 7 \\
\rowcolor{aaccolor}     & AAC-ST & 32 & 16 & 64 & $90.8 \pm 1.4$ & 4 \\
\rowcolor{aaccolor}     & AAC-ST & 64 & 16 & 64 & $89.0 \pm 1.3$ & 4 \\
\rowcolor{aaccolor}     & AAC-ST & 64 & 32 & 128 & $93.3 \pm 1.0$ & 3 \\
   & ALT & -- & -- & 32 & $\mathbf{88.1 \pm 0.0}$ & 5 \\
   & ALT & -- & -- & 64 & $\mathbf{92.1 \pm 0.0}$ & 3 \\
   & ALT & -- & -- & 128 & $\mathbf{94.2 \pm 0.0}$ & 2 \\
   & ALT & -- & -- & 256 & $\mathbf{95.5 \pm 0.0}$ & 2 \\
\midrule
\rowcolor{aaccolor}    Los Angeles (92K) & AAC-ST & 16 & 8 & 32 & $82.0 \pm 2.2$ & 19 \\
\rowcolor{aaccolor}     & AAC-ST & 32 & 8 & 32 & $81.4 \pm 3.7$ & 20 \\
\rowcolor{aaccolor}     & AAC-ST & 32 & 16 & 64 & $89.9 \pm 2.0$ & 11 \\
\rowcolor{aaccolor}     & AAC-ST & 64 & 16 & 64 & $88.1 \pm 1.1$ & 13 \\
\rowcolor{aaccolor}     & AAC-ST & 64 & 32 & 128 & $93.9 \pm 0.6$ & 7 \\
   & ALT & -- & -- & 32 & $\mathbf{84.2 \pm 0.0}$ & 17 \\
   & ALT & -- & -- & 64 & $\mathbf{91.2 \pm 0.0}$ & 9 \\
   & ALT & -- & -- & 128 & $\mathbf{94.3 \pm 0.0}$ & 6 \\
   & ALT & -- & -- & 256 & $\mathbf{96.0 \pm 0.0}$ & 5 \\
\bottomrule
\end{tabular}
}
\par\smallskip
\begin{minipage}{\linewidth}
\footnotesize\raggedright
We compare expansion reduction (\%) on OSMnx road networks at matched memory budgets. For AAC-ST, we select $K_0$ separately for each seed on validation queries, then evaluate both methods on disjoint test queries. Bold marks the higher mean per graph--budget pair. We observe 0 admissibility violations across all configurations. Entries are mean $\pm$ std over 5 seeds. The full configuration sweep appears below the summary.
\end{minipage}
\end{table}

\subsection{Hybrids and admissibility under early stopping}
\label{app:hybrid-admissibility}

We compare the hybrid $\max(\hA{}, \hFPS{})$ with single methods at matched total memory.
Table~\ref{tab:matched-hybrid} covers non-road graphs; Table~\ref{tab:hybrid} covers directed road graphs.
We report the early-stopping admissibility checks in Table~\ref{tab:admissibility-early-stopping}.

\begin{table}[!htbp]
\centering
\caption{\tabletag{Main} Matched-memory hybrids on non-road graphs.}
\label{tab:matched-hybrid}
\footnotesize
\setlength{\tabcolsep}{3pt}
\adjustbox{max width=\linewidth}{\begin{tabular}{@{}l r r
    S[table-format=2.2(2)] S[table-format=2.2]
    S[table-format=2.2(2)] S[table-format=2.2(2)] S[table-format=2.2(2)]@{}}
\toprule
\textbf{Graph} & \textbf{$\Bmem$ (B/v)} & \textbf{Dij.\ exp.}
  & {\textbf{Pure ALT}}      & {\textbf{Greedy-Max}}
  & {\textbf{Pure AAC-ST}}      & {\textbf{Hybrid-half}}
  & {\textbf{Hybrid ($2\Bmem$)}} \\
\midrule
\multirow{3}{*}{SBM 10K}        &  32 & 5{,}330   & \bfseries 89.86 \pm 0.00 & 90.81 & \cellcolor{aaccolor} 88.78 \pm 0.58 & \cellcolor{aaccolor} 88.42 \pm 0.40 & \cellcolor{aaccolor} 93.56 \pm 0.20 \\
                                &  64 & 5{,}330   & 93.88 \pm 0.00 & 94.87 & \cellcolor{aaccolor} \bfseries 94.01 \pm 0.30 & \cellcolor{aaccolor} 93.56 \pm 0.20 & \cellcolor{aaccolor} 96.46 \pm 0.15 \\
                                & 128 & 5{,}330   & \bfseries 97.01 \pm 0.00 & 97.01 & \cellcolor{aaccolor} 96.81 \pm 0.05 & \cellcolor{aaccolor} 96.46 \pm 0.15 & \cellcolor{aaccolor} 98.17 \pm 0.04 \\
\midrule
\multirow{3}{*}{BA 10K}         &  32 & 4{,}764   & \bfseries 89.23 \pm 0.00 & 87.58 & \cellcolor{aaccolor} 88.29 \pm 0.16 & \cellcolor{aaccolor} 87.75 \pm 0.33 & \cellcolor{aaccolor} 92.99 \pm 0.19 \\
                                &  64 & 4{,}764   & \bfseries 93.85 \pm 0.00 & 93.68 & \cellcolor{aaccolor} 93.81 \pm 0.37 & \cellcolor{aaccolor} 92.99 \pm 0.19 & \cellcolor{aaccolor} 96.35 \pm 0.15 \\
                                & 128 & 4{,}764   & \bfseries 96.81 \pm 0.00 & 96.50 & \cellcolor{aaccolor} 96.46 \pm 0.11 & \cellcolor{aaccolor} 96.35 \pm 0.15 & \cellcolor{aaccolor} 97.95 \pm 0.03 \\
\midrule
\multirow{3}{*}{OGB-arXiv 169K} &  32 & 78{,}617  & 85.40 \pm 0.00 & 86.89 & \cellcolor{aaccolor} \bfseries 86.57 \pm 1.30 & \cellcolor{aaccolor} 85.64 \pm 0.73 & \cellcolor{aaccolor} 91.05 \pm 0.56 \\
                                &  64 & 78{,}617  & 91.81 \pm 0.00 & 93.33 & \cellcolor{aaccolor} \bfseries 92.93 \pm 0.70 & \cellcolor{aaccolor} 91.05 \pm 0.56 & \cellcolor{aaccolor} 95.24 \pm 0.36 \\
                                & 128 & 78{,}617  & 95.52 \pm 0.00 & 96.23 & \cellcolor{aaccolor} \bfseries 95.97 \pm 0.22 & \cellcolor{aaccolor} 95.24 \pm 0.36 & \cellcolor{aaccolor} 97.27 \pm 0.10 \\
\bottomrule
\end{tabular}
}
\par\smallskip
\begin{minipage}{\linewidth}
\footnotesize\raggedright
We compare hybrids on undirected non-road graphs under the memory accounting in Section~\ref{sec:setup}. Pure AAC-ST uses $m{=}\Bmem/4$ with $K_0{=}4m$; pure ALT uses $K{=}\Bmem/4$. Hybrid-half splits memory equally: $m{=}\Bmem/8$ AAC-ST labels and $K{=}\Bmem/8$ ALT landmarks. It returns their pointwise maximum. \par Greedy-Max selects $m$ landmarks from the same $K_0{=}4m$ FPS pool to maximize the average ALT bound. We draw its selection queries disjoint from the test queries, with the same sample size. Its feasible subset gives a lower bound on the pool's optimum; another selector can exceed it. The doubled-memory hybrid $\max(\hA{},\hFPS{})$ runs both full-budget heuristics at deployment. It uses $2\Bmem$ B/v and twice the label lookups. \par Entries are mean expansion reduction over Dijkstra (\%) across 5 seeds per graph--budget pair. ALT has zero standard deviation because FPS is deterministic. Bold compares only pure AAC-ST with pure ALT at $\Bmem$. The greedy and hybrid columns are descriptive. Section~\ref{sec:matched-hybrid} discusses the comparisons.
\end{minipage}
\end{table}

\begin{table}[!htbp]
\centering
\caption{\tabletag{Main} Hybrid heuristics on directed road graphs.}
\label{tab:hybrid}
\small
\adjustbox{max width=\linewidth}{\begin{tabular}{l l c S[table-format=2.1] S[table-format=1.1]}
\toprule
\textbf{Graph} & \textbf{Method} & \textbf{B/v}
  & {\textbf{Exp.\ Red.\ (\%)}} & {\textbf{Std}} \\
\midrule
\multirow{9}{*}{Modena}     & \cellcolor{aaccolor} AAC-ST ($m{=}8$)        & \cellcolor{aaccolor}  32 & \cellcolor{aaccolor} 78.3 & \cellcolor{aaccolor} 3.3 \\
                            & ALT ($K{=}4$)                             &  32 & \bfseries 83.9 & 0.0 \\
\cmidrule{2-5}
                            & \cellcolor{aaccolor} AAC-ST ($m{=}16$)       & \cellcolor{aaccolor}  64 & \cellcolor{aaccolor} 87.1 & \cellcolor{aaccolor} 1.8 \\
                            & \cellcolor{aaccolor} AAC-ST+ALT ($m{=}8, K{=}4$) & \cellcolor{aaccolor}  64 & \cellcolor{aaccolor} 88.2 & \cellcolor{aaccolor} 0.6 \\
                            & ALT ($K{=}8$)                             &  64 & \bfseries 89.0 & 0.0 \\
\cmidrule{2-5}
                            & \cellcolor{aaccolor} AAC-ST+ALT ($m{=}16, K{=}4$) & \cellcolor{aaccolor}  96 & \cellcolor{aaccolor} 90.3 & \cellcolor{aaccolor} 0.8 \\
                            & ALT ($K{=}12$)                            &  96 & \bfseries 90.4 & 0.0 \\
\cmidrule{2-5}
                            & \cellcolor{aaccolor} AAC-ST ($m{=}32$)       & \cellcolor{aaccolor} 128 & \cellcolor{aaccolor} 91.2 & \cellcolor{aaccolor} 1.1 \\
                            & ALT ($K{=}16$)                            & 128 & \bfseries 92.1 & 0.0 \\
\midrule
\multirow{9}{*}{Manhattan}  & \cellcolor{aaccolor} AAC-ST ($m{=}8$)        & \cellcolor{aaccolor}  32 & \cellcolor{aaccolor} 75.2 & \cellcolor{aaccolor} 3.6 \\
                            & ALT ($K{=}4$)                             &  32 & \bfseries 80.6 & 0.0 \\
\cmidrule{2-5}
                            & \cellcolor{aaccolor} AAC-ST ($m{=}16$)       & \cellcolor{aaccolor}  64 & \cellcolor{aaccolor} 82.5 & \cellcolor{aaccolor} 5.1 \\
                            & \cellcolor{aaccolor} AAC-ST+ALT ($m{=}8, K{=}4$) & \cellcolor{aaccolor}  64 & \cellcolor{aaccolor} 85.4 & \cellcolor{aaccolor} 1.7 \\
                            & ALT ($K{=}8$)                             &  64 & \bfseries 87.0 & 0.0 \\
\cmidrule{2-5}
                            & \cellcolor{aaccolor} AAC-ST+ALT ($m{=}16, K{=}4$) & \cellcolor{aaccolor}  96 & \cellcolor{aaccolor} \bfseries 88.8 & \cellcolor{aaccolor} 2.5 \\
                            & ALT ($K{=}12$)                            &  96 & 88.5 & 0.0 \\
\cmidrule{2-5}
                            & \cellcolor{aaccolor} AAC-ST ($m{=}32$)       & \cellcolor{aaccolor} 128 & \cellcolor{aaccolor} \bfseries 91.0 & \cellcolor{aaccolor} 1.1 \\
                            & ALT ($K{=}16$)                            & 128 & 90.0 & 0.0 \\
\midrule
\multirow{9}{*}{NY}         & \cellcolor{aaccolor} AAC-ST ($m{=}8$)        & \cellcolor{aaccolor}  32 & \cellcolor{aaccolor} 81.3 & \cellcolor{aaccolor} 2.5 \\
                            & ALT ($K{=}4$)                             &  32 & \bfseries 85.5 & 0.0 \\
\cmidrule{2-5}
                            & \cellcolor{aaccolor} AAC-ST ($m{=}16$)       & \cellcolor{aaccolor}  64 & \cellcolor{aaccolor} 89.4 & \cellcolor{aaccolor} 0.3 \\
                            & \cellcolor{aaccolor} AAC-ST+ALT ($m{=}8, K{=}4$) & \cellcolor{aaccolor}  64 & \cellcolor{aaccolor} 89.8 & \cellcolor{aaccolor} 0.2 \\
                            & ALT ($K{=}8$)                             &  64 & \bfseries 90.2 & 0.0 \\
\cmidrule{2-5}
                            & \cellcolor{aaccolor} AAC-ST+ALT ($m{=}16, K{=}4$) & \cellcolor{aaccolor}  96 & \cellcolor{aaccolor} 90.9 & \cellcolor{aaccolor} 0.3 \\
                            & ALT ($K{=}12$)                            &  96 & \bfseries 92.8 & 0.0 \\
\cmidrule{2-5}
                            & \cellcolor{aaccolor} AAC-ST ($m{=}32$)       & \cellcolor{aaccolor} 128 & \cellcolor{aaccolor} 92.8 & \cellcolor{aaccolor} 0.4 \\
                            & ALT ($K{=}16$)                            & 128 & \bfseries 93.6 & 0.0 \\
\bottomrule
\end{tabular}
}
\par\smallskip
\begin{minipage}{\linewidth}
\footnotesize\raggedright
Expansion reduction over Dijkstra (\%) on directed road networks, as mean $\pm$ standard deviation over 5 seeds. AAC-ST uses $K_0{=}16,32,64$ for $m{=}8,16,32$, respectively. The hybrid $\max(\hA{},\hFPS{})$ is admissible as a maximum of admissible heuristics. Its budget is the sum of AAC-ST and FPS-ALT storage. Rows are grouped by total budget (B/v) within each graph; \textbf{bold marks the highest reduction in each group}. Pure FPS-ALT expands fewer nodes than the hybrid in 5 of 6 matched-budget comparisons.
\end{minipage}
\end{table}

\begin{table}[!htbp]
\centering
\caption{\tabletag{Ablation} Admissibility under early stopping.}
\label{tab:admissibility-early-stopping}
\small
\adjustbox{max width=\linewidth}{\begin{tabular}{r S[table-format=2.1(1)] c l}
\toprule
\textbf{Epochs} & {\textbf{Exp.\ Red.\ (\%)} $\uparrow$}
  & \textbf{Violations} & \textbf{Status} \\
\midrule
\rowcolor{aaccolor}   1 & 84.9 \pm 2.0 & 0 & \checkmark \\
\rowcolor{aaccolor}   5 & 84.7 \pm 2.1 & 0 & \checkmark \\
\rowcolor{aaccolor}  10 & 85.2 \pm 1.9 & 0 & \checkmark \\
\rowcolor{aaccolor}  50 & 85.3 \pm 1.6 & 0 & \checkmark \\
\rowcolor{aaccolor} 200 & 87.1 \pm 0.3 & 0 & \checkmark \\
\bottomrule
\end{tabular}
}
\par\smallskip
\begin{minipage}{\linewidth}
\footnotesize\raggedright
We train AAC-ST on Modena with $K_0{=}64$, $m{=}16$, and epoch budgets $\{1, 5, 10, 50, 200\}$. Each budget has an independent run for each of 5 seeds, evaluated on 100 queries. We check admissibility by comparing $h(u,t)$ with Dijkstra-computed $d(u,t)$ for all query pairs. We observe zero violations at every budget. These numerical checks agree with the exact-distance guarantee of Proposition~\ref{thm:admissibility}. Expansion reduction is $84.9\%$ at epoch~1 and highest at the longest budget, epoch~200 ($87.1\%$).
\end{minipage}
\end{table}

\subsection{Query distributions, preprocessing, and latency}
\label{app:query-distribution-cost}

We evaluate matched and shifted query distributions in Tables~\ref{tab:query-distribution} and~\ref{tab:query-distribution-cross}.
We compare hybrids with single selectors at $2\Bmem$ bytes per vertex in Table~\ref{tab:hybrid-stability}.
Table~\ref{tab:multi-axis-cost} reports preprocessing and query costs at $\Bmem{=}64$.

\paragraph{Query-distribution sensitivity and shift.} We use three query distributions on the
largest connected component (strongly connected for directed graphs). \emph{Uniform} samples
endpoints independently and uniformly. \emph{Hotspot} samples each endpoint uniformly from an
$8$-vertex neighborhood with probability $0.85$, and uniformly from the component otherwise.
We start at the vertex with the largest in-degree plus out-degree.
The neighborhood contains the first $8$ vertices of a breadth-first search that ignores arc directions.
\emph{Power-law} samples each endpoint with probability proportional to its degree to the power
$1.5$.

We evaluate NY, Manhattan, and Modena at $32$ and $64$\,B/v with $3$ seeds.
Each seed uses $3{,}200$ training queries and $500$ test queries. Uniform and power-law training samples exclude every
test pair. We draw hotspot training and test samples independently, so they can overlap.
The $56$ ordered pairs within the neighborhood carry most of the probability mass.
Every test sample contains all of these pairs. Excluding these pairs from training would remove the hotspot.
We therefore report the share of test queries that repeat a training pair.

The query-adaptive selectors use the same pool and training objective. FPS-ALT uses the query-independent FPS prefix.

With matched training and test distributions, AAC-CG's expansion-reduction margin over FPS-ALT ranges from $-0.005$ to $+2.1$ points on NY.
The ranges are $+3.2$ to $+7.0$ on Manhattan and $+1.1$ to $+2.0$ on Modena.
FPS-ALT expands fewer nodes than AAC-ST in $17$ of the $18$ combinations of graph, budget, and distribution.
The exception is Manhattan hotspot at $64$\,B/v. Under hotspot queries, AAC-ST has the largest
seed variance of the four variants in $5$ of $6$ graph--budget pairs.
Section~\ref{sec:robustness} reports the win counts for greedy and AAC-CG and the results under
distribution shift. All $1{,}296$ records report zero suboptimal paths.

\begin{table}[!htbp]
\centering
\caption{\tabletag{Main} Matched training and test distributions.}
\label{tab:query-distribution}
\small
\setlength{\tabcolsep}{3pt}
\adjustbox{max width=\linewidth}{\begin{tabular}{l r l S[table-format=2.1(3)] S[table-format=2.1(3)] S[table-format=2.1(3)] S[table-format=2.1(3)] S[table-format=+2.1]}
\toprule
\multicolumn{3}{l}{\emph{Train $=$ test}} & \multicolumn{4}{c}{\textbf{Reduction (\%)}} & {\textbf{gain (pp)}} \\
\cmidrule(lr){4-7}
\textbf{Graph} & {\textbf{B/v}} & \textbf{Mode}
  & {FPS-ALT} & {Greedy} & {AAC-ST} & {AAC-CG} & {AAC-CG$-$FPS} \\
\midrule
\multirow{6}{*}{NY}        &  32 & Uniform   & 86.3 \pm 0.3 & \bfseries 87.3 \pm 0.3 & \cellcolor{aaccolor} 79.7 \pm 4.5 & \cellcolor{aaccolor} 86.3 \pm 0.4 & \cellcolor{aaccolor} -0.0 \\
                           &  32 & Hotspot   & 90.8 \pm 0.5 & 91.4 \pm 0.7 & \cellcolor{aaccolor} 86.2 \pm 1.4 & \cellcolor{aaccolor} \bfseries 92.9 \pm 0.2 & \cellcolor{aaccolor} +2.1 \\
                           &  32 & Power-law & 86.9 \pm 0.4 & \bfseries 87.9 \pm 0.6 & \cellcolor{aaccolor} 79.3 \pm 3.8 & \cellcolor{aaccolor} 87.3 \pm 0.1 & \cellcolor{aaccolor} +0.4 \\
                           &  64 & Uniform   & 90.6 \pm 0.3 & \bfseries 91.9 \pm 0.3 & \cellcolor{aaccolor} 89.3 \pm 0.5 & \cellcolor{aaccolor} 91.8 \pm 0.2 & \cellcolor{aaccolor} +1.1 \\
                           &  64 & Hotspot   & 93.4 \pm 0.2 & \bfseries 95.3 \pm 0.0 & \cellcolor{aaccolor} 93.4 \pm 0.8 & \cellcolor{aaccolor} 95.2 \pm 0.4 & \cellcolor{aaccolor} +1.8 \\
                           &  64 & Power-law & 91.1 \pm 0.4 & \bfseries 92.5 \pm 0.3 & \cellcolor{aaccolor} 89.8 \pm 0.3 & \cellcolor{aaccolor} 92.3 \pm 0.7 & \cellcolor{aaccolor} +1.1 \\
\midrule
\multirow{6}{*}{Manhattan} &  32 & Uniform   & 77.3 \pm 0.8 & \bfseries 82.5 \pm 0.2 & \cellcolor{aaccolor} 71.7 \pm 4.4 & \cellcolor{aaccolor} 81.7 \pm 1.3 & \cellcolor{aaccolor} +4.4 \\
                           &  32 & Hotspot   & 72.2 \pm 1.4 & \bfseries 79.8 \pm 1.3 & \cellcolor{aaccolor} 67.4 \pm 5.0 & \cellcolor{aaccolor} 79.1 \pm 0.6 & \cellcolor{aaccolor} +7.0 \\
                           &  32 & Power-law & 77.1 \pm 0.0 & \bfseries 81.7 \pm 1.3 & \cellcolor{aaccolor} 72.1 \pm 5.4 & \cellcolor{aaccolor} 80.7 \pm 0.9 & \cellcolor{aaccolor} +3.6 \\
                           &  64 & Uniform   & 85.2 \pm 0.4 & \bfseries 88.6 \pm 0.4 & \cellcolor{aaccolor} 83.6 \pm 1.8 & \cellcolor{aaccolor} 88.4 \pm 0.3 & \cellcolor{aaccolor} +3.2 \\
                           &  64 & Hotspot   & 83.8 \pm 0.6 & \bfseries 91.0 \pm 0.8 & \cellcolor{aaccolor} 86.2 \pm 3.7 & \cellcolor{aaccolor} 90.8 \pm 0.6 & \cellcolor{aaccolor} +7.0 \\
                           &  64 & Power-law & 84.9 \pm 0.3 & \bfseries 88.3 \pm 0.2 & \cellcolor{aaccolor} 83.0 \pm 1.5 & \cellcolor{aaccolor} 88.2 \pm 0.3 & \cellcolor{aaccolor} +3.3 \\
\midrule
\multirow{6}{*}{Modena}    &  32 & Uniform   & 83.4 \pm 0.7 & \bfseries 86.0 \pm 0.5 & \cellcolor{aaccolor} 79.3 \pm 1.8 & \cellcolor{aaccolor} 84.6 \pm 0.3 & \cellcolor{aaccolor} +1.2 \\
                           &  32 & Hotspot   & 89.4 \pm 1.0 & \bfseries 90.8 \pm 1.3 & \cellcolor{aaccolor} 85.6 \pm 3.4 & \cellcolor{aaccolor} 90.6 \pm 1.0 & \cellcolor{aaccolor} +1.2 \\
                           &  32 & Power-law & 83.4 \pm 0.8 & \bfseries 85.8 \pm 0.7 & \cellcolor{aaccolor} 78.9 \pm 1.2 & \cellcolor{aaccolor} 85.3 \pm 0.4 & \cellcolor{aaccolor} +1.9 \\
                           &  64 & Uniform   & 88.9 \pm 0.5 & \bfseries 91.1 \pm 0.4 & \cellcolor{aaccolor} 86.0 \pm 3.5 & \cellcolor{aaccolor} 90.9 \pm 0.4 & \cellcolor{aaccolor} +2.0 \\
                           &  64 & Hotspot   & 92.6 \pm 0.4 & \bfseries 94.1 \pm 0.1 & \cellcolor{aaccolor} 92.3 \pm 0.4 & \cellcolor{aaccolor} 93.7 \pm 0.4 & \cellcolor{aaccolor} +1.1 \\
                           &  64 & Power-law & 88.9 \pm 0.6 & \bfseries 91.1 \pm 0.3 & \cellcolor{aaccolor} 86.3 \pm 3.2 & \cellcolor{aaccolor} 91.0 \pm 0.5 & \cellcolor{aaccolor} +2.0 \\
\bottomrule
\end{tabular}
}
\par\smallskip
\begin{minipage}{\linewidth}
\footnotesize\raggedright
We compare query-adaptive selectors with the same FPS pool, objective, and 3{,}200 training queries. FPS-ALT uses the query-independent FPS prefix. All variants deploy ALT on $m$ labels, with 500 test queries per seed and 3 seeds. We match memory: $K_0{=}16, m{=}8$ at $32$\,B/v and $K_0{=}32, m{=}16$ at $64$\,B/v. \par Uniform queries sample endpoints independently over the LCC. Power-law queries use endpoint probabilities proportional to degree$^{1.5}$. Hotspot queries concentrate on a breadth-first neighborhood of the highest-degree vertex (Appendix~\ref{app:query-distribution-cost}). Training and test pairs are disjoint except under hotspot, where we draw them independently to preserve the cluster's query mass. Under hotspot, $68.3$\,\% of test queries repeat a training pair. Both endpoints lie in the cluster for $69.1$\,\% of training queries and $68.2$\,\% of test queries. \par With matched distributions (train $=$ test), entries are mean expansion reduction over Dijkstra (\% $\pm$ std across seeds). Bold marks the highest mean. The gain is AAC-CG minus FPS-ALT in percentage points. It is positive in 17 of 18 combinations of graph, budget, and distribution (NY $-0.0$ to $+2.1$; Manhattan $+3.2$ to $+7.0$; Modena $+1.1$ to $+2.0$). \par Table~\ref{tab:query-distribution-cross} reports shifted distributions (train $\neq$ test). We define the shift penalty as the reduction after training on the test distribution minus the reduction after training on another distribution. Entries are mean $\pm$ std over the two mismatched training distributions and seeds. FPS-ALT has zero penalty because its selection is query-independent. \par Across 36 combinations of graph, budget, training distribution, and test distribution, the largest seed-mean penalties are 5.4\,pp for greedy and 1.9\,pp for AAC-CG. Averaging over the two mismatched training distributions gives maximum row means of 2.7 and 1.8\,pp, respectively. Under shift, AAC-CG expands fewer nodes than FPS-ALT in 31 of 36 combinations; the companion table's last column reports this gain. All 1{,}296 records report zero suboptimal paths. Table~\ref{tab:hybrid-stability} compares hybrids and single selectors at $2\Bmem$.
\end{minipage}
\end{table}
\begin{table}[!htbp]
\centering
\caption{\tabletag{Main} Shifted training and test distributions.}
\label{tab:query-distribution-cross}
\small
\setlength{\tabcolsep}{3pt}
\adjustbox{max width=\linewidth}{\begin{tabular}{l r l S[table-format=+1.1(1)] S[table-format=+1.1(1)] S[table-format=+2.1(1)]}
\toprule
\multicolumn{3}{l}{\emph{Train $\neq$ test}} & \multicolumn{2}{c}{\textbf{Shift penalty (pp)}} & {\textbf{AAC-CG gain}} \\
\cmidrule(lr){4-5}
\textbf{Graph} & {\textbf{B/v}} & \textbf{Test mode} & {Greedy} & {AAC-CG} & {AAC-CG$-$FPS (pp)} \\
\midrule
\multirow{6}{*}{NY}        &  32 & Uniform   & 2.7 \pm 3.0 & \cellcolor{aaccolor} 0.4 \pm 1.4 & \cellcolor{aaccolor} -0.4 \pm 1.4 \\
                           &  32 & Hotspot   & 0.5 \pm 0.6 & \cellcolor{aaccolor} 1.6 \pm 0.4 & \cellcolor{aaccolor} 0.5 \pm 0.3 \\
                           &  32 & Power-law & 2.3 \pm 2.6 & \cellcolor{aaccolor} 1.1 \pm 1.1 & \cellcolor{aaccolor} -0.7 \pm 1.0 \\
                           &  64 & Uniform   & 0.5 \pm 0.9 & \cellcolor{aaccolor} 0.2 \pm 0.4 & \cellcolor{aaccolor} 1.0 \pm 0.4 \\
                           &  64 & Hotspot   & 0.6 \pm 0.5 & \cellcolor{aaccolor} 1.2 \pm 0.6 & \cellcolor{aaccolor} 0.6 \pm 0.3 \\
                           &  64 & Power-law & 0.7 \pm 0.9 & \cellcolor{aaccolor} -0.0 \pm 0.3 & \cellcolor{aaccolor} 1.2 \pm 0.3 \\
\midrule
\multirow{6}{*}{Manhattan} &  32 & Uniform   & 1.2 \pm 0.8 & \cellcolor{aaccolor} 1.0 \pm 1.0 & \cellcolor{aaccolor} 3.4 \pm 1.0 \\
                           &  32 & Hotspot   & -0.0 \pm 0.8 & \cellcolor{aaccolor} 0.2 \pm 1.1 & \cellcolor{aaccolor} 6.8 \pm 1.8 \\
                           &  32 & Power-law & 0.1 \pm 0.9 & \cellcolor{aaccolor} -0.4 \pm 0.8 & \cellcolor{aaccolor} 4.0 \pm 1.0 \\
                           &  64 & Uniform   & 1.1 \pm 1.2 & \cellcolor{aaccolor} 0.7 \pm 1.0 & \cellcolor{aaccolor} 2.5 \pm 1.1 \\
                           &  64 & Hotspot   & -1.0 \pm 0.4 & \cellcolor{aaccolor} -1.2 \pm 0.1 & \cellcolor{aaccolor} 8.3 \pm 1.1 \\
                           &  64 & Power-law & 1.2 \pm 1.4 & \cellcolor{aaccolor} 0.9 \pm 1.2 & \cellcolor{aaccolor} 2.4 \pm 1.1 \\
\midrule
\multirow{6}{*}{Modena}    &  32 & Uniform   & 1.5 \pm 1.7 & \cellcolor{aaccolor} 0.3 \pm 0.7 & \cellcolor{aaccolor} 0.9 \pm 1.0 \\
                           &  32 & Hotspot   & -0.3 \pm 0.9 & \cellcolor{aaccolor} 1.8 \pm 1.4 & \cellcolor{aaccolor} -0.6 \pm 1.3 \\
                           &  32 & Power-law & 1.8 \pm 1.9 & \cellcolor{aaccolor} 1.0 \pm 0.9 & \cellcolor{aaccolor} 1.0 \pm 1.0 \\
                           &  64 & Uniform   & 0.5 \pm 0.6 & \cellcolor{aaccolor} 0.7 \pm 0.8 & \cellcolor{aaccolor} 1.3 \pm 0.8 \\
                           &  64 & Hotspot   & 0.3 \pm 0.1 & \cellcolor{aaccolor} -0.1 \pm 0.2 & \cellcolor{aaccolor} 1.2 \pm 0.2 \\
                           &  64 & Power-law & 0.3 \pm 0.6 & \cellcolor{aaccolor} 0.6 \pm 0.6 & \cellcolor{aaccolor} 1.4 \pm 0.6 \\
\bottomrule
\end{tabular}
}
\par\smallskip
\begin{minipage}{\linewidth}
\footnotesize\raggedright
Graphs, memory budgets, query distributions, and metric definitions are those of Table~\ref{tab:query-distribution}.
\end{minipage}
\end{table}

\begin{table}[!htbp]
\centering
\caption{\tabletag{Main} Hybrids at doubled memory across query distributions.}
\label{tab:hybrid-stability}
\footnotesize
\setlength{\tabcolsep}{2.5pt}
\adjustbox{max width=\linewidth}{\begin{tabular}{l r l S[table-format=2.1] S[table-format=2.1] S[table-format=2.1] S[table-format=2.1] S[table-format=2.1] S[table-format=2.1] S[table-format=2.1] S[table-format=2.1]}
\toprule
& & & \multicolumn{4}{c}{\textbf{Reduction at $2\Bmem$ (\%)}} & \multicolumn{4}{c}{\textbf{Std across seeds}} \\
\cmidrule(lr){4-7}\cmidrule(lr){8-11}
\textbf{Graph} & {\textbf{$\Bmem$}} & \textbf{Mode}
  & {FPS-ALT} & {AAC-ST+ALT} & {AAC-CG+ALT} & {AAC-CG}
  & {FPS-ALT} & {AAC-ST+ALT} & {AAC-CG+ALT} & {AAC-CG} \\
\midrule
\multirow{6}{*}{NY}        &  32 & Uniform   & 90.6 & \cellcolor{aaccolor} 90.0 & \cellcolor{aaccolor} 87.9 & \cellcolor{aaccolor} \bfseries 91.8 & 0.3 & \cellcolor{aaccolor} 0.5 & \cellcolor{aaccolor} 0.9 & \cellcolor{aaccolor} \bfseries 0.2 \\
                           &  32 & Hotspot   & 93.4 & \cellcolor{aaccolor} 94.0 & \cellcolor{aaccolor} 94.9 & \cellcolor{aaccolor} \bfseries 95.2 & \bfseries 0.2 & \cellcolor{aaccolor} 0.3 & \cellcolor{aaccolor} 0.3 & \cellcolor{aaccolor} 0.4 \\
                           &  32 & Power-law & 91.1 & \cellcolor{aaccolor} 90.6 & \cellcolor{aaccolor} 89.4 & \cellcolor{aaccolor} \bfseries 92.3 & \bfseries 0.4 & \cellcolor{aaccolor} 0.6 & \cellcolor{aaccolor} 0.9 & \cellcolor{aaccolor} 0.7 \\
                           &  64 & Uniform   & 94.0 & \cellcolor{aaccolor} 92.8 & \cellcolor{aaccolor} 92.9 & \cellcolor{aaccolor} \bfseries 94.3 & 0.2 & \cellcolor{aaccolor} 0.3 & \cellcolor{aaccolor} 0.2 & \cellcolor{aaccolor} \bfseries 0.2 \\
                           &  64 & Hotspot   & 96.1 & \cellcolor{aaccolor} 95.4 & \cellcolor{aaccolor} 95.9 & \cellcolor{aaccolor} \bfseries 96.4 & 0.2 & \cellcolor{aaccolor} 0.4 & \cellcolor{aaccolor} 0.3 & \cellcolor{aaccolor} \bfseries 0.1 \\
                           &  64 & Power-law & 94.4 & \cellcolor{aaccolor} 93.4 & \cellcolor{aaccolor} 93.6 & \cellcolor{aaccolor} \bfseries 95.0 & 0.3 & \cellcolor{aaccolor} 0.6 & \cellcolor{aaccolor} 0.5 & \cellcolor{aaccolor} \bfseries 0.2 \\
\midrule
\multirow{6}{*}{Manhattan} &  32 & Uniform   & 85.2 & \cellcolor{aaccolor} 82.2 & \cellcolor{aaccolor} 84.7 & \cellcolor{aaccolor} \bfseries 88.4 & 0.4 & \cellcolor{aaccolor} 1.7 & \cellcolor{aaccolor} 1.8 & \cellcolor{aaccolor} \bfseries 0.3 \\
                           &  32 & Hotspot   & 83.8 & \cellcolor{aaccolor} 78.9 & \cellcolor{aaccolor} 82.1 & \cellcolor{aaccolor} \bfseries 90.8 & \bfseries 0.6 & \cellcolor{aaccolor} 3.2 & \cellcolor{aaccolor} 0.7 & \cellcolor{aaccolor} 0.6 \\
                           &  32 & Power-law & 84.9 & \cellcolor{aaccolor} 82.2 & \cellcolor{aaccolor} 84.1 & \cellcolor{aaccolor} \bfseries 88.2 & 0.3 & \cellcolor{aaccolor} 2.0 & \cellcolor{aaccolor} 0.6 & \cellcolor{aaccolor} \bfseries 0.3 \\
                           &  64 & Uniform   & 89.3 & \cellcolor{aaccolor} 89.9 & \cellcolor{aaccolor} 89.7 & \cellcolor{aaccolor} \bfseries 93.0 & 0.1 & \cellcolor{aaccolor} 0.2 & \cellcolor{aaccolor} 0.2 & \cellcolor{aaccolor} \bfseries 0.1 \\
                           &  64 & Hotspot   & 87.0 & \cellcolor{aaccolor} 91.0 & \cellcolor{aaccolor} 91.7 & \cellcolor{aaccolor} \bfseries 94.0 & \bfseries 0.4 & \cellcolor{aaccolor} 2.2 & \cellcolor{aaccolor} 0.4 & \cellcolor{aaccolor} 0.9 \\
                           &  64 & Power-law & 89.4 & \cellcolor{aaccolor} 89.4 & \cellcolor{aaccolor} 89.4 & \cellcolor{aaccolor} \bfseries 92.8 & 0.2 & \cellcolor{aaccolor} 0.3 & \cellcolor{aaccolor} \bfseries 0.1 & \cellcolor{aaccolor} 0.1 \\
\midrule
\multirow{6}{*}{Modena}    &  32 & Uniform   & 88.9 & \cellcolor{aaccolor} 88.2 & \cellcolor{aaccolor} 88.3 & \cellcolor{aaccolor} \bfseries 90.9 & 0.5 & \cellcolor{aaccolor} 0.9 & \cellcolor{aaccolor} 0.5 & \cellcolor{aaccolor} \bfseries 0.4 \\
                           &  32 & Hotspot   & 92.6 & \cellcolor{aaccolor} 92.1 & \cellcolor{aaccolor} 92.5 & \cellcolor{aaccolor} \bfseries 93.7 & \bfseries 0.4 & \cellcolor{aaccolor} 0.6 & \cellcolor{aaccolor} 1.2 & \cellcolor{aaccolor} 0.4 \\
                           &  32 & Power-law & 88.9 & \cellcolor{aaccolor} 88.2 & \cellcolor{aaccolor} 88.5 & \cellcolor{aaccolor} \bfseries 91.0 & 0.6 & \cellcolor{aaccolor} \bfseries 0.1 & \cellcolor{aaccolor} 0.4 & \cellcolor{aaccolor} 0.5 \\
                           &  64 & Uniform   & 92.3 & \cellcolor{aaccolor} 91.4 & \cellcolor{aaccolor} 91.9 & \cellcolor{aaccolor} \bfseries 93.6 & 0.3 & \cellcolor{aaccolor} 0.8 & \cellcolor{aaccolor} 0.4 & \cellcolor{aaccolor} \bfseries 0.3 \\
                           &  64 & Hotspot   & 95.1 & \cellcolor{aaccolor} 94.7 & \cellcolor{aaccolor} 94.7 & \cellcolor{aaccolor} \bfseries 95.4 & 0.1 & \cellcolor{aaccolor} 0.3 & \cellcolor{aaccolor} 0.3 & \cellcolor{aaccolor} \bfseries 0.1 \\
                           &  64 & Power-law & 92.5 & \cellcolor{aaccolor} 91.7 & \cellcolor{aaccolor} 92.0 & \cellcolor{aaccolor} \bfseries 93.8 & \bfseries 0.3 & \cellcolor{aaccolor} 0.6 & \cellcolor{aaccolor} 0.3 & \cellcolor{aaccolor} 0.4 \\
\bottomrule
\end{tabular}
}
\par\smallskip
\begin{minipage}{\linewidth}
\footnotesize\raggedright
We compare two hybrids with single selectors at the same total memory, $2\Bmem$ B/v. AAC-ST+ALT deploys $\max(\hA{},\hFPS{})$; AAC-CG+ALT deploys $\max(h_{\mathrm{CG}},\hFPS{})$. Each combines two $\Bmem$-budget heuristics. The single FPS-ALT uses $2m$ labels: $m$ landmarks stored in both directions. The single AAC-CG uses $2m$ labels from a $2K_0$ pool. We use the graphs, budgets, and matched query distributions from Table~\ref{tab:query-distribution}, with 3 seeds. \par The left panel reports mean expansion reduction (\%); bold marks the highest mean. The right panel reports standard deviation across seeds; bold marks the lowest value. The single AAC-CG has higher mean reduction than AAC-ST+ALT in 18 of 18 combinations and lower standard deviation in 15. The single FPS-ALT has lower standard deviation than AAC-ST+ALT in 17 combinations. Averaging the standard deviations over combinations gives FPS-ALT 0.32, AAC-ST+ALT 0.87, AAC-CG+ALT 0.56, AAC-CG 0.34\,pp.
\end{minipage}
\end{table}

\paragraph{Multi-axis cost breakdown.} We match deployed label memory, but measure preprocessing time and query latency separately.
Table~\ref{tab:multi-axis-cost} reports these costs on four DIMACS graphs at $\Bmem{=}64$\,B/v.
This separation follows learned-distance-index comparisons~\citep{choudhary2026empirical}.

AAC-ST preprocessing uses $2(K_0{+}1)$ SSSPs plus compressor training; FPS-ALT uses $2(K{+}1)$ SSSPs.
Both compute forward and reverse distances. The extra pair of SSSPs selects the first landmark from the start vertex.
AAC-ST takes $7.3$ times as long as FPS-ALT on FLA and $10.1$ times as long on NY.
Teacher-distance computation takes $75$ to $96$\,\% of AAC-ST's preprocessing time, so pool size accounts for most of this cost.
The latency columns use the measurements from Section~\ref{sec:robustness}.

In the shared-layout study (Table~\ref{tab:latency-layout}), AAC-CG uses a $K_0{=}32$ pool and
$T{=}100$ steps. At $64$\,B/v, selection takes $6.5$\,\% (FLA) to $23$\,\% (NY) of its
preprocessing time; teacher-distance computation takes the rest. AAC-ST and AAC-CG each store
$8$ floats per direction, for $m{=}16$ in total. Under this shared layout, AAC-ST's higher
expansion count is accompanied by higher mean query latency than FPS-ALT's.

\begin{table}[!htbp]
\centering
\caption{\tabletag{Main} Preprocessing, memory, and query costs.}
\label{tab:multi-axis-cost}
\small
\adjustbox{max width=\linewidth}{\begin{tabular}{l l r
    S[table-format=2.1]
    S[table-format=2.1]
    S[table-format=3.1]
    S[table-format=4.1]}
\toprule
\textbf{Graph} & \textbf{Method} & \textbf{Online (B/v)}
  & {\textbf{Offline (s)}} & {\textbf{AAC-ST/ALT}} & {\textbf{p50 query (ms)}} & {\textbf{p95 query (ms)}} \\
\midrule
\multirow{2}{*}{NY}  & \cellcolor{aaccolor} AAC-ST      & \cellcolor{aaccolor} 64 & \cellcolor{aaccolor} 3.8 & \cellcolor{aaccolor} 10.1 & \cellcolor{aaccolor} 18.3 & \cellcolor{aaccolor} 84.4 \\
                     & FPS-ALT                       & 64 & \bfseries 0.4 & {--} & \bfseries 12.3 & \bfseries 64.6 \\
\multirow{2}{*}{BAY} & \cellcolor{aaccolor} AAC-ST      & \cellcolor{aaccolor} 64 & \cellcolor{aaccolor} 3.7 & \cellcolor{aaccolor} 8.2 & \cellcolor{aaccolor} \bfseries 19.0 & \cellcolor{aaccolor} 118.4 \\
                     & FPS-ALT                       & 64 & \bfseries 0.5 & {--} & 21.8 & \bfseries 91.7 \\
\multirow{2}{*}{COL} & \cellcolor{aaccolor} AAC-ST      & \cellcolor{aaccolor} 64 & \cellcolor{aaccolor} 5.0 & \cellcolor{aaccolor} 7.7 & \cellcolor{aaccolor} 41.5 & \cellcolor{aaccolor} 183.6 \\
                     & FPS-ALT                       & 64 & \bfseries 0.6 & {--} & \bfseries 30.2 & \bfseries 158.3 \\
\multirow{2}{*}{FLA} & \cellcolor{aaccolor} AAC-ST      & \cellcolor{aaccolor} 64 & \cellcolor{aaccolor} 12.2 & \cellcolor{aaccolor} 7.3 & \cellcolor{aaccolor} 92.7 & \cellcolor{aaccolor} 649.0 \\
                     & FPS-ALT                       & 64 & \bfseries 1.7 & {--} & \bfseries 73.5 & \bfseries 319.9 \\
\bottomrule
\end{tabular}
}
\par\smallskip
\begin{minipage}{\linewidth}
\footnotesize\raggedright
We compare costs on DIMACS at $\Bmem{=}64$\,B/v. AAC-ST uses $K_0{=}64,m{=}16$; ALT uses $K{=}8$. Online memory is deployed label memory per vertex, matched across methods. Offline preprocessing uses $2(K_0{+}1)$ Dijkstra runs plus compressor training for AAC-ST, and $2(K{+}1)$ Dijkstra runs for ALT. Both counts include forward and reverse searches. The extra pair locates the first landmark from the start vertex. FPS retains these computed labels, so neither method repeats a landmark search. \par Query latency comes from Table~\ref{tab:latency}: p50 and p95 over 100 queries. The \textbf{AAC-ST/ALT} column divides AAC-ST's offline time by ALT's, using unrounded seconds. Bold marks the smaller value for each graph and cost metric.
\end{minipage}
\end{table}

\subsection{DIMACS configuration sweep and paired tests}
\label{app:dimacs-retrospective}

We list every AAC-ST $(K_0, m)$ configuration beside ALT at matched label memory on NY and FLA (Table~\ref{tab:main_results}).
This table reports the full sweep. Table~\ref{tab:valsplit} evaluates the validation-selected $K_0$ on disjoint test queries.

Table~\ref{tab:dimacs-wilcoxon-percell} reports one Wilcoxon $p$-value per graph--budget pair, using the seed-averaged difference for each query.
It also reports crossed seed-by-query intervals for the mean difference, followed by effect sizes.

\begin{table}[!htbp]
\centering
\caption{\tabletag{Main} DIMACS configuration sweep.}
\label{tab:main_results}
\small
\adjustbox{max width=\linewidth}{\begin{tabular}{l l l r S[table-format=2.1(1)] S[table-format=6.0] S[table-format=3.1]}
\toprule
\textbf{Graph} & \textbf{Method} & \textbf{Config} & \textbf{B/v}
  & {\textbf{Exp.\ Red.\ (\%)}} & {\textbf{Mean Exp.}} & {\textbf{Preproc.\ (s)}} \\
\midrule
\multirow{12}{*}{NY} & \cellcolor{aaccolor} AAC-ST & \cellcolor{aaccolor} $K_0{=}32$ & \cellcolor{aaccolor}  32 & \cellcolor{aaccolor} 79.9 \pm 1.9 & \cellcolor{aaccolor}  26742 & \cellcolor{aaccolor}   2.3 \\
                     & \cellcolor{aaccolor} AAC-ST & \cellcolor{aaccolor} $K_0{=}64$ & \cellcolor{aaccolor}  32 & \cellcolor{aaccolor} 81.8 \pm 1.7 & \cellcolor{aaccolor}  24333 & \cellcolor{aaccolor}   3.7 \\
                     & \cellcolor{aaccolor} AAC-ST & \cellcolor{aaccolor} $K_0{=}128$ & \cellcolor{aaccolor}  32 & \cellcolor{aaccolor} 79.1 \pm 3.5 & \cellcolor{aaccolor}  27888 & \cellcolor{aaccolor}   7.5 \\
                     & ALT & $K{=}4$ &  32 & \bfseries 85.5 \pm 0.0 &  19277 &   0.3 \\
                     & \cellcolor{aaccolor} AAC-ST & \cellcolor{aaccolor} $K_0{=}32$ & \cellcolor{aaccolor}  64 & \cellcolor{aaccolor} 89.4 \pm 0.3 & \cellcolor{aaccolor}  14122 & \cellcolor{aaccolor}   2.2 \\
                     & \cellcolor{aaccolor} AAC-ST & \cellcolor{aaccolor} $K_0{=}64$ & \cellcolor{aaccolor}  64 & \cellcolor{aaccolor} 88.6 \pm 0.8 & \cellcolor{aaccolor}  15242 & \cellcolor{aaccolor}   3.7 \\
                     & \cellcolor{aaccolor} AAC-ST & \cellcolor{aaccolor} $K_0{=}128$ & \cellcolor{aaccolor}  64 & \cellcolor{aaccolor} 87.8 \pm 1.4 & \cellcolor{aaccolor}  16207 & \cellcolor{aaccolor}   7.5 \\
                     & ALT & $K{=}8$ &  64 & \bfseries 90.2 \pm 0.0 &  13032 &   0.5 \\
                     & \cellcolor{aaccolor} AAC-ST & \cellcolor{aaccolor} $K_0{=}32$ & \cellcolor{aaccolor} 128 & \cellcolor{aaccolor} 93.2 \pm 0.3 & \cellcolor{aaccolor}   9042 & \cellcolor{aaccolor}   2.2 \\
                     & \cellcolor{aaccolor} AAC-ST & \cellcolor{aaccolor} $K_0{=}64$ & \cellcolor{aaccolor} 128 & \cellcolor{aaccolor} 92.8 \pm 0.4 & \cellcolor{aaccolor}   9539 & \cellcolor{aaccolor}   3.8 \\
                     & \cellcolor{aaccolor} AAC-ST & \cellcolor{aaccolor} $K_0{=}128$ & \cellcolor{aaccolor} 128 & \cellcolor{aaccolor} 92.3 \pm 0.4 & \cellcolor{aaccolor}  10280 & \cellcolor{aaccolor}   7.6 \\
                     & ALT & $K{=}16$ & 128 & \bfseries 93.6 \pm 0.0 &   8527 &   0.9 \\
\midrule
\multirow{12}{*}{FLA} & \cellcolor{aaccolor} AAC-ST & \cellcolor{aaccolor} $K_0{=}32$ & \cellcolor{aaccolor}  32 & \cellcolor{aaccolor} 76.8 \pm 5.4 & \cellcolor{aaccolor} 119050 & \cellcolor{aaccolor}   6.7 \\
                     & \cellcolor{aaccolor} AAC-ST & \cellcolor{aaccolor} $K_0{=}64$ & \cellcolor{aaccolor}  32 & \cellcolor{aaccolor} 77.3 \pm 2.6 & \cellcolor{aaccolor} 116541 & \cellcolor{aaccolor}  13.3 \\
                     & \cellcolor{aaccolor} AAC-ST & \cellcolor{aaccolor} $K_0{=}128$ & \cellcolor{aaccolor}  32 & \cellcolor{aaccolor} 72.0 \pm 3.8 & \cellcolor{aaccolor} 143751 & \cellcolor{aaccolor}  28.2 \\
                     & ALT & $K{=}4$ &  32 & \bfseries 84.6 \pm 0.0 &  78886 &   1.0 \\
                     & \cellcolor{aaccolor} AAC-ST & \cellcolor{aaccolor} $K_0{=}32$ & \cellcolor{aaccolor}  64 & \cellcolor{aaccolor} 84.2 \pm 4.9 & \cellcolor{aaccolor}  80986 & \cellcolor{aaccolor}   6.8 \\
                     & \cellcolor{aaccolor} AAC-ST & \cellcolor{aaccolor} $K_0{=}64$ & \cellcolor{aaccolor}  64 & \cellcolor{aaccolor} 84.0 \pm 2.2 & \cellcolor{aaccolor}  82281 & \cellcolor{aaccolor}  13.4 \\
                     & \cellcolor{aaccolor} AAC-ST & \cellcolor{aaccolor} $K_0{=}128$ & \cellcolor{aaccolor}  64 & \cellcolor{aaccolor} 84.3 \pm 3.8 & \cellcolor{aaccolor}  80358 & \cellcolor{aaccolor}  28.2 \\
                     & ALT & $K{=}8$ &  64 & \bfseries 87.8 \pm 0.0 &  62816 &   1.9 \\
                     & \cellcolor{aaccolor} AAC-ST & \cellcolor{aaccolor} $K_0{=}32$ & \cellcolor{aaccolor} 128 & \cellcolor{aaccolor} 91.8 \pm 0.4 & \cellcolor{aaccolor}  42282 & \cellcolor{aaccolor}   7.0 \\
                     & \cellcolor{aaccolor} AAC-ST & \cellcolor{aaccolor} $K_0{=}64$ & \cellcolor{aaccolor} 128 & \cellcolor{aaccolor} 91.4 \pm 1.1 & \cellcolor{aaccolor}  44269 & \cellcolor{aaccolor}  13.5 \\
                     & \cellcolor{aaccolor} AAC-ST & \cellcolor{aaccolor} $K_0{=}128$ & \cellcolor{aaccolor} 128 & \cellcolor{aaccolor} 90.1 \pm 0.8 & \cellcolor{aaccolor}  50755 & \cellcolor{aaccolor}  28.3 \\
                     & ALT & $K{=}16$ & 128 & \bfseries 92.5 \pm 0.0 &  38528 &   3.8 \\
\bottomrule
\end{tabular}
}
\par\smallskip
\begin{minipage}{\linewidth}
\footnotesize\raggedright
We report every AAC-ST $(K_0,m)$ configuration and ALT at matched deployed label memory on DIMACS NY and FLA. Entries are means $\pm$ std across 5 seeds. We select no $K_0$ in this table. Table~\ref{tab:valsplit} reports validation-selected configurations on disjoint test queries. Bold marks the highest expansion reduction per graph--budget pair. Table~\ref{tab:dimacs-wilcoxon-percell} gives a separate controlled Wilcoxon comparison.
\end{minipage}
\end{table}

\begin{table}[!htbp]
\centering
\caption{\tabletag{Main} Paired comparisons on DIMACS.}
\label{tab:dimacs-wilcoxon-percell}
\small
\setlength{\tabcolsep}{4pt}
\adjustbox{max width=\linewidth}{\begin{tabular}{l r S[table-format=+6.0] l c S[table-format=+6.0] l S[table-format=1.1e+2] l}
\toprule
\textbf{Graph} & \textbf{B/v} & {\textbf{$\Delta$}} & \textbf{crossed 95\,\% CI} & \textbf{ALT seeds}
  & {\textbf{HL}} & \textbf{95\,\% CI} & {\textbf{$p$}} & \textbf{Favors} \\
\midrule
\multirow{3}{*}{NY}  &  32 &  -5057 & [-8120, -1993]     & 5/5 &  +3094 & [1637, 5333]       & 1.9e-6  & ALT \\
                     &  64 &  -2210 & [-3622, -799]      & 5/5 &  +1735 & [1066, 2744]       & 2.5e-7  & ALT \\
                     & 128 &  -1012 & [-1865, -159]      & 5/5 &   +792 & [422, 1356]        & 3.8e-5  & ALT \\
\midrule
\multirow{3}{*}{BAY} &  32 & -22343 & [-34294, -10392]   & 5/5 & +15329 & [8660, 22384]      & 1.4e-11 & ALT \\
                     &  64 &  -6660 & [-12444, -876]     & 5/5 &  +3743 & [2271, 7054]       & 1.8e-9  & ALT \\
                     & 128 &  -1990 & [-4811, 832]       & 4/5 &  +1206 & [493, 2147]        & 3.0e-4  & ALT \\
\midrule
\multirow{3}{*}{COL} &  32 & -21635 & [-28611, -14659]   & 5/5 & +16777 & [12846, 22182]     & 6.6e-15 & ALT \\
                     &  64 & -10243 & [-14247, -6239]    & 5/5 &  +7535 & [4521, 11329]      & 4.3e-12 & ALT \\
                     & 128 &  -1800 & [-4929, 1328]      & 4/5 &  +1288 & [325, 2565]        & 2.7e-3  & ALT \\
\midrule
\multirow{3}{*}{FLA} &  32 & -37655 & [-58558, -16752]   & 5/5 & +18464 & [10556, 30731]     & 1.5e-7  & ALT \\
                     &  64 & -19465 & [-36887, -2044]    & 4/5 &  +5556 & [751, 11315]       & 1.2e-2  & ALT \\
                     & 128 &  -5741 & [-12682, 1200]     & 5/5 &  +2948 & [888, 6827]        & 2.8e-3  & ALT \\
\bottomrule
\end{tabular}
}
\par\vspace{7pt}
{\small \textbf{Effect sizes.} The following block uses the single-seed DIMACS timing study of Table~\ref{tab:latency} (AAC-ST vs.\ FPS-ALT at $\Bmem{=}64$\,B/v, 100 queries per graph), a separate training run from the block above. We report the matched-pairs rank-biserial correlation $r$ and the Hodges--Lehmann paired shift (AAC-ST $-$ ALT; positive means more expansions or longer latency for AAC-ST), with a 10{,}000-resample percentile bootstrap 95\,\% CI over queries. These intervals condition on one trained model per graph. On expansions ALT's advantage is significant on 4 of 4 graphs with $r$ from $+0.24$ to $+0.52$ (small-to-large effects) and a median shift of 1{,}346 to 11{,}497 expansions; the CI excludes zero on 4 graphs. The latency rows time both methods through the single heuristic factory of Table~\ref{tab:latency}.}\par\vspace{4pt}
\small
\setlength{\tabcolsep}{4pt}
\adjustbox{max width=\linewidth}{\begin{tabular}{l l r S[table-format=5.1] S[table-format=5.1] S[table-format=1.1e-1] S[table-format=+1.2] l}
\toprule
\textbf{Graph} & \textbf{Metric} & \textbf{$n$} & {\textbf{median AAC-ST}} & {\textbf{median ALT}}
  & {\textbf{Wilcoxon $p$}} & {\textbf{$r_{\mathrm{rb}}$}} & \textbf{HL shift [95\,\% CI]} \\
\midrule
\multirow{2}{*}{NY}  & expansions                       & 100 & \cellcolor{aaccolor} 9692 &     6713 & 1.1e-2  & +0.29 & $1346$ [243, 3613] \\
                     & latency (ms)                     & 100 & \cellcolor{aaccolor} 18.3 &     12.3 & 1.4e-2  & +0.28 & $2.4$ [0.4, 6.8] \\
\midrule
\multirow{2}{*}{BAY} & expansions                       & 100 & \cellcolor{aaccolor} 9883 &    11212 & 3.6e-2  & +0.24 & $1555$ [52, 4524] \\
                     & latency (ms)                     & 100 & \cellcolor{aaccolor} 19.0 &     21.8 & 2.5e-2  & +0.26 & $3.2$ [0.2, 8.3] \\
\midrule
\multirow{2}{*}{COL} & expansions                       & 100 & \cellcolor{aaccolor} 23183 &    16432 & 8.3e-6  & +0.52 & $5998$ [2313, 10688] \\
                     & latency (ms)                     & 100 & \cellcolor{aaccolor} 41.5 &     30.2 & 1.2e-5  & +0.50 & $9.8$ [3.9, 17.3] \\
\midrule
\multirow{2}{*}{FLA} & expansions                       & 100 & \cellcolor{aaccolor} 49424 &    36706 & 3.9e-4  & +0.41 & $11497$ [3292, 21340] \\
                     & latency (ms)                     & 100 & \cellcolor{aaccolor} 92.7 &     73.5 & 1.6e-4  & +0.43 & $22.0$ [6.5, 39.0] \\
\bottomrule
\end{tabular}
}
\par\smallskip
\begin{minipage}{\linewidth}
\footnotesize\raggedright
We compare AAC-ST with FPS-ALT at matched deployed label memory for each DIMACS graph--budget pair (100 queries, 5 seeds). The seeds share one query sample, and we run ALT once per budget. We therefore average the per-query differences over seeds before testing. For each graph--budget pair, we apply a two-sided Wilcoxon signed-rank test over 100 queries, with Benjamini--Hochberg FDR $q{=}0.05$ across pairs. \par \textbf{HL} is the Hodges--Lehmann estimate of AAC-ST expansions minus ALT expansions. Positive values mean AAC-ST expands more nodes. Its percentile interval uses 10{,}000 bootstrap resamples over queries. \textbf{$\Delta$} is the mean expansion difference, ALT minus AAC-ST; negative values favor ALT. Its crossed random-effects 95\,\% $t$ interval accounts for seed variation, query variation, and shared-query covariance. It excludes zero in ALT's favor in 9 of 12 graph--budget pairs. \par \textbf{ALT seeds} counts how many of the 5 seeds ALT won. Across 12 graph--budget pairs, the Wilcoxon test favors ALT in 12, AAC-ST in 0, and neither in 0. AAC-ST deploys $K_0{=}64, m{=}8$ at $32$\,B/v; $K_0{=}64, m{=}16$ at $64$\,B/v; $K_0{=}64, m{=}32$ at $128$\,B/v, one $K_0{=}64$ pool at every budget. Section~\ref{sec:dimacs} gives the limitations.
\end{minipage}
\end{table}

\subsection{Mechanism diagnostics: ALT-pool, coverage-aware regularizer, training drift}
\label{app:mechanism-diagnostics}

We test whether AAC-ST's expansion deficit comes from the architecture, the training objective, or the training settings.
The ALT-pool control tests whether the architecture can reproduce FPS-ALT.
We then add a coverage penalty, repeat the drift study on a road graph, and vary the hyperparameters.

\paragraph{ALT-pool control.} We evaluate FPS-ALT on the first $K$ landmarks of AAC-ST's $K_0$-landmark pool (Section~\ref{sec:training-drift}).
We use $K{=}m$ on undirected graphs and $K{=}m/2$ on directed graphs, storing both directions in the latter case.
The control matches pure FPS-ALT to two decimal places at every tested budget across eight directed and undirected graphs.
Thus the architecture can reproduce FPS-ALT at matched memory.

\subsubsection{Coverage-aware regularizer sweep}
\label{app:coverage-aware}
We test whether penalizing the distance from each vertex to its nearest selected landmark narrows the road-network gap. We add $\lambda_{\mathrm{cov}} \cdot \tilde{r}$ to the standard distillation loss. For each vertex, we compute a log-sum-exp soft-min of its distances to the selected landmarks.
On directed graphs, we take the larger of the forward and backward soft-mins.
The penalty $\tilde{r}$ averages these values over a random vertex batch.
This mean is a $k$-median-style surrogate; the covering radius takes a maximum over vertices.
At $\lambda_{\mathrm{cov}}{=}0$, we recover standard AAC-ST.

We evaluate Modena and Manhattan over five seeds at $K_0{=}64, m{=}16$ (64\,B/v).
We use $\lambda_{\mathrm{cov}} \in \{0,\, 10^{-3},\, 10^{-2},\, 10^{-1}\}$.
On Modena, mean expansion reduction falls from $87.1$ to $87.0$\,\% at every non-zero strength.
The gap to FPS-ALT increases from $1.8$ to $2.0$ points, a change smaller than the seed standard deviation ($0.3$ points).
On Manhattan, mean expansion reduction falls from $81.3$ to $79.7$--$80.4$\,\%, with lower reduction on four of five seeds.
The gap increases from $5.6$ to $6.5$--$7.3$ points; the seed standard deviation is $3.4$--$4.3$ points.

The penalty does not narrow the gap to FPS-ALT on either graph. We observe no admissibility violations.

Table~\ref{tab:coverage-aware-nonroad} reports the same sweep on SBM, BA and OGB-arXiv. All five seeds use one shared query sample. We therefore report mean shifts and agreement across seeds without combining $p$-values. Every non-zero $\lambda_{\mathrm{cov}}$ lowers mean expansion reduction on all three graphs.
The shifts are $-0.18$ points on SBM, $-0.06$ to $-0.07$ on BA, and $-0.11$ on OGB-arXiv.
Only three of five seeds agree on the sign of each shift.
The penalty does not improve mean expansion reduction on any of the five tested graphs.
Because the penalty averages distances, this experiment does not test regularization by the covering radius itself.

\begin{table}[!htbp]
\centering
\caption{\tabletag{Ablation} Coverage regularization on non-road graphs.}
\label{tab:coverage-aware-nonroad}
\small
\setlength{\tabcolsep}{4pt}
\adjustbox{max width=\linewidth}{\begin{tabular}{l c
    S[table-format=2.2] S[table-format=2.2] l}
\toprule
\textbf{Graph} & \textbf{$\lambda_{\mathrm{cov}}$}
  & {\textbf{AAC-ST red. (\%)}} & {\textbf{ALT red. (\%)}}
  & \textbf{Shift (pp), seeds agreeing} \\
\midrule
\multirow{4}{*}{SBM ($5{\times}2000$)}    & $0$       & \cellcolor{aaccolor} 88.78 & 89.86 & \cellcolor{aaccolor} {--} (baseline) \\
                                          & $10^{-3}$ & \cellcolor{aaccolor} 88.60 & 89.86 & \cellcolor{aaccolor} $-0.18$, 3/5 \\
                                          & $10^{-2}$ & \cellcolor{aaccolor} 88.60 & 89.86 & \cellcolor{aaccolor} $-0.18$, 3/5 \\
                                          & $10^{-1}$ & \cellcolor{aaccolor} 88.60 & 89.86 & \cellcolor{aaccolor} $-0.18$, 3/5 \\
\midrule
\multirow{4}{*}{BA ($10{,}000$, $m_{\mathrm{attach}}{=}5$)} & $0$       & \cellcolor{aaccolor} 88.29 & 89.23 & \cellcolor{aaccolor} {--} (baseline) \\
                                          & $10^{-3}$ & \cellcolor{aaccolor} 88.23 & 89.23 & \cellcolor{aaccolor} $-0.07$, 3/5 \\
                                          & $10^{-2}$ & \cellcolor{aaccolor} 88.23 & 89.23 & \cellcolor{aaccolor} $-0.07$, 3/5 \\
                                          & $10^{-1}$ & \cellcolor{aaccolor} 88.24 & 89.23 & \cellcolor{aaccolor} $-0.06$, 3/5 \\
\midrule
\multirow{4}{*}{OGB-arXiv (symmetrized LCC)} & $0$       & \cellcolor{aaccolor} 86.57 & 85.40 & \cellcolor{aaccolor} {--} (baseline) \\
                                          & $10^{-3}$ & \cellcolor{aaccolor} 86.46 & 85.40 & \cellcolor{aaccolor} $-0.11$, 3/5 \\
                                          & $10^{-2}$ & \cellcolor{aaccolor} 86.46 & 85.40 & \cellcolor{aaccolor} $-0.11$, 3/5 \\
                                          & $10^{-1}$ & \cellcolor{aaccolor} 86.46 & 85.40 & \cellcolor{aaccolor} $-0.11$, 3/5 \\
\bottomrule
\end{tabular}
}
\par\smallskip
\begin{minipage}{\linewidth}
\footnotesize\raggedright
We train AAC-ST with the coverage penalty $\tilde r$, a mean soft nearest-landmark distance, at strengths $\lambda_{\mathrm{cov}}\in\{0,10^{-3},10^{-2},10^{-1}\}$. We compare it with FPS-ALT at matched deployed label memory. We use $K_0{=}32$, $m{=}8$, $200$ epochs, batch size $256$, learning rate $10^{-3}$, and 100 queries per seed. We use 5 seeds on SBM; 5 seeds on BA; 5 seeds on OGB-arXiv. Seeds of each graph share one query sample, so we do not treat them as independent replicates or combine their $p$-values. The last column gives the mean change relative to the same-seed $\lambda_{\mathrm{cov}}{=}0$ baseline and the number of seeds agreeing on its sign. \par On SBM, every nonzero penalty lowers expansion reduction ($-0.18$\,pp). AAC-ST's expansion reduction stays below that of FPS-ALT at every strength. On BA, every nonzero penalty lowers expansion reduction ($-0.07$ to $-0.06$\,pp). AAC-ST's expansion reduction stays below that of FPS-ALT at every strength. On OGB-arXiv, every nonzero penalty lowers expansion reduction ($-0.11$\,pp). AAC-ST's expansion reduction stays above that of FPS-ALT at every strength. The road-network sweep in Appendix~\ref{app:coverage-aware} leaves AAC-ST below FPS-ALT at every strength. The expansion-reduction gaps (ALT minus AAC-ST) are Manhattan: $5.6$ without the penalty and $6.5$--$7.3$ with it; Modena: $1.8$ without the penalty and $2.0$ with it percentage points.
\end{minipage}
\end{table}

\subsubsection{Training-drift diagnostic on a road graph}
\label{app:training-drift-road}
We compare the forced-prefix control with the trained selector on DIMACS NY (264K nodes, directed) at $K_{\mathrm{ALT}}{=}8$ and $\Bmem{=}64$\,B/v.
We use three seeds (42, 123, 456) and the same five epoch budgets as in the SBM/BA experiment.
Each epoch budget is an independent run. Table~\ref{tab:training-drift-road} reports the seed means and standard deviations.

\begin{table}[!htbp]
\centering
\caption{\tabletag{Ablation} Training-objective drift on NY.}
\label{tab:training-drift-road}
\small
\begin{tabularx}{\textwidth}{@{}X
    S[table-format=2.2(2)] S[table-format=2.2(2)] S[table-format=2.2(2)] S[table-format=2.2(2)] S[table-format=2.2(2)]@{}}
\toprule
\textbf{Variant}
  & {\textbf{ep.\ 0}} & {\textbf{ep.\ 50}} & {\textbf{ep.\ 200}} & {\textbf{ep.\ 500}} & {\textbf{ep.\ 1000}} \\
\midrule
FPS-ALT $K{=}8$ (matched-memory reference)                    & 90.23          & 90.23          & 90.23          & 90.23          & 90.23          \\
\rowcolor{aaccolor} AAC-ST (trained, default; mean over 3 seeds) & 88.00 \pm 0.70 & 88.10 \pm 0.31 & 89.43 \pm 0.28 & 89.83 \pm 0.48 & 90.03 \pm 0.37 \\
\rowcolor{aaccolor} AAC-ST forced-first-$m$ (diagnostic control)                          & 90.23          & 90.23          & 90.23          & 90.23          & 90.23          \\
\midrule
\rowcolor{aaccolor} gap (AAC-ST $-$ forced-first-$m$, pp)                            & {$-2.23$} & {$-2.12$} & {$-0.80$} & {$-0.40$} & {$-0.19$} \\
\bottomrule
\end{tabularx}
\par\smallskip
\begin{minipage}{\linewidth}
\footnotesize\raggedright
We repeat the training-drift diagnostic on DIMACS NY under directed accounting (Section~\ref{sec:training-drift}). We match memory at $\Bmem{=}64$\,B/v: AAC-ST uses $K_0{=}32$ and $m_{\mathrm{fwd}}{+}m_{\mathrm{bwd}}{=}8{+}8$; ALT uses $K{=}8$. Entries are mean $\pm$ std of test expansion reduction (\%) over 3 seeds (42, 123, 456), with 100 queries per seed. Each of the 5 epoch budgets is an independent run with its own temperature schedule. \par The \textbf{forced-first-$m$} control uses one-hot rows on the first $K_{\mathrm{ALT}}{=}8$ pool indices in each direction. It equals FPS-ALT $K{=}8$ in every run ($90.23$\,\% reduction). The trained compressor stays $0.2$--$2.2$\,pp \emph{below} this control across the $0$--$1000$-epoch budgets. The seed-mean differences are $-2.23$, $-2.12$, $-0.80$, $-0.40$, $-0.19$ at budgets $0$, $50$, $200$, $500$, $1000$, respectively. The deficit is present at every epoch budget and narrows monotonically with longer budgets.
\end{minipage}
\end{table}

\subsubsection{Hyperparameter and initialization robustness for training drift}
\label{app:training-drift-hp}
We vary learning rate, batch size, candidate-pool size $K_0$, and initialization on SBM at $32$\,B/v (Table~\ref{tab:training-drift-hp}).
Section~\ref{sec:training-drift} summarizes the results. The table includes the
identity-on-first-$m$ rows that match the FPS-ALT control at the reported precision.

\begin{table}[!htbp]
\centering
\caption{\tabletag{Ablation} Initialization and hyperparameter sensitivity on SBM at 32 B/v.}
\label{tab:training-drift-hp}
\small
\adjustbox{max width=\linewidth}{\begin{tabular}{S[table-format=1.0e+1] r l
    S[table-format=2.2(2)] S[table-format=+1.2]}
\toprule
{\textbf{lr}} & \textbf{batch} & \textbf{init}
  & {\textbf{trained red.\ (\%)}} & {\textbf{trained gap (pp)}} \\
\midrule
\rowcolor{aaccolor} 3e-4 & 128 & block\_sparse      & 89.48 \pm 0.87 & -0.38 \\
\rowcolor{aaccolor} 3e-4 & 128 & identity\_first\_m & 89.86 \pm 0.00 & +0.00 \\
\rowcolor{aaccolor} 3e-4 & 256 & block\_sparse      & 89.48 \pm 0.87 & -0.38 \\
\rowcolor{aaccolor} 3e-4 & 256 & identity\_first\_m & 89.86 \pm 0.00 & +0.00 \\
\rowcolor{aaccolor} 3e-4 & 512 & block\_sparse      & 89.29 \pm 1.08 & -0.57 \\
\rowcolor{aaccolor} 3e-4 & 512 & identity\_first\_m & 89.86 \pm 0.00 & +0.00 \\
\rowcolor{aaccolor} 1e-3 & 128 & block\_sparse      & 89.51 \pm 0.84 & -0.35 \\
\rowcolor{aaccolor} 1e-3 & 128 & identity\_first\_m & 89.86 \pm 0.00 & +0.00 \\
\rowcolor{aaccolor} 1e-3 & 256 & block\_sparse      & 89.69 \pm 0.74 & -0.17 \\
\rowcolor{aaccolor} 1e-3 & 256 & identity\_first\_m & 89.86 \pm 0.00 & +0.00 \\
\rowcolor{aaccolor} 1e-3 & 512 & block\_sparse      & 89.49 \pm 0.86 & -0.37 \\
\rowcolor{aaccolor} 1e-3 & 512 & identity\_first\_m & 89.86 \pm 0.00 & +0.00 \\
\rowcolor{aaccolor} 3e-3 & 128 & block\_sparse      & 89.78 \pm 0.63 & -0.08 \\
\rowcolor{aaccolor} 3e-3 & 128 & identity\_first\_m & 89.86 \pm 0.00 & +0.00 \\
\rowcolor{aaccolor} 3e-3 & 256 & block\_sparse      & 89.67 \pm 0.90 & -0.18 \\
\rowcolor{aaccolor} 3e-3 & 256 & identity\_first\_m & 89.86 \pm 0.00 & +0.00 \\
\rowcolor{aaccolor} 3e-3 & 512 & block\_sparse      & 89.68 \pm 0.69 & -0.18 \\
\rowcolor{aaccolor} 3e-3 & 512 & identity\_first\_m & 89.86 \pm 0.00 & +0.00 \\
\bottomrule
\end{tabular}
}
\par\smallskip
\begin{minipage}{\linewidth}
\footnotesize\raggedright
Each row reports mean $\pm$ std over 3 seeds for one combination of learning rate, batch size, $K_0$, and initialization. We train every configuration for $200$ epochs. ``Trained gap'' is the expansion reduction of trained AAC-ST minus the forced-prefix control, in percentage points. The control gives $89.86$\,\% reduction at $K_0{=}32$, $m{=}8$. \par \emph{block\_sparse} denotes the default initialization. \emph{identity\_first\_m} initializes the logits to favor the first $m$ pool landmarks, reproducing the control under hard argmax. Across learning rates and batch sizes, the default-initialization gap is $-0.6$ to $-0.1$\,pp. With prefix initialization, the trained selector matches the forced-first-$m$ control at every tested learning rate and batch size ($+0.00$\,pp, std $0.00$). \par We show the $K_0{=}32$ subset of 36 tested combinations: $3$ learning rates, $3$ batch sizes, $2$ pool sizes, and $2$ initializations. At $K_0{=}64$, the default-initialization gap is $-1.9$ to $-0.5$\,pp. With prefix initialization, the trained selector matches the forced-first-$m$ control at every tested learning rate and batch size ($+0.00$\,pp, std $0.00$).
\end{minipage}
\end{table}

\subsection{CDH configurations}
\label{app:cdh-sweep}
Table~\ref{tab:cdh-sweep} gives the selected pivot pools and goal-side search costs for the
CDH comparison in Section~\ref{sec:cdh-headtohead}. The supplementary records contain the full sweep.

\begin{table}[!htbp]
\centering
\caption{\tabletag{Ablation} Compressed differential heuristic configurations.}
\label{tab:cdh-sweep}
\footnotesize
\setlength{\tabcolsep}{3pt}
\adjustbox{max width=\linewidth}{\begin{tabular}{l r S[table-format=-3.1] S[table-format=-3.1] S[table-format=5.0] S[table-format=-3.1] S[table-format=-3.1] S[table-format=5.0] l}
\toprule
\textbf{Graph} & {\textbf{B/v}} & {\textbf{FPS-ALT}} & {\textbf{CDH ($P{=}64$)}} & {\textbf{bound.}} & {\textbf{+BPMX}} & {\textbf{CDH best}} & {\textbf{bound.}} & \textbf{at} \\
\midrule
\multirow{5}{*}{SBM} &   8 & 66.8 & 8.0 & 287 & 13.3 & 44.7 & 24 & $P{=}8$, round-robin, BPMX, $r{=}2$ \\
                        &  16 & 82.7 & 29.5 & 125 & 39.3 & 69.1 & 9 & $P{=}8$, round-robin, BPMX, $r{=}4$ \\
                        &  32 & 89.9 & 55.9 & 58 & 62.8 & 88.0 & 0 & $P{=}8$, round-robin, $r{=}8$ \\
                        &  64 & 93.9 & 78.6 & 24 & 81.7 & 93.3 & 0 & $P{=}16$, round-robin, $r{=}16$ \\
                        & 128 & 97.0 & 92.5 & 9 & 93.2 & 96.4 & 0 & $P{=}32$, round-robin, $r{=}32$ \\
\midrule
\multirow{4}{*}{Modena} &  16 & 67.8 & 73.1 & 234 & {--} & 82.0 & 21 & $P{=}8$, round-robin, $r{=}2$ \\
                        &  32 & 83.9 & 85.9 & 107 & {--} & 89.1 & 21 & $P{=}16$, round-robin, $r{=}4$ \\
                        &  64 & 89.0 & 91.4 & 47 & {--} & 92.2 & 21 & $P{=}32$, round-robin, $r{=}8$ \\
                        & 128 & 92.1 & 94.3 & 21 & {--} & 94.3 & 21 & $P{=}64$, round-robin, $r{=}16$ \\
\midrule
\multirow{3}{*}{NY} &  32 & 85.5 & 70.4 & 96 & {--} & 85.3 & 8 & $P{=}8$, round-robin, $r{=}4$ \\
                        &  64 & 90.2 & 87.4 & 45 & {--} & 90.7 & 8 & $P{=}16$, round-robin, $r{=}8$ \\
                        & 128 & 93.6 & 93.6 & 20 & {--} & 93.6 & 20 & $P{=}64$, round-robin, $r{=}16$ \\
\bottomrule
\end{tabular}
}
\par\smallskip
\begin{minipage}{\linewidth}
\footnotesize\raggedright
We vary pivot-pool size $P_{\mathrm{cdh}}\in\{8, 16, 32, 64, 128, 256\}$ (on NY only $\{8, 16, 64\}$) and use 4 storage rules: farthest, nearest, random, round-robin. We use the queries and memory budgets of Table~\ref{tab:cdh-reference}. Reduction columns report the mean percentage of Dijkstra expansions saved. \emph{bound.} is the mean number of states settled by the goal-side search; we include these in total expansions. The search stops once each missing pivot has been seen at 4 states that store it. \par \emph{best} selects the configuration with the highest seed-mean reduction. \emph{at} gives its pool, storage rule, and retained count $r$. Each vertex stores the largest subset that fits the budget, capped at the pool size. We store no pivot-to-pivot table. The $P{=}64$ columns use round-robin storage. \par We use BPMX only on undirected SBM; a dash means we do not use it. Deterministic rules use one seed; random storage averages 3 seeds. Section~\ref{sec:cdh-headtohead} gives the protocol.
\end{minipage}
\end{table}

\section{Ablation studies}
\label{app:ablation-studies}

We run two ablations on the OSMnx city graphs (Modena and Manhattan). Both use $K_0{=}64$, seeds $\{42, 123, 456, 789, 1024\}$ and $100$ random query pairs per configuration.

\subsection{Selection strategy}

Table~\ref{tab:selection-ablation} compares the five strategies discussed in
Section~\ref{sec:selection-ablation}, at $64$\,B/v and $m=16$.
Its notes specify the sampling and validation protocol for each strategy.

\begin{table}[!htbp]
\centering
\caption{\tabletag{Ablation} Landmark-selection strategies.}
\label{tab:selection-ablation}
\small
\adjustbox{max width=\linewidth}{\begin{tabular}{l l S[table-format=2.1] S[table-format=1.1]}
\toprule
\textbf{Graph} & \textbf{Method}
  & {\textbf{Exp.\ Red.\ (\%)} $\uparrow$} & {\textbf{Std}} \\
\midrule
\multirow{5}{*}{Modena}    & \cellcolor{aaccolor} AAC-ST (default)                   & \cellcolor{aaccolor} 87.1  & \cellcolor{aaccolor} 0.3 \\
                           & FPS-Subset (ALT $K{=}8$, canonical seed)             & 89.0                       & 0.0 \\
                           & FPS-RR$_{10}$ (val-selected best of 10 restarts)     & 88.8                       & 0.1 \\
                           & Greedy-Max$^\dagger$                                 & \bfseries 90.6             & 0.4 \\
                           & Random-Subset                                        & 85.9                       & 1.7 \\
\midrule
\multirow{5}{*}{Manhattan} & \cellcolor{aaccolor} AAC-ST (default)                   & \cellcolor{aaccolor} 81.3  & \cellcolor{aaccolor} 3.4 \\
                           & FPS-Subset (ALT $K{=}8$, canonical seed)             & 87.0                       & 0.0 \\
                           & FPS-RR$_{10}$ (val-selected best of 10 restarts)     & 87.0                       & 0.0 \\
                           & Greedy-Max$^\dagger$                                 & \bfseries 88.2             & 0.2 \\
                           & Random-Subset                                        & 83.0                       & 2.1 \\
\bottomrule
\multicolumn{4}{@{}l}{\footnotesize $^\dagger$ Greedy comparator: selects from the $K_0$ pool on separate selection queries, not on the evaluation queries.} \\
\end{tabular}
}
\par\smallskip
\begin{minipage}{\linewidth}
\footnotesize\raggedright
We compare selection strategies at $K_0{=}64$, $m{=}16$, and 64~B/v (Section~\ref{sec:methodological-protocol}). AAC-ST uses the default trained row-stochastic selector. Random-Subset independently samples $m/2$ forward and $m/2$ backward labels from the $K_0$ pool. FPS-Subset uses standard ALT with $K{=}m/2{=}8$ FPS landmarks, without a teacher pool. Its start vertex is fixed within the largest strongly connected component. \par FPS-RR$_{10}$ runs FPS from $R{=}10$ random start vertices. We select the best run on 100 validation queries and evaluate it on 100 disjoint test queries. Greedy-Max selects $m/2$ landmarks from the $K_0$ pool and stores both directions of each. It maximizes the mean ALT bound on a per-seed selection sample, disjoint from evaluation and equal in size to the test set. \par Entries are mean $\pm$ std over 5 seeds and 100 test queries. Bold marks the highest mean expansion reduction per graph. The AAC-ST results at $K_0{=}64$, $m{=}16$ are identical, seed for seed, across this ablation and Tables~\ref{tab:compression-curve} and~\ref{tab:osmnx}.
\end{minipage}
\end{table}

\subsection{Compression efficiency curve}

We vary $m \in \{4, 8, 16, 32, 64\}$ at fixed $K_0{=}64$ and compare AAC-ST with ALT at matched memory (Table~\ref{tab:compression-curve}).
AAC-ST stores $m$ directional labels per vertex: $m/2$ forward and $m/2$ backward.
ALT uses $K{=}m/2$ landmarks with both directions, for the same $m$ floats per vertex.

AAC-ST has a higher mean expansion reduction than FPS-ALT only at $m=64$ on Modena
and at $m=32,64$ on Manhattan. Table~\ref{tab:compression-curve} gives the means and seed variation.

\paragraph{Interpreting the crossover.} Proposition~\ref{thm:admissibility} bounds AAC-ST by
its full $K_0$-landmark teacher. The matched-memory baseline uses fewer landmarks.
At $m=64$ and $K_0=64$, AAC-ST selects $32$ forward and $32$ backward labels from the pool,
while FPS-ALT retains $32$ landmarks in both directions. The two directional selections may
differ, so AAC-ST can exceed the baseline while remaining below its teacher.
Identity selection of the full directed teacher uses $m=2K_0=128$ labels.

\begin{table}[!htbp]
\centering
\caption{\tabletag{Ablation} Performance versus compression budget.}
\label{tab:compression-curve}
\small
\begin{subtable}[t]{0.48\textwidth}
\centering
\caption{Modena (26K nodes)}
\adjustbox{max width=\linewidth}{\begin{tabular}{r r S[table-format=2.1(1)] S[table-format=2.1]}
\toprule
{$m$} & \textbf{B/v} & {\textbf{AAC-ST (\%)}} & {\textbf{ALT (\%)}} \\
\midrule
 4 &  16 & \cellcolor{aaccolor} 63.3 \pm 7.3 & \bfseries 67.8 \\
 8 &  32 & \cellcolor{aaccolor} 79.4 \pm 2.5 & \bfseries 83.9 \\
16 &  64 & \cellcolor{aaccolor} 87.1 \pm 0.3 & \bfseries 89.0 \\
32 & 128 & \cellcolor{aaccolor} 91.2 \pm 1.1 & \bfseries 92.1 \\
64 & 256 & \cellcolor{aaccolor} \bfseries 93.9 \pm 0.1 & 93.5 \\
\bottomrule
\end{tabular}
}
\end{subtable}
\hfill
\begin{subtable}[t]{0.48\textwidth}
\centering
\caption{Manhattan (13K nodes)}
\adjustbox{max width=\linewidth}{\begin{tabular}{r r S[table-format=2.1(1)] S[table-format=2.1]}
\toprule
{$m$} & \textbf{B/v} & {\textbf{AAC-ST (\%)}} & {\textbf{ALT (\%)}} \\
\midrule
 4 &  16 & \cellcolor{aaccolor} 60.3 \pm 5.3 & \bfseries 69.7 \\
 8 &  32 & \cellcolor{aaccolor} 71.0 \pm 5.8 & \bfseries 80.6 \\
16 &  64 & \cellcolor{aaccolor} 81.3 \pm 3.4 & \bfseries 87.0 \\
32 & 128 & \cellcolor{aaccolor} \bfseries 91.0 \pm 1.1 & 90.0 \\
64 & 256 & \cellcolor{aaccolor} \bfseries 93.9 \pm 0.2 & 93.8 \\
\bottomrule
\end{tabular}
}
\end{subtable}
\par\smallskip
\begin{minipage}{\linewidth}
\footnotesize\raggedright
We compare AAC-ST with $K_0{=}64$ teachers and $m\in\{4, 8, 16, 32, 64\}$ dimensions against ALT with $K{=}m/2$ FPS landmarks, at matched memory. Entries are expansion reduction over Dijkstra (\%, mean $\pm$ std across 5 seeds, 100 queries). Bold marks the higher mean per budget. ALT leads Modena at $m{=}4$, $m{=}8$, $m{=}16$, $m{=}32$; Manhattan at $m{=}4$, $m{=}8$, $m{=}16$; AAC-ST leads Modena at $m{=}64$; Manhattan at $m{=}32$, $m{=}64$. At $m{=}64$ ($m{=}K_0$), these directed graphs still compress. The budget gives $32$ dimensions per direction, each selecting from $64$ landmarks. Identity selection of the full teacher requires $m{=}2K_0$ labels with permutation rows (Proposition~\ref{prop:alt-special-case}). The AAC-ST results at $K_0{=}64$, $m{=}16$ are identical, seed for seed, across this ablation and Tables~\ref{tab:selection-ablation} and~\ref{tab:osmnx}.
\end{minipage}
\end{table}

\section{FastMap: implementation and memory accounting}
\label{app:fastmap}

We use FastMap's linear residual projection \citep{cohen2018fastmap}, which makes the
embedding's $L_1$ distance a lower bound on symmetric graph distances. The original
squared-residual projection \citep{faloutsos1995fastmap} lacks this guarantee.
For each dimension, we use $10$ random starts on the current residual metric.
From each start, we run $10$ alternating farthest-vertex sweeps and retain the farthest pair across starts.
We seed the random starts per run and store coordinates in float32.
For each dimension, we subtract two ULPs at the largest coordinate magnitude, using the coarser of input and stored precision.
This allowance does not bound all errors accumulated during coordinate construction and summation.
The memory accounting charges $4$ bytes per coordinate.

FM+DH follows the equal-memory split and pointwise maximum of \citet[Table~1]{cohen2018fastmap}:
half the stored values go to differential-heuristic pivots and half to FastMap coordinates.
We use FPS for the differential pivots and charge one value per pivot on undirected graphs
and two on directed graphs. These placement and directed-accounting choices are ours;
\citet{cohen2018fastmap} study undirected graphs. The maximum preserves admissibility for
symmetric distances. On general directed graphs, FM+DH inherits FastMap's lack of a guarantee.

Although stored as directed arc lists, all four DIMACS instances have symmetric distances:
every arc in NY, BAY, COL, and FLA has an equal-weight reverse. Thus $d_{\mathrm{out}}$ and
$d_{\mathrm{in}}$ coincide, and one stored copy suffices. Using the saved space for more
landmarks cuts FPS-ALT's expansions by $20$--$39$\,\% at every tested budget on NY and FLA,
with zero observed admissibility violations. On NY at $32$\,B/v, expansion reduction rises
from $85.55$ to $90.23$\,\%, compared with FastMap's $81.42$\,\%.

FPS-ALT expands fewer nodes than FastMap at every tested budget on NY and FLA, even when both directions are stored separately. Sharing the tables widens that lead. It also changes the comparison with
FM+DH: FPS-ALT trails on FLA at $64$ and $256$\,B/v under separate storage, but leads at every
tested budget on both graphs under shared storage. For example, at $64$\,B/v on FLA, the
expansion reductions are $92.49$\,\% for FPS-ALT and $90.21$\,\% for FM+DH.

Sharing the two equal distance tables doubles the number of landmarks available to ALT at a
fixed byte budget. For a selector, it also changes the ground set from direction-tagged items
to landmarks. A selector that retained both directions previously paid for two equal tables,
just as paired ALT did. Table~\ref{tab:symmetric-accounting} evaluates all three methods under
the shared-table rule on the eleven affected graph--budget pairs and retains the original accounting on
the eight unaffected graph--budget pairs.

For every run, we record admissibility violations and compare returned path costs with Dijkstra's
(Section~\ref{sec:admissibility-verification}).

\section{Contextual variant: pipeline specification and Warcraft compatibility study}
\label{app:contextual}

This appendix specifies the differentiable pipeline and reports the Warcraft $12{\times}12$ compatibility study.

\subsection{Pipeline specification}

\paragraph{Pipeline.} We train and deploy the model as follows:
\begin{enumerate}
\item An encoder $f_\phi$ (CNN or ResNet) predicts a per-vertex cost $\hat{c} = f_\phi(\mathrm{context})$ from input features.
\item We convert vertex costs to edge weights on the 8-connected grid: $w(u,v) = \frac{1}{2}(\hat{c}_u + \hat{c}_v)$.
We use no extra factor for diagonal moves. The dataset's solver charges every visited cell, including the start.
Thus a path's native cost equals its length under $w$ plus $\frac{1}{2}(c_s + c_t)$.
This offset is constant for a fixed source and target, so both conventions give the same optimal paths.
\item We compute smoothed distances from $K_0$ anchors with a smoothed Bellman--Ford layer.
We differentiate through its fixed point as described below.
\item During training, we apply either the row-stochastic compressor or the continuous-greedy layer to the smoothed anchor distances.
For validation and deployment, we recompute exact anchor distances on the predicted-cost graph and evaluate ALT on the selected items.
The resulting heuristic is admissible on that graph (Theorem~\ref{thm:alt-admissibility}).
\end{enumerate}

\paragraph{Smoothed Bellman--Ford.} Exact distances from anchor $l$ satisfy $d(l,l)=0$ and $d(l,v)=\min_{(u,v)\in E}\{d(l,u)+w(u,v)\}$ for $v\neq l$.
We replace the hard minimum over the $n_v$ in-arcs of $v$ with the normalized soft-min of Theorem~\ref{thm:smooth}(a).
For $\beta>0$, we solve for the fixed point
\[
\tilde d(l,l) = 0,
\qquad
\tilde d(l,v) = -\frac{1}{\beta}\log\Big(\frac{1}{n_v}\sum_{(u,v) \in E} e^{-\beta(\tilde d(l,u) + w(u,v))}\Big)
\quad (v \neq l),
\]
over the vertices $l$ reaches, with $n_v$ counting the in-arcs from those vertices. Write $\Phi_v$ for the right-hand side when $v\neq l$ and set $\Phi_l=0$. For every non-negative weighting the fixed point exists, is unique, and brackets the exact distance:
\[
d(l,v) \;\leq\; \tilde d(l,v) \;\leq\; d(l,v) + \frac{1}{\beta}\sum_{j=1}^{k} \log n_{p_j}
\]
along any shortest path $l = p_0, p_1, \dots, p_k = v$. The total error is bounded by the sum along that path and tends to zero as $\beta\to\infty$.

\emph{Existence and the distance bound.} Since $\Phi_v$ is $\smin_\beta$ plus
$\log(n_v)/\beta$, Theorem~\ref{thm:smooth}(a) gives $\min_u (x_u + w(u,v)) \leq \Phi_v(x) \leq \min_u (x_u + w(u,v)) + \log(n_v)/\beta$, and $\Phi$ is non-decreasing in $x$. So $\Phi(d) \geq d$, and the shortest-path distance $U$ under the arc lengths $w(u,v) + \log(n_v)/\beta$ satisfies $\Phi(U) \leq U$ and $U \geq d$. Iterating $\Phi$ from $d$ gives a non-decreasing sequence bounded by $U$, whose limit is a fixed point by continuity.

\emph{Uniqueness.} Suppose two distinct fixed points $x, x'$ exist.
We label them so that $D := \max_v (x_v - x'_v) > 0$ and choose a maximizing vertex $v^*$.
The soft-min gradient is a softmax, so its weights are positive.
By the mean value theorem, $x_{v^*} - x'_{v^*} = \Phi_{v^*}(x) - \Phi_{v^*}(x')$ is a weighted average of the differences at the in-neighbors of $v^*$. Each of those differences is at most $D$, so all equal $D$. Walking back along a path from $l$ to $v^*$ reaches $l$, where the difference is $0$, a contradiction. The unique fixed point is therefore the limit above, between $d$ and $U$. \qed

The factor $1/n_v$ guarantees existence for every non-negative weighting.
Without it, $z_v := e^{-\beta \tilde d(l,v)}$ would have to satisfy $z_v = \sum_{(u,v) \in E} e^{-\beta w(u,v)} z_u$ for $v \neq l$.
Consider a set $C \not\ni l$ where each vertex has at least $n$ in-arcs from $C$, all with weights below $\log(n)/\beta$.
The vertex with the smallest $z$ in $C$ would satisfy $z_v > z_v$, a contradiction.
On the 8-connected grid, a $2{\times}2$ block away from the anchor gives $n = 3$.
Predicted costs below $\log(3)/\beta$ therefore rule out a finite fixed point. This threshold is about $1.1$ at the initial $\beta = 1$.

We solve the fixed point numerically by policy iteration.
We initialize with the uniform distribution over tight predecessor arcs of the exact distances.
At each step, we evaluate the current predecessor distribution with a linear solve, then update it to the soft-min distribution. At convergence, let $P_{vu}$ be the softmin probability of predecessor $u$ for $v\neq l$, and set the anchor row $P_{l\cdot}=0$. The matrix $I-P$ is invertible.
If $z=Pz$ had a nonzero solution, a vertex maximizing $|z_v|$ would force the same value at every predecessor.
Following a path back to the anchor would then contradict $z_l=0$. For $g=\partial\mathcal L/\partial\tilde d$, implicit differentiation therefore gives
\[
(I-P^\top)\lambda=g,
\qquad
\frac{\partial\mathcal L}{\partial w(u,v)}=P_{vu}\lambda_v\quad(v\neq l).
\]
This is the derivative of the converged fixed point; numerical evaluation is subject to the solver tolerance.

\paragraph{Admissibility during training and deployment.} The smoothed distances overestimate the exact distances: $\tilde{d}(l_k, v) \geq d(l_k, v)$.
ALT uses distance differences, and subtracting two overestimates need not give a lower bound.
We therefore use the smoothed labels only as a differentiable training proxy.
For validation and deployment, we recompute exact distances on the predicted-cost graph.
ALT on the selected items is then admissible on that graph (Theorem~\ref{thm:alt-admissibility}).

We anneal $\beta$ from $1.0$ to $30.0$ by a factor of $1.05$ per iteration.
Increasing $\beta$ reduces smoothing and the per-hop overestimate $\log(n_v)/\beta$.

\paragraph{Combined loss.}
\begin{equation}
\label{eq:ctx-loss}
\mathcal{L}_{\mathrm{ctx}} = \underbrace{(d_{\mathrm{true}} - \tilde{h}^{\mathrm{ctx}}_{A})_+}_{\text{gap-closing}} + \alpha_{\mathrm{cost}} \cdot \underbrace{\tfrac{1}{|V|}\|\hat{c} - c_{\mathrm{gt}}\|^2}_{\text{cost supervision}} + \lambda_{\mathrm{ent}} \cdot R_{\mathrm{ent}},
\end{equation}
where $d_{\mathrm{true}}$ is the shortest-path distance on the true-cost graph under the edge weights of item~2.
We compute these distances offline and average the gap-closing term over $32$ random query pairs per map.
For the straight-through compressor, $\tilde{h}^{\mathrm{ctx}}_{A}$ is the smoothed heuristic of Theorem~\ref{thm:smooth}(b).
We construct it from smoothed Bellman--Ford distances on the predicted-cost graph.

Cost supervision is the mean squared error over the $|V|$ cells.
The term $R_{\mathrm{ent}}$ is the compressor's row entropy, with $\lambda_{\mathrm{ent}} = 0.01$.
For the continuous-greedy ablation, we replace $\tilde{h}^{\mathrm{ctx}}_{A}$ with the fractional heuristic $h_{x_T}$ (Proposition~\ref{prop:matroid-cg}(c)).
We use the same smoothed distances and the final entropic-oracle iterate $x_T$.
This selection layer has no trainable parameters, and we omit $R_{\mathrm{ent}}$.

The terms $d_{\mathrm{true}}$ and $\tilde{h}^{\mathrm{ctx}}_{A}$ use different graphs: true costs for the distance and predicted costs for the heuristic.
Their difference is therefore not an admissibility bound.
Cost supervision trains the encoder directly. The gap-closing term trains both the encoder and the compressor.

\paragraph{Encoder architectures.} We evaluate two encoders on 96$\times$96 RGB terrain images:
\begin{itemize}
\item \textbf{CNN}: 3 convolutional layers (3$\to$32$\to$64$\to$128 channels), AdaptiveAvgPool to grid size, 1$\times$1 conv to scalar costs, $\mathrm{softplus}(\cdot)+0.01$ for strict positivity.
\item \textbf{ResNet}: 5$\times$5 stem, 3 residual blocks (64 channels each), group normalization with $32$ groups after the stem and after each convolution of the residual blocks, AdaptiveAvgPool, two 1$\times$1 convolutions, $\mathrm{softplus}(\cdot)+0.01$. Group normalization keeps each map's costs independent of the batch.
\end{itemize}

\subsection{Warcraft compatibility study}
\label{sec:contextual-study}

We test whether AAC can train an encoder through a differentiable planning pipeline and deploy an admissible heuristic on the predicted-cost graph.
We compare AAC-CG with AAC-ST under the same training protocol.

\paragraph{Setup.} We use Warcraft $12{\times}12$ grids~\citep{vlastelica2020differentiation},
with $10{,}000$ training and $1{,}000$ test maps. Each graph has $144$ vertices and 8-connected edges.
A CNN or ResNet predicts cell costs from $96{\times}96$ RGB terrain images.
Item~2 of the pipeline converts these costs to edge weights. AAC then selects landmarks for
an admissible A* heuristic on the predicted-cost graph.

We train for at most $100$ iterations, with one optimizer step on $24$ maps per iteration.
We sample without replacement within each batch, for at most $2{,}400$ map draws per run.
We also record the number of distinct maps.

We stop early with patience $15$ on a fixed $24$-map subset of the dataset's validation split.
We validate with exact distances and restore the best validation weights for deployment.
We measure exact path match and Jaccard against the dataset's label paths.
Jaccard is the intersection size divided by the union size of the two paths' edge sets.
For a path $p$ returned on the predicted-cost graph and an optimal path $p^*$ under true costs, cost regret is
$(C_{\mathrm{true}}(p)-C_{\mathrm{true}}(p^*))/C_{\mathrm{true}}(p^*)$.
Here $C_{\mathrm{true}}$ sums the true costs of all visited cells, including the start.
We average this relative excess cost over test maps; a value of $0.094$ means $9.4$\,\% mean excess cost.

Table~\ref{tab:contextual-ablation} reports eight-seed CNN and three-seed ResNet results
at $K_0=16,m=8$. With cost supervision ($\alpha_{\mathrm{cost}}=10$) and a frozen compressor, ResNet reaches $85.3$\,\% optimal-cost accuracy and $(17.5\pm0.8)$\,\% exact path match.
We observe zero admissibility violations on every seed's $1{,}000$ test maps.
\citet[Table~2]{vlastelica2020differentiation} report $(96.0\pm0.3)$\,\% optimal-cost accuracy using
blackbox differentiation through Dijkstra. Their training protocol differs from ours, so this comparison does not isolate the selection layer.
Our results establish that the pipeline trains and deploys admissible landmark bounds;
they do not establish competitive planning performance.

\begin{table}[!htbp]
\centering
\caption{\tabletag{Ablation} Warcraft compatibility study.}
\label{tab:contextual-ablation}
\small
\setlength{\tabcolsep}{4pt}
\adjustbox{max width=\linewidth}{\begin{tabular}{llcccc}
\toprule
\textbf{Method} & \textbf{Configuration} & \textbf{Path Match} $\uparrow$ & \textbf{Jaccard} $\uparrow$ & \textbf{Opt.\ cost} $\uparrow$ & \textbf{Adm.\ violations} $\downarrow$ \\
\midrule
\multicolumn{6}{l}{\textit{AAC-ST with learned costs and a frozen compressor; $O(m)$ per heuristic evaluation}} \\
\rowcolor{aaccolor} \quad ResNet$^\ddagger$, $\alpha_{\mathrm{cost}}{=}10$ & encoder only (frozen comp.) & $(17.5 \pm 0.8)$\,\% & $(54.4 \pm 0.4)$\,\% & $(85.3 \pm 2.1)$\,\% & 0 \\
\rowcolor{aaccolor} \quad CNN,    $\alpha_{\mathrm{cost}}{=}1$  & encoder only (frozen comp.) & $(2.9 \pm 0.4)$\,\%            & $(25.4 \pm 3.1)$\,\%           & $(11.8 \pm 5.7)$\,\%           & 0 \\
\midrule
\multicolumn{6}{l}{\textit{Published cost-learning reference (different training protocol)}} \\
\quad BB$^\dagger$       & blackbox diff.\ solver       & --               & --    & $(96.0 \pm 0.3)$\,\% & -- \\
\bottomrule
\multicolumn{6}{p{0.95\textwidth}}{\footnotesize $^\dagger$ Vlastelica et al.\ (ICLR 2020), their Table~2, on Warcraft $12{\times}12$; they report optimal-cost accuracy and no exact-path-match figure.} \\
\multicolumn{6}{p{0.95\textwidth}}{\footnotesize $^\ddagger$ ResNet uses a 5$\times$5 stem, 3 residual blocks at 64 channels, adaptive average pooling, and 1$\times$1 convolutions (Appendix~\ref{app:contextual}).} \\
\end{tabular}
}
\par\vspace{6pt}
{\small\raggedright \textbf{End-to-end selection layers.} We use $\alpha_{\mathrm{cost}}=1$ for CNNs, with 8 seeds at $K_0=16,m=8$ and 3 seeds at $K_0=5$. For ResNet, we use $\alpha_{\mathrm{cost}}=10$ and 3 seeds. We sample training maps from 10{,}000 available maps. AAC-CG replaces the Gumbel compressor with a parameter-free continuous-greedy layer that selects for each map. The rest of the training pipeline is unchanged. AAC-CG$_\tau$ uses the entropic oracle at deployment; AAC-CG uses the exact oracle. \par $F/F_{\mathrm{FPS}}$ compares the test objective with the FPS prefix on separate random query pairs. Across 48{,}000 test instances, we observe 0 vertex-level admissibility violations and 0 suboptimal paths on predicted-cost graphs. Expansion counts use one fixed corner-to-corner query whose target is the first FPS landmark. Retaining that landmark gives an exact heuristic in real arithmetic, so these counts do not establish general search-efficiency gains.\par}\vspace{4pt}
\footnotesize
\setlength{\tabcolsep}{3pt}
\adjustbox{max width=\linewidth}{\begin{tabular}{l l c c S[table-format=1.3] S[table-format=3.0] l S[table-format=1.3] c}
\toprule
\textbf{Variant} & \textbf{Config} & \textbf{Path Match} & \textbf{Jaccard} & \textbf{Cost regret}
  & \textbf{Exp.} & \textbf{Exp.\ red.} & {\textbf{$F/F_{\mathrm{FPS}}$}} & \textbf{Adm.} \\
\midrule
\rowcolor{aaccolor} AAC-ST       & $K_0{=}5,\,m{=}2$  & $(3.2 \pm 0.5)$\,\%          & $(23.6 \pm 0.3)$\,\%         & 0.127 & 81 & 38.5\,\% & 1.001 & 0 \\
\rowcolor{aaccolor} AAC-CG       & $K_0{=}5,\,m{=}2$  & $(3.3 \pm 0.5)$\,\%          & $(24.4 \pm 0.7)$\,\%         & 0.126 & 36 & 72.9\,\% & 1.048 & 0 \\
\cmidrule{1-9}
\rowcolor{aaccolor} AAC-ST       & $K_0{=}5,\,m{=}3$  & $(3.0 \pm 0.4)$\,\%          & $(23.5 \pm 0.5)$\,\%         & 0.129 & 64 & 51.4\,\% & 1.032 & 0 \\
\rowcolor{aaccolor} AAC-CG       & $K_0{=}5,\,m{=}3$  & $(3.0 \pm 0.5)$\,\%          & $(23.6 \pm 0.7)$\,\%         & 0.119 & 26 & 80.3\,\% & 1.051 & 0 \\
\cmidrule{1-9}
\rowcolor{aaccolor} AAC-ST       & $K_0{=}5,\,m{=}4$  & $(3.0 \pm 0.6)$\,\%          & $(23.6 \pm 1.0)$\,\%         & 0.120 & 15 & 88.6\,\% & 0.966 & 0 \\
\rowcolor{aaccolor} AAC-CG       & $K_0{=}5,\,m{=}4$  & $(2.9 \pm 0.5)$\,\%          & $(23.5 \pm 1.0)$\,\%         & 0.116 & 16 & 87.7\,\% & 1.002 & 0 \\
\cmidrule{1-9}
\rowcolor{aaccolor} AAC-ST       & $K_0{=}16,\,m{=}8$ & $(2.9 \pm 0.4)$\,\%          & $(25.4 \pm 3.1)$\,\%         & 0.094 & 33 & 74.9\,\% & 0.989 & 0 \\
\rowcolor{aaccolor} AAC-CG       & $K_0{=}16,\,m{=}8$ & $(2.5 \pm 0.4)$\,\%          & $(24.5 \pm 2.6)$\,\%         & 0.094 & 17 & 86.8\,\% & 1.046 & 0 \\
\rowcolor{aaccolor} AAC-CG$_\tau$ & $K_0{=}16,\,m{=}8$ & $(2.5 \pm 0.4)$\,\%          & $(24.5 \pm 2.6)$\,\%         & 0.094 & 18 & 86.0\,\% & 1.044 & 0 \\
\cmidrule{1-9}
\rowcolor{aaccolor} AAC-ST       & $K_0{=}16,\,m{=}8$ (ResNet) & $(16.9 \pm 0.5)$\,\%         & $(53.9 \pm 0.5)$\,\%         & 0.005 & 36 & 71.0\,\% & 0.981 & 0 \\
\rowcolor{aaccolor} AAC-CG       & $K_0{=}16,\,m{=}8$ (ResNet) & $(17.9 \pm 0.6)$\,\%         & $(54.8 \pm 0.4)$\,\%         & 0.004 & 16 & 86.7\,\% & 1.050 & 0 \\
\bottomrule
\end{tabular}
}
\par\smallskip
\begin{minipage}{\linewidth}
\footnotesize\raggedright
At $K_0=16,m=8$, we report mean $\pm$ std over 8 CNN seeds and 3 ResNet seeds. The dataset contains 10{,}000 training maps and 1{,}000 test maps. \textbf{Adm.\ violations} counts vertices where $h(v,t)>d(v,t)$ on predicted-cost test graphs. In the upper block, A* returns 0 suboptimal paths across 11{,}000 test instances. \par \textbf{Cost regret} is the mean relative excess true path cost (Appendix~\ref{sec:contextual-study}). \textbf{Opt.\ cost} is optimal-cost accuracy under true cell costs, as reported by BB. A dash means the metric was not reported. BB learns costs through an exact shortest-path solver under a different training protocol (Appendix~\ref{app:contextual}). \par We also freeze the compressor at initialization, at the FPS prefix, at a pretrained subset, or at a random subset. Path match stays within one standard deviation of full end-to-end training at 4 of 4 CNN ratios $m/K_0\in\{2/5, 3/5, 4/5, 8/16\}$. For the CNN at $K_0=16,m=8$, mean cost regret is $0.625$ with the encoder frozen at random initialization, $0.094$ when both components are trained, and $0.094$ when only the encoder is trained. The last two means agree to three decimal places.
\end{minipage}
\end{table}

\paragraph{Selection layer: continuous greedy against straight-through.} In the second block of Table~\ref{tab:contextual-ablation}, we change the selection layer.
We keep the CNN architecture, $\alpha_{\mathrm{cost}}{=}1$, and the rest of the training pipeline fixed. One variant trains the
straight-through row-Gumbel compressor. The other replaces it with the parameter-free soft
continuous-greedy layer (Proposition~\ref{prop:eps-lmo}), which selects anew on every map and deploys
the pipage-rounded base.

\emph{Training and deployment oracles.} The entropic oracle lets the selection objective train
the upstream encoder. We train with $\tau{=}0.1$ and $T{=}100$ and backpropagate through the entropic-oracle trajectory to the encoder.
At deployment, we re-select on each held-out map with the exact oracle at the same $T$.
We then apply pipage rounding and run A* with the selected base.
We also report deployment with the entropic oracle in a separate row to measure the effect of changing the oracle.

At $K_0{=}16, m{=}8$, AAC-CG has the highest test objective among the nine CNN configurations.
Its ratio to the FPS prefix is $F/F_{\mathrm{FPS}}{=}1.046$; every variant without continuous
greedy is at or below $1.000$. Search effort gives a different ordering. AAC-CG averages $16.84$ expansions and AAC-ST $32.57$.
AAC-ST needs $59$--$64$ expansions on three of eight seeds, but fewer expansions than AAC-CG on the other five.
The frozen FPS control trains the encoder while keeping the compressor at the FPS prefix.
It averages $15.48$ expansions and expands fewer nodes than AAC-CG on every seed.

Running continuous greedy once offline and freezing its selection gives $15.57$ expansions,
also fewer than per-map selection on every seed. Re-selecting on each map improves the objective
without reducing expansions on the fixed corner-to-corner query.

For configurations with a trained encoder, differences in cost regret, path match, and Jaccard
are smaller than their variation across seeds. We compute these metrics using Dijkstra on the
predicted-cost graph. Holding that graph fixed, landmark selection affects search effort and
the objective but cannot change these path metrics.

At $K_0{=}5$ and $m\in\{2,3,4\}$, AAC-CG again exceeds AAC-ST on the test objective at all three
label counts and expands fewer nodes at two. Frozen FPS expands fewer nodes than AAC-CG at all
three. AAC-ST has a slightly lower mean expansion count than frozen FPS only at $K_0=5,m=4$:
$14.96$ versus $15.03$. With ResNet at $K_0{=}16,m{=}8$, AAC-CG also has the highest objective
($F/F_{\mathrm{FPS}}{=}1.050$). It averages $16.5$ expansions, compared with AAC-ST's $36.0$ and
frozen FPS's $15.5$. Cost-regret differences remain smaller than seed variation. Every variant
has zero observed admissibility violations and zero suboptimal paths at every seed.

\paragraph{Fixed-target limitation.} Every path query runs from vertex $0$ to vertex $143$.
FPS on the first training map places vertex $143$ first in both candidate pools. If a selector
retains this landmark, its ALT term is $|d(143,u)-d(143,143)|=d(u,143)$, up to the deployment
rounding allowance. The frozen FPS prefix always retains it. Expansion differences depend on target retention, learned costs, and the other selected landmarks.
We therefore use these counts only to diagnose this fixed-target task.
We evaluate $F/F_{\mathrm{FPS}}$ on separate random query pairs and compute path-imitation metrics with Dijkstra.
Neither metric depends on retaining the path query's target landmark.

\paragraph{Effects of encoder and compressor training.} With the compressor frozen at initialization,
we replace the 3-layer CNN with ResNet and increase cost supervision from $\alpha_{\mathrm{cost}}=1$ to $10$ (Table~\ref{tab:contextual-ablation}).
Jaccard rises from $(25.4\pm3.1)$\,\% to $(54.4\pm0.4)$\,\%, and exact path match rises from $(2.9\pm0.4)$\,\% to $(17.5\pm0.8)$\,\%.
We observe zero admissibility violations in both configurations.
Because we change the encoder and supervision together, these runs do not separate their effects.

For the CNN at $K_0=16,m=8$ and $\alpha_{\mathrm{cost}}=1$, training the compressor gives no clear improvement in the path metrics.
Over eight seeds, training both components gives mean cost regret $0.0938\pm0.0326$.
Training the encoder with the compressor frozen at initialization gives $0.0938\pm0.0322$.
The means agree to four decimal places. Training only the compressor, with the encoder frozen at random initialization,
gives mean cost regret $0.6247$ (Table~\ref{tab:contextual-ablation}, notes).

In a separate $100$-iteration CNN run with the same configuration and seed $0$, both components receive non-zero gradients at every iteration.
The ratio of encoder to compressor gradient norms has median $47.3$ and minimum $5.0$.

\paragraph{Related planning methods.}
Neural A*~\citep{yonetani2021path} and iA*~\citep{chen2025ia} learn search guidance
without an architectural admissibility guarantee. Neither cited study evaluates on Warcraft. Blackbox differentiation
\citep{vlastelica2020differentiation}, differentiable perturbed optimizers
\citep{berthet2020learning}, and I-MLE~\citep{niepert2021implicit} report results on this
benchmark by learning costs and using a shortest-path solver. As with AAC, these solvers return an optimal
path for their predicted costs; prediction error determines accuracy under the true costs.
Our runs use direct cell-cost supervision and at most $2{,}400$ training-map draws.
The protocols differ, so the published results provide task context without isolating the effect of the heuristic.

\section{Reproducibility}
\label{app:reproducibility}

\paragraph{Data and code availability.} We release the implementation and source records at \url{https://github.com/anindex/aac} and in the supplementary archive.
The index \texttt{results/README.md} links every table, figure, and reported measurement to its source record and regeneration command.
We use Python 3.14, PyTorch 2.14.1, NumPy 2.5, and SciPy 1.18.

\paragraph{OGB-arXiv instance.} We symmetrize the citation graph ($169{,}343$ papers,
$1{,}166{,}243$ directed edges). Its largest connected component contains every vertex.
We draw edge weights once, uniformly in $[1,10]$, with seed $42$ in NetworkX edge order.
All OGB-arXiv results use this weighted instance and $100$ uniform source--target queries per seed, over five seeds.
At $32/64/128$\,B/v, AAC-ST uses $K_0/m=32/8,64/16,128/32$ and FPS-ALT uses $K=8,16,32$.
Dijkstra expands $78{,}617$ vertices on average. Section~\ref{sec:ogb-main} reports the results.

\paragraph{Hyperparameters.} Across the experiments, we use $K_0 \in \{5, 8, 16, 32, 64, 128, 256\}$ and $m \in \{2, 3, 4, 8, 16, 32, 64\}$.
The nine-vertex witness uses $K_0{=}3$, $m{=}2$.
Unless stated otherwise, we train static selectors for $200$ epochs with learning rate $10^{-3}$ and AdamW at weight decay $0$.
Each epoch is one optimizer step on $256$ sampled query pairs.

For contextual training, we use Adam for at most $100$ iterations, with one optimizer step on $24$ maps per iteration.
At weight decay $0$, Adam and AdamW coincide.
We stop early with patience $15$ on a fixed $24$-map validation set and restore the best validation weights (Appendix~\ref{app:contextual}).

The static matched-memory experiments use $K_0 \in \{16, 32, 64, 128\}$ with $m \in \{8, 16, 32\}$. The contextual study of Appendix~\ref{app:contextual} runs $K_0{=}16, m{=}8$ and the $K_0{=}5$ sweep $m \in \{2,3,4\}$. The finite-sample study of Proposition~\ref{prop:erm} runs at $K_0{=}16, m{=}4$. All reported runs use $\lambda_{\mathrm{wd}}{=}0$, so weight decay cannot explain the training drift
in Section~\ref{sec:training-drift}.

\paragraph{Sentinel rule.} We store unreachable landmark distances as a sentinel ($10^{18}$).
If only one distance in a landmark term is infinite, the term either proves the target unreachable
or is omitted from the maximum. In the first case, the heuristic returns $+\infty$.
A term with two infinite distances contributes $0$ (Appendix~\ref{app:directed}).
The sentinel's magnitude does not enter any theorem. The surviving finite terms are admissible
by Theorem~\ref{thm:alt-admissibility}, and the returned $+\infty$ is exact.

Unreachable-landmark handling depends on reachability between landmarks and the vertices visited
by A*. It can be needed even when the two query endpoints are mutually reachable.
On all eleven benchmark graphs, the endpoints lie in one component for $100/100$ sampled queries.
We use strongly connected components on directed graphs and connected components on undirected graphs. Unreachable-landmark handling is needed on every OSMnx graph,
but on none of the four DIMACS or three undirected graphs.
Each OSMnx graph has $3$--$4$ distinct masks indicating which pool landmarks are usable.
Across $81$--$400$ arcs, the two endpoints have different masks.
The rule preserves consistency in exact arithmetic even on these arcs
(Appendix~\ref{app:directed}).

We also sample the consistency residual $h(u,t) - w(u,v) - h(v,t)$ on arcs using the deployed
float32 heuristic. Rounding can make this residual positive. The largest sampled residual per
OSMnx graph ranges from $3.8\times10^{-3}$ to $1.2\times10^{-2}$ meters.
A* reopens closed nodes, so optimality requires admissibility rather than consistency
(Section~\ref{sec:graphs-astar}).

\section{Broader impact}
\label{app:broader-impact}

AAC preserves admissibility with respect to the supplied graph, assuming exact landmark distances
and arithmetic (Theorem~\ref{thm:alt-admissibility}). The relaxations preserve this property at
every parameter value, independently of training quality, early stopping, or the selector's
query distribution (Propositions~\ref{thm:admissibility} and~\ref{prop:matroid-cg}(c)).
The graph model and numerical implementation still require validation. In the contextual
pipeline, an optimal path under predicted costs can be suboptimal under true costs.
In the CNN ablation at $K_0=16,m=8$, mean cost regret (relative excess true path cost) is $0.625$ with the encoder frozen at random initialization
and $0.094$ when both components are trained. Training the encoder with a frozen compressor also gives $0.094$
to three decimal places (Appendix~\ref{sec:contextual-study}; Table~\ref{tab:contextual-ablation}, notes). Admissible selection
does not correct encoder error or guarantee path qualities such as smoothness or obstacle clearance.

Fewer A* expansions can reduce search work; latency also depends on lookup cost and memory
layout. We did not measure energy use. Precomputed landmark distances require the graph to be
known offline. Admissibility alone does not guarantee safe deployment; application safety depends on the graph model, costs, and operating constraints.

\end{document}